\documentclass[10pt,journal]{IEEEtran}
\usepackage[nocompress]{cite}
\usepackage{booktabs}
\usepackage{url}
\usepackage{float}
\usepackage{textcomp}
\usepackage{gensymb}
\usepackage{comment}
\usepackage[pdftex]{graphicx}
\usepackage{tabularx}
\usepackage[cmex10]{amsmath}
\usepackage{amsfonts,amsthm}
\usepackage{algorithm,algorithmic}
\usepackage[caption=false,font=footnotesize]{subfig}
\usepackage{bm}
\usepackage{pgfplots}
\usepackage{tikz}
\usepackage{tkz-euclide}
\usepackage{xr}
\usepackage[inline]{enumitem}
\usetikzlibrary{positioning,calc,spy}
\usetkzobj{all}
\newif\ifblackandwhitecycle
\gdef\patternnumber{0}

\newcommand\phantomimage{%
    \phantom{%
        \rule{\imagewidth}{\imageheight}%
    }%
}
\newcommand\zoombox[2][]{
    \begin{scope}[zoombox paths]
        \pgfmathsetmacro\xpos{
            (\columncount-1)*(\pgfkeysvalueof{/tikz/zoomboxarray width} / \pgfkeysvalueof{/tikz/zoomboxarray columns} + \pgfkeysvalueof{/tikz/zoomboxarray inner gap} / \pgfkeysvalueof{/tikz/zoomboxarray columns} ) + \pgflinewidth
        }
        \pgfmathsetmacro\ypos{
            (\rowcount-1)*( \pgfkeysvalueof{/tikz/zoomboxarray height} / \pgfkeysvalueof{/tikz/zoomboxarray rows} + \pgfkeysvalueof{/tikz/zoomboxarray inner gap} / \pgfkeysvalueof{/tikz/zoomboxarray rows} ) + 0.5*\pgflinewidth
        }
        \edef\dospy{\noexpand\spy [
            #1,
            zoombox paths/.append style={
                black and white pattern=\patternnumber
            },
            every spy on node/.append style={#1},
            x=\imagewidth,
            y=\imageheight
        ] on (#2) in node [anchor=north west] at ($(zoomboxes container.north west)+(\xpos pt,-\ypos pt)$);}
        \dospy
        \pgfmathtruncatemacro\pgfmathresult{ifthenelse(\columncount==\pgfkeysvalueof{/tikz/zoomboxarray columns},\rowcount+1,\rowcount)}
        \global\let\rowcount=\pgfmathresult
        \pgfmathtruncatemacro\pgfmathresult{ifthenelse(\columncount==\pgfkeysvalueof{/tikz/zoomboxarray columns},1,\columncount+1)}
        \global\let\columncount=\pgfmathresult
        \ifblackandwhitecycle
            \pgfmathtruncatemacro{\newpatternnumber}{\patternnumber+1}
            \global\edef\patternnumber{\newpatternnumber}
        \fi
    \end{scope}
}

\pgfkeys{/tikz/.cd,
    zoombox paths/.style={%
        draw=orange,
        very thick
    },
    black and white/.is choice,
    black and white/.default=static,
    black and white/static/.style={%
        draw=white,
        zoombox paths/.append style={%
            draw=white,
            postaction={%
                draw=black,
                loosely dashed
            }
        }
    },
    black and white/static/.code={%
        \gdef\patternnumber{1}
    },
    black and white/cycle/.code={%
        \blackandwhitecycletrue
        \gdef\patternnumber{1}
    },
    black and white pattern/.is choice,
    black and white pattern/0/.style={},
    black and white pattern/1/.style={%
            draw=white,
            postaction={%
                draw=black,
                dash pattern=on 2pt off 2pt
            }
    },
    black and white pattern/2/.style={%
            draw=white,
            postaction={%
                draw=black,
                dash pattern=on 4pt off 4pt
            }
    },
    black and white pattern/3/.style={%
            draw=white,
            postaction={%
                draw=black,
                dash pattern=on 4pt off 4pt on 1pt off 4pt
            }
    },
    black and white pattern/4/.style={%
            draw=white,
            postaction={%
                draw=black,
                dash pattern=on 4pt off 2pt on 2 pt off 2pt on 2 pt off 2pt
            }
    },
    zoomboxarray inner gap/.initial=5pt,
    zoomboxarray columns/.initial=2,
    zoomboxarray rows/.initial=2,
    zoomboxarray width/.initial=\imagewidth,
    zoomboxarray height/.initial=\imageheight,
    subfigurename/.initial={},
    figurename/.initial={zoombox},
    zoomboxarray/.style={%
        execute at begin picture={%
            \begin{scope}[
                spy using outlines={%
                    zoombox paths,
                    width=\pgfkeysvalueof{/tikz/zoomboxarray width} / \pgfkeysvalueof{/tikz/zoomboxarray columns} - (\pgfkeysvalueof{/tikz/zoomboxarray columns} - 1) / \pgfkeysvalueof{/tikz/zoomboxarray columns} * \pgfkeysvalueof{/tikz/zoomboxarray inner gap} -\pgflinewidth,
                    height=\pgfkeysvalueof{/tikz/zoomboxarray height} / \pgfkeysvalueof{/tikz/zoomboxarray rows} - (\pgfkeysvalueof{/tikz/zoomboxarray rows} - 1) / \pgfkeysvalueof{/tikz/zoomboxarray rows} * \pgfkeysvalueof{/tikz/zoomboxarray inner gap}-\pgflinewidth,
                    magnification=3,
                    every spy on node/.style={%
                        zoombox paths
                    },
                    every spy in node/.style={%
                        zoombox paths
                    }
                }
            ]
        },
        execute at end picture={%
            \end{scope}
     \gdef\patternnumber{0}
        },
        spymargin/.initial=0.5em,
        zoomboxes xshift/.initial=1,
        zoomboxes right/.code=\pgfkeys{/tikz/zoomboxes xshift=1},
        zoomboxes left/.code=\pgfkeys{/tikz/zoomboxes xshift=-1},
        zoomboxes yshift/.initial=0,
        zoomboxes above/.code={%
            \pgfkeys{/tikz/zoomboxes yshift=1},
            \pgfkeys{/tikz/zoomboxes xshift=0}
        },
        zoomboxes below/.code={%
            \pgfkeys{/tikz/zoomboxes yshift=-0.87},
            \pgfkeys{/tikz/zoomboxes xshift=0}
        },
        caption margin/.initial=4ex,
    },
    adjust caption spacing/.code={},
    image container/.style={%
        inner sep=0pt,
        at=(image.north),
        anchor=north,
        adjust caption spacing
    },
    zoomboxes container/.style={%
        inner sep=0pt,
        at=(image.north),
        anchor=north,
        name=zoomboxes container,
        xshift=\pgfkeysvalueof{/tikz/zoomboxes xshift}*(\imagewidth+\pgfkeysvalueof{/tikz/spymargin}),
        yshift=\pgfkeysvalueof{/tikz/zoomboxes yshift}*(\imageheight+\pgfkeysvalueof{/tikz/spymargin}+\pgfkeysvalueof{/tikz/caption margin}),
        adjust caption spacing
    },
    calculate dimensions/.code={%
        \pgfpointdiff{\pgfpointanchor{image}{south west} }{\pgfpointanchor{image}{north east} }
        \pgfgetlastxy{\imagewidth}{\imageheight}
        \global\let\imagewidth=\imagewidth
        \global\let\imageheight=\imageheight
        \gdef\columncount{1}
        \gdef\rowcount{1}
        
    },
    image node/.style={%
        inner sep=0pt,
        name=image,
        anchor=south west,
        append after command={%
            [calculate dimensions]
            node [image container,subfigurename=\pgfkeysvalueof{/tikz/figurename}-image] {\phantomimage}
            node [zoomboxes container,subfigurename=\pgfkeysvalueof{/tikz/figurename}-zoom] {\phantom{\rule{\pgfkeysvalueof{/tikz/zoomboxarray width}}{\pgfkeysvalueof{/tikz/zoomboxarray height}}}}
        }
    },
    color code/.style={%
        zoombox paths/.append style={draw=#1}
    },
    connect zoomboxes/.style={%
    spy connection path={\draw[draw=none,zoombox paths] (tikzspyonnode) -- (tikzspyinnode);}
    },
    help grid code/.code={%
        \begin{scope}[
                x={(image.south east)},
                y={(image.north west)},
                font=\footnotesize,
                help lines,
                overlay
            ]
            \foreach \x in {0,1,...,9} {
                \draw(\x/10,0) -- (\x/10,1);
                \node [anchor=north] at (\x/10,0) {0.\x};
            }
            \foreach \y in {0,1,...,9} {
                \draw(0,\y/10) -- (1,\y/10);                        \node [anchor=east] at (0,\y/10) {0.\y};
            }
        \end{scope}
    },
    help grid/.style={%
        append after command={%
            [help grid code]
        }
    },
}

\def\bI{\ensuremath{{\bf I}}}
\def\bD{\ensuremath{{\bf D}}}
\def\bB{\ensuremath{{\bf B}}}
\def\bF{\ensuremath{{\bf F}}}
\def\bJ{\ensuremath{{\bf J}}}
\def\bX{\ensuremath{{\bf X}}}

\def\bE{\ensuremath{{\bf E}}}
\def\bM{\ensuremath{{\bf M}}}
\def\bv{\ensuremath{{\bf v}}}
\def\bu{\ensuremath{{\bf u}}}
\def\bL{\ensuremath{{\bf L}}}
\def\bQ{\ensuremath{{\bf Q}}}
\def\bR{\ensuremath{{\bf R}}}
\def\bA{\ensuremath{{\bf A}}}
\def\bB{\ensuremath{{\bf B}}}
\def\bW{\ensuremath{{\bf W}}}
\def\bP{\ensuremath{{\bf P}}}
\def\bS{\ensuremath{{\bf S}}}
\def\bG{\ensuremath{{\bf G}}}
\def\bF{\ensuremath{{\bf F}}}
\def\bH{\ensuremath{{\bf H}}}
\def\bN{\ensuremath{{\bf N}}}
\def\bp{\ensuremath{{\bf p}}}
\def\br{\ensuremath{{\bf r}}}
\def\bt{\ensuremath{{\bf t}}}
\def\bw{\ensuremath{{\bf w}}}
\def\bs{\ensuremath{{\bf s}}}
\def\be{\ensuremath{{\bf e}}}
\def\bq{\ensuremath{{\bf q}}}
\def\vec{\ensuremath{{\mathrm{vec}}}}
\def\cP{\ensuremath{{\mathcal P}}}
\def\cI{\ensuremath{{\mathcal I}}}
\def\cO{\ensuremath{{\mathcal O}}}
\def\cU{\ensuremath{{\mathcal U}}}
\def\cV{\ensuremath{{\mathcal V}}}
\def\cE{\ensuremath{{\mathcal E}}}
\def\cN{\ensuremath{{\mathcal N}}}
\def\cR{\ensuremath{{\mathcal R}}}
\def\cT{\ensuremath{{\mathcal T}}}

\def\cS{\ensuremath{{\mathcal S}}}
\def\supp{\ensuremath{{\mathrm{Supp}}}}

\def\dt{\ensuremath{\Delta\tau}}
\def\0{\ensuremath{{\bf 0}}}
\def\rank{\ensuremath{\mathrm{rank}}}

\def\tcr{\textcolor{black}}
\def\tcred{\textcolor{black}}

\newtheorem*{remark}{Remark}
\newtheorem{theorem}{Theorem}
\newtheorem{lemma}{Lemma}

\begin{document}

\title{Robust Alignment for Panoramic Stitching\\via An Exact Rank Constraint}

\author{Yuelong~Li,~\IEEEmembership{Student~Member,~IEEE,}
        Mohammad~Tofighi,~\IEEEmembership{Student~Member,~IEEE,}
        and~Vishal~Monga,~\IEEEmembership{Senior~Member,~IEEE}

	    	\thanks{Y. Li, M. Tofighi, and V. Monga are with Department of Electriccal Engineering, The Pennsylvania State University, University Park,
	    		PA, 16802 USA, Emails: liyuelongee@gmail.com, tofighi@psu.edu, vmonga@engr.psu.edu
	    	}
	    	\thanks{Research supported by a National Science Foundation CAREER Award.}}

\markboth{IEEE Transactions on Image Processing, Accepted for Publication, March 2019}%
{Shell \MakeLowercase{\textit{et al.}}: Bare Demo of IEEEtran.cls for IEEE Journals}

\maketitle

\begin{abstract}
We study the problem of image alignment for panoramic
    stitching. Unlike most existing approaches that are feature-based, our
    algorithm works on pixels directly, and accounts for errors across the
    whole images globally. Technically, we formulate the alignment problem
    as rank-1 and sparse matrix decomposition over transformed images,
    and develop an efficient algorithm for solving this challenging
    non-convex optimization problem. The algorithm reduces to solving a sequence
    of subproblems, where we analytically establish exact recovery
    conditions, convergence and optimality, together with convergence rate
    and complexity. We generalize it to simultaneously align multiple
    images and recover multiple homographies, extending its application
    scope towards vast majority of practical scenarios.  Experimental
    results demonstrate that the proposed algorithm is capable of more
   accurately aligning the images and generating higher quality stitched
    images than state-of-the-art methods.
\end{abstract}

\begin{IEEEkeywords}
Image alignment, stitching, pixel-based alignment, alternating minimization, rank-1 and sparse decomposition
\end{IEEEkeywords}

\IEEEpeerreviewmaketitle

\section{Introduction}\label{sec:introduction}
\IEEEPARstart{P}{anoramic} stitching refers to the process of stitching several
images together, assuming these images contain pairwise overlapping regions.
Despite their variety, panoramic stitching algorithms generally follow the same
pipelines~\cite{szeliski_image_2006}: first, estimate the parametric
correspondences between input images, and then warp them towards a common
canvas; this step is called alignment.  Second, compose these aligned images
together deliberately to conceal visual artifacts; this step is called
stitching. Afterwards, several post-processing steps can be incorporated to
further improve visual quality.  Among these procedures, alignment serves as an
essential step since properly aligned images can significantly reduce the
burden of subsequent procedures.  Besides, other applications such as object
recognition can greatly benefit from more accurately aligned
images~\cite{liu_sift_2011,peng_rasl:_2012}.

\subsection{Related Previous Work}\label{subsec:review}
Broadly speaking, alignment methods could be divided into two categories:
pixel-based vs feature-based~\cite{szeliski_image_2006}.  Pixel-based
approaches typically estimate a dense correspondence field per pixel, while
feature-based approaches generally rely on feature extraction, detection and
matching. Despite extensive studies about pixel-based
approaches~\cite{szeliski_creating_1997,bartoli_groupwise_2008,evangelidis_parametric_2008,gay-bellile_direct_2010}
especially in the early stage of image alignment, nowadays feature-based
methods are much more widely used since pixel-based methods require close
initializations under large displacements. To enhance the alignment accuracy of
feature-based methods, a post-processing strategy called bundle
adjustment~\cite{triggs_bundle_1999} is usually performed.

In the seminal work of~\cite{brown_automatic_2007}, a fully automated panorama
system called AutoStitch is introduced. It employs a
SIFT~\cite{lowe_distinctive_2004}-based alignment technique, followed by
RANSAC~\cite{fischler_random_1981} and probabilistic matching verification to
enhance robustness. In a natural evolution of panoramic stitching, it was
observed that a single homography motion model is only accurate for planar and
rotational scenes~\cite{zaragoza_as-projective-as-possible_2014}. For casually
taken images, this model is error-prone. Consequently, the stitched images may
contain ghosting artifacts or distorted objects. Instead of a single global
homography, in~\cite{gao_constructing_2011} the authors propose a
dual-homography model, i.e., a composition of two distinct homographies. As a
further generalization, in~\cite{lin_smoothly_2011} a spatially varying affine
model is proposed. In this approach, local deviations are superposed to the
global affine transformations to correct local misalignment errors. However,
this method suffers from shape distortions in the non-overlapping regions as
affine transformation may be suboptimal for extrapolation. Therefore,
in~\cite{zaragoza_as-projective-as-possible_2014}, Zaragoza {\it et al.}
estimate spatially moving homographies instead of affine transformations.

In a parallel direction, Gao {\it et al.}~\cite{gao_seam-driven_2013} propose
to select the homography that produces the ``best'' seam in the subsequent
seam-cutting procedure~\cite{szeliski_image_2006} by minimizing an 
energy function. However, this method is inapplicable to cases where the single
homography model is highly inaccurate. Lin {\it et al.}~\cite{lin_seam_2016}
advance the idea by combining it with content-preserving
warps~\cite{liu_content-preserving_2009} to deal with large parallax. Other
works inspired by content-preserving warps
include~\cite{zhang_parallax-tolerant_2014,li_dual-feature_2015,lin_direct_2017},
etc.

Other feature-based methods have focused on aspects of research somewhat
adjacent to the topic of this paper such as 3D reconstruction, improving the
aesthetic value of the stitched images, \tcr{handling specific scenarios such as large parallax and low textures} etc. These
include~\cite{chang_shape-preserving_2014, li_quasi-homography_2018,
	lin_adaptive_2015, hu_multi-objective_2015, chen_natural_2016,
agarwala_photographing_2006,qi_zhi_toward_2012,zhang_multi-viewpoint_2016,li_parallax-tolerant_2018,xiang_image_2018}, to
name a few.

In the face alignment literature, a pixel-based method called
RASL~\cite{peng_rasl:_2012} has been shown to be quite successful. It formulates the alignment
problem as a low-rank and sparse matrix decomposition problem on the warped
images.  RASL exhibits robustness to object occlusions and
photometric differences (in the image set) can be overcome to some extent.
However, RASL relaxes the true optimization problem and relies heavily on close
initializations that can be unrealistic. Furthermore, the assumption
that all aligned images overlap on a common region can be overly restrictive in
practice.

\subsection{Motivation and Contributions}\label{subsec:motivation}
There are two fundamental elements in an alignment algorithm: 
\begin{enumerate*}
	\item motion models to represent geometric transformations across images and
	\item methods for fitting the motion models.
\end{enumerate*}
Motivated by the limitations of the traditional single homography model, recent
studies are focused on development of more flexible geometric transformation
models and have reduced misalignment errors to some extent.

Nevertheless, methods for fitting the newly-invented models are commonly built
on the ground of feature matching, which may suffer from several drawbacks.
First, feature points are less accurately localized than pixels. Additionally,
feature detection may also lead to loss of information.  In particular, the
edge thresholding procedure in detecting SIFT features may lead to
underemphasis on line and edge structures, which contribute to evident
artifacts in the final results. Finally, the feature points are spatially
non-uniform. Consequently, absence of feature points in certain local areas may
cause instabilities when fitting locally adaptive models. Therefore, however
general the geometric transformation model is, there exist irreducible errors
that cannot be overcome by the nature of feature-based methods. It is therefore
interesting to investigate possible performance gains achievable through
alternative pixel-based strategies.

To address the aforementioned drawbacks, we develop a new pixel-based method. Specifically, we make the following
\textbf{contributions} in this work\footnote{Preliminary version of this work
	has been presented in~\cite{li_siasm:_2016}. In this paper, we provide novel, more effective and efficient optimization algorithms and associated theoretical analysis. The experimental comparisons are also significantly expanded.
}:

\begin{enumerate}
	\item We propose a novel image alignment algorithm based on decomposing a
		set of panoramic images (data matrix) as the sum of a rank-1 and a
		sparse matrix. Our work extends the framework of low-rank plus sparse
		matrix decomposition now widely used in computer vision, except that in
		our approach an exact capture of the rank (as opposed to approximations
		via the nuclear norm) plays a crucial role.  We show that explicitly
		forcing the low-rank component to be of rank-1 is not only physically
		meaningful but also practically beneficial.

	\item Our analytical contributions involve solving an inherently non-convex
		problem (induced by the exact rank constraint) and analyzing
		convergence properties of the proposed algorithm as well as optimality
		of the final solution. Recent advances in {\it low-rank matrix
		recovery\/}~\cite{jain_low-rank_2013,netrapalli_non-convex_2014}
		facilitate our development of efficient algorithms and solutions. As
		opposed to~\cite{wright_compressive_2013}, our results rely on
		deterministic conditions and are more specific to the alignment
		problem.

	\item As our practical contributions, we generalize the aforementioned
		alignment strategy to handle realistic scenarios such as multiple
		overlapping regions and multiple underlying homographies. The complete
		algorithm is called {\it Bundle Robust Alignment and Stitching\/}
		(BRAS). The complexity is linear in both the number of pixels and
		images, ensuring scalability.  We verify the enhanced alignment
		accuracy for panoramic stitching applications compared to many
		state-of-the-art methods through extensive experiments.

	\item Finally, our code and datasets can be accessed freely online for
		reproducibility~\footnote{\url{signal.ee.psu.edu/panorama.html}}.
\end{enumerate}

The rest of the paper is organized as follows: we provide a concrete
mathematical formulation of the robust alignment problem in
Section~\ref{sec:robust_alignment}, in particular focusing on decomposing the
data matrix as the sum of a rank-1 and a sparse matrix. Our iterative alignment
algorithm is introduced here along with convergence analysis and examination of
the optimality of the resulting solution. We then describe a complete panoramic
composition system in Section~\ref{sec:bundle_alignment}, which includes issues
of initialization, handling of multiple overlapping regions and dealing with
multiple homographies. Experimental validation on various popular datasets and
real world applications to panoramic stitching are presented in
Section~\ref{sec:results}. Finally, Section~\ref{sec:conclusion} concludes the
paper.

\section{Robust Image Alignment via Rank-1 and Sparse Decomposition}\label{sec:robust_alignment}

For ease of exposition, we first describe our approach for the case where all input images overlap on a single region. We formulate it as a rank-constrained sparse error minimization
problem in~\ref{subsec:formulate} and then discuss an efficient solution
in~\ref{subsec:solution} and associated theoretical analysis
in~\ref{subsec:analysis}. We also discuss its relationship to other pixel-based methods in~\ref{subsec:relation}.

{\bf Notation:} We will adopt the following notation henceforth: we use $\|\cdot\|_{\ell_p}$ to denote the $\ell_p$ norm on
vectors; when acting on matrices, we first stack the elements into a
vector and then operate on it. $\|\cdot\|_p$ denotes the induced $p$ norm
on matrices. $\|\bX\|_F=\sqrt{\sum_{ij}\bX_{ij}^2}$ is the Frobenius norm
of matrix $\bX$. We let ${\{\be_i\}}_{i=1}^n$ be the standard basis vectors
in $\mathbb{R}^n$ where all elements except the $i$-th are zero. $\bX^\dag$ denotes the pseudo-inverse of matrix $\bX$. $\supp(\bX)=\{(i,j):\bX_{ij}\neq 0\}$ denotes the support of $\bX$. $\circ$ denotes functional composition (image warping in practice). A listing of the symbols can be found in  the supplementary document.

\subsection{Alignment Model and Problem Formulation}\label{subsec:formulate}
Our goal is to estimate the geometric transformations
associated with several input images to align the images. These transformations are from the
original image plane towards a common reference coordinate axis (we call it
canvas). In this work we focus on the case where all the images are taken
on the same natural scene.  Automatic algorithms for discovering different
scenes have been proposed in~\cite{brown_automatic_2007}.

\begin{figure}
    \centering
    \includegraphics[width=\linewidth]{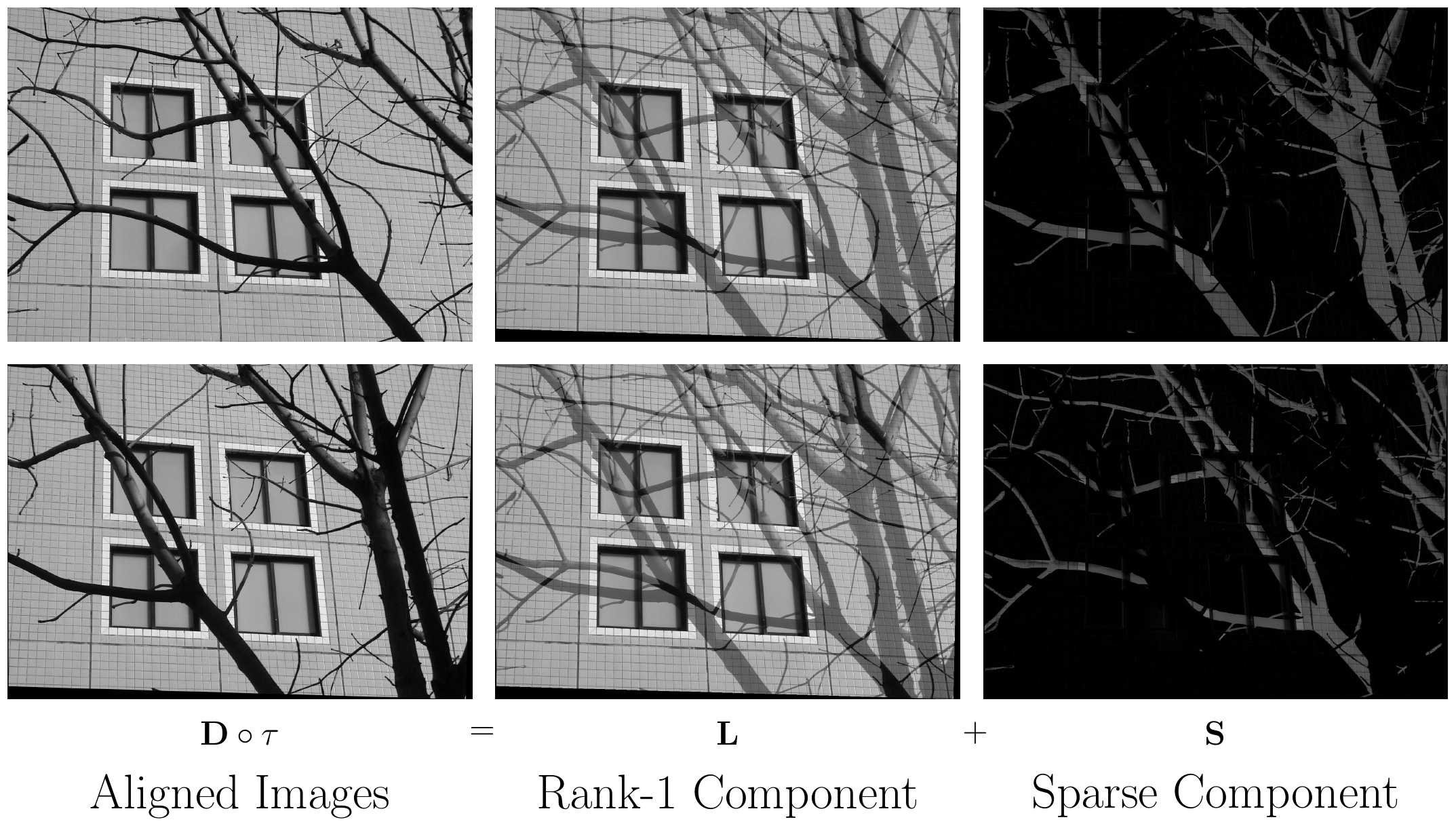}
	\caption{Rank-1 and sparse decomposition on an aligned image pair in the {\it windows\/} dataset~\cite{peng_rasl:_2012}: note that the two images as they appear in the rows above (second column) are nearly identical --- emphasizing that $\mathbf L$ is rank-1, while the error component ($\mathbf S$) is spatially sparse.}\label{fig:example}
\end{figure}

Given $n$ grayscale images ${\{I_i\}}_{i=1}^n$ of the same
real world scene, let $t_i:\mathbb{R}^2\rightarrow\mathbb{R}^2$ be the geometric
transformation warping $I_i$ to the canvas axis, and let $M_i\in\mathbb{R}^{h\times w}$ be the
corresponding warped image on canvas.\footnote{\tcr{In practice, warped images usually contain out-of-domain elements and the observed portions are generally not rectangular, can be seen in Fig.~\ref{fig:example}. We assume fully-observed rectangular $M_i$s in this Section. Handling more general (and practical) cases is discussed in Section~\ref{sec:bundle_alignment}.}} \tcr{Concretely, we implement the canvas as the minimum rectangle enclosing each warped image.} In most practical scenarios, the geometric transformations ${\{t_i\}}_{i=1}^n$ admit low-dimensional parametric representations. In particular, for planar and rotational scenes they can be well-approximated as 8-parameter plane homographies~\cite{szeliski_image_2006}. We hereafter represent $t_i$ with its $d$ parameters $\tau_i\in\mathbb{R}^d$, and express it as $t(\cdot;\tau_i)$.
Therefore,
\begin{equation}
	\left(I_i\circ\tau_i\right)(x,y) = I_i(t((x,y);\tau_i)) = M_i(x,y),
    \label{eqn:transform}
\end{equation}
$\forall (x,y)\in\mathbb{R}^2$. 
As we postulate that ${\{M_i\}}_i$ are fully overlapping on a common region,
under perfect alignment the content will appear largely similar. For images $M_1$, $M_2$ that differ only in photometric differences, a widely used model is the following ``gain and bias'' model~\cite{bartoli_groupwise_2008,evangelidis_parametric_2008}:
\[
	M_1=gM_2+b,
\]
where $g,b$ are scalar constants that model contrast and brightness changes, respectively, and the operations between matrices and scalars act elementwise. In this work, we assume brightness changes are negligible, i.e., $b\approx 0$. In doing so, the pixel values in the same position $(x,y)$ from different images $M_i$ will only differ by the gain factors:
\begin{equation*}
	M_i(x,y)=g_{i}B(x,y)
\end{equation*}
where $B\in\mathbb{R}^{h\times w}$ is the underlying background. Further deviations may come from moving objects, parallax, noises and errors due to non-linearity of the camera response curve~\cite{debevec_recovering_1997,lee_ghost-free_2014}. In consideration of such deviations, the following model may be used instead:
\begin{equation}
	M_i = g_{i}B + S_i,\qquad i=1,\ldots,n,\label{eqn:model}
\end{equation}
where $g_i>0$ is the gain factor, and $S_i\in\mathbb{R}^{h\times w}$ is the error term.
To capture the background similarity between $n$ images, we re-express the $n$ equations in~(\ref{eqn:model}) jointly as follows:
\begin{equation}
    \bM = \bL + \bS,
    \label{eqn:vec}
\end{equation}
where $\bM = \left[\vec(M_1), \vec(M_2), \dots,
\vec(M_n)\right]\in\mathbb{R}^{m\times n}$, $m=hw$, $\bL=[g_1\vec(B),g_2\vec(B),\dots,g_n\vec(B)]\in\mathbb{R}^{m\times n}$, $\bS=[\vec(S_1),\vec(S_2),\dots,\vec(S_n)]\in\mathbb{R}^{m\times n}$ and $\vec:\mathbb{R}^{h\times w}\longrightarrow\mathbb{R}^m$ is the linear operator that stacks the elements of a matrix into a vector. Plugging~(\ref{eqn:transform}) into~(\ref{eqn:vec}), we get
\begin{equation*}
    \bD\circ\tau=\bL+\bS,
\end{equation*}
where
$\bD\circ\tau=[\vec(I_1\circ\tau_1),\vec(I_2\circ\tau_2),\dots,\vec(I_n\circ\tau_n)]$,$\tau=[\tau_1,\tau_2,\dots,\tau_n]\in\mathbb{R}^{d \times n}$.
In all cases of interest, $g_i\neq 0$ and $B\neq 0$, and thus $\bL$ is clearly a rank-1 matrix; furthermore, we observe that
the errors $\{S_i\}$ usually appear spatially sparse, i.e., most of their elements
have very small magnitudes while a small number of elements can appear
quite large in magnitude (please refer to Fig.~\ref{fig:example} for a concrete example). Therefore, we employ the $\ell_1$ norm as the error
metric since it is both robust against gross sparse corruptions and stable
against small but dense noises~\cite{peng_rasl:_2012}.Consolidating the above models, we formulate the alignment problem as the following minimization problem:
\begin{align}
    \min_{\bL,\bS,\tau}\quad \|\bS\|_{\ell_1}\label{eqn:nonlinear}&\\
    \text{subject to }\quad\bD\circ\tau&=\bL+\bS,~\rank(\bL)=1,\nonumber
\end{align}
i.e., we try to fit a rank-1 matrix that deviates the warped images
(collectively) as small as possible, and simultaneously seek for a collection of
aligning geometric transformations.

In case of brightness changes, the $\bS$ component will accommodate the
differences as long as they are relatively small in magnitude.  When they become
large, we may revise the optimization problem~\eqref{eqn:nonlinear}
accordingly but this is an investigation outside the scope of this work.


\subsection{Efficient Optimization Algorithm}\label{subsec:solution}

A common practice
for handling non-linearity in numerical analysis and optimization is to
iteratively linearize the problem and solve the linearized version.
Indeed, this technique has been widely used in the image alignment
literature~\cite{baker_lucas-kanade_2004,peng_rasl:_2012}. Specifically,
under small perturbation $\Delta\tau\in\mathbb{R}^{d\times n}$ of $\tau$, we can expand
$\bD\circ(\tau+\Delta\tau)\approx
\bD\circ\tau+\sum_{i=1}^n\bJ_i\Delta\tau_i\be_i^T$, where
$\bJ_i=\left.\frac{\partial}{\partial\xi}\vec\left(I_i\circ\xi\right)\right|_{\xi=\tau_i}\in\mathbb{R}^{m\times
    d}$ is the Jacobian matrix of the $i$-th image with respect to its
	transformation parameters $\tau_i$. Problem~\eqref{eqn:nonlinear}
    then reduces to:
\begin{align}
	\min_{\bL,\bS,\Delta\tau} \quad &\|\bS\|_{\ell_1}\label{eqn:linearized}\\
	\text{subject to }\quad\bD_\tau+\sum_{i=1}^n\bJ_i\Delta\tau_i\be_i^T&=\bL+\bS,\label{eqn:eq_constraint}\\
					  \quad\rank(\bL)&=1,\nonumber
\end{align}
where we write $\bD_\tau=\bD\circ\tau$ for notational brevity.

For relatively large deviation of $\tau$, we can successively
solve~\eqref{eqn:linearized} and apply the update\footnote{Convergence of such
	linearization schemes is discussed in~\cite{peng_rasl:_2012}.}
$\tau\gets\tau+\dt$.
The remaining problem is how to efficiently solve~(\ref{eqn:linearized}).

Problem~\eqref{eqn:linearized} is non-smooth due to the $\ell_1$ norm, and
the dimensionality of the variables $\bL$ and $\bS$ can be very large.
Therefore, the solution method needs to be derivative-free and scalable.
For such purposes we choose the penalty method~\cite{nocedal_numerical_2006}. To
this end, we first form the penalty function
\[
    \mathrm{P}_{\zeta}(\bL,\bS,\dt)=\|\bS\|_{\ell_1}+\frac{1}{2\zeta}\left\|\bD_\tau+\sum_{i=1}^n\bJ_i\Delta\tau_i\be_i^T-\bL-\bS\right\|_F^2,
\]
where $\zeta>0$ is a constant parameter to control the strength of the
constraint~\eqref{eqn:eq_constraint}, and $\|\cdot\|_F$ denotes the Frobenius norm. When
$\zeta\rightarrow0$ the original equality constraint is strictly enforced.
Based on this fact, standard penalty method requires iteratively
solving the problem
\begin{equation}
    \min_{\bL,\bS,\dt}\mathrm{P}_{\zeta_k}(\bL,\bS,\dt)\quad\text{subject to}\quad\rank(\bL)=1,
    \label{eqn:penalty}
\end{equation}
for some positive real sequence ${\{\zeta_k\}}_k$ obeying $\zeta_k\rightarrow0$ as $k\rightarrow\infty$. However, directly solving the joint minimization problem~(\ref{eqn:penalty}) is difficult due to the rank-1 constraint, and we apply an alternating minimization scheme as follows:
\begin{align*}
    \bL^{k}&=\arg\min_{\bL:\rank(\bL)=1}\mathrm{P}_{\zeta_k}(\bL,\bS^{k-1},\dt^{k-1}),\\
    \bS^{k}&=\arg\min_{\bS}\mathrm{P}_{\zeta_k}(\bL^{k},\bS,\dt^{k-1}),\\
    \dt^{k}&=\arg\min_{\dt}\mathrm{P}_{\zeta_k}(\bL^{k},\bS^{k},\dt).
\end{align*}
Closed form solutions for every subproblems are available. We begin by defining the soft-thresholding operator $\cS_\zeta:\mathbb{R}^{m\times n}\longrightarrow\mathbb{R}^{m\times n}$ as
\begin{equation}
	\cS_\zeta\{\bX\}=\mathrm{sgn}(\bX)\cdot\max\{|\bX|-\zeta,0\},\label{eqn:sthr}
\end{equation}
where $\mathrm{sgn}(x)=1$ if $x>0$ and $0$ otherwise, $\cdot$ denotes elementwise product, and the sgn, max and $\vert {\mathbf X} \vert$ (absolute value) operators act elementwise. We also define the rank-1 projection operator $\cT_1:\mathbb{R}^{m\times n}\rightarrow\mathbb{R}^{m\times n}$ as
\begin{equation*}
    \cT_1\{\bX\}=\sigma_1\bu_1\bv_1^T,
\end{equation*}
where $\sigma_1$ is the largest singular value of $\bX$ and $\bu_1\in\mathbb{R}^m$,
$\bv_1\in\mathbb{R}^n$ are the corresponding left and right singular vectors,
respectively. 
\begin{algorithm}[t!]
    \renewcommand{\algorithmicrequire}{\textbf{Input:}}
    \renewcommand{\algorithmicensure}{\textbf{Output:}}
    \caption{Alternating Minimization for Solving~\eqref{eqn:linearized}}

    \begin{algorithmic}[1]
		\REQUIRE{$\bD_\tau\in\mathbb{R}^{m\times n}$, ${\{\bJ_i\in\mathbb{R}^{m\times d}\}}_{i=1}^n$,$\beta_0$,$\beta_1>0$,$q\in(0,1)$.}
		\STATE{$\bL^0\gets 0$,$\dt^0\gets 0$,$\zeta_0=\beta_0\frac{\|\bD_\tau\|_2}{\sqrt{mn}}$,$\bS^0\gets\cS_{\zeta_0}(\bD_\tau)$.\label{initialization}}
        \FOR{$k=1, 2, \dots, K$}
		\STATE{$\bL^{k}\gets\cT_1\left\{\bD_\tau+\sum_{i=1}^n\bJ_i\Delta\tau_i^{k-1}\be_i^T-\bS^{k-1}\right\}$,\label{lupdate1}}
		\STATE{$\zeta_{k}\gets\beta_1\frac{q^{k}}{\sqrt{mn}}\|\bL^{k}\|_2$,\label{parameter1}}
		\STATE{$\bS^{k}\gets\cS_{\zeta_{k}}\left\{\bD_\tau+\sum_{i=1}^n\bJ_i\Delta\tau_i^{k-1}\be_i^T-\bL^{k}\right\}$,\label{supdate1}}
		\STATE{$\dt^{k}\gets\sum_{i=1}^n\bJ_i^\dag(\bL^{k}+\bS^{k}-\bD_\tau)\be_i\be_i^T$.\label{dtupdate}}
        \ENDFOR
		\ENSURE{$\widehat{\bL}\gets\bL^K$, $\widehat{\bS}\gets\bS^K$, $\widehat{\dt}\gets\dt^K$.}
    \end{algorithmic}\label{alg:single}
\end{algorithm}

The complete algorithm is summarized in
Algorithm~\ref{alg:single}.
Interestingly, Algorithm~\ref{alg:single} is closely
related to recent studies about non-convex robust principal component
analysis~\cite{netrapalli_non-convex_2014}. In particular, if we set all the $\dt^k$ to $0$, and replace soft-thresholding with hard-thresholding, then
Algorithm~\ref{alg:single} reduces to Algorithm 1
in~\cite{netrapalli_non-convex_2014} {\it exactly\/} for the rank-1 case.

The choice of
${\{\zeta_k\}}_k$ as in Step 4 of Algorithm~\ref{alg:single} is crucial in our work. Justification for this choice is provided in Theorem~\ref{thm:convergence} which essentially guarantees that, with $\zeta_k$ chosen as in step~\ref{parameter1}, the sequences ${\{\bL^k\}}_k$, ${\{\bS^k\}}_k$ and
${\{\dt^k\}}_k$
converge to $\bL^\ast$, $\bS^\ast$ and $\dt^\ast$ (to be defined in Section~\ref{subsec:analysis}) under certain conditions and proper
choices of parameters.\footnote{For simplicity, we provide an
analysis assuming a noiseless model; given our result,
stability under small dense noise can be derived in  a manner similar to~\cite{netrapalli_non-convex_2014}}. Additionally, the maximum number of iterations $K$ in Algorithm~\ref{alg:single} can be predetermined given desired accuracy $\varepsilon>0$. 

\subsection{Convergence Analysis}\label{subsec:analysis}
To begin with, we suppose there
are underlying rank-1 matrix $\bL^\ast\in\mathbb{R}^{m\times n}$, sparse matrix $\bS^\ast\in\mathbb{R}^{m\times n}$ and
incremental transformation parameters
$\dt^\ast=[\Delta\tau^\ast_1,\Delta\tau^\ast_2,\dots,\Delta\tau^\ast_n]\in\mathbb{R}^{d\times n}$
acting together to generate the data matrix $\bD_\tau$:
\begin{equation}
    \bD_\tau=\bL^\ast+\bS^\ast-\sum_{i=1}^n\bJ_i\Delta\tau_i^\ast\be_i^T.
    \label{eqn:observation}
\end{equation}

Our next goal is to specify the conditions under which we can
recover $\bL^\ast$, $\bS^\ast$ and $\dt^\ast$ {\it exactly\/} from the
observed $\bD_\tau$. To quantify those conditions, we invoke the Singular Value
Decomposition (SVD) of $\bL^\ast:\bL^\ast=\sigma^\ast\bu^\ast{\bv^\ast}^T$ (note
$\bL^\ast$ is rank-1) and the QR decomposition of
$\bJ_i:\bJ_i=\bQ_i\bR_i$ where $\bQ_i\in\mathbb{R}^{m\times d}$ and $\bR_i\in\mathbb{R}^{d\times d},i=1,2,\dots,n$. Recall that $\dt^\ast$ is one
small incremental step of transformation parameters, so it is reasonable to assume
that it has small size; in particular, we assume
\begin{enumerate}
    \item [A.1] $\max_i\|\bR_i\Delta\tau^\ast_i\|_{\ell_2}\leq\gamma\frac{\sigma^\ast}{\sqrt{nd}}$ for some $\gamma>0$
\end{enumerate}
to ensure relatively small Frobenius norm of $\sum_{i=1}^n\bJ_i\dt_i^\ast\be_i^T$ compared to $\bL^\ast$.
On the other hand, we don't want to lose any generality on $\dt^\ast$, so
we won't pose any other assumptions on $\dt^\ast$. Therefore, for the
special case $\dt^\ast=0$, $\bL^\ast$ and $\bS^\ast$ must still be
recoverable. This falls back to the well-studied {\it robust principal
component analysis\/} problem~\cite{wright_robust_2009,candes_robust_2011}
and different sets of conditions~\cite{recht_guaranteed_2010,chandrasekaran_rank-sparsity_2011} on $\bL^\ast$ and $\bS^\ast$ have been proposed.
Intuitively they guard $\bL^\ast$ against being ``sparse'' and $\bS^\ast$
against being ``low-rank'' to resolve the identifiability issue.
We adopt the ones
discussed in~\cite{netrapalli_non-convex_2014}:
\begin{enumerate}
    \item [A.2] $\max_i|\bu_i^\ast|\leq\frac{\mu}{\sqrt{m}}, \max_i|\bv_i^\ast|\leq\frac{\mu}{\sqrt{n}}$ for some $\mu>0$;
    \item [A.3] $\bS^\ast$ has a fraction of at most $\alpha_1,\alpha_2$ non-zeros in each column and row respectively, for some $\alpha_1,\alpha_2\in(0,1)$.
\end{enumerate}

The parameter $\mu$ in A.2 is commonly referred to as {\it incoherence parameter\/}~\cite{candes_robust_2011};
intuitively it measures the ``similarity'' between the singular
vectors and the standard basis vectors ${\{\be_i\}}_i$.  As standard basis vectors are sparse, a  small incoherence parameter typically implies ``dense'' singular vectors and effectively prevents $\bL^\ast$ being sparse. A.3 prevents the nonzeros entries in $\bS^\ast$ gathering in the same rows or columns, and in turn prevents $\bS^\ast$ being low-rank.

\begin{figure}
    \subfloat[\label{fig:corrupt_L}] {\includegraphics[width=0.45\linewidth]{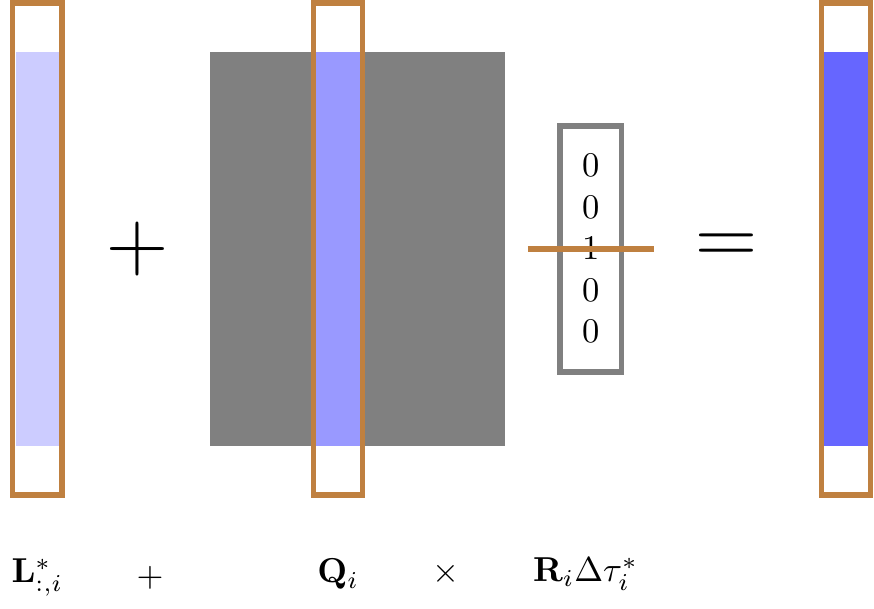}}\hfill
    \subfloat[\label{fig:corrupt_S}] {\includegraphics[width=0.45\linewidth]{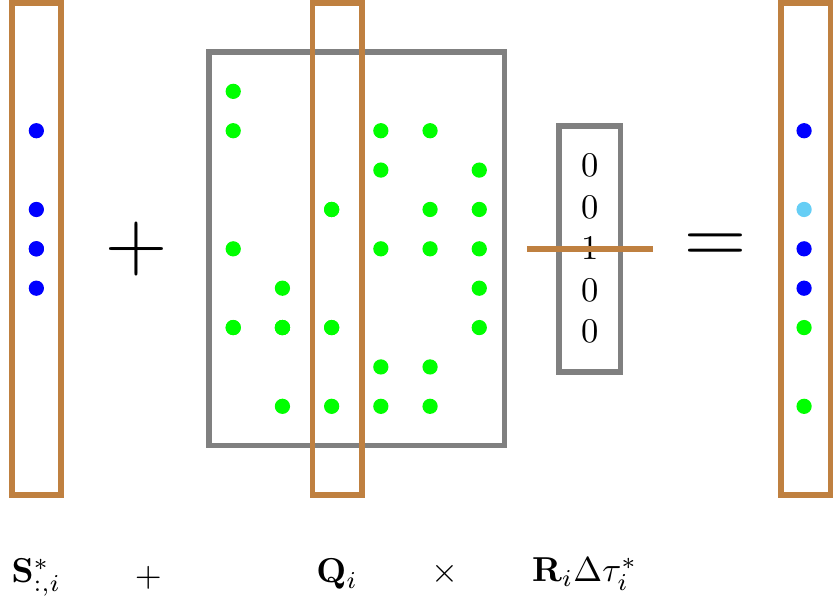}}\\

    \subfloat[\label{fig:low_rank}] {\includegraphics[width=\linewidth]{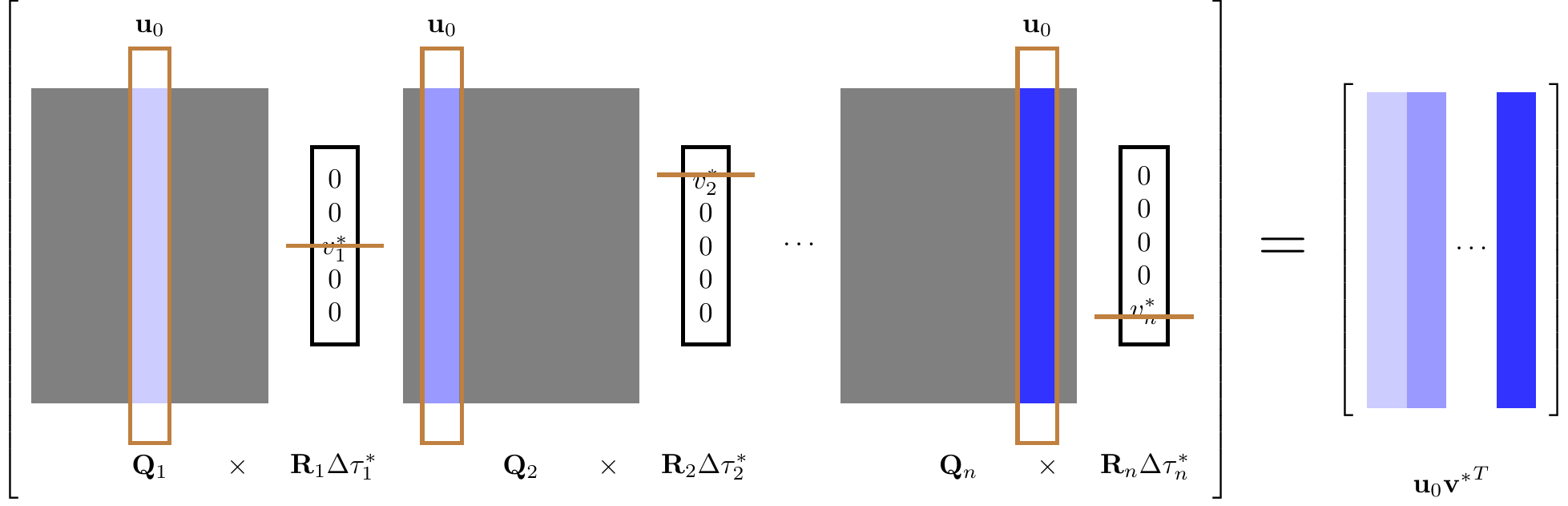}}
	\caption{Illustration of the technical assumptions A.4-A.6: (a) a $\bQ_i$ with some columns similar to $\bu^\ast$ may induce a vector $\bQ_i\bR_i\dt_i^\ast$ unidentifiable from $\bL_{:, i}^\ast$, the $i$-th column of $\bL$; (b) a sparse (coherent) $\bQ_i$ may induce a vector $\bQ_i\bR_i\dt^\ast_i$ unidentifiable from $\bS^\ast_{:,i}$, the $i$-th column of $\bS^\ast$; (c) highly correlated $\bQ_i$'s (with a common column $\bu_0\in\mathbb{R}^m$) may induce a rank-1 matrix $\sum_{i=1}^n\bJ_i\Delta\tau^\ast_i\be_i^T=\sum_{i=1}^n\bQ_i\bR_i\dt_i^\ast\be_i^T=\bu_0{\bv^\ast}^T$ that is unidentifiable from $\bL^\ast=\sigma^\ast\bu^\ast{\bv^\ast}^T$. }\label{fig:theory}\vspace{-10pt}
\end{figure}

Generally when $\dt^\ast\neq 0$ we have an additional component
$\sum_{i=1}^n\bJ_i\Delta\tau_i^\ast\be_i^T$, and in the same spirit we
need to ensure it is identifiable from both $\bL^\ast$ and $\bS^\ast$.
Note A.1 does not prevent $\bR_i\Delta\tau_i^\ast$ from moving towards
any directions; in particular, it may be aligned with any one of the
standard basis vectors in $\mathbb{R}^d$. In such a case
$\bJ_i\Delta\tau_i^\ast$ will be aligned with one of the columns in
$\bQ_i$, and it turns out either a sparse $\bQ_i$ or a $\bQ_i$ with some
columns similar to $\bu^\ast$ can potentially cause identifiability issues with certain columns of
$\bL^\ast$ or $\bS^\ast$ (see Fig.~\ref{fig:corrupt_L} and
Fig.~\ref{fig:corrupt_S} for a depiction). To avoid such cases we need two additional assumptions:
\begin{enumerate}
    \item [A.4] $\max_{ij}\|\bQ_j^T\be_i\|_{\ell_2}\leq\nu\sqrt{\frac{d}{m}}$ for some $\nu>0$;
    \item [A.5] $\max_i\|\bQ_i^T\bu^\ast\|_{\ell_2}\leq\kappa\sqrt{\frac{d}{m}}$ for some $\kappa>0$;
\end{enumerate}
here we follow the idea of incoherence in A.4 to measure non-sparsity of
$\bQ_i$, and in A.5 to measure the dissimilarity between $\bQ_i$ and
$\bu^\ast$. Finally, note that highly-correlated column vectors in different
$\bQ_i$'s may raise identifiability issues between $\bL^\ast$ and
$\sum_{i=1}^n\bJ_i\Delta\tau_i^\ast\be_i^T$ (see Fig.~\ref{fig:low_rank}
for an example); hence we further
assume small {\it first principal
angles\/}~\cite{soltanolkotabi_robust_2014} between $\bQ_i$'s (on average):
\begin{enumerate}
    \item [A.6] $\frac{1}{n-1}\sum_{j\neq i}\|\bQ_j^T\bQ_i\|_2\leq \delta$ for some $\delta>0$.
\end{enumerate}

Equipped with assumptions A.1-A.6, we are now ready
to state our major theorem as follows:
\begin{theorem}
	There are positive constants $C_\delta$, $C_q<1$, $C_d=\mathcal{O}(\sqrt{m})$, $C_{\alpha_1}=\mathcal{O}\left(d^{-1}\right)$, $C_{\alpha_2}=\mathcal{O}\left(d^{-\frac{1}{2}}\right)$ such that,
	if $\alpha_1\leq C_{\alpha_1}$, $\alpha_2\leq C_{\alpha_2}$, $d\leq C_d$, $\delta\leq C_\delta$, $C_q\leq q < 1$  then $\forall\beta_0\in\left[\frac{\widetilde{\mu}\sigma^\ast}{\|\bD_\tau\|_2},\frac{2\widetilde{\mu}\sigma^\ast}{\|\bD_\tau\|_2}\right]$, $\beta_1\in\left[2\widetilde{\mu},
	2.2\widetilde{\mu}\right]$ where $\widetilde{\mu}=\mu^2+\gamma\nu$, each $\bS^k
	(k\geq 0)$ in Algorithm~\ref{alg:single}~\footnote{Note that initialization of Algorithm~\ref{alg:single} is fixed as shown in Step~\ref{initialization}.} satisfies
	$\supp(\bS^k)\subset\supp(\bS^\ast)$; furthermore, $\forall\varepsilon>0$ and $K\geq\log_q\left(\frac{\varepsilon}{\beta_0\|\bD_\tau\|_2}\right)$,\footnote{Note that $K$ is decreasing w.r.t.\ increase in $\varepsilon$ as $q<1$.}
	$\|\widehat{\bL}-\bL^\ast\|_F\leq\varepsilon$,
	$\|\widehat{\bS}-\bS^\ast\|_{\ell_\infty}\leq\frac{5\varepsilon}{\sqrt{mn}}$,
	$\|\widehat{\dt}-\dt^\ast\|_F\leq M\varepsilon$
	for some $M>0$.\label{thm:convergence}
\end{theorem}
\begin{proof}
	We define $\bE^k=\bS^\ast-\bS^k$,
	$\bF^k=\sum_{i=1}^n\bQ_i\bQ_i^T(\bS^k-\bS^\ast)\be_i\be_i^T$,
	$\bG^k=\sum_{i=1}^n\bQ_i\bQ_i^T(\bL^k-\bL^\ast)\be_i\be_i^T$ and
	$\bH^k=\bF^k+\bG^k$. We are going to show the two real sequences
	$s_k:=\|\bE^k\|_{\ell_\infty}$ and
	$l_k:=\max_i\|\bG^k\be_i\|_{\ell_2}$ decrease exponentially, and this
	will establish a claim about linear convergence. We proceed by
	induction. For the basic case $k=0$, $\forall(i,
	j)\not\in\supp(\bS^\ast)$, ${(\bS^\ast)}_{ij}=0$ and it follows that
	\begin{align*}
		\left|{(\bD_\tau)}_{ij}\right|&\leq\|\bL^\ast\|_{\ell_\infty}+\left\|\sum_{i=1}^n\bJ_i\dt^\ast\be_i\be_i^T\right\|_{\ell_\infty}\\
						   &\leq\sigma^\ast\|\bu^\ast\|_{\ell_\infty}\|\bv^\ast\|_{\ell_\infty}+\max_{i,j}\left|\be_i^T\bQ_i\bR_i\dt^\ast\be_j\right|\\
						   &\leq\sigma^\ast\frac{\mu^2}{\sqrt{mn}}+\nu\sqrt{\frac{d}{m}}\frac{\sigma^\ast}{\sqrt{nd}}\gamma\leq\zeta_0,
	\end{align*}
	and thus ${\left(\bS^0\right)}_{ij}=0$ and $\supp\left(\bS^0\right)\subset\supp\left(\bS^\ast\right)$.
	Further, for $d\leq\frac{\sqrt{m}}{16\mu^2\nu\kappa}$,
	$l_0=\sigma^\ast\max_i\left\|\bQ_i^T\bu^\ast\right\|_{\ell_2}\left|{(\bv^\ast)}_i\right|\leq\frac{1}{16\mu\nu\sqrt{d}}\frac{\sigma^\ast}{\sqrt{n}}$
	and
	$s_0=\|\bE^0\|_{\ell_\infty}\leq\zeta^0+\|\bL^\ast\|_{\ell_\infty}+\|\sum_{i=1}^n\bJ_i\dt^\ast\be_i\be_i^T\|_{\ell_\infty}\leq\frac{3\sigma^\ast\widetilde{\mu}}{\sqrt{mn}}$.
	Assume that $l_k\leq\frac{1}{16\mu\nu\sqrt{d}}\frac{\sigma^\ast q^k}{\sqrt{n}}$,
	$\supp(\bS^k)\subset\supp(\bS^\ast)$ and $s_k\leq\frac{5\widetilde{\mu}\sigma^\ast q^k}{\sqrt{mn}}$; we can show that
	$l_{k+1}\leq\frac{1}{16\mu\nu\sqrt{d}}\frac{\sigma^\ast q^{k+1}}{\sqrt{n}}$,
	$\supp(\bS^{k+1})\subset\supp(\bS^\ast)$ and $s_{k+1}\leq\frac{5\widetilde{\mu}\sigma^\ast q^{k+1}}{\sqrt{mn}}$. We defer this procedure to the supplementary document.
	For $K\geq\log_q\left(\frac{\varepsilon}{\beta_0\|\bD_\tau\|_2}\right)$, and by~\eqref{ineq:const} in the supplementary document, $\|\bL^K-\bL^\ast\|_{\ell_\infty}\leq\frac{\mu^2\sigma^\ast q^K}{\sqrt{mn}}\leq\beta_0\frac{\|\bD_\tau\|_2q^K}{\sqrt{mn}}\leq\frac{\varepsilon}{\sqrt{mn}}$
	and $\|\widehat{\bS}-\bS^\ast\|_{\ell_\infty}=\|\bS^K-\bS^\ast\|_{\ell_\infty}\leq5\beta_0\frac{\|\bD_\tau\|_2q^K}{\sqrt{mn}}\leq\frac{5\varepsilon}{\sqrt{mn}}$;
	thus $\|\widehat{\bL}-\bL^\ast\|_F=\|\bL^K-\bL^\ast\|_F\leq\sqrt{n}l_K\leq\frac{\varepsilon}{16\mu^3\nu\sqrt{d}}$ and $\|\bS^K-\bS^\ast\|_F\leq\sqrt{\alpha_1mn}s_K$, and therefore
	\begin{align*}
		\|\widehat{\dt}-\dt^\ast\|_F&=\|\dt^K-\dt^\ast\|_F\\
									&=\sqrt{\sum_{i=1}^n\|\bJ_i^\dag(\bL^K-\bL^\ast+\bS^K-\bS^\ast)\be_i\|_{\ell^2}^2}\\
	&\leq\left(\max_{1\leq i\leq n}\|\bJ_i^\dag\|_2\right)\left(\frac{1}{16\mu^3\nu\sqrt{d}}+5\sqrt{\alpha_1}\right)\varepsilon.
	\end{align*}\vspace{-2mm}
\end{proof}
\noindent We only present a sketch of the key arguments above, the detailed proof is available in the supplementary document.

To get a sense for computational complexity, note that step~\ref{lupdate1} in Algorithm~\ref{alg:single} can be done
in $\cO(mn)$ time without invoking full SVD~\cite{golub_MC_1996}. The approximate complexity
for Algorithm~\ref{alg:single} is thus
$\cO(mnd^2\log\left(\frac{1}{\varepsilon}\right))$ to achieve $\varepsilon$ accuracy.

\begin{remark}
	The assumptions stated above are in fact quite reasonable.	Assumption A.1 is reasonable as $\Delta\tau$ is expected to be relatively small (note this condition is on $\Delta\tau$ and not $\tau$) by design. Assumptions A.2 and A.3 are widely accepted in the relevant literature~\cite{jain_low-rank_2013,netrapalli_non-convex_2014,wright_compressive_2013,wright_robust_2009}. Assumptions A.4 to A.6 serve to avoid identifiability issues. We must emphasize that A.4 --- A.6 may not always hold for a ${\mathbf D}_{\tau}$ formed from practical datasets. Note however, that these conditions when fully satisfied provide rather strong exact recovery guarantees as confirmed via Theorem~\ref{thm:convergence} (or perfect alignment). In practice, even with departures from some of these conditions, enforcing a rank-1 constraint is highly meritorious and leads to improved performance. Experimental studies that corroborate this follow in Section~\ref{subsec:rasl}.
\end{remark}

\subsection{Relationship to Traditional Pixel-Based Methods}\label{subsec:relation}
There is a large variety of pixel-based methods in the image alignment literature as~\cite{szeliski_creating_1997,bartoli_groupwise_2008,evangelidis_parametric_2008,gay-bellile_direct_2010}. Invariably, they work on aligning image pairs by solving:
	\begin{equation}
		\min_{\tau,\cP}\rho(I_2\circ \tau,\cP(I_1))\label{eqn:pixel_based}
	\end{equation}
	where $I_1$, $I_2$ are given image pairs to be aligned, and $\tau$, $\cP$
are the geometric and photometric transformations, respectively. The real-valued function $\rho$ obeys $\rho(I_1,I_2)\geq 0$ and equals $0$ if and only if $I_1=I_2$. For
instance, the method in~\cite{szeliski_creating_1997} corresponds to
$\rho(I_1,I_2)=\|I_1-I_2\|^2_2$ and $\cP(I)=I$, whereas in~\cite{bartoli_groupwise_2008} an affine model on $I$ is used. Regularizers may also be added. For
instance, the method in~\cite{gay-bellile_direct_2010} encourages
smoothness of the warp and shrinkage at the self-occlusion boundary.
To handle occlusions, error functions such as the Huber loss function~\cite{huber_robust_1964} may be used.

As a clear benefit, our method can of course work on aligning a batch of images. We can draw an analogy between traditional pixel-based methods and ours when $n=2$. We let $\bD=[\vec(I_2),\vec(I_1)]$. Through eliminating $\bS$, we may rewrite~(\ref{eqn:nonlinear}) as
	\begin{align}
		\min_{\bL,\tau}\|\bD\circ\tau-\bL\|_{\ell^1},\quad\text{ subject to }\rank(\bL)=1,\label{eqn:ll1}
	\end{align}
	Note that $\rank(\bL)=1$ if and only if its two columns $L_1,L_2\in\mathbb{R}^m$ are linearly dependent. In most cases of interest, $\bL=0$ will not be the minimizer to~(\ref{eqn:nonlinear}), and without loss of generality we can write $L_1=gL_2$ for some constant $g>0$. Therefore, we may drop the rank-1 constraint and rewrite~(\ref{eqn:ll1}) as
	\begin{align}
		\min_{g,L_2,\tau}\|I_2\circ\tau-gL_2\|_{\ell^1}+\|I_1-L_2\|_{\ell^1}.\label{eqn:pairwise}
	\end{align}
	Thus our method acts like ``mean absolute deviation'' (around the median) of $I_2\circ\tau$ and $I_1$, with $g$ accounting for the photometric differences. When there are multiple images ($n\geq 3$), pairwise registration may be suboptimal~\cite{szeliski_image_2006}. It is less obvious to extend~\eqref{eqn:pixel_based} to such cases than~\eqref{eqn:ll1}. From this perspective, our method provides a principled approach of jointly aligning multiple images and accounts for photometric differences and occlusions simultaneously.

\section{Bundle Robust Alignment for Practical Image Stitching}\label{sec:bundle_alignment}


The flowchart of our panoramic image composition scheme is illustrated in Fig.~\ref{fig:flowchart}. 

\begin{figure}
	\includegraphics[width=\linewidth]{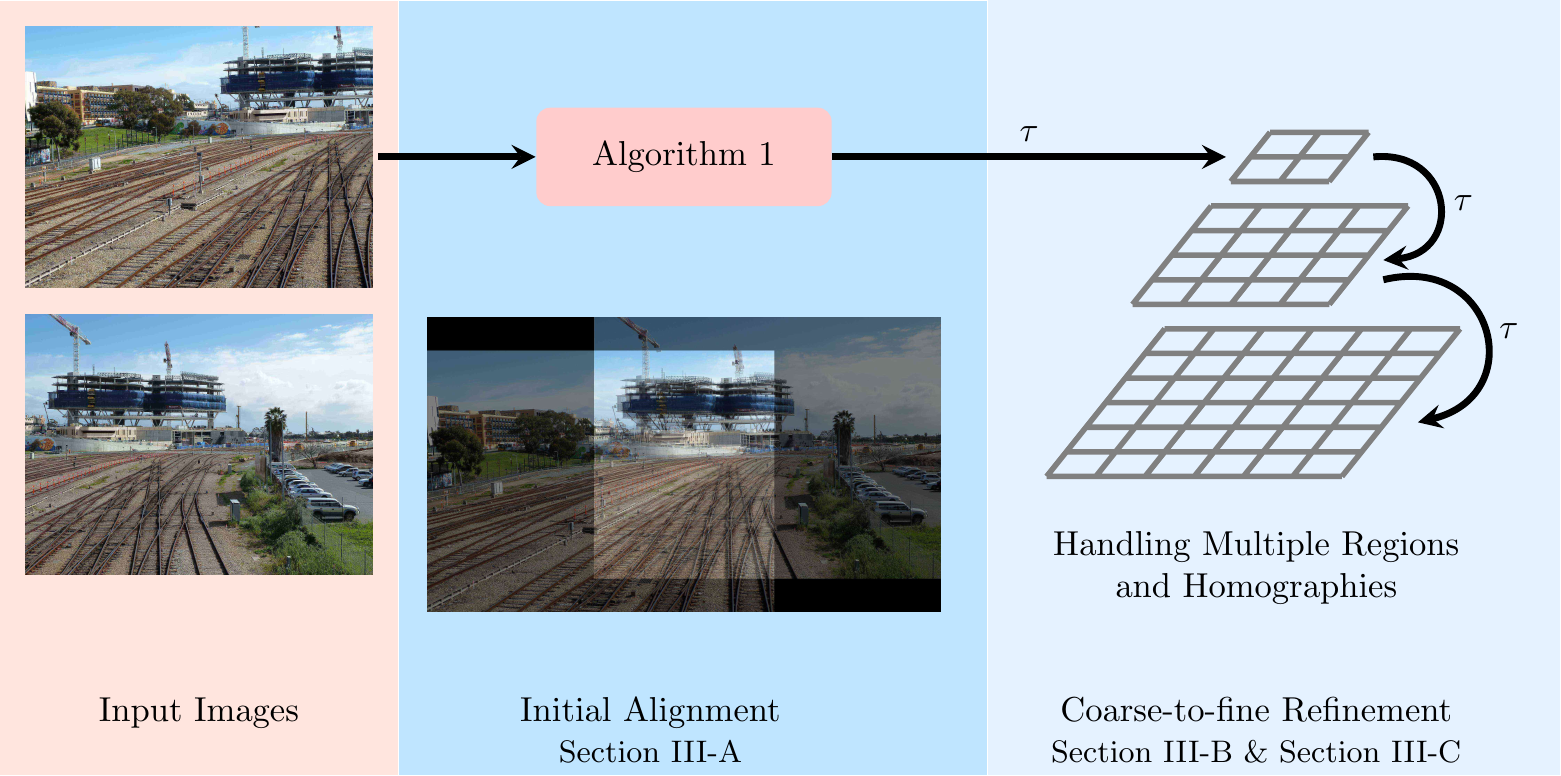}
	\caption{Flowchart of our panoramic alignment for stitching pipeline. See Section~\ref{subsec:initial} to~\ref{subsec:multiple} for details.}\label{fig:flowchart}
\end{figure}

\subsection{Estimating Initial Transformation Parameters}\label{subsec:initial}
\begin{figure}
	\centering
	\subfloat[Input]{\begin{tikzpicture}
			\node (railtracks01) {\includegraphics[width=0.25\linewidth]{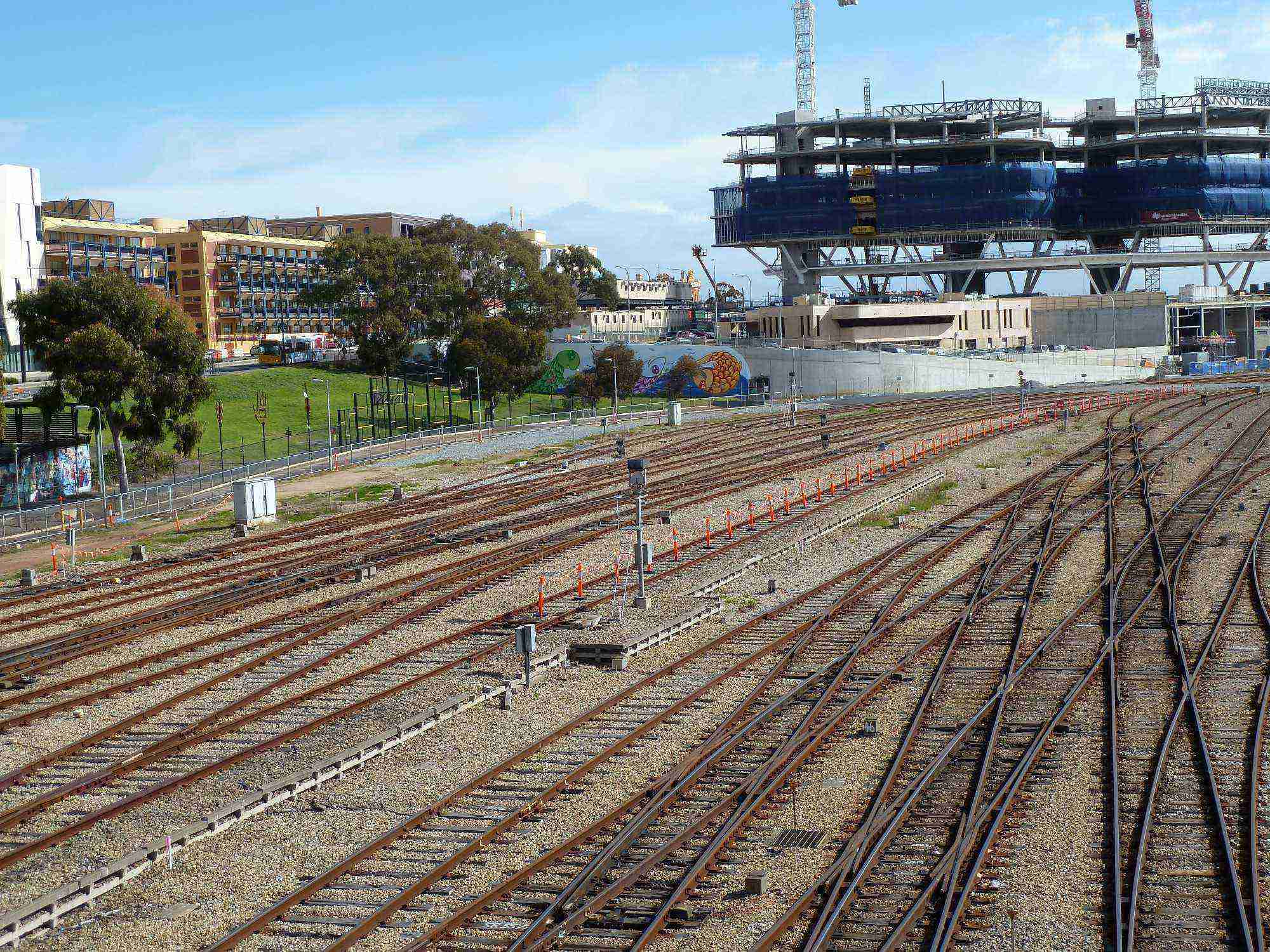}};
			\node [below=-0.3em of railtracks01](railtracks02) {\includegraphics[width=0.25\linewidth]{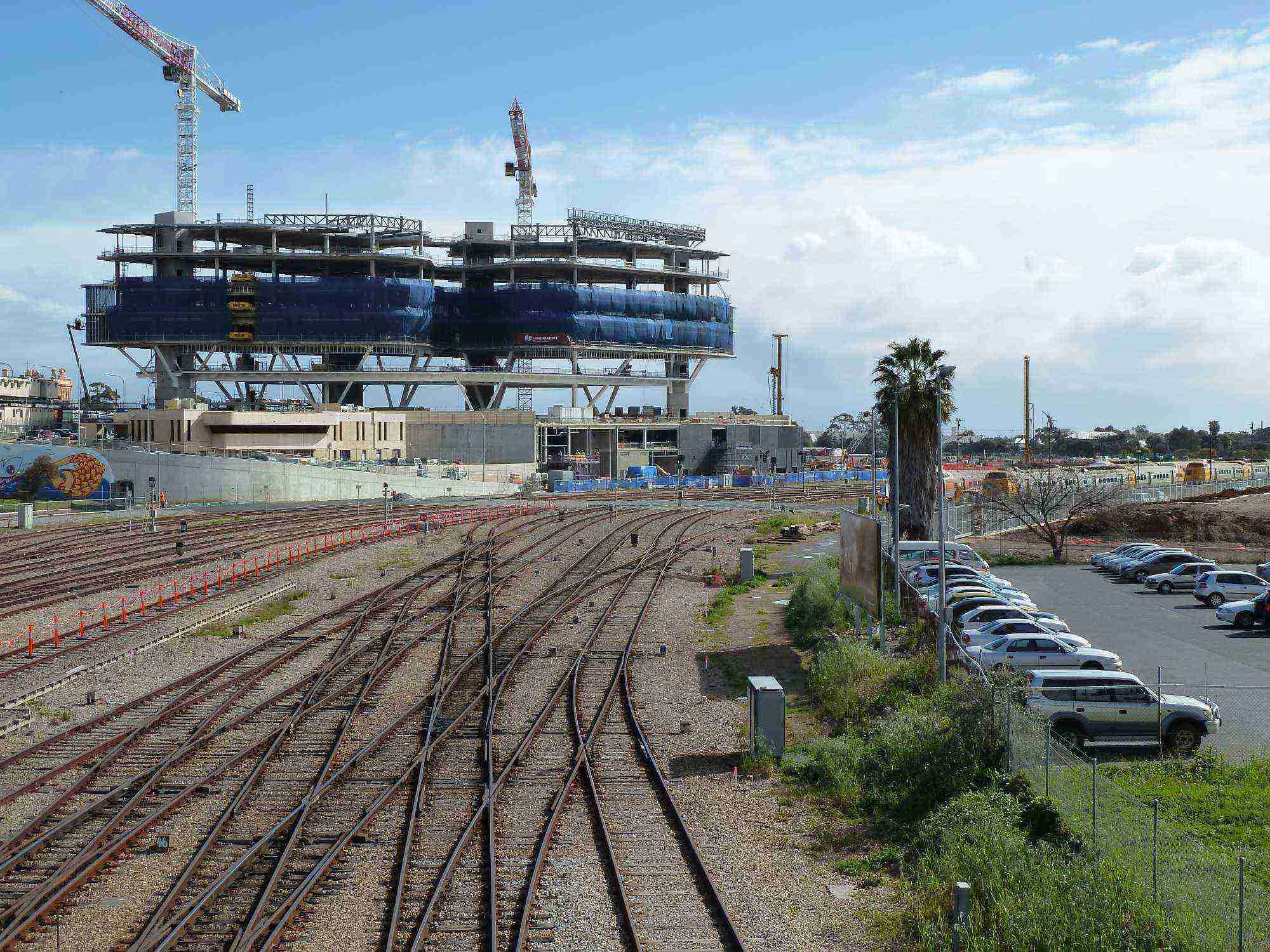}};
		\end{tikzpicture}
	}
	\subfloat[Initially Aligned]{\includegraphics[width=0.7\linewidth]{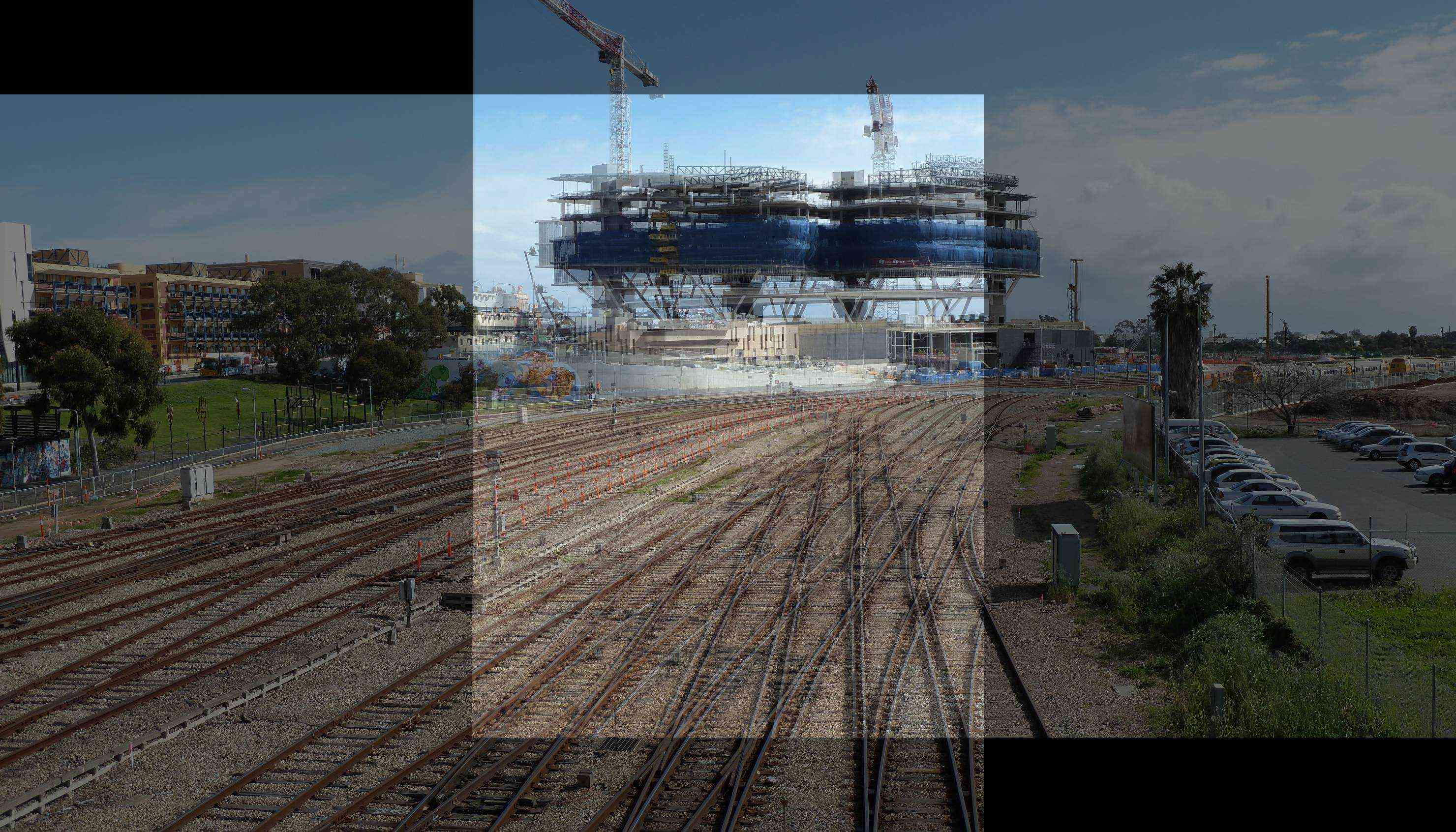}}
	\caption{Initially aligned and overlayed images on the {\it railtracks\/}~\cite{zaragoza_as-projective-as-possible_2014} dataset, using the method in Section~\ref{subsec:initial}. Rescaled for better display.}\label{fig:initial}
\end{figure}
The iterative linearization scheme discussed in
Section~\ref{subsec:formulate} becomes invalid when $\tau$ is large. To mitigate this, we adopt a coarse-to-fine pyramidal
implementation: we decompose the images into Gaussian pyramids, and
progressively refine $\tau$ at each scale.  In image stitching, the
displacements between images are often  large, and we still
need to properly initialize $\tau$ at the beginning of the coarsest scale. One way of accomplishing this is by traditional pixel based methods but they cannot handle large displacements. Therefore,  in the following, we develop a probabilistic approach.

For multiple images, we estimate the parameters between consecutive pairs of images, and
then chain them together. For each pair of images, we extract their
corresponding SIFT images~\cite{liu_sift_2011} $s_1$ and $s_2$ where
$s_i(i=1,2)$ comprises a 128-dimensional SIFT feature vector per pixel.
Let $\tau_{12}$ be the transformation (homography) parameters from $s_2$ to $s_1$. Our idea is to
estimate $\tau_{12}$ from dense correspondences (motion field) between $s_1$
and $s_2$. We leverage a robust error function as the difference measure between $s_1$ and $s_2$ to address large displacements and significant occlusions. We also enforce spatial contiguity on the motion field to enhance robustness.

Let $\bs=\{s_1,s_2\}$. For tractability of minimization, we model the relationship between $\bs$ and $\tau_{12}$ using
Hidden Markov Random Field (HMRF): we view $\bs$ as observed variables, and
the motion vector $\bw(\bp)=(u(\bp),v(\bp))$ at every pixel $\bp$ as
latent variables. Here $u(\bp),v(\bp)$ take integer values between $-L$ and $L$ for a given integer $L$ that serves as an upper bound on the pixel displacements. Let $\bw={\{\bw(\bp)\}}_\bp$. We then set up the following probabilistic models:
\begin{align}
    p(\bs|\bw)\propto\exp\biggl\{-&\sum_{\bp}\min(\|s_1(\bp)-s_2(\bp+\bw(\bp))\|_{\ell_1}, \kappa)\biggr\},\label{eqn:data_cost}\\
	p(\bw|\tau_{12})\propto\exp\biggl\{-&\eta\sum_{\bp}\|\bw(\bp)-t(\bp;\tau_{12})\|_{\ell_1}\label{eqn:fidelty_global}\\
    -&\alpha\sum_{(\bp,\bq)\in\mathcal{N}}\|\bw(\bp)-\bw(\bq)\|_{\ell_2}^2\biggr\},\label{eqn:motion_smoothness}
\end{align}
where $\kappa,\eta,\alpha$ are positive constants, and $\mathcal{N}$ comprises
all pairs of 4-connected neighboring pixels.
(\ref{eqn:data_cost}) enforces
robust feature matching similar to~\cite{liu_sift_2011}; (\ref{eqn:fidelty_global})
encourages small deviations of $\bw(\bp)$ from $t(\bp;\tau_{12})$, the coordinate of $\bp$ under the transformation represented by $\tau_{12}$, while~\eqref{eqn:motion_smoothness} encourages small differences
of $\bw$ between adjacent pixels $\bp$, $\bq$. We estimate $\tau_{12}$ following the maximum likelihood estimation principle:
\begin{equation}
	\widehat{\tau_{12}}\gets\arg\max_{\tau_{12}}\sum_\bw p(\bs|\bw)p(\bw|\tau_{12}).\label{eqn:mle}
\end{equation}
A standard approach to solving~(\ref{eqn:mle}) is the
expectation-maximization (EM) algorithm~\cite{bishop_2006_PRML}, as listed in
Algorithm~\ref{alg:em}, where we let $t(\bp;\tau_{12})=(t_u(\bp;\tau_{12}),t_v(\bp;\tau_{12}))$. To monitor the convergence behavior, we define $d_{h,w}(\tau_1,\tau_2)$ to measure the difference between transformation parameters $\tau_1$ and $\tau_2$:
\begin{equation}
	d_{h, w}(\tau_1,\tau_2)=\frac{1}{hw}\sum_{x=1}^w\sum_{y=1}^h\left\|t((x,y);\tau_1)-t((x,y);\tau_2)\right\|_{\ell_2}^2,
    \label{eqn:rmse}
\end{equation}
i.e., it measures the mean squared error between transformed
coordinates $t((x,y);\tau_1)$ and $t((x,y);\tau_2)$ by applying $t(\cdot;\tau_1)$ and $t(\cdot;\tau_2)$ to an $h\times w$ image. Fig.~\ref{fig:initial} shows an example of the initialization. The images are rescaled to the original resolution for display purposes. After running Algorithm~\ref{alg:em}, the images are coarsely aligned but Algorithm~\ref{alg:region} in Section~\ref{subsec:overlap} (which employ the rank-1 constraint) is required for precise alignment.

\begin{algorithm}[!]
    \renewcommand{\algorithmicrequire}{\textbf{Input:}}
    \renewcommand{\algorithmicensure}{\textbf{Output:}}
    \caption{Initial Transformation Parameter Estimation}

    \begin{algorithmic}[1]
		\REQUIRE{SIFT images $s_1, s_2\in\mathbb{R}^{w_2\times h_2\times 128}$, $\epsilon>0,L\in\mathbb{N}$.}
		\STATE{Initialize $\tau_{12}^0$ to identity transformation, $k\gets0$.}
		\REPEAT{}
		\STATE{Use belief propagation~\cite{felzenszwalb_efficient_2006} to compute probabilities $\omega_{\bp,l}^u\gets p(u(\bp)=l|\bs,\tau_{12}^k)$, $\omega_{\bp,l}^v\gets p(v(\bp)=l|\bs,\tau_{12}^k),\forall$ pixel $\bp,l=-L,-L+1,\dots,L$.}
		\STATE{$k\gets k+1$.}
		\STATE{$\tau_{12}^k\gets\arg\min_{\tau_{12}}\sum_{\bp}\sum_{l=-L}^L|t_u(\bp;\tau_{12})-l|\omega^u_{\bp, l}+|t_v(\bp;\tau_{12})-l|\omega_{\bp, l}^v$.}
        \UNTIL{$d_{h_2,w_2}(\tau_{12}^k,\tau_{12}^{k-1})<\epsilon$}
		\ENSURE{estimated transformation parameters $\widehat{\tau_{12}}$.}
    \end{algorithmic}\label{alg:em}
\end{algorithm}

\subsection{Handling Multiple Overlapping Regions}\label{subsec:overlap}
When there are multiple input images they usually overlap across multiple
regions. \tcr{We have experimentally observed that simply zero-padding the
warped images to the size of canvas often causes divergence; on the other
hand,} the naive way of applying pairwise alignment consecutively is suboptimal
since the alignment errors may propagate and
accumulate~\cite{szeliski_image_2006}. Therefore, it is essential to develop a
method to account for the complicated overlapping relationships simultaneously.

The basic idea is to apply the rank-1 and sparse decomposition discussed in
Section~\ref{sec:robust_alignment} on every overlapping region. A visualization of the model is in Fig.~\ref{fig:region_model}. We again let $m$ be the number of pixels on canvas and $n$ be the number of images. The $m\times n$ matrix thus formed is called the canvas matrix. We will consistently use $r$ as region index
and $i$ as image index. $\sum_r$ and $\sum_i$ means summing over each region and each image, respectively. Assuming
predetermined regions, we reformulate
Equation~\eqref{eqn:nonlinear} as
\begin{align}
	\min_{{\{\bL_r\}}_r,{\{\bS_r\}}_r,\tau}\quad&\sum_r\left\|\bS_r\right\|_{\ell_1}\label{eqn:region}\\
    \text{subject to}\quad\cR_r(\bD\circ\tau)&=\bL_r+\bS_r,\quad \forall r,\nonumber\\
	\mathrm{rank}(\bL_r)&=1,\qquad\forall r,\nonumber
\end{align}
where $\cR_r:\mathbb{R}^{m\times n}\rightarrow\mathbb{R}^{m_r\times n_r}$ is the operator that extracts the portion of $\bD\circ\tau$ belonging to the $r$-th region and $\bL_r$, $\bS_r\in\mathbb{R}^{m_r\times n_r}$ are the corresponding
rank-1 and sparse components. The linearized problem is
\begin{align}
	\min_{{\{\bL_r\}}_r,{\{\bS_r\}}_r,\dt}\quad\sum_r\left\|\bS_r\right\|_{\ell_1}&\nonumber\\
    \text{subject to}\quad\bD_r + \sum_{i\in\cI_r}\bJ_{r,i}\Delta\tau_i\be_{f_r^i}^T&=\bL_r+\bS_r,\quad \forall r,\nonumber\\
    \mathrm{rank}(\bL_r)&=1,\qquad \forall r,\label{eqn:linearized_region}
\end{align}
where we write $\cR_r(\bD\circ\tau)$ as $\bD_r$ for brevity;
$\bJ_{r,i}\in\mathbb{R}^{m_r\times d}$ is the Jacobian of the $i$-th image
restricted to the $r$-th region. We let $\cI_r$ be the indices of images
contributing to the $r$-th overlapping regions, $\bar{\cI}_i$ be the indices of
regions in the $i$-th image, and $f_r^i$ be the column index of the $i$-th
image in $\bD_r$. Algorithm~\ref{alg:region} is used to
solve~\eqref{eqn:linearized_region}, and the accompanying theoretical analysis
is included in the supplementary document.

\begin{figure}[h!]
    \centering
    \includegraphics[width=\linewidth]{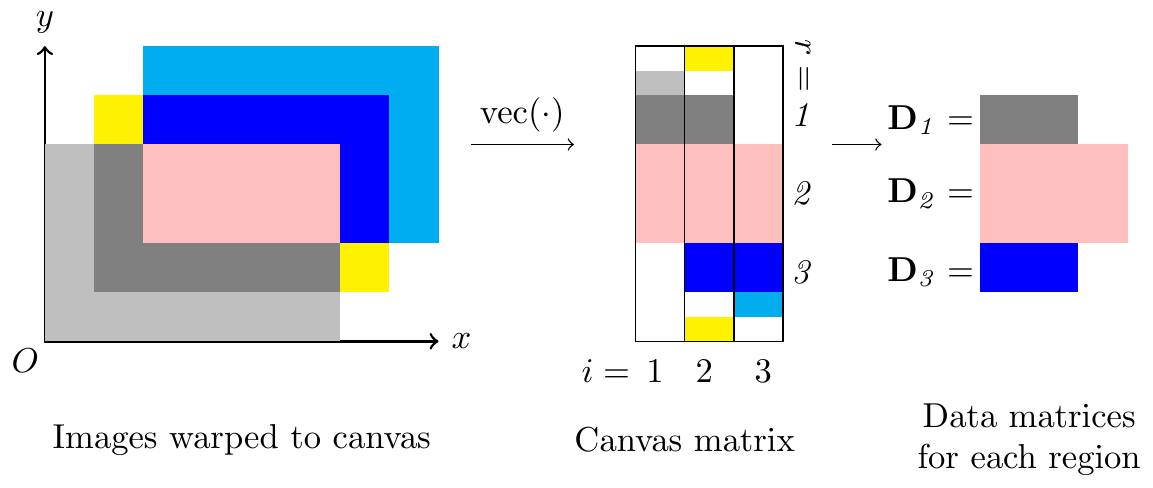}
	\vspace{-7mm}
	\caption{The overlapping region model and relevant notations: images on canvas may overlap in different regions, according to which we may determine the data matrices $\bD_\mathit{1}$,$\bD_\mathit{2}$ and $\bD_\mathit{3}$. Here
		$\cI_{\mathit{1}}=\{1, 2\}$,
		$\cI_{\mathit{2}}=\{1, 2, 3\}$;
        $\bar{\cI}_1=\{\mathit{1},
        \mathit{2}\}$,$\bar{\cI}_2=\{\mathit{1},\mathit{2},\mathit{3}\}$; $f_\mathit{3}^2=1$, $f_\mathit{3}^3=2$,
    etc.}\label{fig:region_model}
\end{figure}

In practice, we determine the regions approximately by updating them at each linearization
step before solving~\eqref{eqn:linearized_region}. Specifically, we
warp the images using the currently estimated $\tau$ and repeatedly scope out the overlapping regions. Each time we pick the region with the highest number of
contributing images until no overlaps remain.\footnote{As a concrete example, in Fig.~\ref{fig:region_model} region $\mathit{2}$ is discovered first as it has the highest number of overlapping images (3).}

\begin{algorithm}[t]
    \renewcommand{\algorithmicrequire}{\textbf{Input:}}
    \renewcommand{\algorithmicensure}{\textbf{Output:}}
    \caption{Alternating Minimization for Solving~(\ref{eqn:linearized_region})}

    \begin{algorithmic}[1]
		\REQUIRE{${\{\bD_r\in\mathbb{R}^{m_r\times n_r}\}}_r$, ${\{\bJ_{r,i}\in\mathbb{R}^{m_r\times d}\}}_{r,i}$,$\beta_0$,$\beta_1$,$q>0$.}
        \FORALL{$r$}
		\STATE{$\bS_r^0\gets\cS_{\zeta_r^0}(\bD_r)$ where $\zeta_r^0=\beta_0\frac{\|\bD_r\|_2}{\sqrt{m_{r}n_r}}$.}
        \ENDFOR
		\STATE{$\dt^0\gets\sum_i\sum_{r\in\bar{\cI}_i}{\left(\sum_{l\in\bar{\cI}_i}\frac{\bJ_{l, i}^T\bJ_{l, i}}{\zeta_l^0}\right)}^\dag\frac{\bJ_{r, i}^T(\bS_r^0-\bD_r)\be_{f_r^i}\be_i^T}{\zeta_r^0}$.}\label{dt0}
        \FOR{$k=1,2, \dots, K$}
        \FORALL{$r$}
		\STATE{$\bL_r^{k}\gets\cT_1\left\{\bD_r+\sum_{i\in\mathcal{I}_r}\bJ_{r,i}\Delta\tau_i^{k-1}\be_{f_r^i}^T-\bS_r^{k-1}\right\}$,}\label{lupdate}
		\STATE{$\zeta_r^{k}\gets\beta_1\frac{q^{k}}{\sqrt{m_{r}n_r}}\|\bL_r^{k}\|_2$,}\label{parameter}
		\STATE{$\bS_r^{k}\gets\cS_{\zeta_r^{k}}\left\{\bD_r+\sum_{i\in\mathcal{I}_r}\bJ_{r,i}\Delta\tau_i^{k-1}\be_{f_r^i}^T-\bL_r^{k}\right\}$.}\label{supdate}
        \ENDFOR
		\STATE{$\dt^{k}\gets\sum_{i,r\in\bar{\cI}_i}\!\!{\left(\sum_{l\in\bar{\cI}_i}\!\!\frac{\bJ_{l, i}^T\bJ_{l, i}}{\zeta_l^{k}}\right)}^\dag\!\frac{\bJ_{r,i}^T(\bL_r^{k}+\bS_r^{k}-\bD_r)\be_{f_r^i}\be_i^T}{\zeta_r^{k}}$.}\label{dtk}
        \ENDFOR
		\ENSURE{$\widehat{\bL_r}\gets\bL_r^K$, $\widehat{\bS_r}=\bS_r^K$, $\widehat{\dt}=\dt^K$.}
    \end{algorithmic}\label{alg:region}
\end{algorithm}

\subsection{Extension to Multiple Homographies}\label{subsec:multiple}
Recent
studies~\cite{gao_constructing_2011,zaragoza_as-projective-as-possible_2014}
have revealed that the classic single homography model is inadequate to
represent camera motions accurately for casually taken photos.
Fig.~\ref{fig:multiple_fail} illustrates one such example. To handle such cases,
we need to extend our method to incorporate more general motion
models.

\begin{figure}[!]
    \centering
    \setcounter{subfigure}{0}
    \subfloat[Input]{\includegraphics[width=0.25\linewidth]{figs/railtracks_01}\,\!
    \includegraphics[width=0.25\linewidth]{figs/railtracks_02}}\,
    \subfloat[Overlayed]{\includegraphics[width=0.45\linewidth]{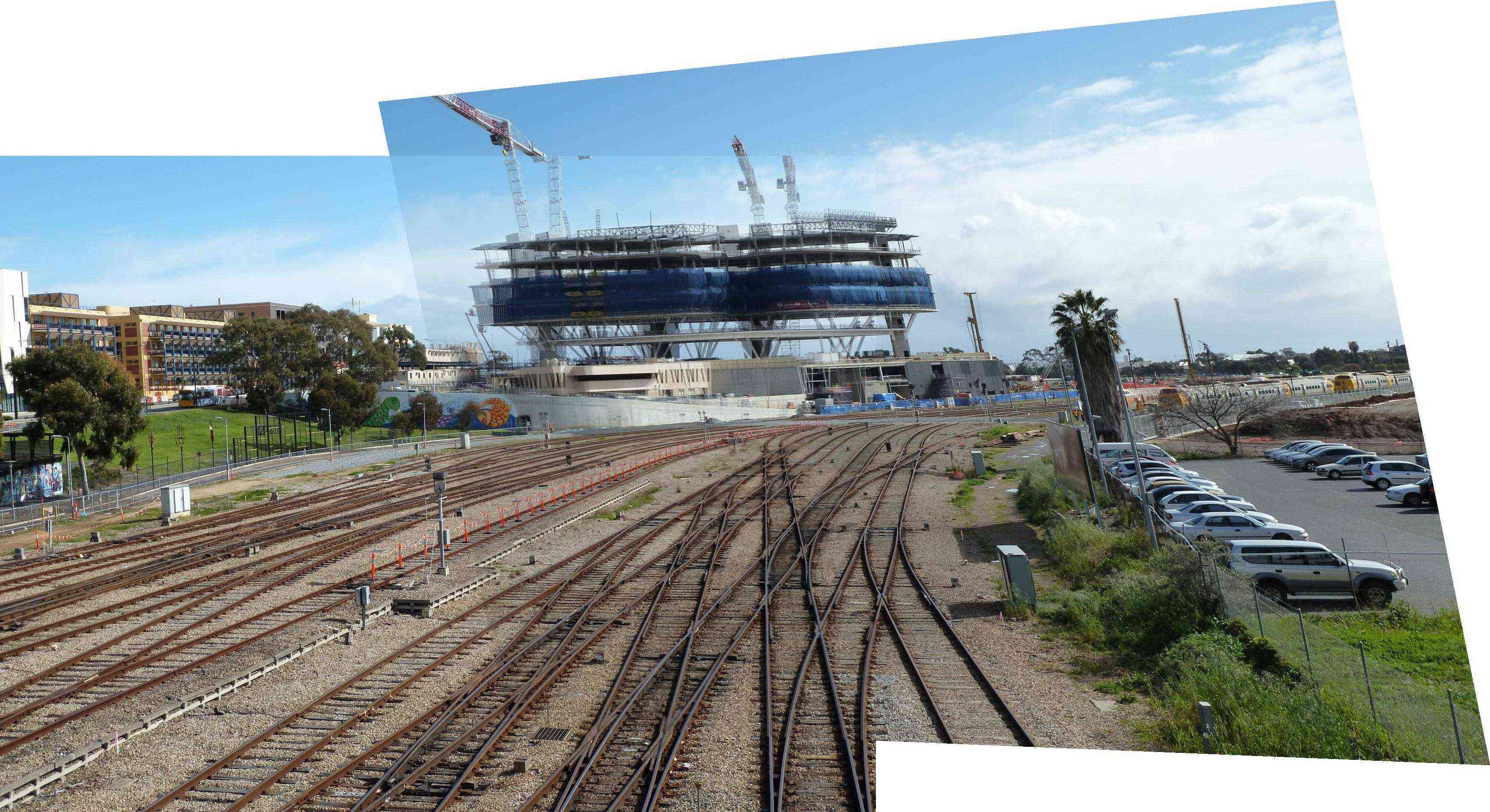}}
	\caption{A failure case of single homography in modeling camera motion: the homography estimated using the ground plane induces apparent misalignments of the building.}\label{fig:multiple_fail}
\end{figure}

\begin{figure}
	\centering
	\includegraphics[width=\linewidth]{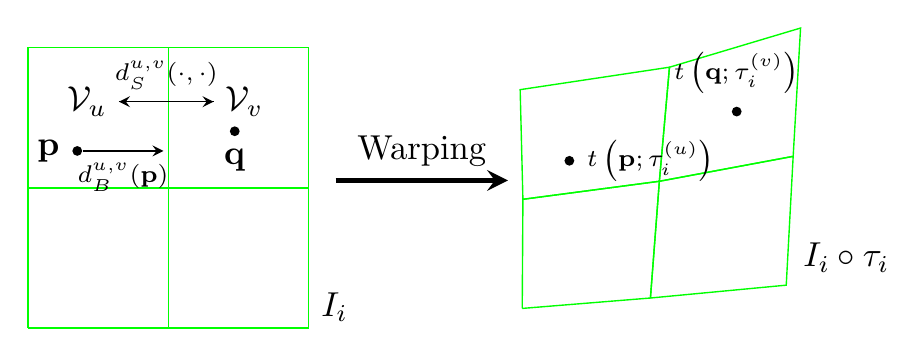}
	\caption{Illustration of the multiple homography model and relevant notations: we divide each image $I_i$ into cells, and associate each cell with its own set of homography parameters. We further encourage smoothness of the homographies in neighboring cells $u$ and $v$.}\label{fig:multiple}
\end{figure}

While our approach may be combined with generic motion models, developing a
dedicated flexible model is outside the scope of this work. Moreover, we would
like to ensure consistency in experimental comparisons. To this end, we
integrate the model in~\cite{zaragoza_as-projective-as-possible_2014} into our
approach. Specifically, we partition each image (indexed by $i$) into
$C_1\times C_2$ cells, and warp each cell (indexed by $u$) using an individual
set of homography parameters $\tau_i^{(u)}$. We stack $\tau_i^{(u)}$s together
as $\tau_i={\left(\tau_i^{(u)}\right)}_u$ and define $I_i\circ\tau_i$ as the
image warped cell by cell. A visual depiction of this model is in
Fig.~\ref{fig:multiple}.

To avoid tearing the objects and maintain stability in extrapolation along
non-overlapping regions, it is necessary to enforce smoothness of the
underlying geometric transformation.  We thus introduce a smoothness term over
$\tau$ and modify problem~\eqref{eqn:region} as:
\begin{align}
	\min_{{\{\bL_r\}}_r,{\{\bS_r\}}_r,\tau}\quad\sum_r\left\|\bS_r\right\|_{\ell_1}&+\frac{\lambda}{2}\sum_{i=1}^n\sum_{(u,v)\in\cE_i}d^{u,v}_S\left(\tau_i^{(u)},\tau_i^{(v)}\right)\nonumber\\
	\text{subject to}\quad\cR_r(\bD\circ\tau)&=\bL_r+\bS_r,\quad\forall r\nonumber\\
	\mathrm{rank}(\bL_r)&=1,\qquad\forall r,\label{eqn:formulate_multiple}
\end{align}
where $\lambda>0$ is a constant parameter that controls the strength of smoothness, $\cE_i$ comprises pairs of adjacent cells in the $i$-th image and $d^{u,v}_S\left(\tau_i^{(u)},\tau_i^{(v)}\right)$ promotes smoothness between $\tau_i^{(u)}$ and $\tau_i^{(v)}$, defined as the following:
\begin{equation}
	\sum_{\bp\in{\cV_{u}\cup\cV_{v}}}\exp{\left\{-\frac{{d_B^{u,v}(\bp)}^2}{\sigma^2}\right\}}\left\|t\left(\bp;\tau_i^{(u)}\right)-t\left(\bp;\tau_i^{(v)}\right)\right\|_{\ell_2}^2,\label{eqn:weight_gaussian}
\end{equation}
where $\sigma>0$ is a constant parameter, $\cV_u$ and $\cV_v$ collect pixels in
cell $u$ and $v$, respectively, and $d_B^{u,v}(\bp)$ is the Euclidean distance
from $\bp$ to the common boundary of cell $u$ and $v$. The reasoning behind
$d^{u,v}_S(\cdot,\cdot)$ is to encourage neighboring homographies to be
consistent on pixels close to the cell boundaries and thus prevent
discontinuities along those boundaries.

To solve~\eqref{eqn:formulate_multiple}, we can adopt the same iterative linearization scheme as in Section~\ref{subsec:solution}; in addition to $I_i\circ\tau_i$, we linearize each $t\left(\bp;\tau^{(u)}_i\right)$ in~\eqref{eqn:weight_gaussian} with respect to $\tau^{(u)}_i$ as:
$t\left(\bp;\tau_i^{(u)}+\Delta\tau_i^{(u)}\right)\approx t\left(\bp;\tau_i^{(u)}\right)+\bJ_i^{(u)}\Delta\tau_i^{(u)}$ where $\bJ_i^{(u)}=\left.\frac{\partial t(\bp;\xi)}{\partial\xi}\right|_{\xi=\tau_i^{(u)}}$. The optimization algorithm for solving the linearized subproblems of~\eqref{eqn:formulate_multiple} resembles Algorithm~\ref{alg:region}, except that in Step~\ref{dt0} and~\ref{dtk} the following formula should be used instead (for $k\geq 0$):
	\[
		\Delta\tau^k\gets\sum_{i=1}^n\!\!{\left(\sum_{l\in\bar{\cI}_i}\!\!\frac{\bJ_{l, i}^T\bJ_{l, i}}{\zeta_l^{k}}+\lambda\bm{\Gamma}_i\right)}^\dag\!\left(\br^k_i+\lambda\bt_i\right)\be_i^T
\]
where $\br^k_i=\sum_{r\in\bar{\cI}_i}\frac{\bJ_{r,i}^T}{\zeta_r^k}\left(\bL_r^{k}+\bS_r^{k}-\bD_r\right)\be_{f_r^i}$ and we define $\bL_r^0=0$. $\bm{\Gamma}_i$, $\bt_i$ are obtained by organizing the following linear system into the form $\bm{\Gamma}_i\Delta\tau_i=\bt_i$:
\begin{align}
	&\sum_{v\in\cN_u}\sum_{\bp\in\cV_u\cup\cV_v}w_{\bp}^{u,v}{\bJ_i^{(u)}}^T\left[\bJ_i^{(u)}\Delta\tau_i^{(u)}-\bJ_i^{(v)}\Delta\tau_i^{(v)}\right]\label{eqn:linear_system}\\
	=&\sum_{v\in\cN_u}\sum_{\bp\in\cV_u\cup\cV_v}w_{\bp}^{u,v}{\bJ_i^{(u)}}^T\left[t\left(\bp;\tau_i^{(v)}\right)-t\left(\bp;\tau_i^{(u)}\right)\right],\,\forall u,\nonumber
\end{align}
where $w_{\bp}^{u,v}=\exp{\left\{-\frac{{d^{u,v}_B(\bp)}^2}{\sigma^2}\right\}}$
and $\cN_u$ comprises cells adjacent to $u$.  After aligning all the images, we
employ seam-cutting~\cite{agarwala_interactive_2004} and gradient-domain
blending~\cite{perez_poisson_2003} to stitch them.

\section{Experimental Results}\label{sec:results}

\subsection{\tcr{Experimental Settings}}\label{sec:setup}

\tcr{To evaluate the influence of different parameter choices on alignment performance, we collect a validation set comprising $5$ sets of images to be aligned. The images are taken from Adobe Panorama Dataset~\cite{brandt_transform_2010}. To quantitatively assess the performance of BRAS, we
introduce truncated $\ell_2$-norm $d_t$ between images $I_1$ and $I_2$:
\begin{align*} d_t(I_1,I_2) = \sqrt{\frac{1}{m}
	\sum_{i=1}^{m}\min(\|I_{1}(\bp_i) - I_{2}(\bp_i)\|^{2}_{\ell_2}, t^2)},
\end{align*}
where $t$ is a positive constant fixed to $25$ intensity values in
all the experiments,\footnote{All images tested are  8-bits per pixel in the
intensity channel.} $I_{1}(\bp_i)$ and $I_{2}(\bp_i)$ are the values of the
$i$-th pixel in the overlapping region (comprising $m$ pixels) from image $I_1$
and $I_2$, respectively. This quantity serves as a robust measure of alignment
accuracy~\cite{ask_optimal_2013}, and smaller value generally indicates higher
alignment accuracy. For more than two images, errors are computed pairwise and then accumulated. The average values of this metric under various parameter settings are summarized in Table~\ref{tab:parameters}.}

\tcr{It can be clearly observed that the performance of BRAS is
robust to small perturbations of its parameter values. Therefore, unless otherwise stated, we}
set $\beta_0=\beta_1=1$, $q=0.7$ in Algorithm~\ref{alg:region}.
In~\eqref{eqn:fidelty_global} $\eta$ takes either $10^{-3}$ or $10^{-4}$ while
in~\eqref{eqn:motion_smoothness} $\alpha$ is chosen as either $1$ or $10$. The
number of cells (as described in~\ref{subsec:multiple}) is fixed to
$C_1=C_2=16$, and we reduce it by half when moving to a higher pyramid level.
In~\eqref{eqn:formulate_multiple} we choose $\lambda=N_p$ or $\lambda=10N_p$
where $N_p$ is the number of pixels in each cell while
in~\eqref{eqn:weight_gaussian} we let $\sigma=0.2$. We choose
$\epsilon=10^{-3}$ in Algorithm~\ref{alg:em}, and $\epsilon=10^{-5}$ in
Algorithm~\ref{alg:region}.  We employ normalized coordinates as suggested
in~\cite{hartley_defense_1997} for better conditioning of the linear system in
Equation~\eqref{eqn:linear_system}.

\begin{table}[h]
	\centering
	\caption{\tcr{Effects of parameter variations on alignment accuracy (measured in truncated $\ell^2$ norm).}}
	\begin{tabular}{ccccccc}
		\toprule
		$\beta_0,\beta_1$ & $q$ & $\eta$ & $\alpha$ & $\sigma$ & $C_1,C_2$ & Truncated $\ell^2$ norm\\
		\toprule
		$1$ & $0.5$ & $10^{-3}$ & $10$ & $0.2$ & $16$ & $7.55$\\
		\midrule
		$1$ & $0.7$ & $10^{-3}$ & $10$ & $0.2$ & $16$ & $7.51$\\
		\midrule
		$1.5$ & $0.5$ & $10^{-3}$ & $10$ & $0.2$ & $16$ & $7.50$\\
		\midrule
		$1$ & $0.5$ & $10^{-4}$ & $10$ & $0.2$ & $16$ & $7.55$\\
		\midrule
		$1$ & $0.5$ & $10^{-3}$ & $1$ & $0.2$ & $16$ & $7.55$\\
		\midrule
		$1$ & $0.5$ & $10^{-3}$ & $10$ & $0.3$ & $16$ & $7.59$\\
		\midrule
		$1$ & $0.5$ & $10^{-3}$ & $10$ & $0.2$ & $8$ & $7.58$\\
		\bottomrule
	\end{tabular}\label{tab:parameters}
\end{table}

For results generated with competing state of the art methods, we either use
publicly available code with instructions made available by the authors of the
work or contacted the authors directly for generation of results. For
fairness in comparison, we consistently apply the same stitching method
described in Section~\ref{subsec:multiple} across all methods.

\subsection{Validation on Random Simulations}\label{subsec:random}
\begin{figure}
    \subfloat[Convergence speed] {\includegraphics[width=0.5\linewidth]{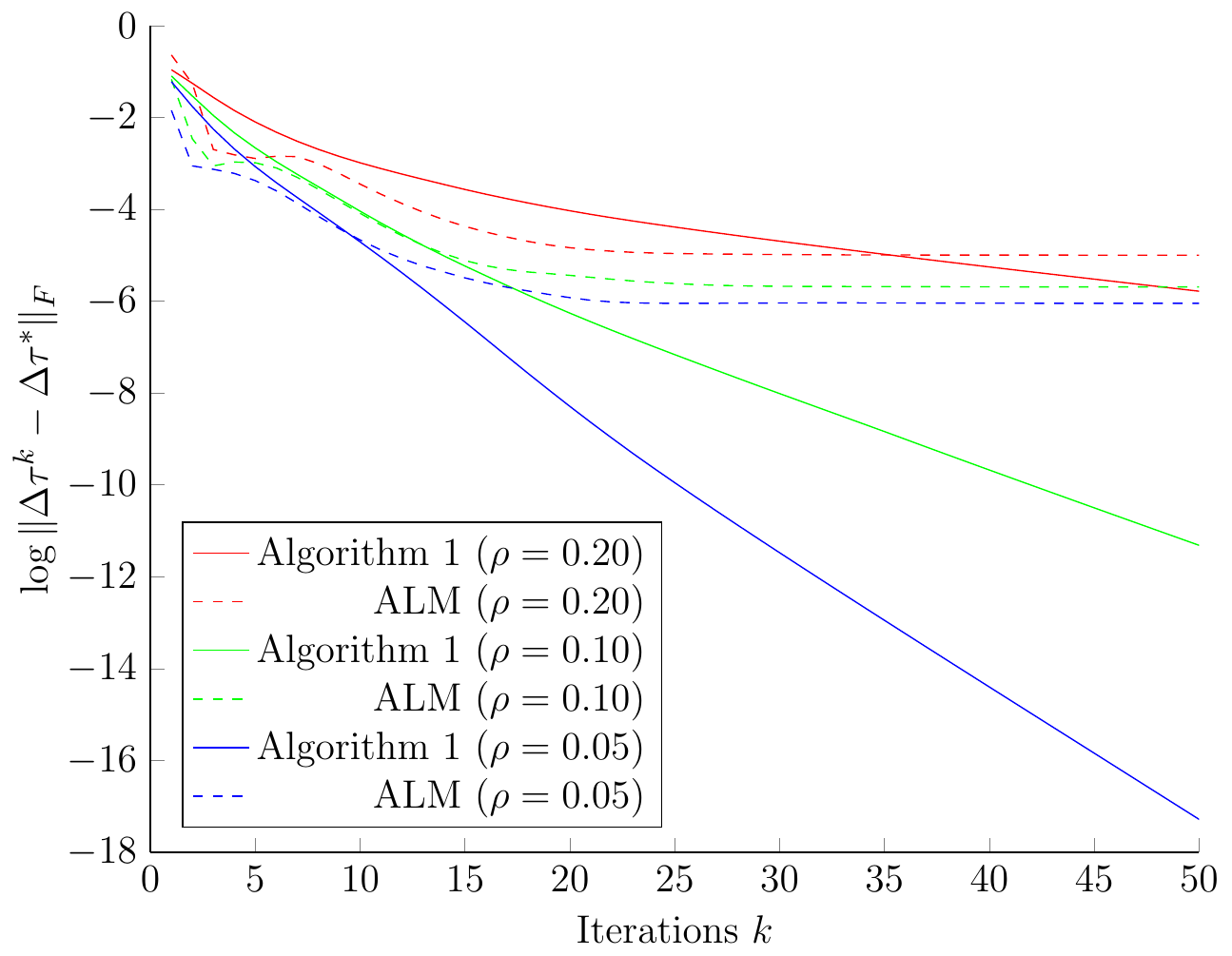}\label{fig:convergence}}
    \subfloat[Stability under noise] {\includegraphics[width=0.5\linewidth]{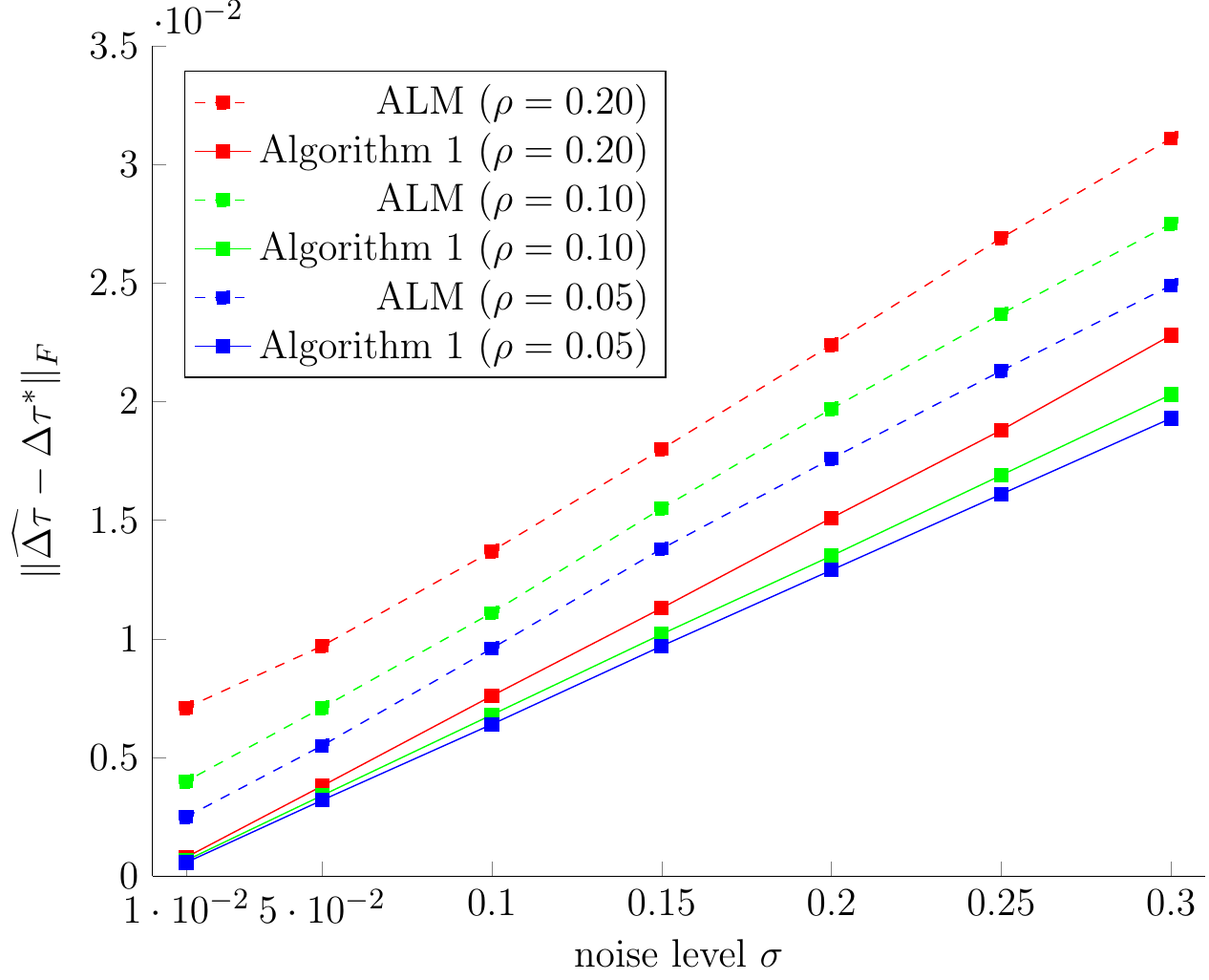}\label{fig:noise}}
	\caption{Performance evaluation under different occlusion levels $\rho$. \protect\subref{fig:convergence} The convergence rate of Algorithm~\ref{alg:single} is linear and it achieves lower final errors than ALM\@; \protect\subref{fig:noise} Algorithm~\ref{alg:single} degrades less with increasing noise levels compared with ALM.}\label{fig:simulation}
\end{figure}
We perform random simulations to quantitatively study the performance of
Algorithm~\ref{alg:single} and verify the assertions of
Theorem~\ref{thm:convergence}. We let $m=500$, $n=10$, and generate $\bD$ according to
formula~(\ref{eqn:observation}). In accordance
with~\cite{candes_robust_2011}, we generate
$\bL^\ast=\bu^\ast{\bv^\ast}^T$ with
$\bu^\ast\sim\cN\left(0,\frac{1}{m}\bI_m\right)$,
$\bv^\ast\sim\cN\left(0,\frac{1}{n}\bI_n\right)$ where $\bI_m$ is the
$m\times m$ identity matrix. The entries of $\bS^\ast$ take $-1$ and $1$
independently and equally likely, while $\supp(\bS^\ast)$ is a randomly
sampled subset of $\{1,\dots,m\}\times\{1,\dots,n\}$ with cardinality
$\rho mn$, $\rho\in(0,1)$.
$\bJ_i\sim\cN\left(0,\frac{d}{m}\bI_{md}\right)$ and
$\bR_i\dt_i^\ast\sim\cN\left(0,\frac{\|\bL^\ast\|_2}{nd}\bI_d\right)$. For
comparison, we include Algorithm 2 in~\cite{peng_rasl:_2012}, where an
Augmented Lagrange Multiplier (ALM) method was proposed to solve the convex
surrogate of~\eqref{eqn:linearized}. We vary $\rho$ to study the influence
of different amount of occlusions. We set $q=0.75,0.85,0.95$ in
Algorithm~\ref{alg:single} when $\rho=0.05,0.1,0.2$, respectively. We
execute both algorithms for $50$ iterations and plot their convergence
curves in Fig.~\ref{fig:convergence}. Both algorithms progress slower
under larger $\rho$, as larger occlusions are generally more difficult to
handle. The errors
$\|\Delta\tau^k-\Delta\tau^\ast\|_F$ for Algorithm~\ref{alg:single} decay
exponentially at rate $q$, in agreement with Theorem~\ref{thm:convergence}. ALM
progresses aggressively up to $20$ iterations, but stays at certain error
levels. In contrast, Algorithm~\ref{alg:single} steadily progresses towards smaller errors.

In practice, it is often the case that $\bD$ is affected by random noise~\cite{zhou_stable_2010};
in particular, some errors may be introduced in the
linearization~(\ref{eqn:linearized}). We study their effects by adding an
noise matrix $\bN$ to $\bD$, where $\bN_{ij}\sim\cN(0,\frac{1}{mn})$. In
all cases we set $q=0.95$ in Algorithm~\ref{alg:single}, and run both
algorithms for $100$ iterations for fairness. We measure the errors in their outputs
$\widehat{\Delta\tau}$ under varying levels of noise, i.e., different
$\sigma$, and different occlusion levels $\rho$. The results are
summarized in Fig.~\ref{fig:noise}. Similar to Fig.~\ref{fig:convergence},
increasing $\rho$ leads to increases in errors in both algorithms.
Clearly, Algorithm~\ref{alg:single} exhibits smaller errors than ALM in
every cases, indicating its higher stability under different amount of
noise. Indeed, as we observed in our experiments on real data, our method can
tolerate larger deviations of $\tau$ than~\cite{peng_rasl:_2012}.

\subsection{Effectiveness of the Rank-1 Constraint}\label{subsec:rasl}
To further verify the effectiveness of the exact rank-1 constraint as opposed to the conventional convex relaxation approaches, we compare with
RASL~\cite{peng_rasl:_2012} over the \emph{temple} dataset~\cite{gao_constructing_2011} under the same settings \emph{except} that the rank-1 constraint is replaced by the convex relaxation. The results are in Fig.~\ref{fig:rasl_align}. It may be inferred from Fig.~\ref{fig:rasl_align} that BRAS generates much better alignment. \tcr{An additional visual comparison is shown in Fig.~\ref{fig:rasl_align_skyscraper} in the supplementary document.}

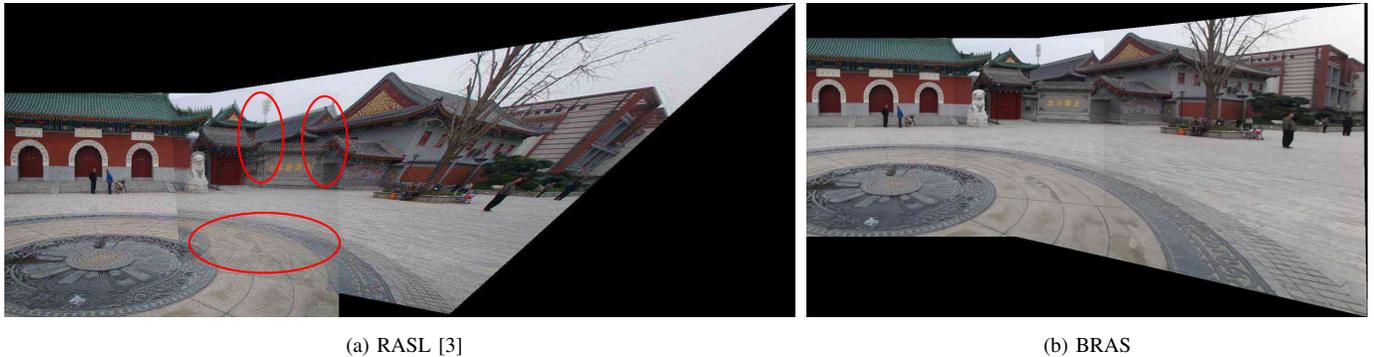
\begin{figure*}
	\subfloat[RASL~\cite{peng_rasl:_2012}]{%
		\begin{tikzpicture}
			\node {\includegraphics[height=0.17\textheight]{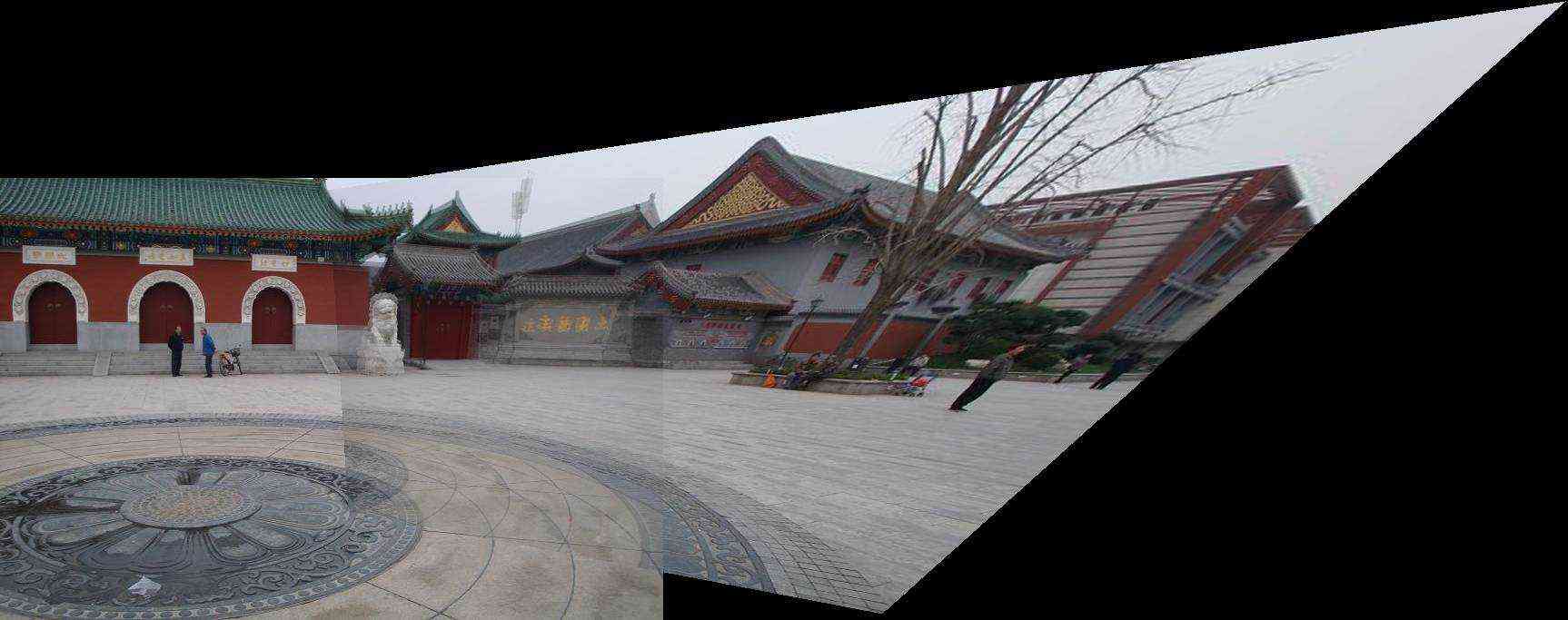}};
			\draw[thick,red] (-1.8,-1.1) ellipse (1cm and 0.4cm);
			\draw[thick,red] (-1,0.25) ellipse (0.3cm and 0.6cm);
			\draw[thick,red] (-1.85,0.3) ellipse (0.3cm and 0.6cm);
		\end{tikzpicture}
	}\hspace{-4mm}
	\subfloat[BRAS]{
		\begin{tikzpicture}
			\node {\includegraphics[height=0.17\textheight]{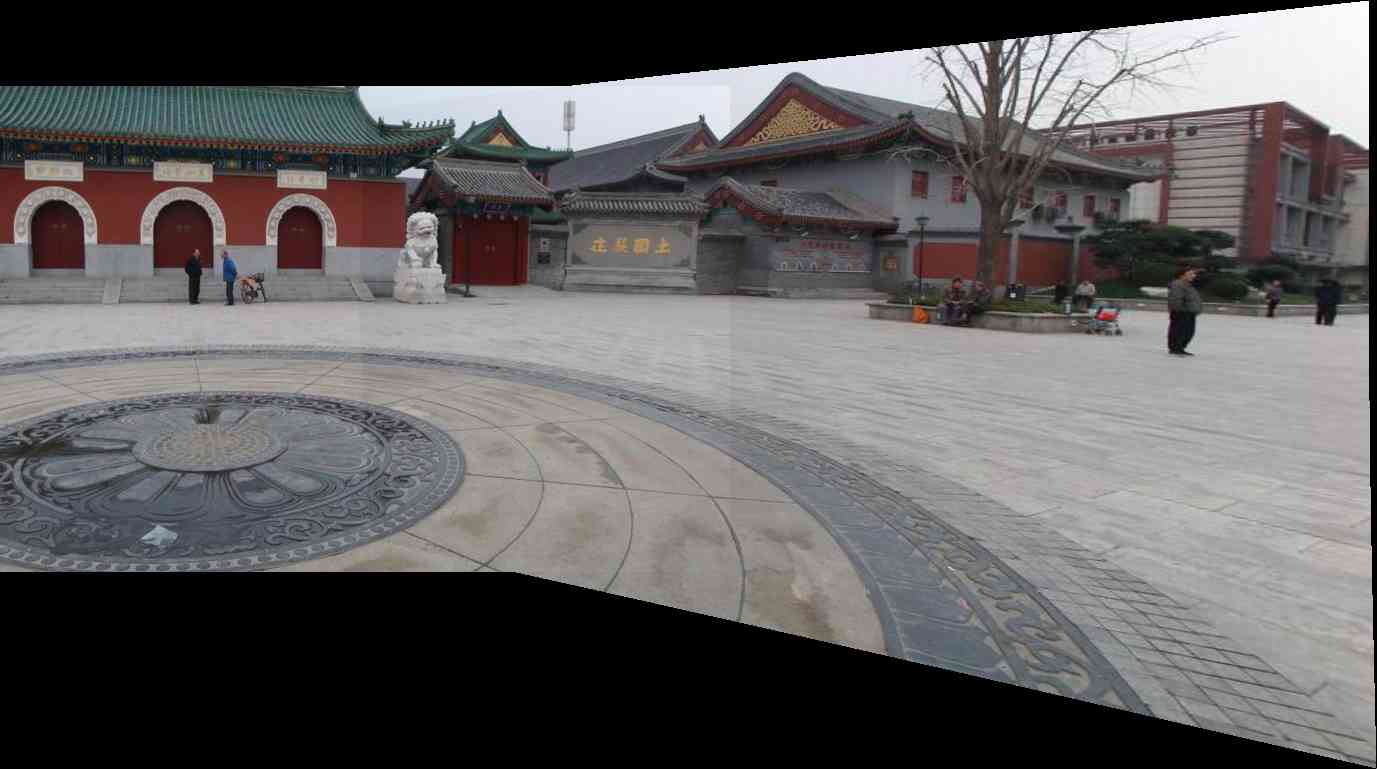}};
		\end{tikzpicture}
	}
	\caption{The effectiveness of the rank-1 constraint: comparison on the \emph{temple} dataset~\cite{gao_constructing_2011} between RASL and BRAS\@. The aligned images are overlayed to show the misaligned regions, which are highlighted in red circles. BRAS aligns the images significantly better, as can be seen by the reduction in ghosting artifacts.}\label{fig:rasl_align}
\end{figure*}

To quantitatively characterize the performance limits of RASL and BRAS, we perform a synthetic experimental comparison. In the same spirit as~\cite{peng_rasl:_2012}, different magnitudes of misalignments (translations in $x$ direction) and different amount of occlusions are studied. Occlusions are simulated by zeroing out pixels at random locations. We consider an alignment successful if for a given translation $t(\cdot;\tau^\ast)$, the final solution
$\widehat{\tau}$ satisfies $d_{h,w}(\tau^\ast,\widehat{\tau})<1$ pixel spacing for
$d_{h,w}(\tau_1,\tau_2)$ defined in~\eqref{eqn:rmse}. We test over $10$
commonly used image sets, and randomly select one image from each image
set. We count the number of successes and
divide it by $10$. The results are summarized in Fig.~\ref{fig:occlusion}, which confirms that BRAS has much higher tolerance against large translations and occlusions.

\begin{figure}
	\subfloat[RASL~\cite{peng_rasl:_2012}]{\includegraphics[width=0.5\linewidth]{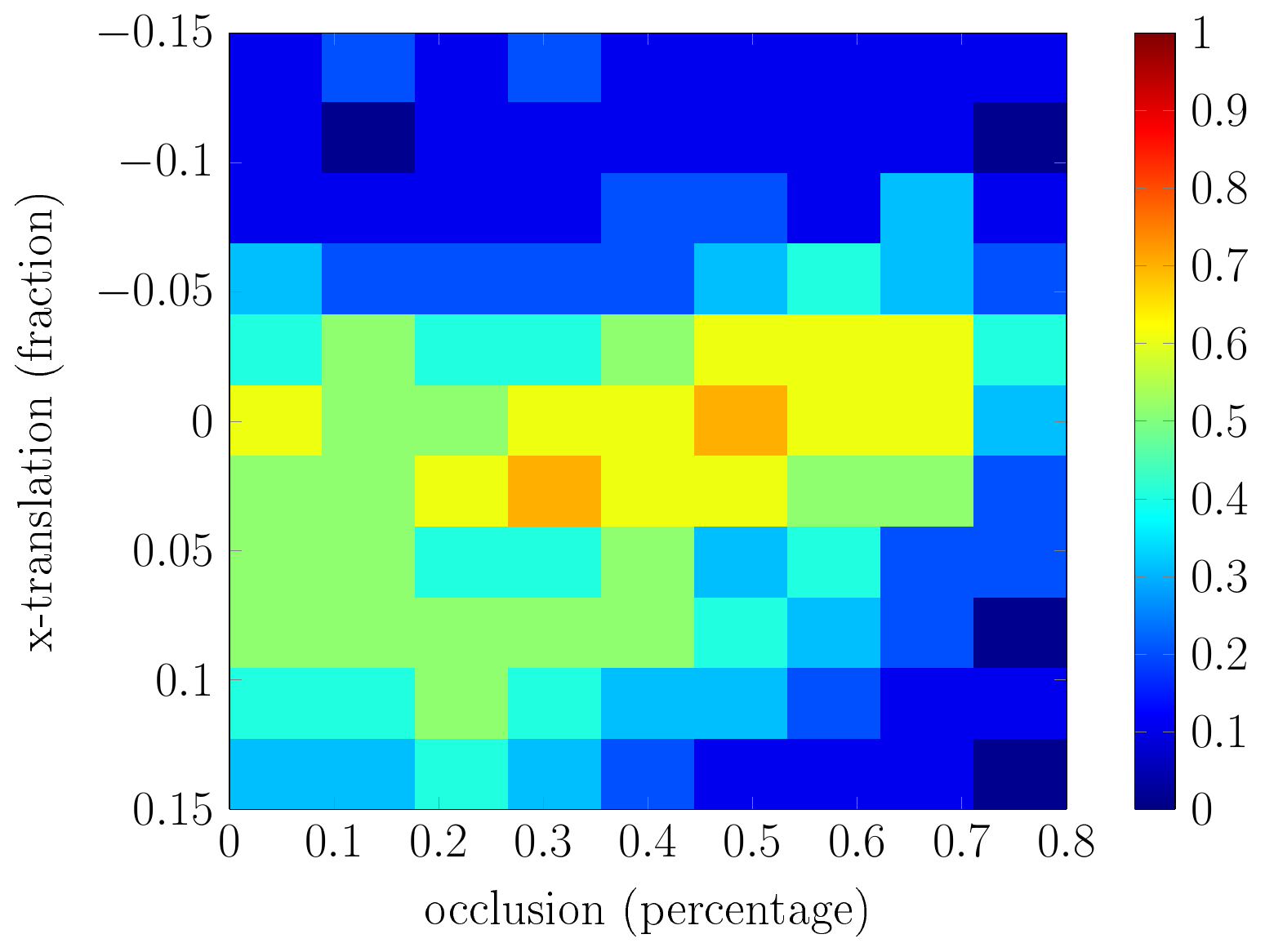}}
	\subfloat[BRAS]{\includegraphics[width=0.5\linewidth]{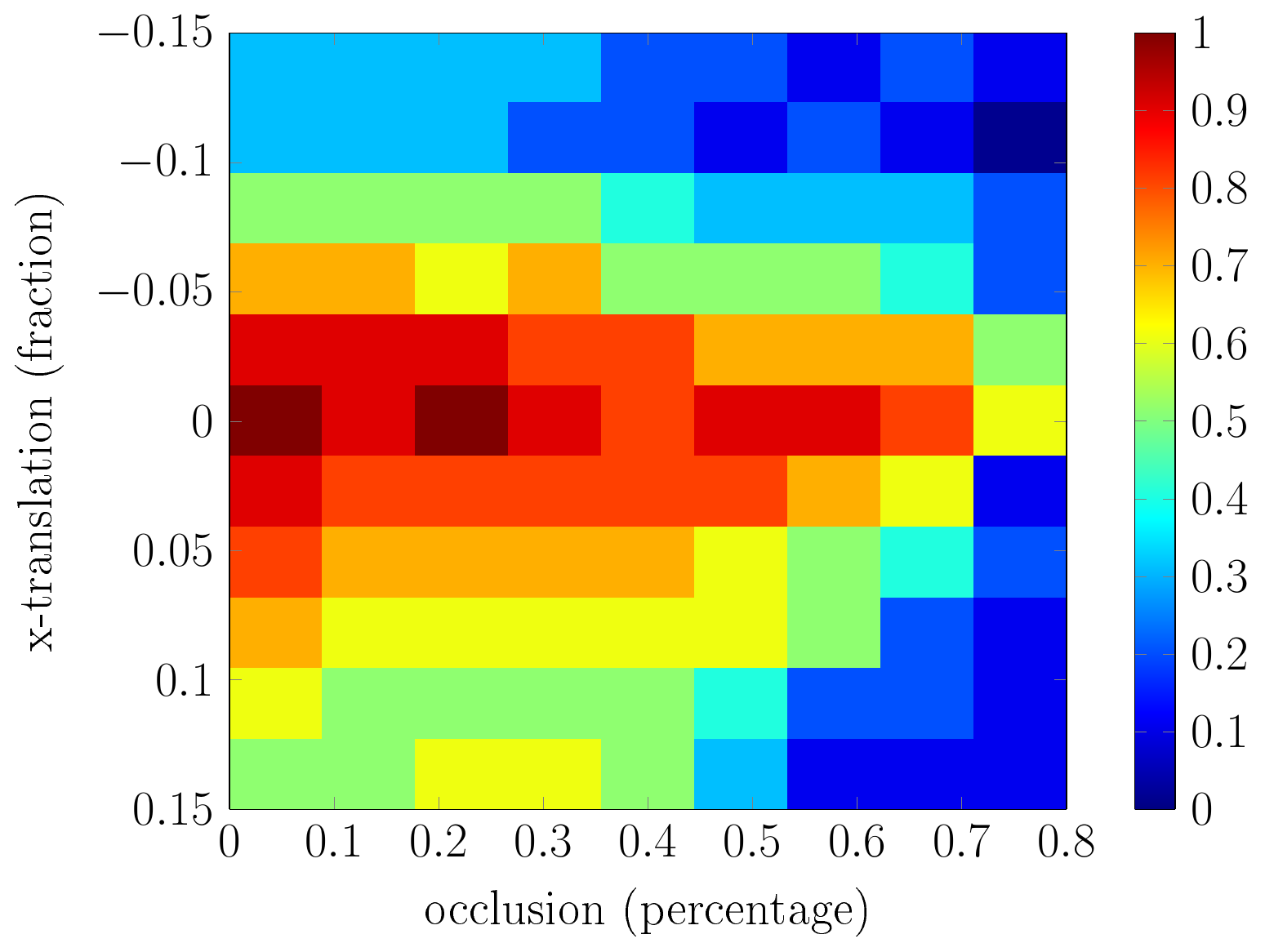}}
	\caption{Fractions of successful alignments are plotted for RASL and BRAS versus different translations and occlusions. The translations are in percentage of the corresponding image dimensions, while the occlusions are in fractions of the number of pixels. }\label{fig:occlusion}
\end{figure}

\subsection{Alignment in Challenging Scenarios}\label{subsec:alignment}
For various reasons, images difficult to align may be produced in real world
photography. However, for practical purposes such as object recognition, it may
be desirable to find a sensible alignment even in such scenarios. We will
analyze two typical examples, and assess the performance of BRAS in terms of
alignment accuracy. For comparisons, we also evaluate a typical feature-based
method with SIFT as descriptors. This method is widely used in panorama
software nowadays~\cite{brown_automatic_2007,hugin}. In all the cases we adhere
to the classic single homography model to rule out the influence of different
motion models.

\begin{figure}
	\centering
	\setcounter{subfigure}{0}
	\subfloat[\label{fig:christSIFT}]{\includegraphics[height=0.1\textheight]{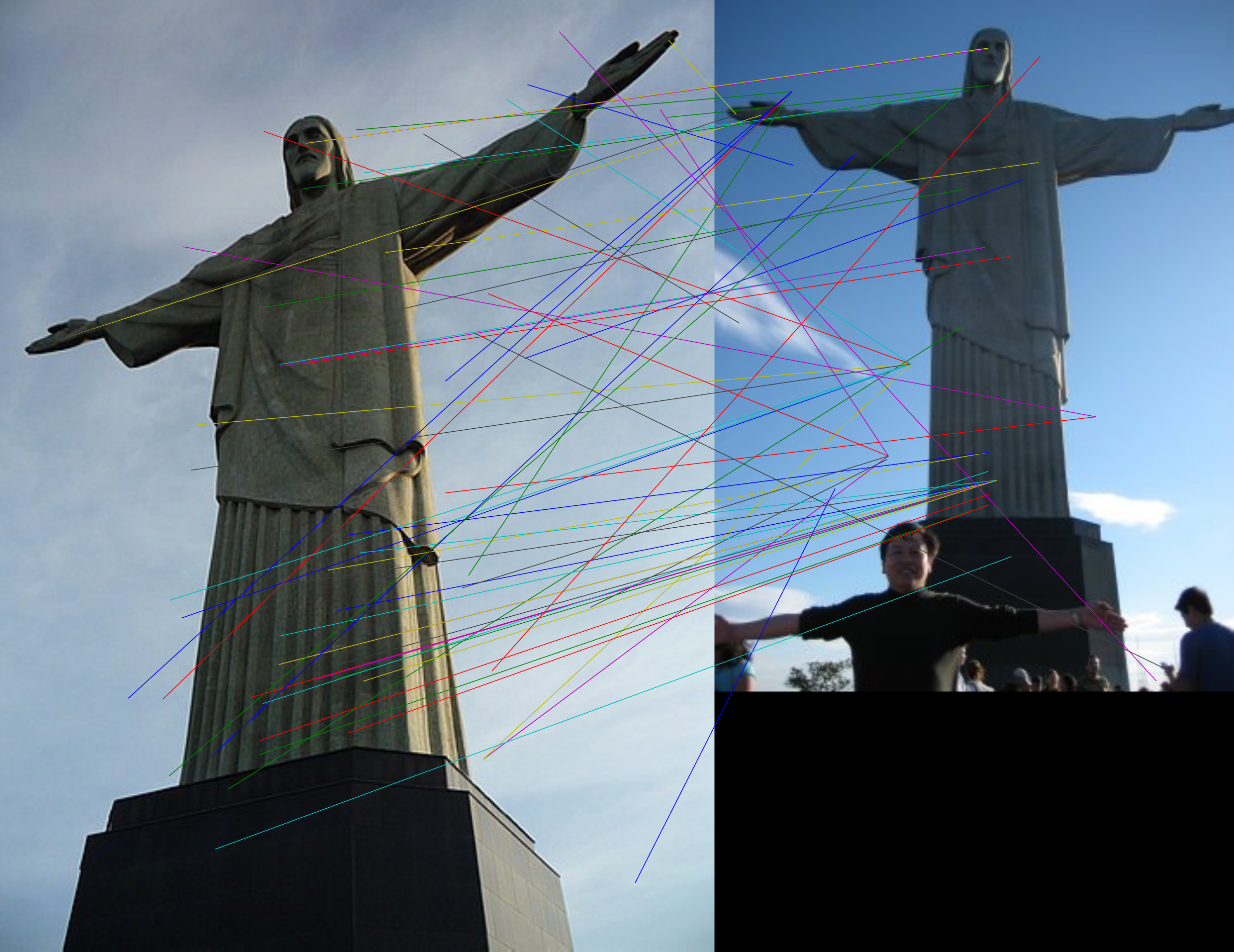}}\,
	\subfloat[\label{fig:christRANSAC}]{\includegraphics[height=0.1\textheight]{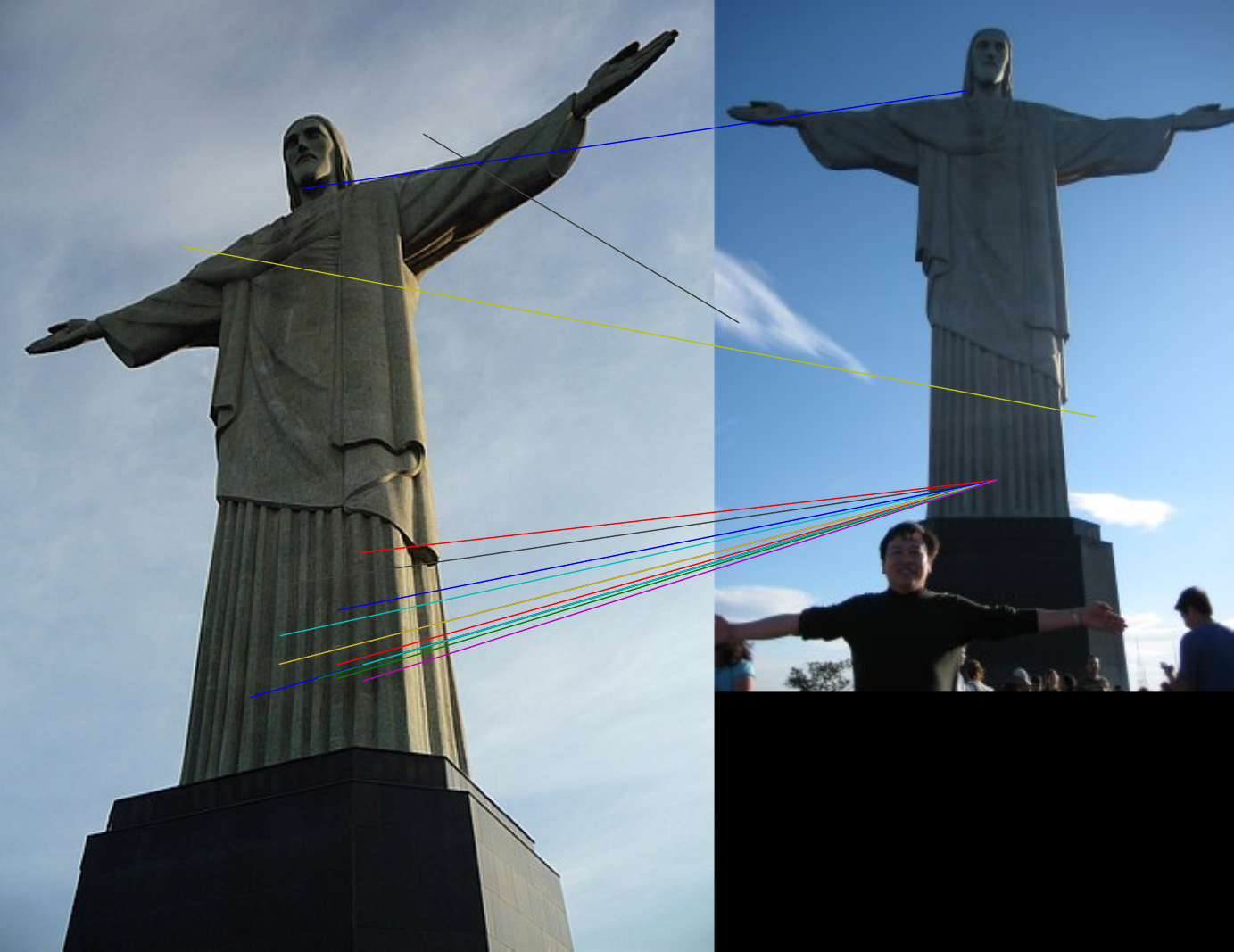}}\,
	\subfloat[\label{fig:christBRAS}]{\includegraphics[height=0.1\textheight, trim={90px 0 0 300px},clip]{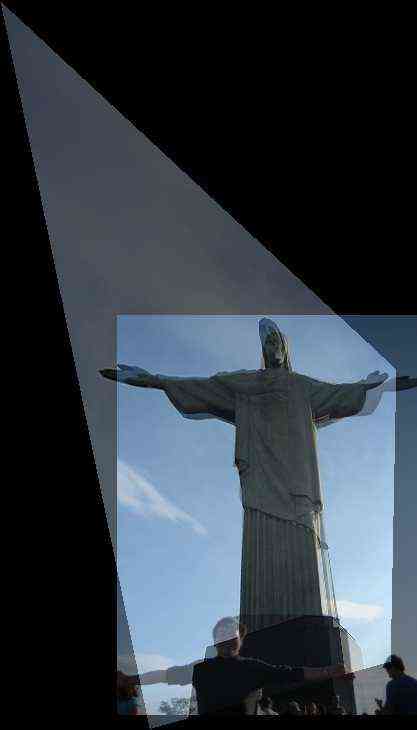}}
	\caption{Alignment results under appearance changes: \protect\subref{fig:christSIFT} appearance changes may mislead feature matching, and \protect\subref{fig:christRANSAC} the RANSAC does not find any sensible alignment; \protect\subref{fig:christBRAS} BRAS successfully aligns the images.}\label{fig:christ}
\end{figure}

\textbf{Appearance changes:} When the same object is captured at different
times, both scene differences and exposure differences may arise. An example is
the {\it christ\/} dataset~\cite{lin_aligning_2012} shown in
Fig.~\ref{fig:christSIFT}. Feature-based method for this dataset breaks down in
10 successive executions. We analyze this phenomenon by showing the SIFT
feature matching in~\ref{fig:christSIFT} and the RANSAC result in
Fig.~\ref{fig:christRANSAC}. It can be observed that, appearance changes
produce enormous amount of falsified feature matches, and RANSAC cannot
faithfully remove the outliers. In contrast, as shown in
Fig.~\ref{fig:christBRAS}, BRAS can still align the images plausibly.

\textbf{Large occlusions:} Occlusions may appear due to object movements,
parallax, etc. Large occlusions contribute to substantial amount of outliers in
feature matching, posing great challenges to RANSAC\@. For instance, in
Fig.~\ref{fig:large_occlusion} the windows and the wall are occluded
significantly by the tree branches. As is shown in~\ref{fig:winSIFT}, a typical
feature-based method poorly aligns the images. In contrast, BRAS still aligns the
images reliably, as shown in Fig.~\ref{fig:winBRAS}.

\tcred{
\subsection{A Panoramic Stitching Example for Composing a Long Image Sequence}\label{subsec:wideview}
We present the results of aligned and stitched images with BRAS for a dataset which includes a large number of images.
The dataset we use is from \cite{CMU}.
The individual image sequence as well as the final stitched result via BRAS are shown in Fig. \ref{fig:NSH}. This verifies the effectiveness of BRAS in handling long image sequences. Note that many state of the art methods such as APAP~\cite{zaragoza_as-projective-as-possible_2014} and CPW~\cite{hu_multi-objective_2015} do not handle such scenarios. Comprehensive comparisons with state of the art panoramic composition methods are reported next.  
}

\begin{figure*}
	\centering
	\subfloat {\includegraphics[width=0.9\textwidth]{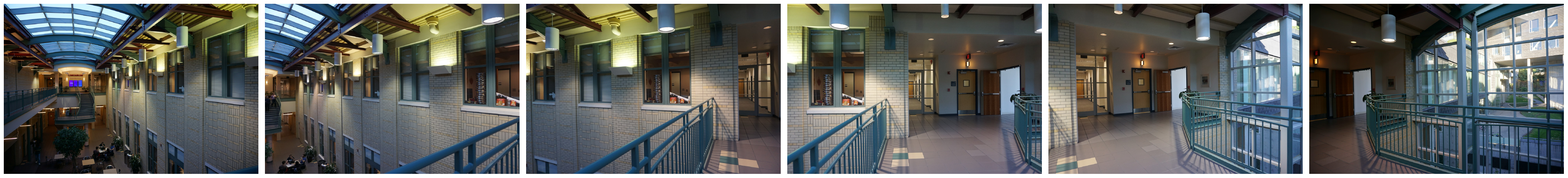}}\\ \vspace{-2mm}
	\subfloat {\includegraphics[width=0.9\textwidth]{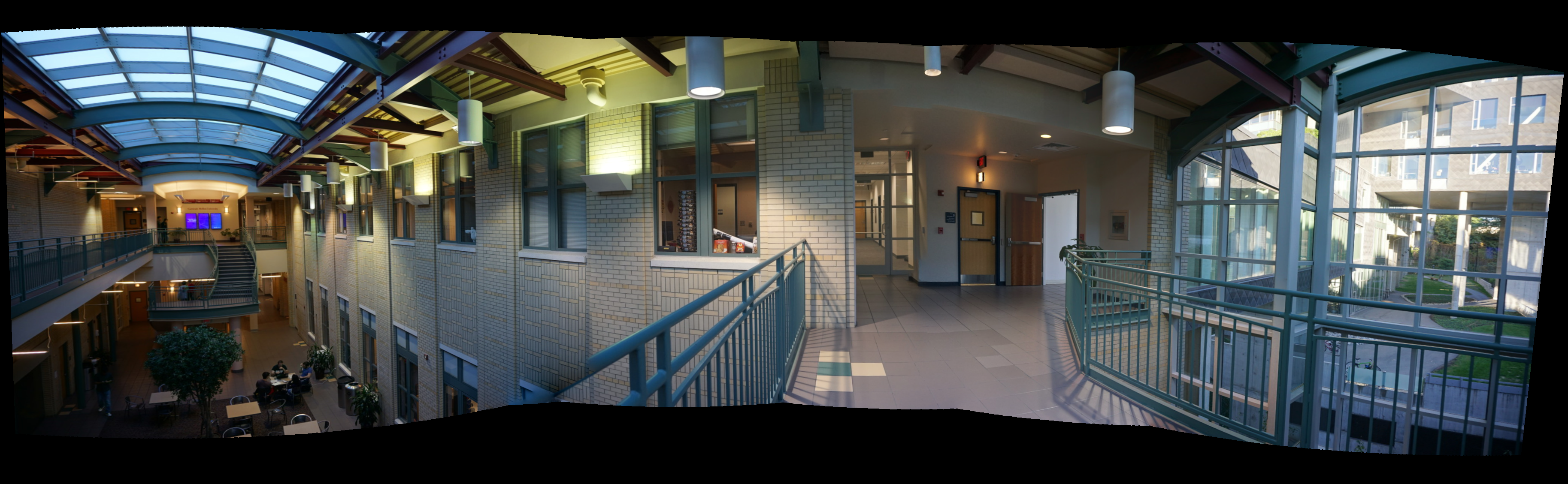}} \vspace{-2mm}
	\caption{\tcred{BRAS panoramic composition for a dataset with a long image sequence. Top: $6$ images from the {\it CMU\/} dataset~\cite{CMU}. Bottom: The final panoramic composition as achieved by BRAS after alignment and stitching.}}\label{fig:NSH}
\end{figure*}

\subsection{Panoramic Stitching Comparisons}\label{subsec:panorama}
We compare BRAS
against recent state-of-the-art methods, including
AutoStitch~\cite{brown_automatic_2007},
APAP~\cite{zaragoza_as-projective-as-possible_2014},
CPW~\cite{hu_multi-objective_2015}, and
SPHP~\cite{chang_shape-preserving_2014}. We also include a cutting-edge
commercial software, Microsoft ICE~\cite{ICE}. We carry out the comparisons over a
large variety of image datasets, all taken from recent publications in the
literature of image alignment and stitching. Because CPW can only handle
pairwise image stitching directly, we only include its results for
the two-image cases.

\subsubsection{Qualitative Evaluation on Image Stitching}\label{subsubsec:qualitative}

Fig.~\ref{fig:railtracks_aligned} shows the aligned images on the
\emph{railtracks} dataset~\cite{zaragoza_as-projective-as-possible_2014}. The
stitched version may be found in the supplementary document. From the insets,
we find that BRAS achieves a superior alignment accuracy than the others, as
evidenced by its significantly reduced ghosting artifacts.

\tcr{Fig.~\ref{fig:skyscraper_aligned} shows an additional alignment example on the \emph{skyscraper} dataset~\cite{chang_shape-preserving_2014}. From the insets, it is clear that BRAS aligns the edges and lines more accurately compared with both SPHP and APAP, and is free from structural distortions which are evident in APAP.}

\begin{figure}[h!]
	\centering
	\setcounter{subfigure}{0}
	\subfloat[SIFT\label{fig:winSIFT}]{%
		\includegraphics[height=0.14\textheight]{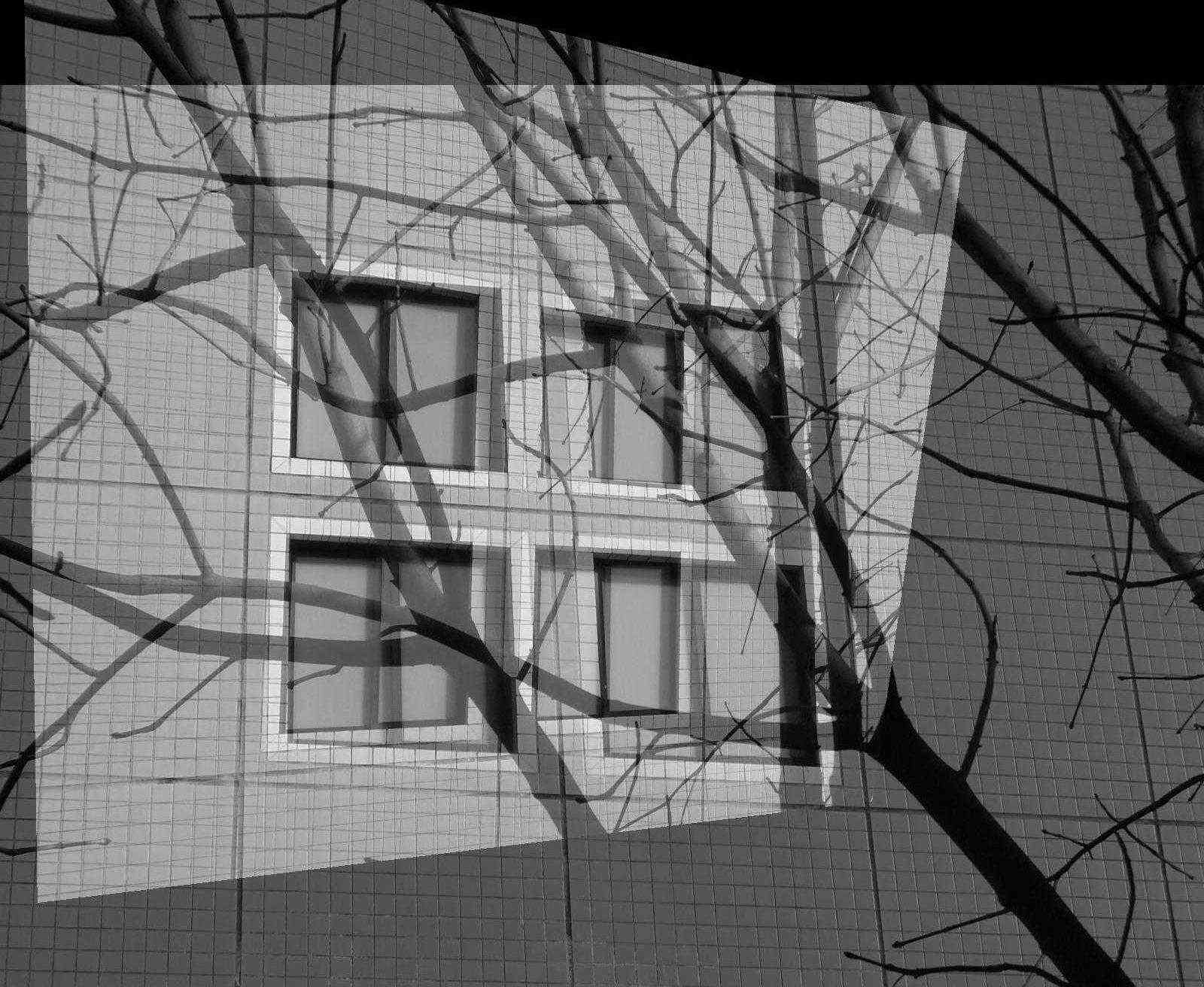}
	}\!\!
	\subfloat[BRAS\label{fig:winBRAS}]{%
		\includegraphics[height=0.14\textheight]{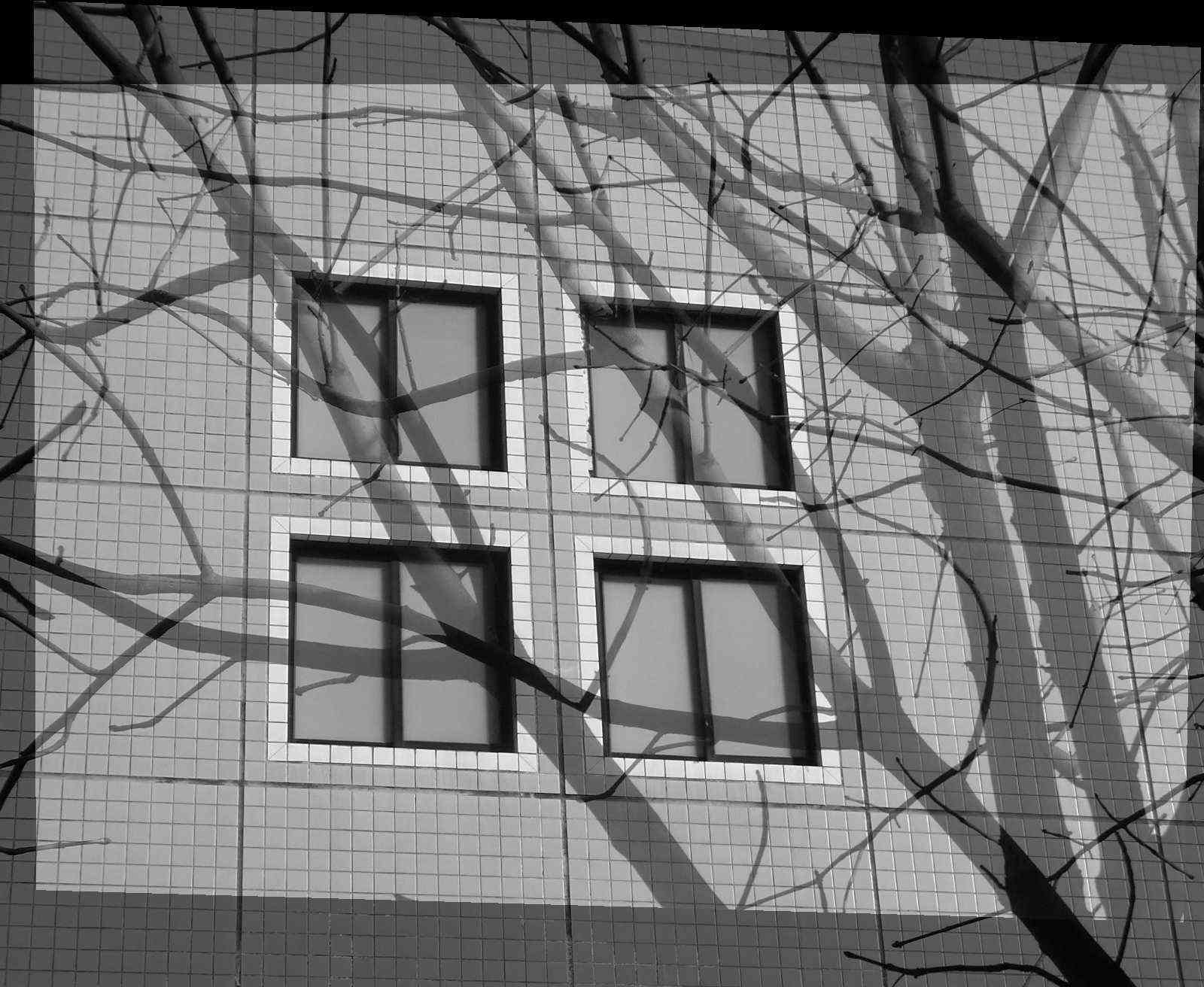}
	}
	\caption{Alignment results under large occlusions. \protect\subref{fig:winSIFT} A representative feature-based method (using SIFT) generates misaligned results because of large occlusions. \protect\subref{fig:winBRAS} BRAS accurately aligns the images. Images are from~\cite{peng_rasl:_2012}.}\label{fig:large_occlusion}
\end{figure}

Fig.~\ref{fig:apartments} shows stitched images for the \textit{apartments}
image set~\cite{gao_constructing_2011}. To show the artifacts we magnify some
regions.  AutoStitch, ICE and SPHP are based on the single homography model,
and thus they perform poorly whenever this simple model is inadequate.
Compared with them, CPW and APAP employ multiple homographies and are more
flexible, but severe artifacts such as shape distortions still remain.  In
contrast, the result of BRAS is significantly more visually appealing.
Additionally, the aligned version (in the supplementary document) shows that
BRAS succeeds in finding a more accurate and sensible alignment than the
others.

\tcr{Finally, we include the stitched images for the \emph{hanger} image
set~\cite{lin_smoothly_2011} in Fig.~\ref{fig:hanger}. As highlighted in the
red circles, AutoStitch and SPHP bend the lines on the wall, while AutoStitch,
ICE and SPHP duplicate part of the bed frame. APAP and CPW perform better, but
the bed frame is distorted by both methods, and appears broken in their stitched
results. The result of BRAS, on the other hand, is clearly of higher visual
quality and free from the aforementioned artifacts.}

\begin{figure}
	\centering
	\subfloat[CPW~\cite{hu_multi-objective_2015}\label{fig:railtracks_CPW}] {%
		\begin{tikzpicture}[zoomboxarray, zoomboxes below, zoomboxarray height=0.12\textheight, zoomboxes yshift=-0.8]
		\node [image node] {\includegraphics[height=0.1\textheight]{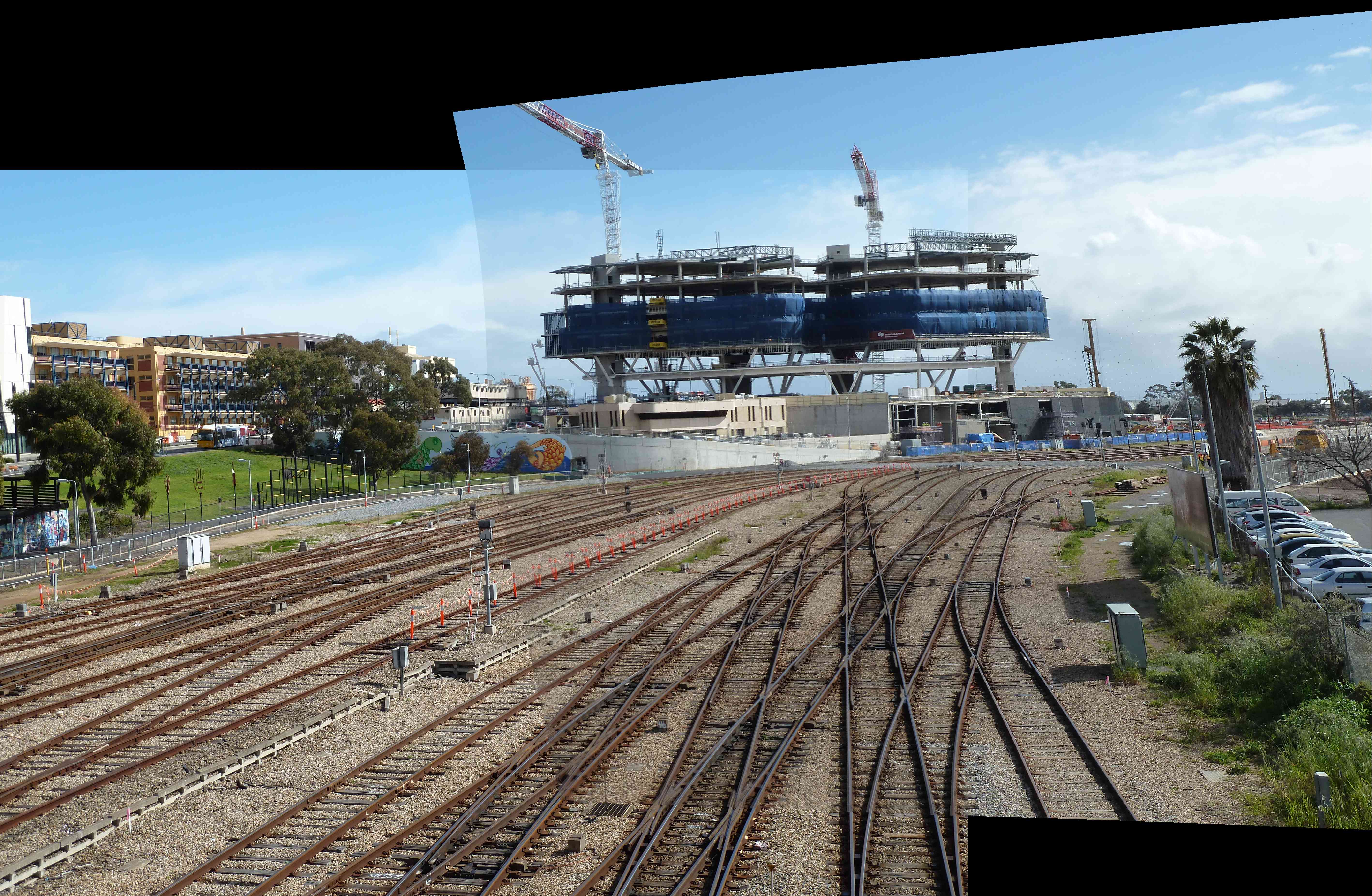}};
		\zoombox[color code=red,magnification=8]{0.54,0.53}
		\zoombox[color code=cyan,magnification=8]{0.44,0.76}
		\zoombox[color code=blue,magnification=8]{0.673,0.53}
		\zoombox[color code=green,magnification=8]{0.685,0.72}
		\end{tikzpicture}
	}\!
	\subfloat[SPHP~\cite{chang_shape-preserving_2014}\label{fig:railtracks_SPHP}] {%
		\begin{tikzpicture}[zoomboxarray, zoomboxes below, zoomboxarray height=0.12\textheight, zoomboxes yshift=-0.8]
		\node [image node] {\includegraphics[height=0.1\textheight]{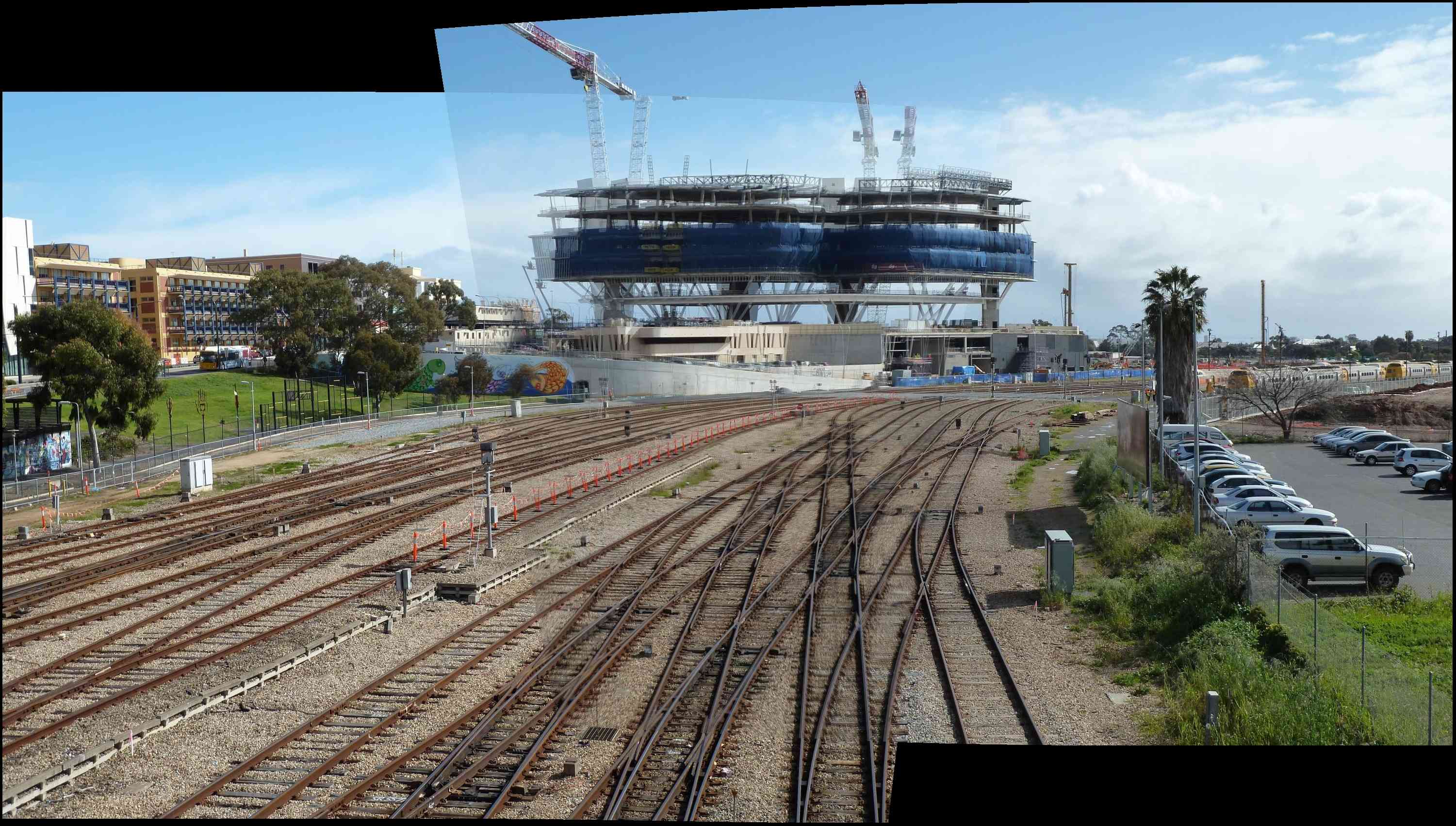}};
		\zoombox[color code=red,magnification=8]{0.51,0.58}
		\zoombox[color code=cyan,magnification=8]{0.43,0.83}
		\zoombox[color code=blue,magnification=8]{0.632,0.58}
		\zoombox[color code=green,magnification=8]{0.64,0.767}
		\end{tikzpicture}
	}\\
	\subfloat[APAP~\cite{zaragoza_as-projective-as-possible_2014}\label{fig:railtracks_APAP}] {%
		\begin{tikzpicture}[zoomboxarray, zoomboxes below, zoomboxarray height=0.12\textheight, zoomboxes yshift=-0.8]
		\node [image node] {\includegraphics[height=0.09\textheight]{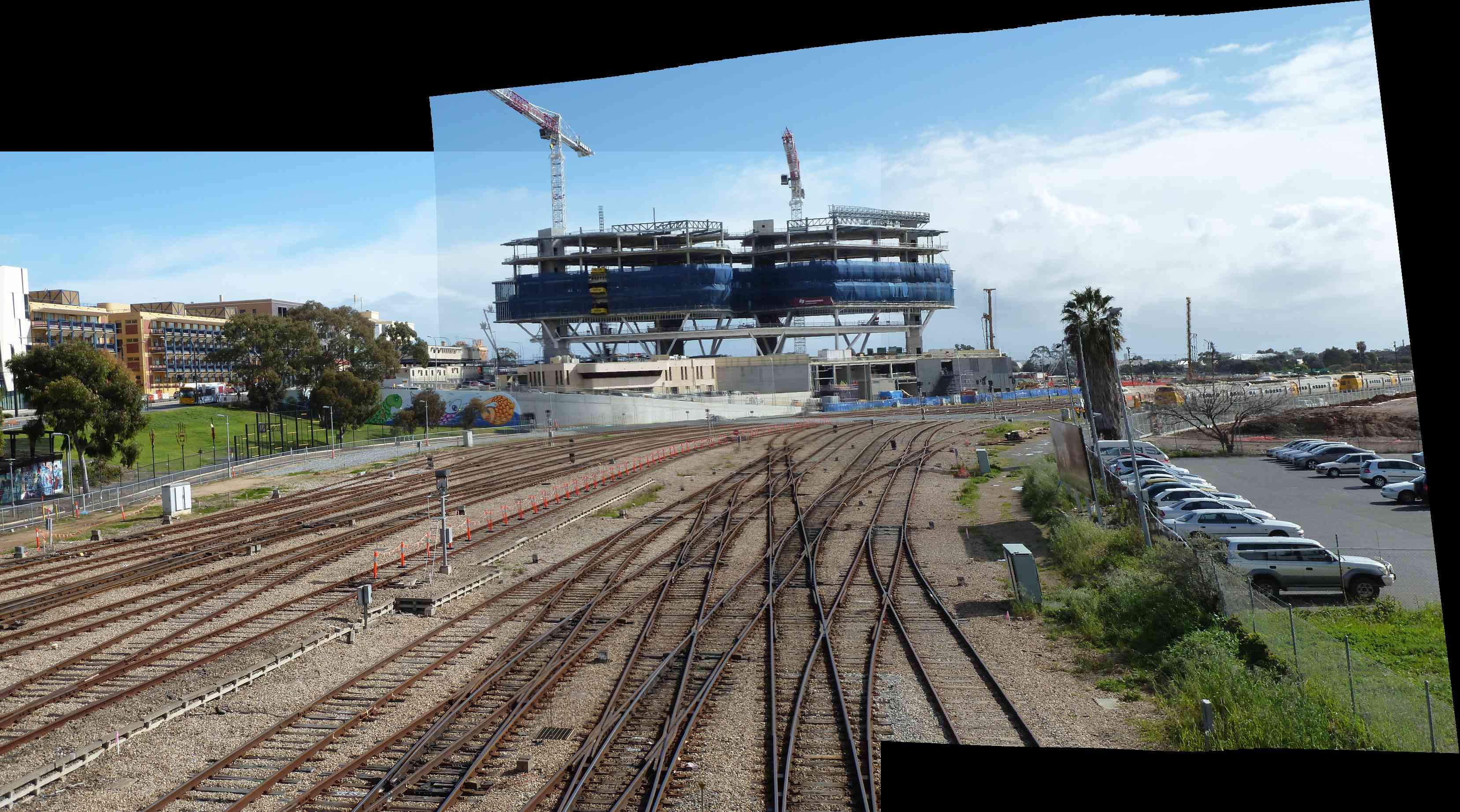}};
		\zoombox[color code=red,magnification=8]{0.458,0.53}
		\zoombox[color code=cyan,magnification=8]{0.38,0.76}
		\zoombox[color code=blue,magnification=8]{0.58,0.53}
		\zoombox[color code=green,magnification=8]{0.59,0.72}
		\end{tikzpicture}
	}\!
	\subfloat[BRAS\label{fig:railtracks_BRAS}] {%
		\begin{tikzpicture}[zoomboxarray, zoomboxes below, zoomboxarray height=0.12\textheight, zoomboxes yshift=-0.8]
		\node [image node] {\includegraphics[height=0.09\textheight]{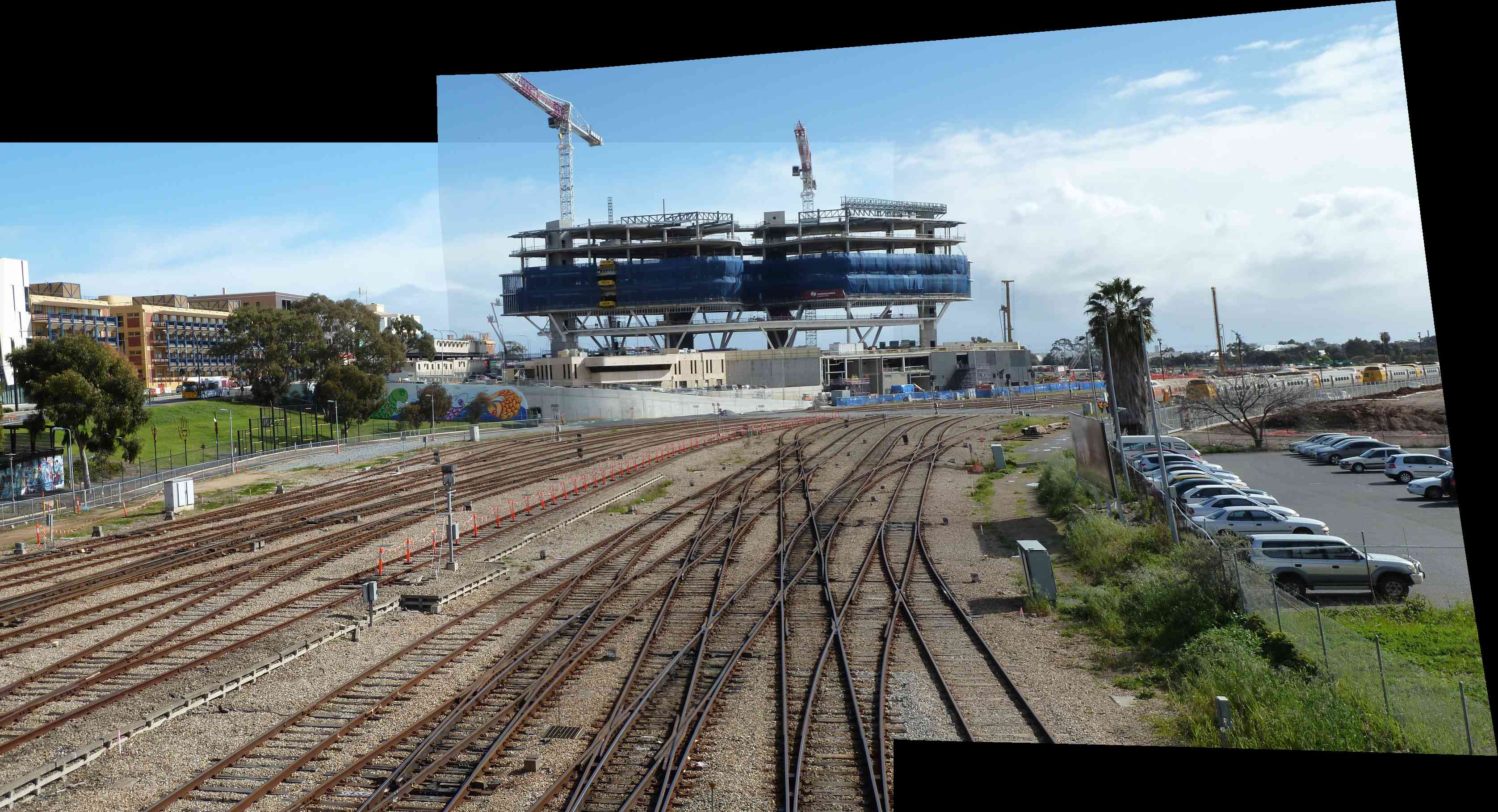}};
		\zoombox[color code=red,magnification=8]{0.454,0.53}
		\zoombox[color code=cyan,magnification=8]{0.38,0.77}
		\zoombox[color code=blue,magnification=8]{0.574,0.53}
		\zoombox[color code=green,magnification=8]{0.584,0.73}
		\end{tikzpicture}
	}\\
	\caption{Aligned and overlayed images on the {\it railtracks\/} dataset~\cite{zaragoza_as-projective-as-possible_2014}. Regions with misalignments are magnified in the insets at the bottom. Please notice the ghosting artifacts produced by \protect\subref{fig:railtracks_CPW} CPW \protect\subref{fig:railtracks_SPHP} SPHP and \protect\subref{fig:railtracks_APAP} APAP, which indicate alignment errors and do not appear in \protect\subref{fig:railtracks_BRAS} BRAS.}\label{fig:railtracks_aligned}
\end{figure}

\begin{figure*}
	\centering
	\subfloat[SPHP~\cite{chang_shape-preserving_2014}] {\includegraphics[height=0.28\textheight]{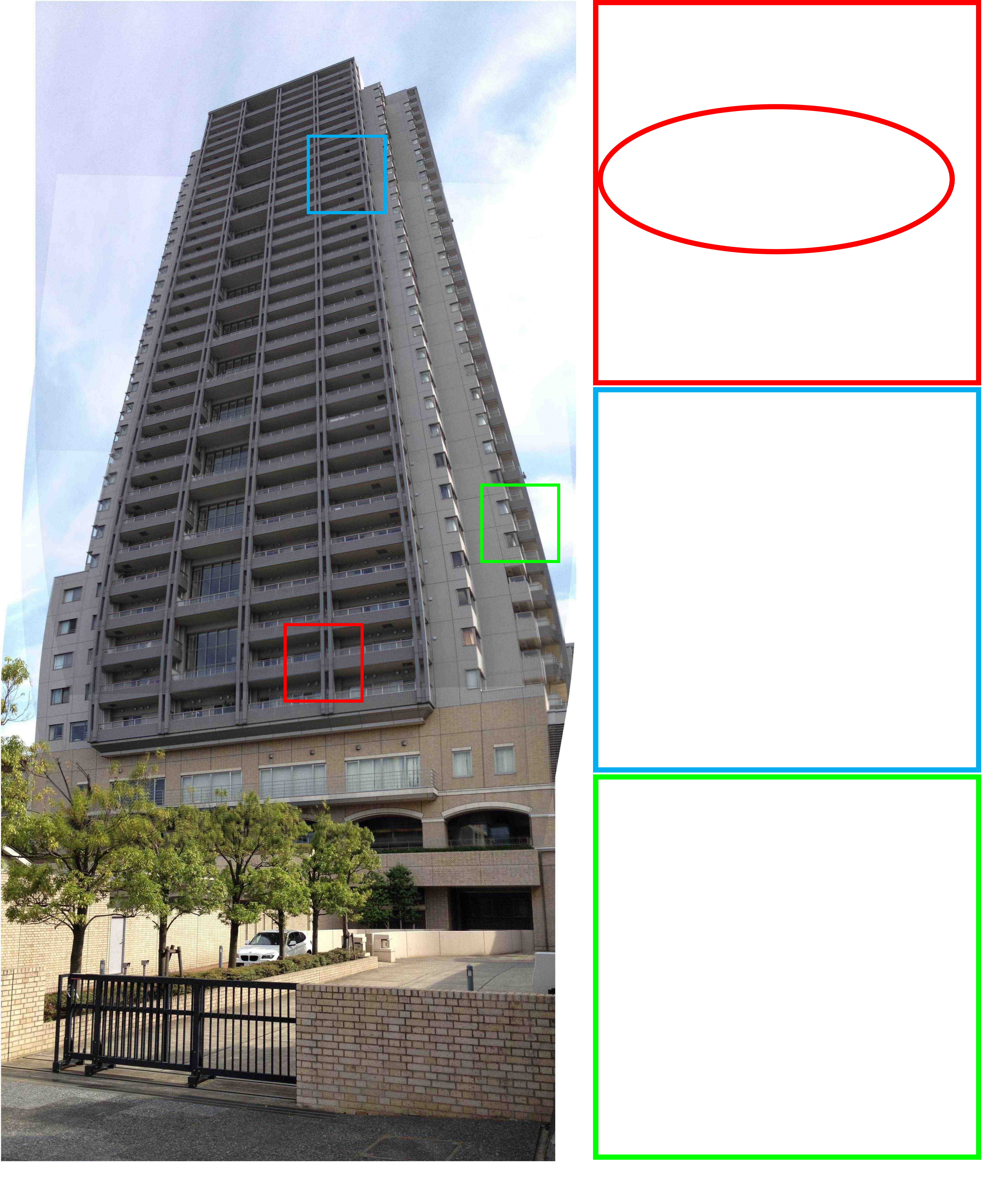}}\,\!
	\subfloat[APAP~\cite{zaragoza_as-projective-as-possible_2014}] {\includegraphics[height=0.28\textheight]{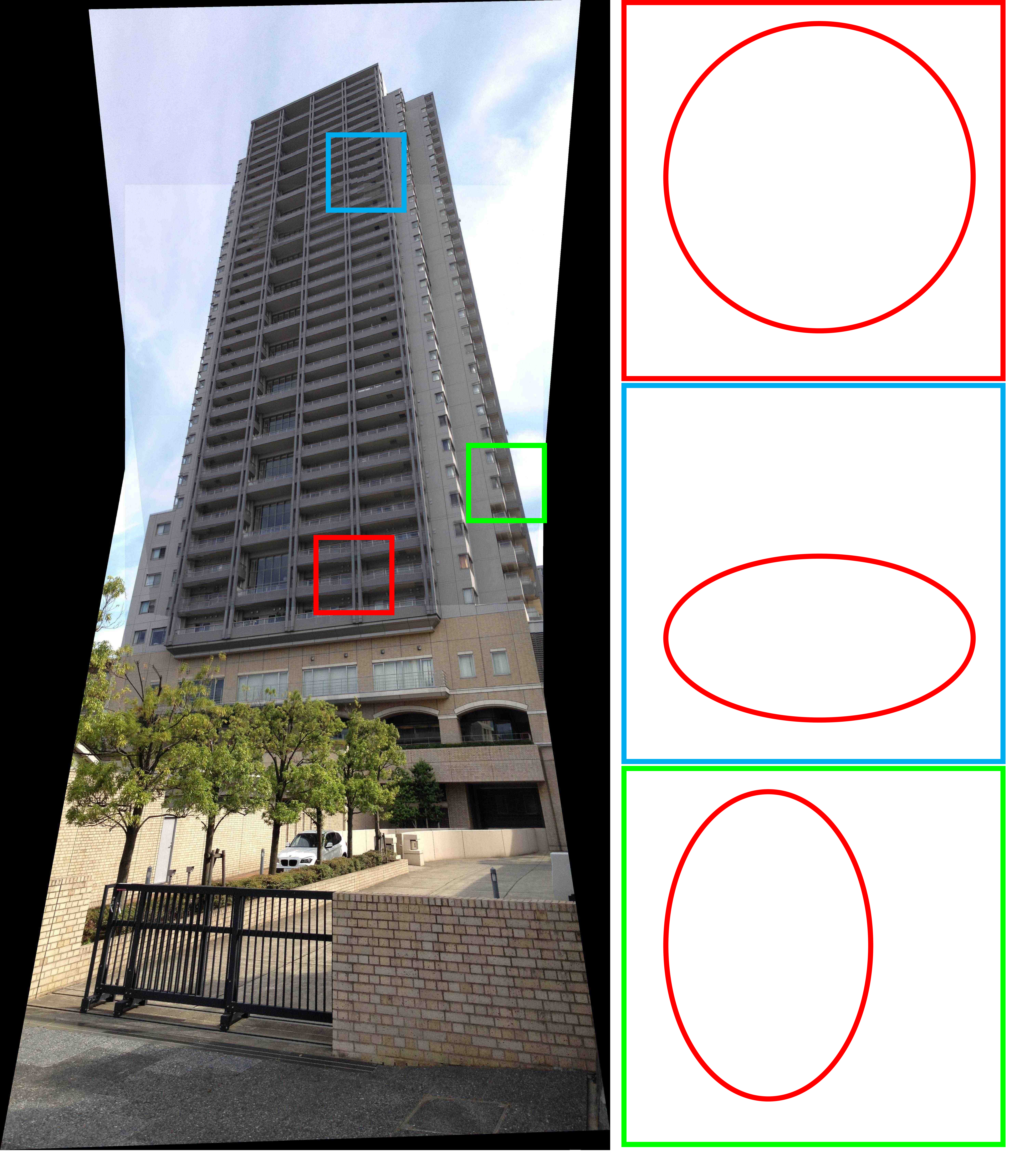}}\,\!
	\subfloat[BRAS] {\includegraphics[height=0.28\textheight]{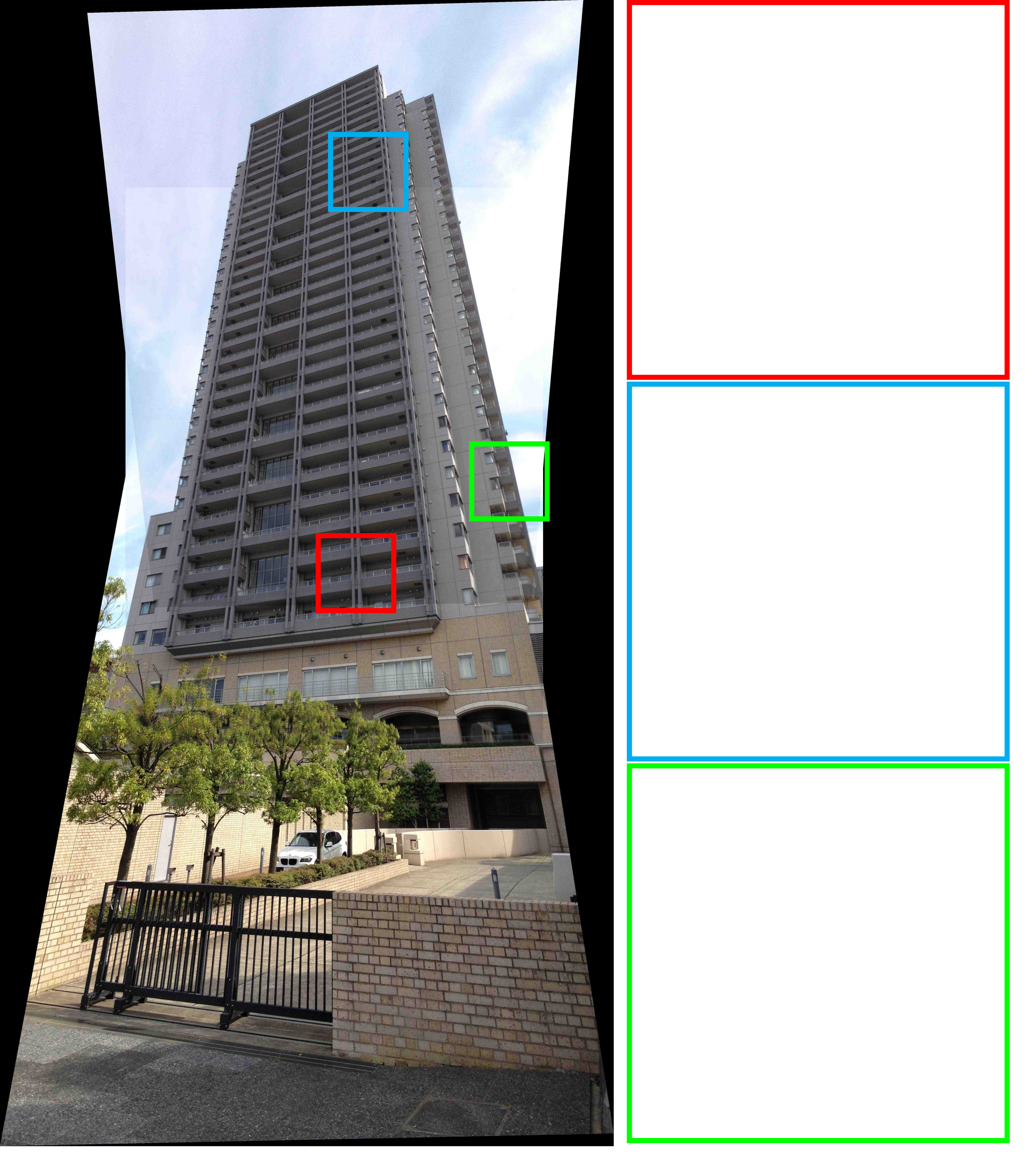}}
	\caption{\tcr{Aligned and overlayed images on the {\it skyscraper\/} dataset~\cite{chang_shape-preserving_2014}. Red circles highlight errors and ghosting indicates misalignments.}}\label{fig:skyscraper_aligned}
\end{figure*}

\begin{figure*}
	\centering
	\subfloat[AutoStitch~\cite{brown_automatic_2007}] {%
		\begin{tikzpicture}[zoomboxarray, zoomboxes below, zoomboxarray rows=1, zoomboxarray height=0.13\textheight, zoomboxes yshift=-0.9, execute at end picture={\draw[thick,red] (2.3,-2.5) circle (.5cm);}, execute at end picture={\draw[thick,red] (4.8,-.8) circle (.5cm);}, execute at end picture={\draw[thick,red] (4.6,-2.9) circle (.5cm);}]
		\node [image node] {\includegraphics[height=0.18\textheight]{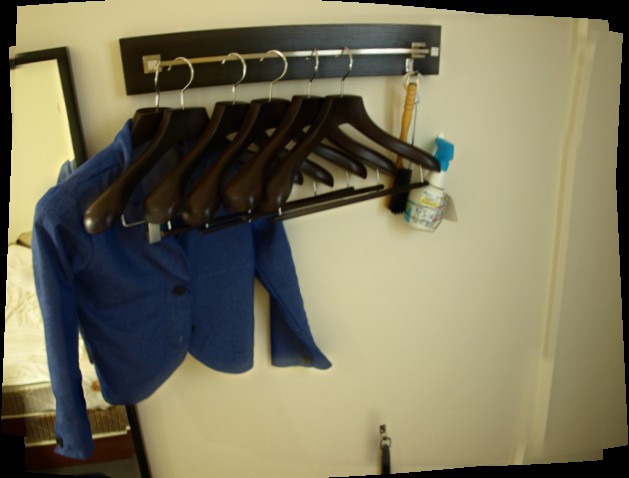}};
		\zoombox[color code=blue]{0.14,0.15}
		\zoombox[color code=green]{0.9,0.8}
		\end{tikzpicture}
	}\,\!
	\subfloat[ICE~\cite{ICE}] {%
		\begin{tikzpicture}[zoomboxarray, zoomboxes below, zoomboxarray rows=1, zoomboxarray height=0.13\textheight, zoomboxes yshift=-0.9, execute at end picture={\draw[rotate=90, thick,red] (-2.2,-.7) ellipse (1cm and .6cm);}]
		\node [image node] {\includegraphics[height=0.18\textheight]{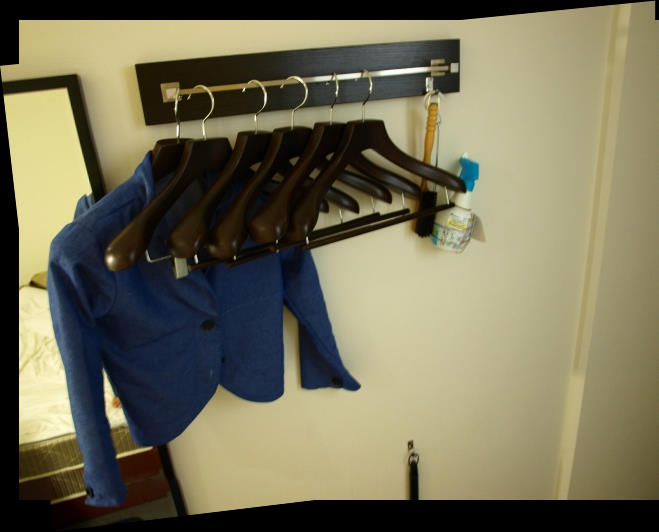}};
		\zoombox[color code=blue]{0.14,0.15}
		\zoombox[color code=green]{0.9,0.8}
		\end{tikzpicture}
	}\,\!
	\subfloat[CPW~\cite{hu_multi-objective_2015}] {%
		\begin{tikzpicture}[zoomboxarray, zoomboxes below, zoomboxarray rows=1, zoomboxarray height=0.13\textheight, zoomboxes yshift=-0.9, execute at end picture={\draw[rotate=70, thick,red] (-1.6,-1.9) ellipse (.6cm and 1cm);}]
		\node [image node] {\includegraphics[height=0.18\textheight]{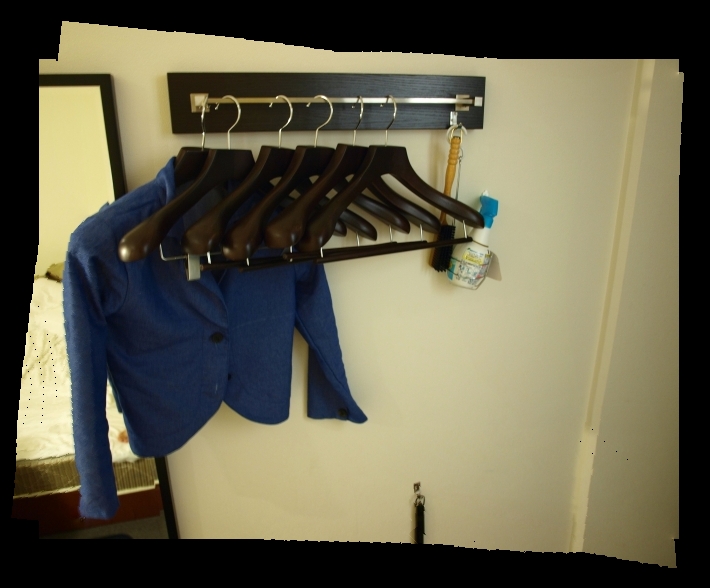}};
		\zoombox[color code=blue]{0.14,0.21}
		\zoombox[color code=green]{0.87,0.75}
		\end{tikzpicture}
	}\\
	
	\subfloat[SPHP~\cite{chang_shape-preserving_2014}] {%
		\begin{tikzpicture}[zoomboxarray, zoomboxes below, zoomboxarray rows=1, zoomboxarray height=0.13\textheight, zoomboxes yshift=-0.9, execute at end picture={\draw[rotate=90, thick,red] (-2,-.7) ellipse (1cm and .6cm);}]
		\node [image node] {\includegraphics[height=0.18\textheight]{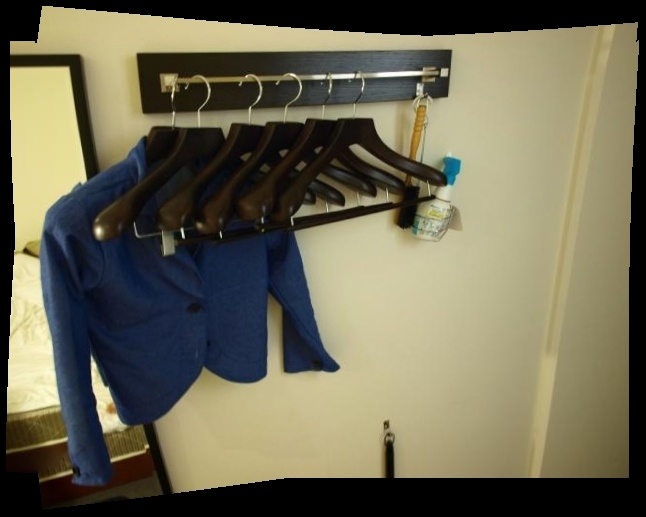}};
		\zoombox[color code=blue]{0.14,0.15}
		\zoombox[color code=green]{0.9,0.8}
		\end{tikzpicture}
	}\,\!
	\subfloat[APAP~\cite{zaragoza_as-projective-as-possible_2014}] {%
		\begin{tikzpicture}[zoomboxarray, zoomboxes below, zoomboxarray rows=1, zoomboxarray height=0.13\textheight, zoomboxes yshift=-0.9, execute at end picture={\draw[rotate=70, thick,red] (-1.6,-2.1) ellipse (.6cm and 1cm);}]
		\node [image node] {\includegraphics[height=0.18\textheight]{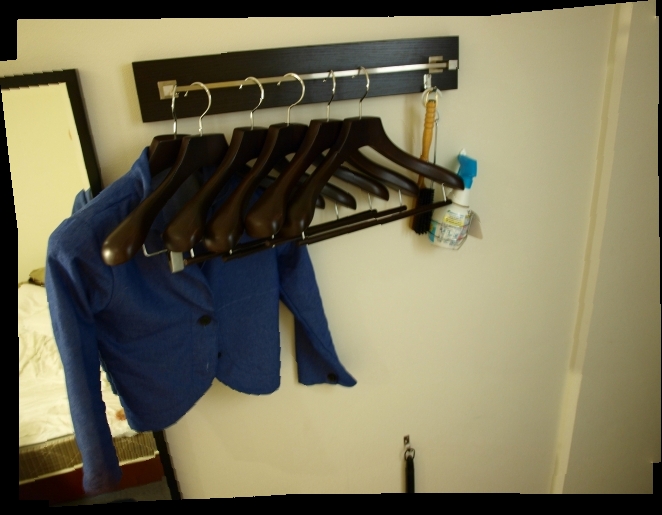}};
		\zoombox[color code=blue]{0.14,0.15}
		\zoombox[color code=green]{0.9,0.8}
		\end{tikzpicture}
	}\,\!
	\subfloat[BRAS] {%
		\begin{tikzpicture}[zoomboxarray, zoomboxes below, zoomboxarray rows=1, zoomboxarray height=0.13\textheight, zoomboxes yshift=-0.9]
		\node [image node] {\includegraphics[height=0.18\textheight]{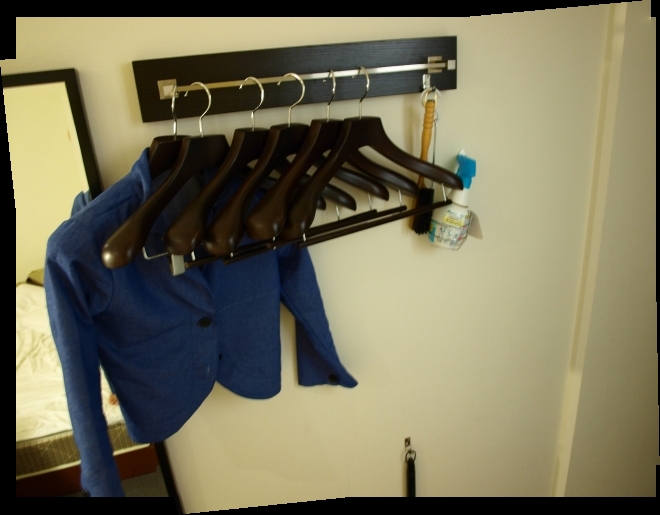}};
		\zoombox[color code=blue]{0.14,0.15}
		\zoombox[color code=green]{0.9,0.8}
		\end{tikzpicture}
	}\\
	\caption{\tcr{Stitched images on the {\it hanger\/} dataset~\cite{lin_smoothly_2011}. Red circles in the insets highlight artifacts.}}\label{fig:hanger}
\end{figure*}
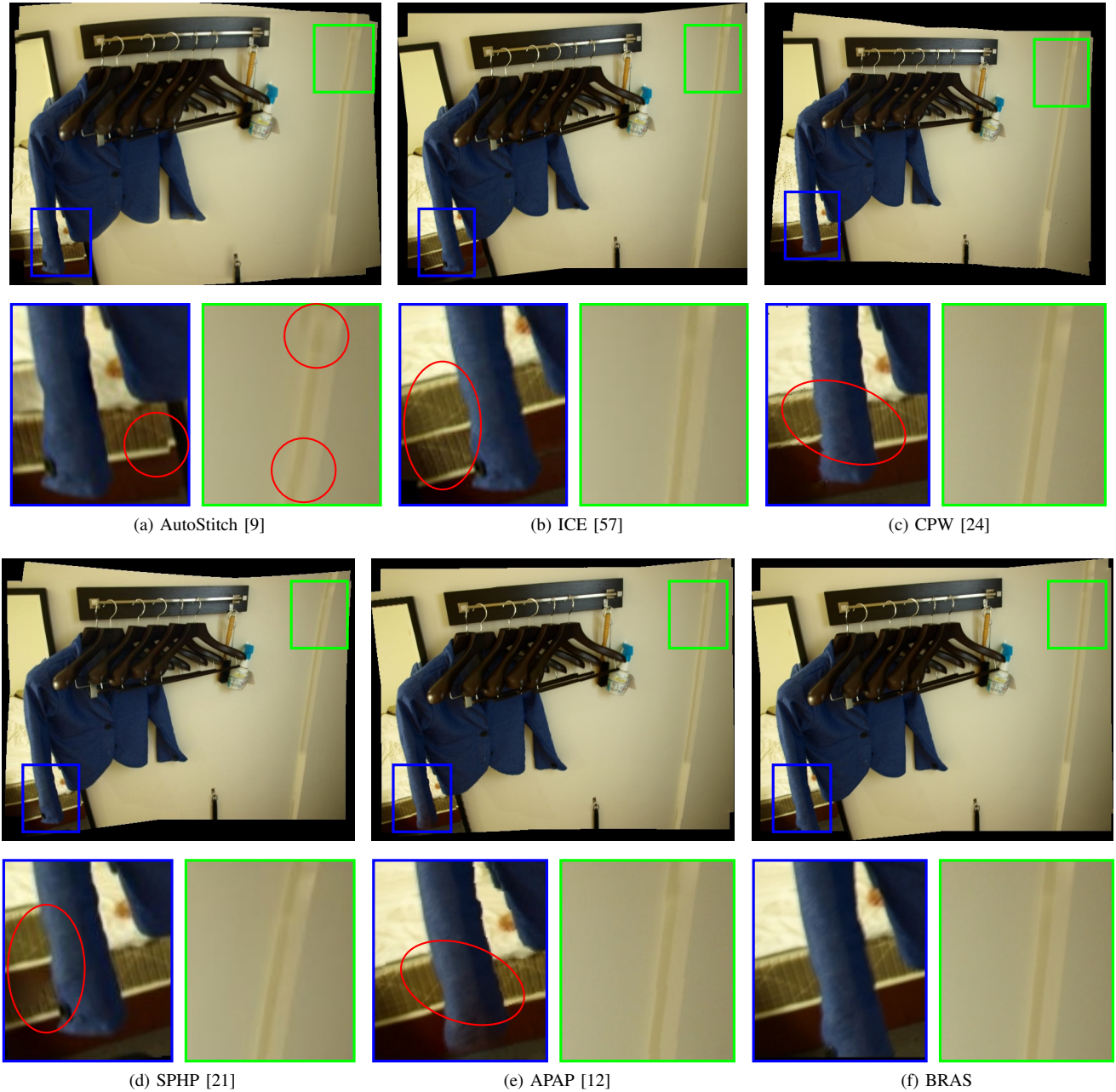

\begin{figure*}
	\centering
	\subfloat[AutoStitch~\cite{brown_automatic_2007}\label{fig:apartments_autostitch}] {%
	\begin{tikzpicture}[zoomboxarray, zoomboxes below, zoomboxarray rows=1, execute at end picture={\draw[thick,red] (1,-1.5) ellipse (0.8cm and 1cm);}]
		\node [image node] {\includegraphics[height=0.12\textheight]{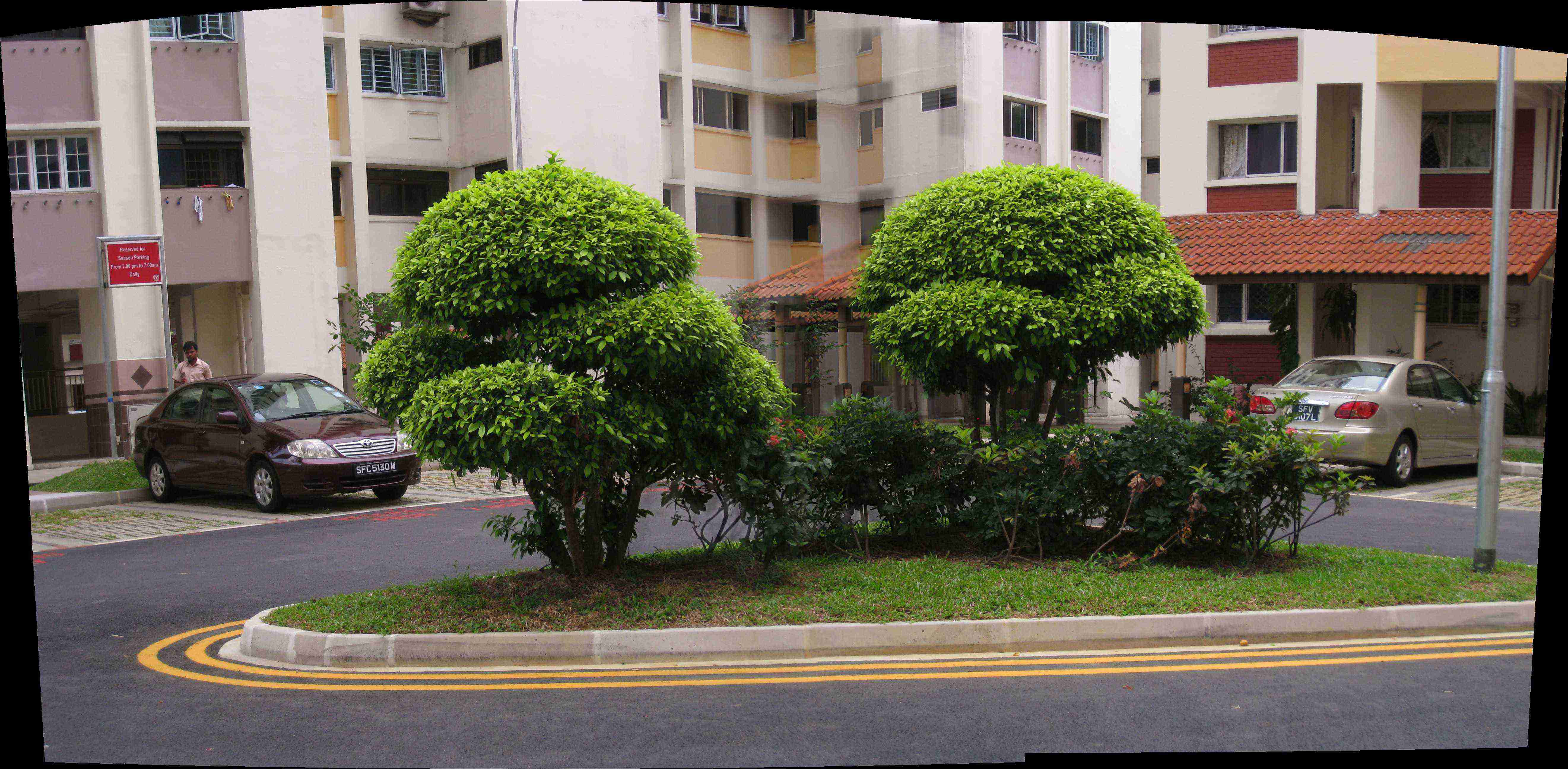}};
		\zoombox[color code=green]{0.55,0.75}
		\zoombox[color code=blue]{0.7,0.8}
		\end{tikzpicture}
	}\,\!
	\subfloat[ICE~\cite{ICE}\label{fig:apartments_ICE}] {%
		\begin{tikzpicture}[zoomboxarray, zoomboxes below, zoomboxarray rows=1, execute at end picture={\draw[thick,red] (4,-1.2) ellipse (1cm and 0.7cm);}]
		\node [image node] {\includegraphics[height=0.12\textheight]{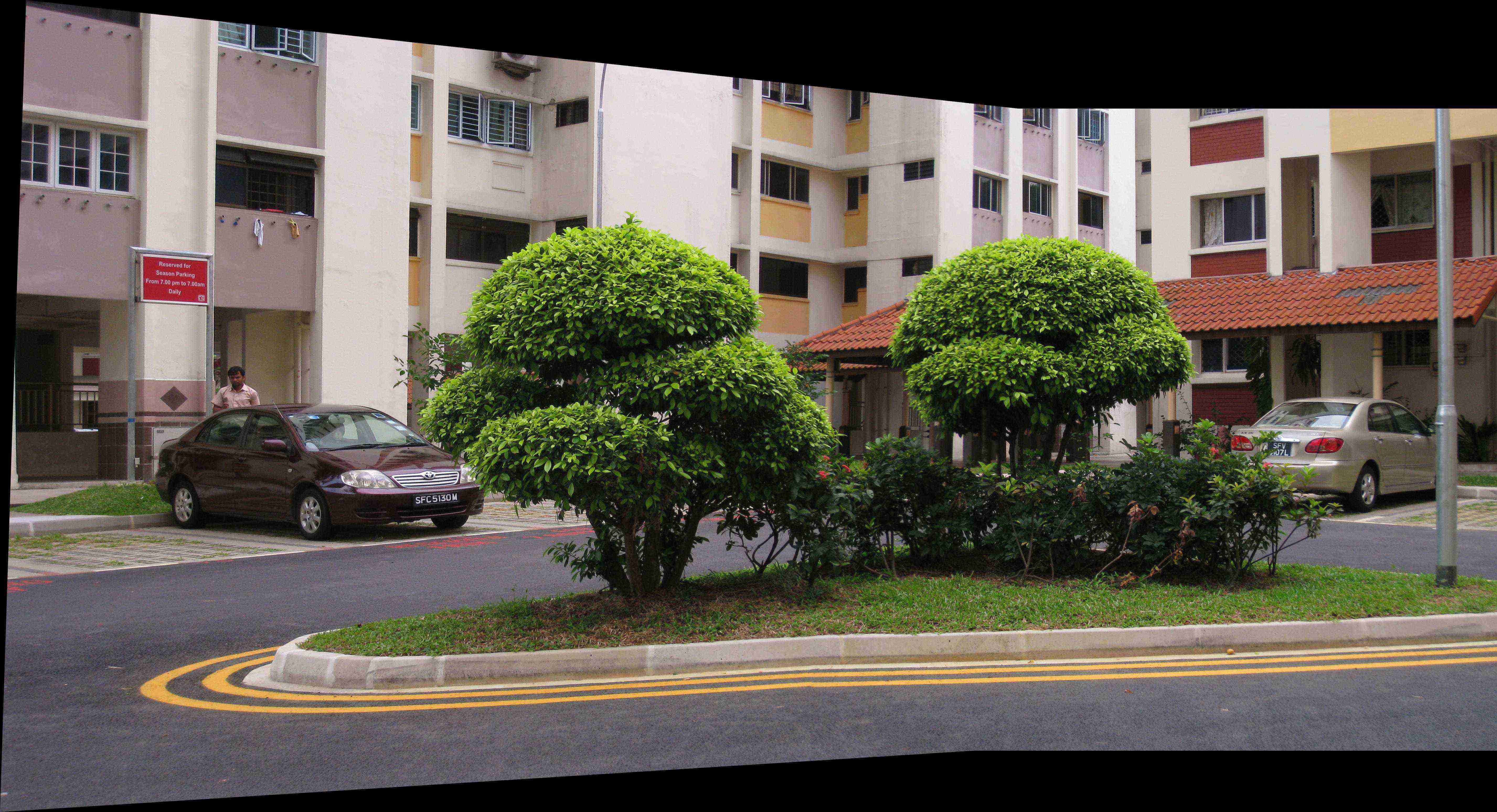}};
		\zoombox[color code=green]{0.57,0.73}
		\zoombox[color code=blue]{0.7,0.72}
		\end{tikzpicture}
	}\,\!
	\subfloat[CPW~\cite{hu_multi-objective_2015}\label{fig:apartments_CPW}] {%
		\begin{tikzpicture}[zoomboxarray, zoomboxes below, zoomboxarray rows=1, execute at end picture={\draw[thick,red] (1.3,-2.3) ellipse (1cm and 0.6cm); \draw[thick,red] (4,-1.5) ellipse (1cm and 0.5cm);}]
		\node [image node] {\includegraphics[height=0.12\textheight]{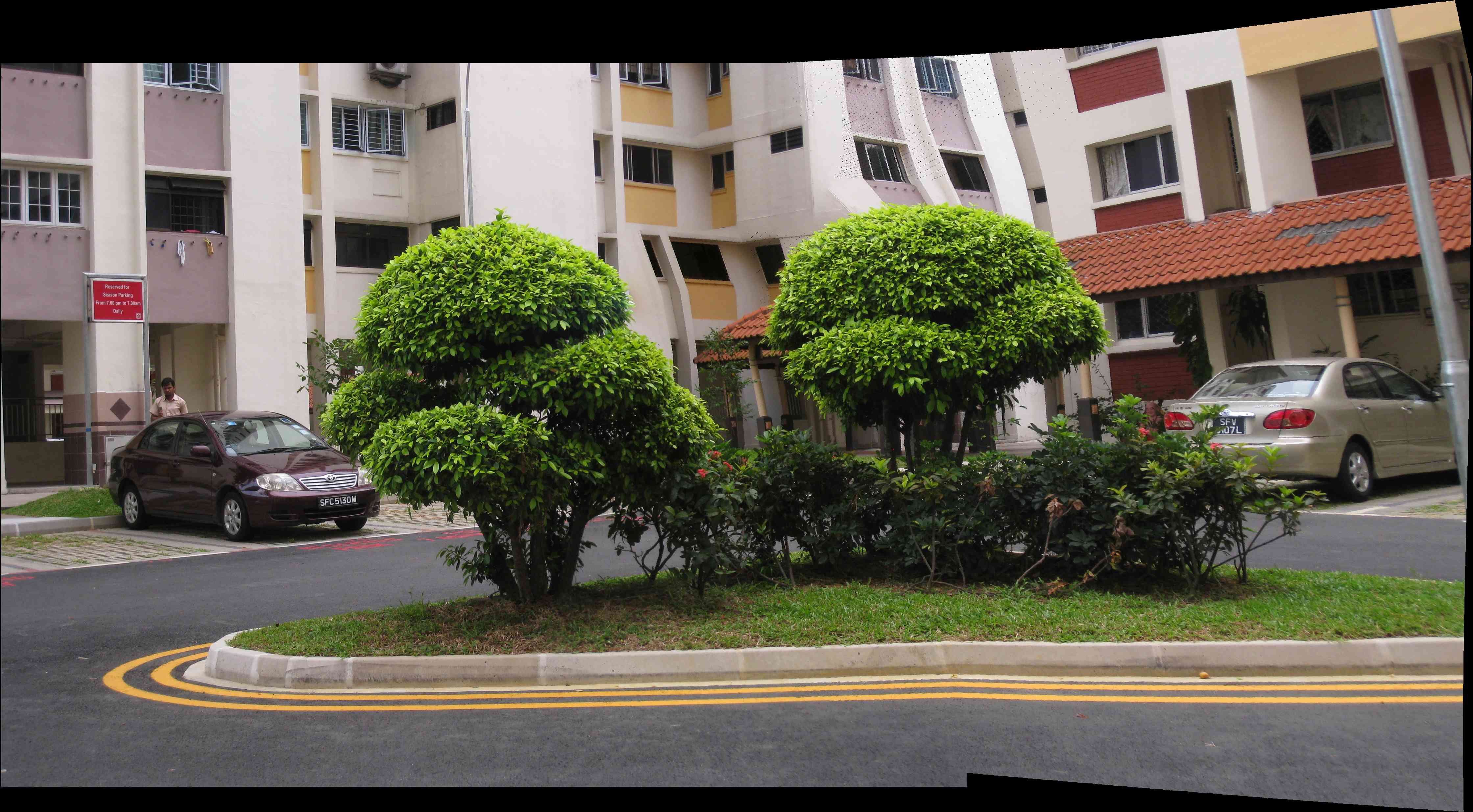}};
		\zoombox[color code=green]{0.47,0.73}
		\zoombox[color code=blue]{0.62,0.73}
		\end{tikzpicture}
	}\\

	\subfloat[SPHP~\cite{chang_shape-preserving_2014}\label{fig:apartments_SPHP}] {%
		\begin{tikzpicture}[zoomboxarray, zoomboxes below, zoomboxarray rows=1, execute at end picture={\draw[thick,red] (1.5,-1.3) ellipse (1.2cm and 0.5cm); \draw[thick,red] (4,-1) ellipse (0.5cm and 0.5cm);}]
		\node [image node] {\includegraphics[height=0.12\textheight]{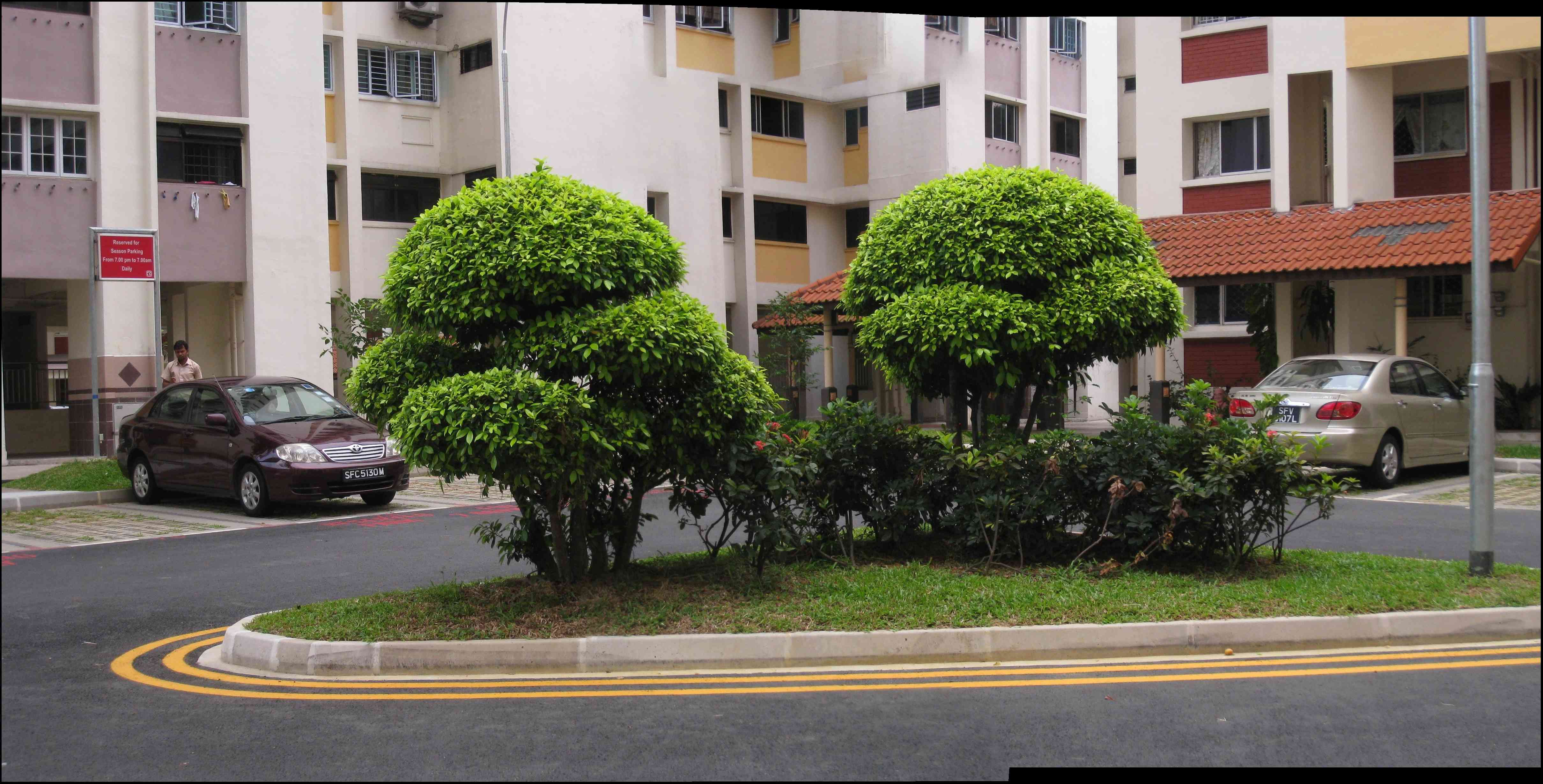}};
		\zoombox[color code=green]{0.5,0.83}
		\zoombox[color code=blue]{0.65,0.79}
		\end{tikzpicture}
	}\,\!
	\subfloat[APAP~\cite{zaragoza_as-projective-as-possible_2014}\label{fig:apartments_APAP}] {%
		\begin{tikzpicture}[zoomboxarray, zoomboxes below, zoomboxarray rows=1, execute at end picture={\draw[thick,red] (4.5,-1) ellipse (1cm and 0.5cm);}]
		\node [image node] {\includegraphics[height=0.12\textheight]{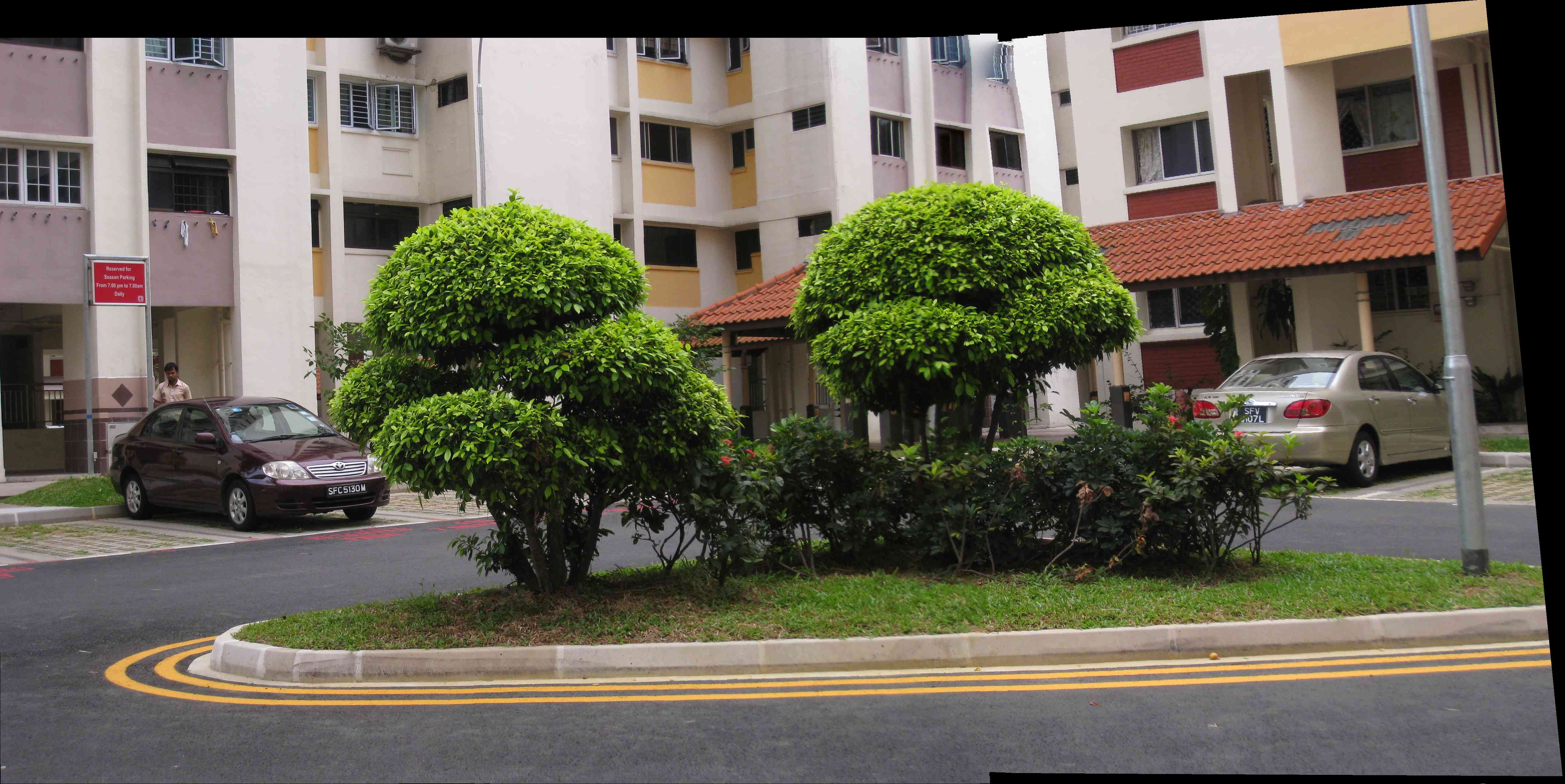}};
		\zoombox[color code=green]{0.5,0.75}
		\zoombox[color code=blue]{0.62,0.82}
		\end{tikzpicture}
	}\,\!
	\subfloat[BRAS\label{fig:apartments_BRAS}] {%
		\begin{tikzpicture}[zoomboxarray, zoomboxes below, zoomboxarray rows=1]
		\node [image node] {\includegraphics[height=0.12\textheight]{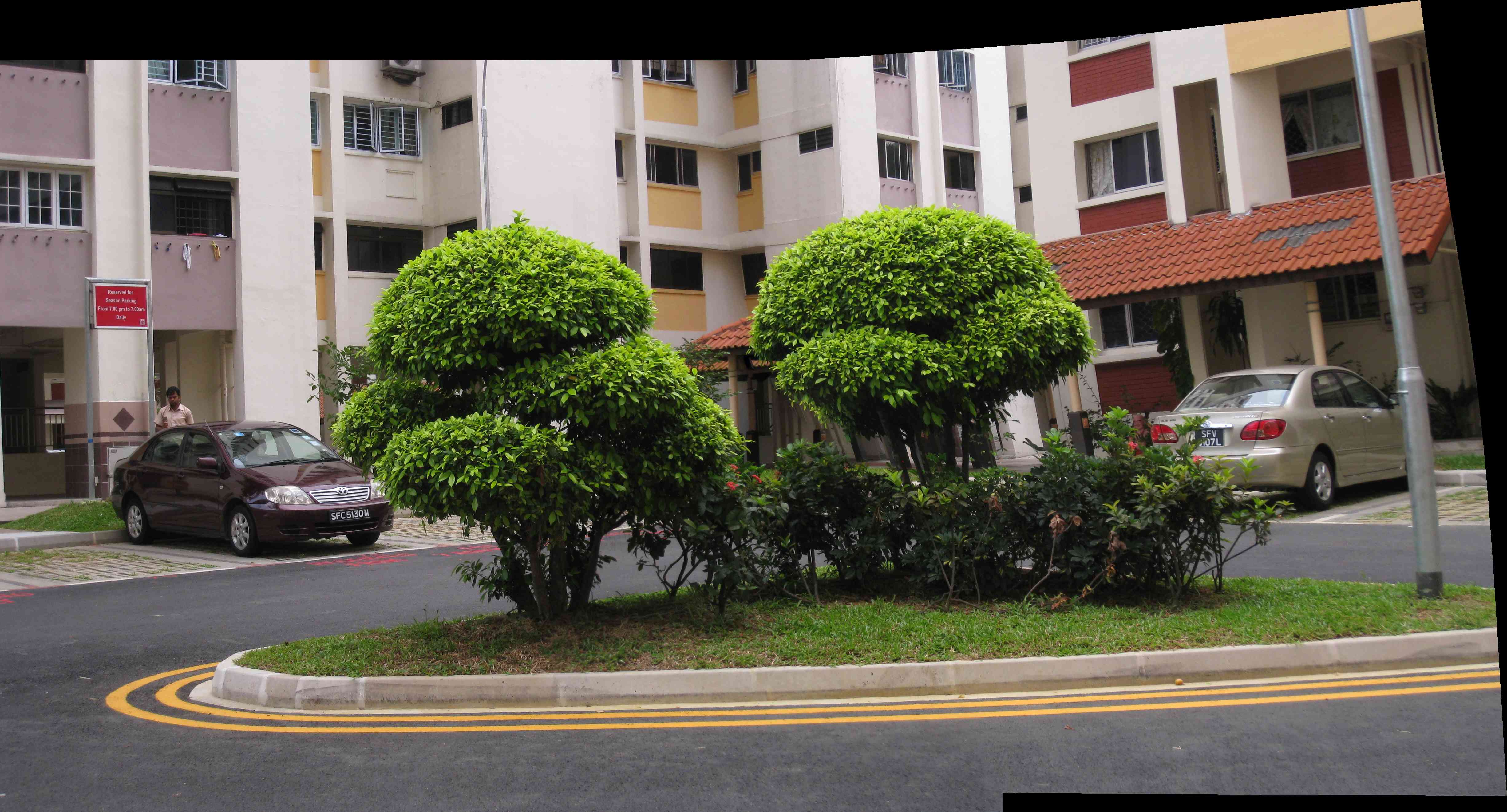}};
		\zoombox[color code=green]{0.5,0.73}
		\zoombox[color code=blue]{0.64,0.77}
		\end{tikzpicture}
	}\\	
	\caption{Qualitative comparison on the stitched images of the {\it apartments\/} dataset~\cite{gao_constructing_2011}. Regions with artifacts are magnified in the insets \tcr{and circled in red} at the bottom. In particular, please notice the duplicated windows in the stitched images of \protect\subref{fig:apartments_ICE} ICE, \protect\subref{fig:apartments_SPHP} SPHP, and \protect\subref{fig:apartments_APAP} APAP (in the blue insets). Also notice the ghosting, line distortion, and misplacement artifacts in the stitched images of \protect\subref{fig:apartments_autostitch} AutoStitch, \protect\subref{fig:apartments_CPW} CPW and \protect\subref{fig:apartments_SPHP} SPHP (in the green insets).}\label{fig:apartments}
\end{figure*}

\subsubsection{Quantitative Evaluation of Alignment Accuracy}\label{subsubsec:quantitative}
The truncated $\ell_2$-norm for several image sets processed by
different algorithms are presented in Table~\ref{tab:quantitive}. \tcr{The corresponding stitched images (other than those already shown) can be found in the supplementary document.} The best
results are highlighted and ``-'' stands for the case where the corresponding
algorithm is not capable of handling multiple images.  It can be observed
	that BRAS consistently achieves the best performance, and usually
outperforms other algorithms by a clear margin.  We note here again that
the underlying motion model for BRAS is essentially the same as APAP;\@ thus
its improved accuracy is attributable to its improved model fitting strategy
which crucially relies on exact capture of the rank. Furthermore, this confirms
the intuition that pixel-based method can achieve superior accuracy as BRAS is
of pixel-based nature (as discussed in Section~\ref{subsec:relation}).

\begin{table}[h!]
    \centering
    \caption{Alignment accuracies of different methods on each dataset, measured in truncated $\ell_2$ norm. Lower values typically indicate higher accuracies.}
	\begin{tabular}{ccccc}
        \toprule
		{\bf Dataset}    & {\bf BRAS}         & {\bf APAP}~\cite{zaragoza_as-projective-as-possible_2014}        & {\bf CPW}~\cite{hu_multi-objective_2015}         & {\bf SPHP}~\cite{chang_shape-preserving_2014}        \\
		\toprule
		{\it skyscraper\/}~\cite{chang_shape-preserving_2014} & $\mathbf{16.83}$ & $17.21$   & $-$           & $17.40$   \\
        \midrule
		{\it railtracks\/}~\cite{zaragoza_as-projective-as-possible_2014} & $\mathbf{16.74}$ & $17.67$   & $18.13$   & $19.23$   \\
        \midrule
		{\it apartments\/}~\cite{gao_constructing_2011} & $\mathbf{14.88}$ & $16.68$   & $15.92$ & $17.45$   \\
        \midrule
		{\it rooftops\/}~\cite{lin_smoothly_2011}   & $\mathbf{16.61}$ & $17.09$   & $-$        & $17.37$   \\
        \midrule
		{\it forest\/}~\cite{zaragoza_as-projective-as-possible_2014}   & $\mathbf{20.53}$ & $20.83$   & $-$           & $20.75$   \\
        \midrule
		{\it carpark\/}~\cite{gao_constructing_2011}   & $\mathbf{13.99}$ & $15.51$   & $16.67$   & $16.52$   \\
        \midrule
		{\it temple\/}~\cite{gao_constructing_2011}   & $\mathbf{12.70}$ & $14.00$   & $14.72$   & $14.43$   \\
        \midrule
		{\it hanger\/}~\cite{lin_smoothly_2011}    & $\mathbf{7.61}$  & $8.15$    & $9.41$    & $10.44$   \\
		\midrule
		{\it couch\/}~\cite{lin_smoothly_2011} & $\mathbf{13.24}$ & $13.75$ & $13.92$ & $14.51$\\
        \bottomrule
	\end{tabular}\label{tab:quantitive}
\end{table}

Finally, we compare in Fig.~\ref{fig:run_time} the running time of each algorithm on \tcr{all the datasets included in Table~\ref{tab:quantitive}. In particular, Fig.\ \ref{fig:run_time} plots the average running times against the total number of pixels in each dataset. The running time for BRAS includes every stage illustrated in Fig.~\ref{fig:flowchart} as described in Section~\ref{subsec:initial} to Section~\ref{subsec:multiple}.
}  AutoStitch and ICE are not included since they are
commercial softwares. We employ a computer with an Intel Core i5--6200, 2.30GHz
CPU and 8GB of RAM.\@ \tcr{SPHP usually runs the fastest, while BRAS usually ranks the second. Moreover, BRAS scales well when the image resolution increases. Note that the \emph{couch} image set (with around $2\times 10^7$ pixels) has relatively small motion and thus all methods run faster on it.}

\begin{figure}[h!]
	\centering
	\includegraphics[width=0.95\linewidth]{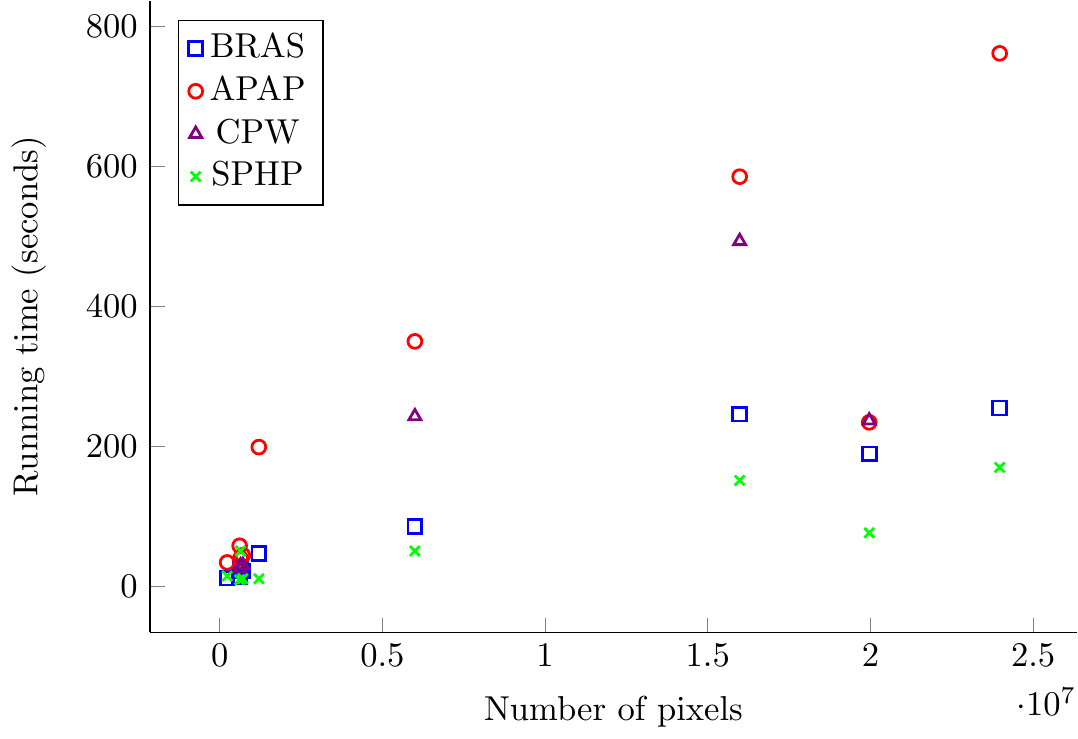}
	\caption{\tcr{Running time of different methods on each dataset in Table~\ref{tab:quantitive}.}}\label{fig:run_time}
\end{figure}

\subsection{Stitching in the Presence of Large Moving Objects}\label{subsec:large_objects}

One of the most difficult cases for panoramic alignment and stitching is when
large moving objects are present in the image set. The \emph{catabus} image
set~\cite{brasWeb} shown in Fig.~\ref{fig:catabusInput} represents such a case.
Note that this example is representative of scenarios where feature based
methods are unlikely to work well. This is because there is a large moving
object in the foreground with significant motion (the bus) while the background
is relatively simple.  That is, feature based methods must largely rely on the
moving object features for alignment and the large motion means that the
accuracy of the alignment is fundamentally limited. 

Fig.~\ref{fig:large_objects} shows the stitched image results generated using
different methods. The corresponding aligned images are in the supplementary
document.  Figs.~\ref{fig:catabusICE},~\ref{fig:catabusCPW}
and~\ref{fig:catabusAPAP}, show results of the most competitive feature based
methods, ICE, CPW and APAP respectively. Distortions due to misalignment can be
seen in the yellow line on the road and the duplication of a tree can be easily
identified.  In contrast, being pixel based, BRAS bases its alignment on common
overlapping regions between the two images in the {\em catabus\/} set, where
the accuracy of the said alignment is enabled by the rank-1 and sparse
decomposition. It is readily apparent from Fig.~\ref{fig:catabusBRAS} that BRAS
composes a realistic stitched image free of the distortions in
Figs.~\ref{fig:large_objects} (b), (c) and (d).

\begin{figure}
	\subfloat[Input\label{fig:catabusInput}]{%
		\begin{tikzpicture}
			\node (input01) {\includegraphics[height=0.17\textheight]{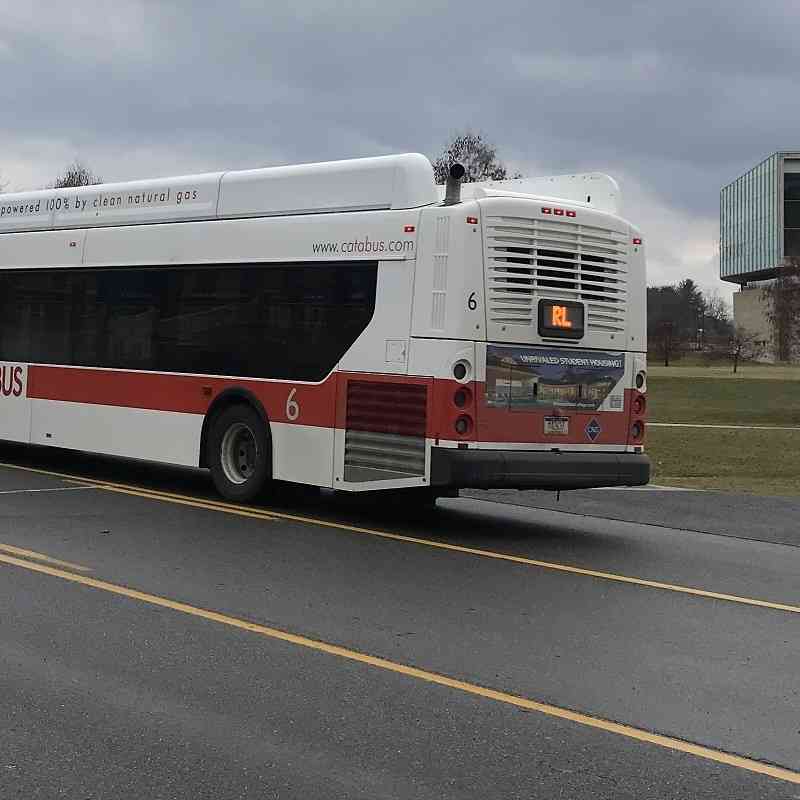}};
			\node[right=-0.3em of input01] {\includegraphics[height=0.17\textheight]{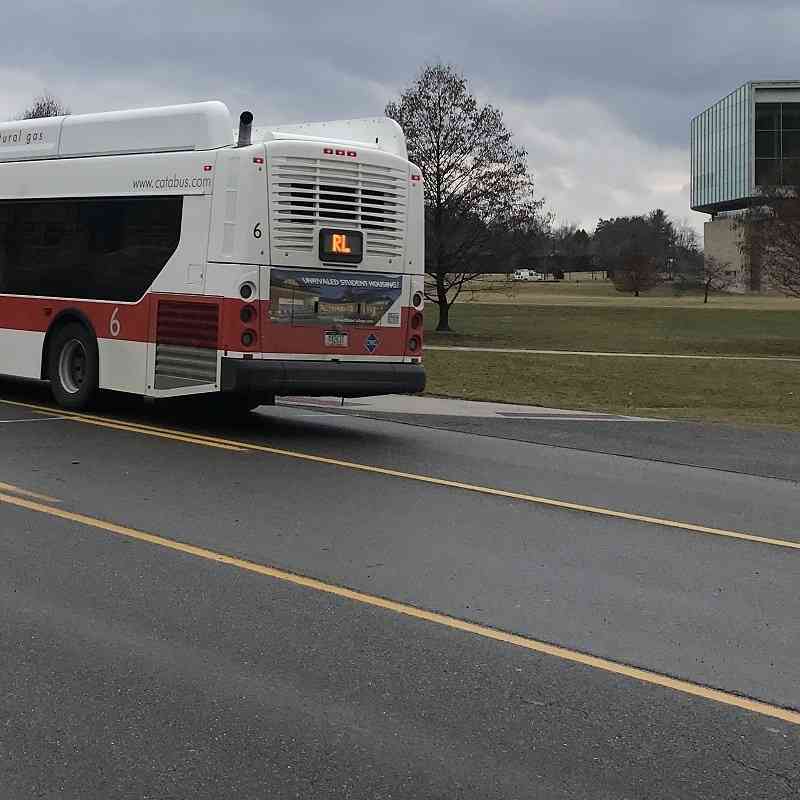}};
		\end{tikzpicture}
	}

	\subfloat[ICE~\cite{ICE}\label{fig:catabusICE}] {%
		\begin{tikzpicture}
			\node {\includegraphics[height=0.163\textheight]{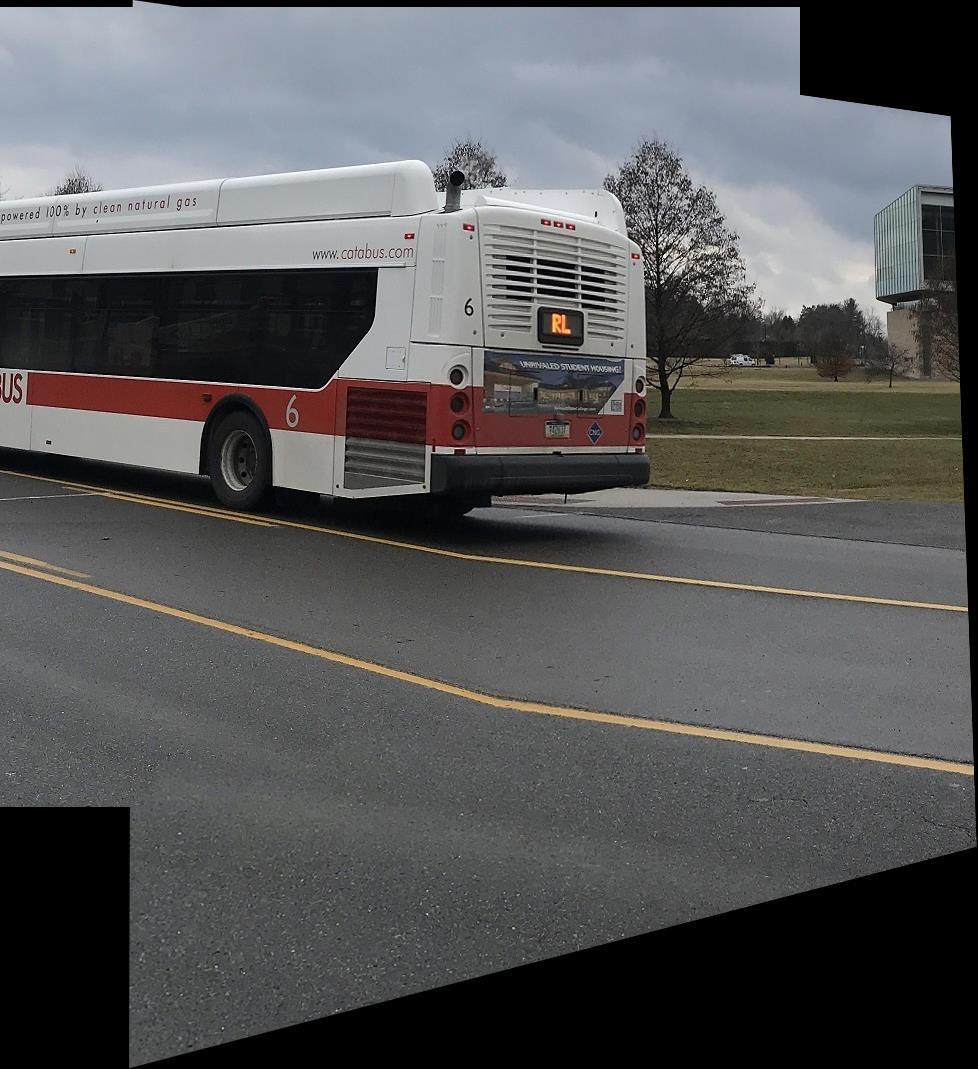}};
			\draw[thin,red] (0.1,-0.7) ellipse (0.4cm and 0.25cm);
			\draw[thin,red] (-0.1,-0.1) ellipse (0.2cm and 0.15cm);
			\draw[thin,red] (-0.1,1.4) ellipse (0.2cm and 0.1cm);
		\end{tikzpicture}
	}\hspace{-3mm}
	\subfloat[CPW~\cite{hu_multi-objective_2015}\label{fig:catabusCPW}]{%
		\begin{tikzpicture}
			\node {\includegraphics[height=0.163\textheight]{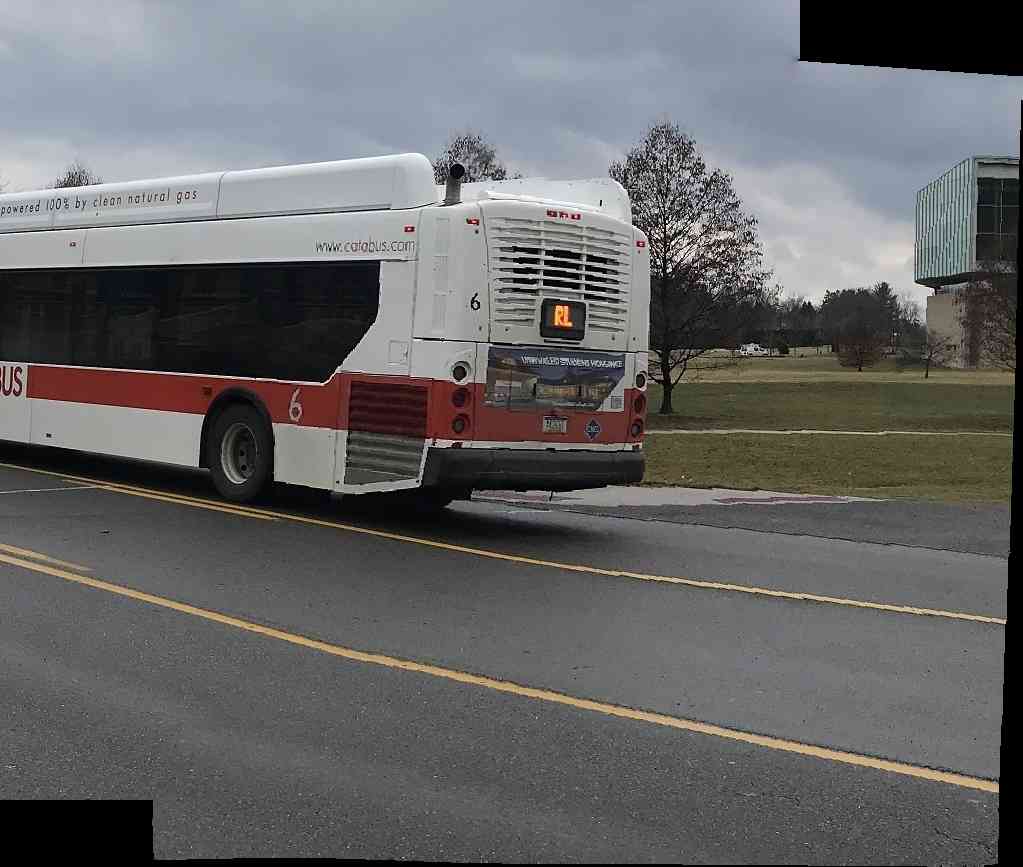}};
			\draw[thin,red] (-0.7,-1) ellipse (0.4cm and 0.25cm);
			\draw[thin,red] (0.3,-0.6) ellipse (0.2cm and 0.15cm);
			\draw[thin,red] (-0.2,1.3) ellipse (0.2cm and 0.1cm);
		\end{tikzpicture}
	}

	\subfloat[APAP~\cite{zaragoza_as-projective-as-possible_2014}\label{fig:catabusAPAP}]{%
		\begin{tikzpicture}
			\node {\includegraphics[height=0.175\textheight]{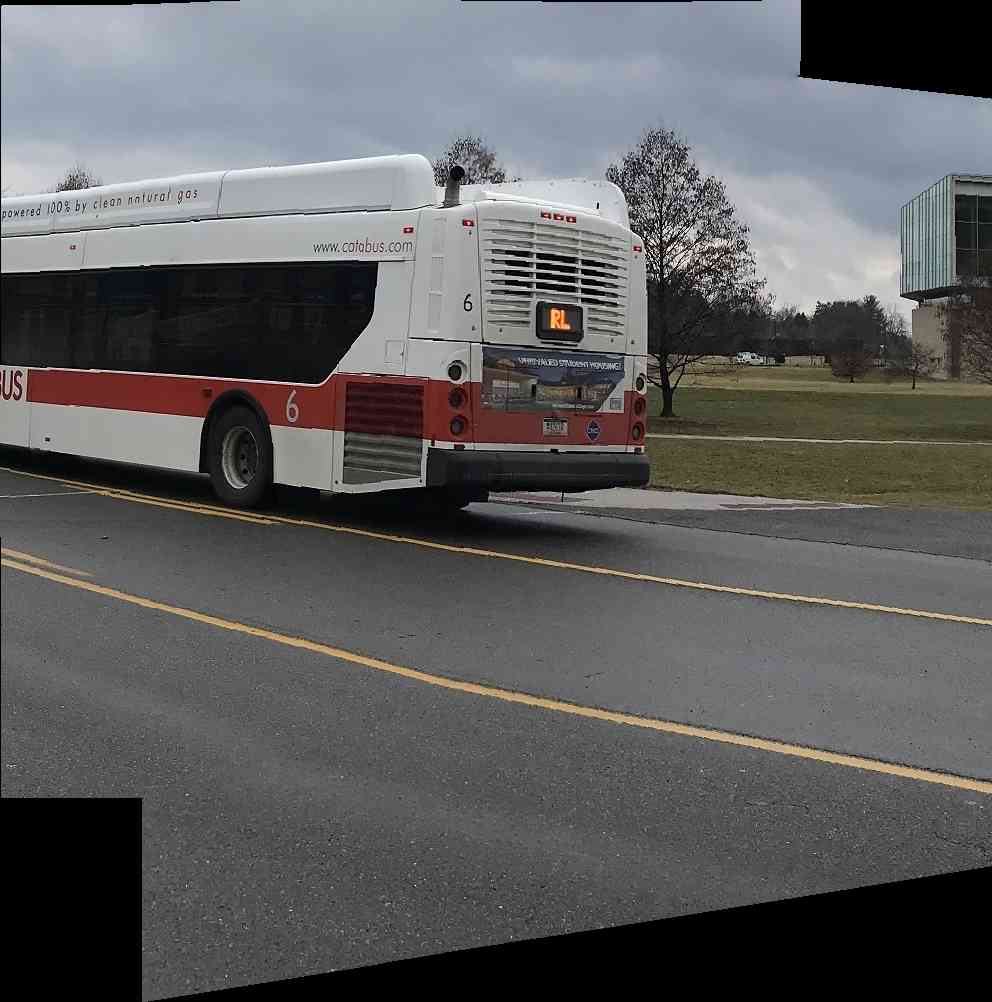}};
			\draw[thin,red] (-0.4,-0.8) ellipse (0.4cm and 0.25cm);
			\draw[thin,red] (0,-0.2) ellipse (0.2cm and 0.15cm);
			\draw[thin,red] (-0.1,1.5) ellipse (0.2cm and 0.1cm);
		\end{tikzpicture}
	}\hspace{-3mm}
	\subfloat[BRAS\label{fig:catabusBRAS}]{%
		\begin{tikzpicture}
			\node {\includegraphics[height=0.175\textheight]{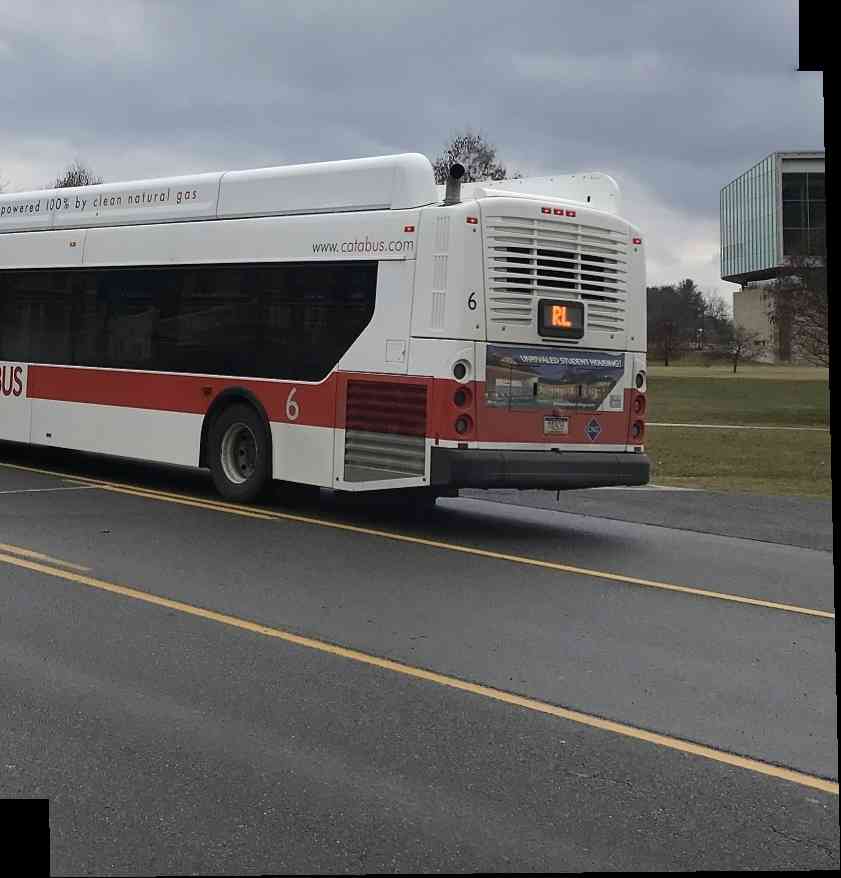}};
		\end{tikzpicture}
	}
	\caption{Stitching results for \emph{catabus} dataset~\cite{brasWeb} -- the case with large moving objects. The red circles highlight line distortions and tree duplication in \protect\subref{fig:catabusICE} ICE, \protect\subref{fig:catabusCPW} CPW and \protect\subref{fig:catabusAPAP} APAP, which are absent in \protect\subref{fig:catabusBRAS} BRAS.}\label{fig:large_objects}
\end{figure}

\subsection{\tcr{Limitations and Future Works}}\label{subsec:limitation}
\tcr{For our implementation, we integrate the geometric transformation
	model proposed in~\cite{zaragoza_as-projective-as-possible_2014}  for its simplicity and relatively high accuracy. Nevertheless,
	visual aspects such as shape/structure preservation were not considered
	in~\cite{zaragoza_as-projective-as-possible_2014}, and shape/structure
	distortions may appear in the stitched images, especially around
	non-overlapping regions. Similarly, our method may exhibit such drawbacks
	as well, as can be seen in Fig.~\ref{fig:railtracks_BRAS} and
	Fig.~\ref{fig:apartments_BRAS}, etc. However, our contributions are
	complementary to those works that focus on improving the geometric
	transformation models, and our framework can be flexibly combined with
	recently developed geometric models. Therefore, addressing shape/structure
	distortions by integrating in the BRAS framework, models such as those in~\cite{chang_shape-preserving_2014,chen_natural_2016} is an interesting direction of
	future exploration.
}

\section{Conclusion}\label{sec:conclusion}

We develop a new bundle robust alignment method for panoramic
stitching (BRAS). We formulate the alignment problem as the recovery of a rank-1
matrix under sparse corruptions in the transformation domain, and develop
efficient algorithms and theoretical guarantees, together with important
generalizations to handle realistic scenarios. Unlike most of the existing
algorithms that employ feature matching, our method works directly on pixels.
In contrast with other panoramic alignment techniques based on matrix
decompositions, exactly forcing a rank-1 constraint (vs.\ existing convex
relaxations) plays a crucial role in ensuring practical successes. Extensive
experiments confirm that BRAS aligns more accurately and often
achieves better visual quality of panoramic stitched images than many
state-of-the art techniques.

\bibliographystyle{IEEEtran}
\bibliography{Panorama,Optimization}

\newpage
\pagebreak

\onecolumn
\begin{center}
	\Huge Supplementary Document
\end{center}
\section{Proof of the Key Induction Step}\label{sec:proof_induction}
\begin{table}[h!]
	\centering
	\caption{Partial list of symbols and their concise meanings used in Section~\ref{sec:proof_induction} and Section~\ref{sec:robust_alignment}. Please refer to the text for detailed explanations. For some other standard definitions, please refer to the {\bf Notation} part of Section~\ref{sec:robust_alignment}.}
	\begin{tabular}{cm{3.5cm}m{10.8cm}}
		\toprule
		\bf Symbols & \bf Space &  \bf Meanings\\
		\toprule
		$m$ & Integers & Number of pixels on canvas\\
		\midrule
		$n$ & Integers & Number of images\\
		\midrule
		$h,w$ & Integers & Dimension of the canvas\\
		\midrule
		$d$ & Integers & Number of transformation parameters per image; equals $8$ for the single homography case\\
		\midrule
		$\beta_0,\beta_1,\zeta,q$ & $\mathbb{R}$ & Positive constant parameters\\
		\midrule
		$\tau$ & $\mathbb{R}^{d\times n}$ & Geometric transformation parameters; represent plane homographies unless otherwise stated\\
		\midrule
		$\dt$ & $\mathbb{R}^{d\times n}$ & Incremental transformation parameters\\
		\midrule
		$\bD$ & $\mathbb{R}^{m\times n}$ & Data matrix: images stacked as columns\\
		\midrule
		$\bL$  & $\mathbb{R}^{m\times n}$ & Rank-1 component \\
		\midrule
		$\bS$  & $\mathbb{R}^{m\times n}$ & Sparse component modeling errors\\
		\midrule
		$\bJ_i$  & $\mathbb{R}^{m\times d}$ & Jacobian of image $i$ with respect to its transformation parameters $\tau_i$\\
		\midrule
		$\bQ_i,\bR_i$	  & $\mathbb{R}^{m\times d},\mathbb{R}^{d\times d}$ & QR decomposition of $\bJ_i$\\
		\bottomrule
	\end{tabular}
	\label{tab:definitions}
\end{table}

\begin{algorithm}
    \renewcommand{\algorithmicrequire}{\textbf{Input:}}
    \renewcommand{\algorithmicensure}{\textbf{Output:}}
	\caption{Alternating Minimization for Solving Linearized Subproblem (Single Overlapping Case)}

    \begin{algorithmic}[1]
		\REQUIRE{$\bD_\tau\in\mathbb{R}^{m\times n}$, ${\{\bJ_i\in\mathbb{R}^{m\times d}\}}_{i=1}^n$,$\beta_0$,$\beta_1>0$,$q\in(0,1)$.}
		\STATE{$\bL^0\gets 0$,$\dt^0\gets 0$,$\zeta_0=\beta_0\frac{\|\bD_\tau\|_2}{\sqrt{mn}}$,$\bS^0\gets\cS_{\zeta_0}(\bD_\tau)$.}
        \FOR{$k=1, 2, \dots, K$}
		\STATE{$\bL^{k}\gets\cT_1\left\{\bD_\tau+\sum_{i=1}^n\bJ_i\Delta\tau_i^{k-1}\be_i^T-\bS^{k-1}\right\}$,\label{rank_projection}}
		\STATE{$\zeta_{k}\gets\beta_1\frac{q^{k}}{\sqrt{mn}}\|\bL^{k}\|_2$,}
		\STATE{$\bS^{k}\gets\cS_{\zeta_{k}}\left\{\bD_\tau+\sum_{i=1}^n\bJ_i\Delta\tau_i^{k-1}\be_i^T-\bL^{k}\right\}$,\label{supdate1_supp}}
		\STATE{$\dt^{k}\gets\sum_{i=1}^n\bJ_i^\dag(\bL^{k}+\bS^{k}-\bD_\tau)\be_i\be_i^T$.\label{dtupdate_supp}}
        \ENDFOR
		\ENSURE{$\widehat{\bL}\gets\bL^K$, $\widehat{\bS}\gets\bS^K$, $\widehat{\dt}\gets\dt^K$.}
    \end{algorithmic}\label{alg:single_supp}
\end{algorithm}

This section is devoted to the proof of the key induction step along the proof of Theorem~\ref{thm:convergence}. Our goal is to prove $l_{k+1}\leq\frac{1}{16\mu\nu\sqrt{d}}\frac{\sigma^\ast q^{k+1}}{\sqrt{n}}$,
	$\supp(\bS^{k+1})\subset\supp(\bS^\ast)$ and $s_{k+1}\leq\frac{5\widetilde{\mu}\sigma^\ast q^{k+1}}{\sqrt{mn}}$ assuming $l_k\leq\frac{1}{16\mu\nu\sqrt{d}}\frac{\sigma^\ast q^k}{\sqrt{n}}$,
    $\supp(\bS^k)\subset\supp(\bS^\ast)$ and $s_k\leq\frac{5\widetilde{\mu}\sigma^\ast q^k}{\sqrt{mn}}$ where $s_k:=\|\bE^k\|_{\ell_\infty}$ and
	$l_k:=\max_i\|\bG^k\be_i\|_{\ell_2}$. We copy Algorithm~\ref{alg:single} in the paper here as Algorithm~\ref{alg:single_supp} for easier reference. We also list some of the symbols and notations in Table~\ref{tab:definitions}. 
Let us define $\bM^k=\bH^k+\bE^k$. We first prove four technical lemmas:
\begin{lemma}
	For $\ k\geq 0$, if $\supp(\bE^k)\subset\supp(\bS^\ast)$, then
	\[
	\|\bQ_i^T\bE^k\be_l\|_{\ell^2}\leq\nu\alpha_1\sqrt{md}s_k,\forall i,l=1, 2, \dots, n.
	\]\label{prop:qle2}
\end{lemma}

\begin{proof}
	Since $\bE^k$ has at most $\alpha_1$ fraction of non-zeros in each column (A.3), using Cauchy-Schwartz inequality,
	\begin{align*}
		\|\bQ_i^T\bE^k\be_l\|_{\ell^2}^2&=\sum_{j_1=1}^m\sum_{j_2=1}^m\be_l^T{(\bE^k)}^T\be_{j_1}\be_{j_1}^T\bQ_i\bQ_i^T\be_{j_2}\be_{j_2}^T\bE^k\be_l\\
										&\leq\sum_{j_1,j_2}|{(\bE^k)}_{j_1, l}||{(\bE^k)}_{j_2, l}|{\left(\max_{1\leq l\leq m}\|\bQ_i^T\be_l\|_{\ell^2}\right)}^2\\
										&\leq{\left(\alpha_1 m\|\bE^k\|_{\ell^\infty}\right)}^2{\left(\nu\sqrt{\frac{d}{m}}\right)}^2={\left(\nu\alpha_1\sqrt{md}s_k\right)}^2.
	\end{align*}
\end{proof}

\begin{lemma}
	For $i=1,2,\dots,n$, $\forall\bv\in\mathbb{R}^n$ and $k\geq 0$, if $\supp(\bE^k)\subset\supp(\bS^\ast)$, $n\geq 5\mu^2$, $0<\delta\leq\frac{1}{20\mu^2}$, then
	\[
	\|\bQ_i^T\bM^k\bv\|_{\ell^2}\leq\left(\frac{1}{4\mu^2}l_k+\frac{5}{4}\alpha_1\nu\sqrt{md}s_k\right)n\|\bv\|_{\ell^\infty}.
	\]\label{lemm:ql2}
\end{lemma}
\begin{proof}
	Under our choice of $\delta$, $1+(n-1)\delta\leq\frac{n}{4\mu^2}$; thus $n+1+(n-1)\delta\leq\frac{5}{4}n$ since $\mu\geq 1$. Therefore,
	\begin{align}
		\|\bQ_i^T\bG^k\bv\|_{\ell^2}\leq&\|\bQ_i^T\bG^k\be_i\|_{\ell^2}|{(\bv)}_i|+\sum_{j\neq i}\|\bQ_i^T\bQ_j\bQ_j^T\bG^k\be_j\|_{\ell^2}|{(\bv)}_j|\nonumber\\
		\leq& l_k\|\bv\|_{\ell^\infty}+\left(\sum_{j\neq i}\|\bQ_i^T\bQ_j\|_2\right)l_k\|\bv\|_{\ell^\infty}
		\leq[1+(n-1)\delta]l_k\|\bv\|_{\ell^\infty}\leq\frac{n}{4\mu^2}l_k\|\bv\|_{\ell^\infty},\label{eqn:qgv1}
	\end{align}
	
	\begin{align}
		\|\bQ_i^T\bF^k\bv\|_{\ell^2}\leq&[1+(n-1)\delta]\left(\max_{1\leq i\leq n}\|\bQ_i^T\bE^k\be_i\|_2\right)\|\bv\|_{\ell^\infty}\nonumber\\
		\stackrel{\mbox{\tiny Lemma 2}}{\leq}&[1+(n-1)\delta](\alpha_1\nu\sqrt{md}s_k)\|\bv\|_{\ell^\infty},\label{eqn:qfv1}
	\end{align}
	\begin{align}
		\|\bQ^T_i\bE^k\bv\|_{\ell^2}^2&=\sum_{j_1=1}^m\sum_{j_2=1}^m\bv^T{(\bE^k)}^T\be_{j_1}\be_{j_1}^T\bQ_i\bQ_i^T\be_{j_2}\be_{j_2}^T\bE^k\bv\nonumber\\
									  &\leq\sum_{j_1,j_2}{\left(\max_{1\leq j\leq m}\|\bQ_i^T\be_j\|_{\ell^2}\right)}^2\left|\be_{j_1}^T\bE^k\bv\right|\left|\be_{j_2}^T\bE^k\bv\right|\nonumber\\
									  &\leq{\left(\nu\sqrt{\frac{d}{m}}\right)}^2{\left(\sum_{i=1}^n\sum_{j=1}^m{(\bE^k)}_{ji}{(\bv)}_i\right)}^2\nonumber\\
									  &={\left(\alpha_1\nu n\sqrt{md}s_k\|\bv\|_{\ell^\infty}\right)}^2.\label{eqn:qev}
	\end{align}
	Adding~(\ref{eqn:qgv1}),~(\ref{eqn:qfv1}),~(\ref{eqn:qev}) up gives the result.
\end{proof}

\begin{lemma}
	For $k\geq 0$, if $\supp(\bE^k)\subset\supp(\bS^\ast)$, then $\forall\bv\in\mathbb{R}^n$
	\[
	\|\bM^k\bv\|_{\ell^\infty}\leq\left[\nu\sqrt{\frac{d}{m}}l_k+(\alpha_2+\alpha_1\nu^2d)s_k\right]n\|\bv\|_{\ell^\infty}.
	\]\label{lemm:ehinf}
\end{lemma}

\begin{proof}
	$\forall i$, adding the following inequalities up	
	\begin{align*}
		\left|\be_i^T\bG^k\bv\right|&\leq\sum_{j=1}^n\|\be_i^T\bQ_j\|_{\ell^2}\|\bQ_j^T\bG^k\be_j\|_{\ell^2}|{(\bv)}_j|\\
		&\leq n\nu\sqrt{\frac{d}{m}}l_k\|\bv\|_{\ell^\infty},\\
	\end{align*}
	\begin{align*}
		|\be_i^T\bE^k\bv|&\leq\|\bE^k\|_\infty\|\bv\|_{\ell^\infty}\leq\alpha_2 ns_k\|\bv\|_{\ell^\infty},\\
		|\be_i^T\bF^k\bv|&\leq \nu\sqrt{\frac{d}{m}}\left(\max_{1\leq j\leq n}\|\bQ_j^T\bE^k\be_j\|_{\ell^2}\right)n\|\bv\|_{\ell^\infty}\\
		&\leq n\nu\sqrt{\frac{d}{m}}(\alpha_1\nu\sqrt{md}s_k)\|\bv\|_{\ell^\infty}=\alpha_1\nu^2nds_k\|\bv\|_{\ell^\infty}.
	\end{align*}
	and taking maximum over $i$ gives the desired result.
\end{proof}

\begin{lemma}
	$\forall k\geq 0$, if $\supp(\bE^k)\subset\supp(\bS^\ast)$, then $\forall\bu\in\mathbb{R}^m$
	\[
		\left\|{(\bM^k)}^T\bu\right\|_{\ell^\infty}\leq\left[\nu\sqrt{\frac{d}{m}}l_k+\alpha_1(1+\nu^2d)s_k\right]m\|\bu\|_{\ell^\infty}.
	\]\label{lemm:ehtinf}
\end{lemma}

\begin{proof}
	For $i=1, \dots, n$, using Cauchy-Schwartz inequality,
	\begin{align*}
		\left\|\bQ_i^T\bu\right\|_{\ell^2}^2&=\sum_{j_1=1}^m\sum_{j_2=1}^m\bu^T\be_{j_1}\be_{j_1}^T\bQ_i\bQ_i^T\be_{j_2}\be_{j_2}^T\bu\\
								 &\leq m^2{\left(\nu\sqrt{\frac{d}{m}}\right)}^2\|\bu\|_{\ell^\infty}^2={(\nu\sqrt{md}\|\bu\|_{\ell^\infty})}^2.
	\end{align*}
	Lemma~\ref{lemm:ehtinf} comes from adding the following inequalities
	\begin{align*}
		|\be_i^T{(\bG^k)}^T\bu|&=|\be_i^T{(\bL^k-\bL^\ast)}^T\bQ_i\bQ_i^T\bu|\leq l_k\|\bQ_i^T\bu\|_{\ell^2}\leq l_k\nu\sqrt{md}\|\bu\|_{\ell^\infty},\\
		|\be_i^T{(\bF^k)}^T\bu|&\leq\left(\max_{1\leq i\leq n}\|\bQ_i^T\bE^k\be_i\|_{\ell^2}\right)\left(\max_{1\leq i\leq n}\|\bQ_i^T\bu\|_{\ell^2}\right)\leq\alpha_1\nu^2mds_k\|\bu\|_{\ell^\infty},\\
		|\be_i^T{(\bE^k)}^T\bu|&\leq\|\bE^k\|_1\|\bu\|_{\ell^\infty}\leq\alpha_1ms_k\|\bu\|_{\ell^\infty},
	\end{align*}
	and taking the maximum over $i$.
\end{proof}

We now start the major procedures. From step~\ref{dtupdate_supp} in Algorithm~\ref{alg:single_supp},
\begin{align*}
	&\sum_{i=1}^n\bJ_i(\dt^k-\dt^\ast)\be_i\be_i^T=\sum_{i=1}^n\bQ_i\bQ_i^T(\bL^k-\bL^\ast+\bS^k-\bS^\ast)\be_i\be_i^T=\bH^k.
\end{align*}
Therefore, from step~\ref{rank_projection} in Algorithm~\ref{alg:single_supp}
\begin{align}
	\sigma_1\bu_1&=\sigma^\ast({\bv^\ast}^T\bv_1)\bu^\ast+\bM^k\bv_1,\label{eqn:uv}\\
	\sigma_1\bv_1^T&=\sigma^\ast(\bu_1^T\bu^\ast){\bv^\ast}^T+\bu_1^T\bM^k\label{eqn:vu}.
\end{align}
where $\bL^{k+1}=\sigma_1\bu_1\bv_1^T$ is the SVD of $\bL^{k+1}$. Thus
\begin{align}
	\frac{\sigma^\ast({\bv^\ast}^T\bv_1)({\bu^\ast}^T\bu_1)}{\sigma_1}-1&=-\frac{\bu_1^T\bM^k\bv_1}{\sigma_1},\nonumber\\
	\bL^{k+1}-\bL^\ast&=\sigma_1\bu_1\bv_1^T-\sigma^\ast\bu^\ast{\bv^\ast}^T\nonumber\\
	&=-\frac{\bu_1^T\bM^k\bv_1}{\sigma_1}\sigma^\ast\bu^\ast{\bv^\ast}^T+\frac{\sigma^\ast({\bv^\ast}^T\bv_1)\bu^\ast\bu_1^T\bM^k}{\sigma_1}+\frac{\sigma^\ast(\bu_1^T\bu^\ast)\bM^k\bv_1{\bv^\ast}^T}{\sigma_1}+\frac{\bM^k\bv_1\bu_1^T\bM^k}{\sigma_1}.\label{eqn:lerr}
\end{align}
By right multiplying~\eqref{eqn:vu} with ${(\bM^k)}^T$ and plugging it into~\eqref{eqn:uv}
\begin{align}
	\sigma_1^2\bu_1-\bM^k{(\bM^k)}^T\bu_1=\sigma_1\sigma^\ast\bu^\ast({\bv^\ast}^T\bv_1)+\sigma^\ast(\bu_1^T\bu^\ast)\bM^k\bv^\ast,\label{eqn:u}
\end{align}
The induction hypotheses assures $\supp(\bE^k)\subset\supp(\bS^\ast)$ and
\begin{align*}
	l_k\leq\frac{1}{16\mu\nu\sqrt{d}}\frac{\sigma^\ast}{\sqrt{n}}, s_k\leq5\widetilde{\mu}\frac{\sigma^\ast}{\sqrt{mn}};
\end{align*}
from step~\ref{rank_projection} in Algorithm~\ref{alg:single_supp} and the inequalities
\begin{align}
	\|\bE^k\|_2&\leq\sqrt{\|\bE^k\|_1\|\bE^k\|_\infty},\nonumber\\
	\|\bH^k\|_F^2&=\sum_{i=1}^n\|\bG^k\be_i\|_{\ell^2}^2+\|\bF^k\be_i\|_{\ell^2}^2\stackrel{\mbox{\tiny Lemma 1}}{\leq}nl_k^2+n{\left(\alpha_1\nu\sqrt{md}s_k\right)}^2\nonumber
\end{align}
we have
\begin{align}
	|\sigma_1-\sigma^\ast|&\leq\|\bM^k\|_2\leq\sqrt{\|\bE^k\|_1\|\bE^k\|_\infty}+\|\bH^k\|_F\nonumber\\
	&\leq\sqrt{\alpha_1\alpha_2mn}s_k+\sqrt{n}l_k+\sqrt{n}\alpha_1\nu\sqrt{md}s_k\label{eqn:eh2}\\
	&\leq\sigma^\ast\left[\frac{1}{16\mu\nu\sqrt{d}}+5\widetilde{\mu}\sqrt{\alpha_1}(\sqrt{\alpha_2}+\nu\sqrt{\alpha_1d})\right]\leq\frac{\sigma^\ast}{8},\nonumber
\end{align}
whenever $\alpha_1\leq\frac{1}{160\widetilde{\mu}\nu\sqrt{d}}$ and $\alpha_2\leq\frac{1}{160\widetilde{\mu}}$. Thus we obtain
\begin{equation}
	\frac{8}{9}\leq\frac{\sigma^\ast}{\sigma_1}\leq\frac{8}{7}.\label{eqn:sigma}
\end{equation}
We next prove that $\bu_1$ and $\bv_1$ are also incoherent vectors. Indeed, from~(\ref{eqn:u}), taking $\|\cdot\|_{\ell^\infty}$ on both sides yields
\begin{align}
	\sigma_1^2\|\bu_1\|_{\ell^\infty}-\left\|\bM^k{(\bM^k)}^T\bu_1\right\|_{\ell^\infty}\leq\sigma_1\sigma^\ast\|\bu^\ast\|_{\ell^\infty}+\sigma^\ast\left\|\bM^k\bv^\ast\right\|_{\ell^\infty}\label{eqn:uincor}.
\end{align}
Using Lemma~\ref{lemm:ehtinf} with $\alpha_1\leq\frac{1}{40\widetilde{\mu}(1+\nu^2d)}$,
\begin{align*}
	\|{(\bM^k)}^T\bu_1\|_{\ell^\infty}\leq\sqrt{\frac{m}{n}}\left[\frac{1}{16\mu}+5\alpha_1(1+\nu^2d)\widetilde{\mu}\right]\sigma^\ast\|\bu_1\|_{\ell^\infty}\leq\frac{1}{4}\sqrt{\frac{m}{n}}\sigma^\ast\|\bu_1\|_{\ell^\infty},
\end{align*}
and Lemma~\ref{lemm:ehinf} with $\alpha_1\leq\frac{1}{80\nu^2d\widetilde{\mu}}$, $\alpha_2\leq\frac{1}{80\widetilde{\mu}}$,
\begin{align}
	\left\|\bM^k{(\bM^k)}^T\bu_1\right\|_{\ell^\infty}&\leq\frac{{\sigma^\ast}^2}{16}\sqrt{\frac{n}{m}}\sqrt{\frac{m}{n}}\|\bu_1\|_{\ell^\infty},\label{eqn:ehehtu}\\
	\left\|\bM^k\bv^\ast\right\|_{\ell^\infty}&\leq\frac{\sigma^\ast}{4}\sqrt{\frac{n}{m}}\|\bv^\ast\|_{\ell^\infty}\leq\frac{\mu\sigma^\ast}{4\sqrt{m}}.\label{eqn:ehv}
\end{align}
Plugging~\eqref{eqn:ehehtu} and~\eqref{eqn:ehv} into~\eqref{eqn:uincor} to obtain (similarly for $\bv_1$)
\begin{equation}
	\|\bu_1\|_{\ell^\infty}\leq \frac{2\mu}{\sqrt{m}}, \|\bv_1\|_{\ell^\infty}\leq \frac{2\mu}{\sqrt{n}}.\label{eqn:u1incor}
\end{equation}
Now using Lemma~\ref{lemm:ehtinf} again, together with~\eqref{eqn:u1incor}
\begin{equation}
	\left\|{(\bM^k)}^T\bu_1\right\|_{\ell^\infty}\leq\left[\frac{1}{8}+2\alpha_1\mu(1+\nu^2d)c_s\right]\frac{\sigma^\ast q^k}{\sqrt{n}},
	\label{eqn:ehu1inf}
\end{equation}
where $c_s:=5\widetilde{\mu}$. Symmetrically,
\begin{equation}
	\|\bM^k\bv_1\|_{\ell^\infty}\leq\left[\frac{1}{8}+2\mu(\alpha_2+\alpha_1\nu^2d)c_s\right]\frac{\sigma^\ast q^k}{\sqrt{m}}.
	\label{eqn:ehv1inf}
\end{equation}
And using Lemma~\ref{lemm:ql2} and~\eqref{eqn:u1incor}, by letting $\delta\leq\frac{1}{20\mu^2}$,
\begin{equation}
	\max_{1\leq i\leq n}\|\bQ_i^T\bM^k\bv_1\|_{\ell^2}\leq\left(\frac{5}{2}\alpha_1\mu\nu\sqrt{d}c_s+\frac{1}{32\mu^2\nu\sqrt{d}}\right)\sigma^\ast q^k.
	\label{eqn:qehv1}
\end{equation}
Also, working in the same manner as in~\eqref{eqn:eh2}, we obtain
\begin{align}
	\left|\bu_1^T\bM^k\bv_1\right|\leq\|\bM^k\|_2&\leq\sigma^\ast q^k\left[c_s\sqrt{\alpha_1}(\sqrt{\alpha_2}+\nu\sqrt{\alpha_1d})+\frac{1}{16\mu\nu\sqrt{d}}\right].\label{eqn:uehv}
\end{align}
For $i=1, 2, \dots, n$, a combination of~\eqref{eqn:lerr}~\eqref{eqn:ehu1inf}~\eqref{eqn:qehv1}~\eqref{eqn:uehv} yields
\begin{align}
	&\left\|\bQ_i^T(\bL^{k+1}-\bL^\ast)\be_i\right\|_{\ell^2}\nonumber\\
	\leq&\frac{\sigma^\ast|\bu_1^T\bM^k\bv_1|}{\sigma_1}\frac{\kappa\mu\sqrt{d}}{\sqrt{mn}}+\frac{\sigma^\ast\kappa}{\sigma_1}\sqrt{\frac{d}{m}}\left\|{(\bM^k)}^T\bu_1\right\|_{\ell^\infty}+\frac{\|\bQ_i^T\bM^k\bv_1\|_{\ell^2}}{\sigma_1}\left[\frac{\sigma^\ast\mu}{\sqrt{n}}+\left\|{(\bM^k)}^T\bu_1\right\|_{\ell^\infty}\right]\nonumber\\
	\leq&\frac{8}{7}\biggl\{\biggl[\frac{1}{16\nu\sqrt{d}}+\frac{1}{8}+c_s\mu\biggl(\sqrt{\alpha_1\alpha_2}+2\alpha_1\biggl(1+\nu^2d+\frac{\nu\sqrt{d}}{2}\biggr)\biggr)\biggr]\kappa\sqrt{\frac{d}{m}}
	+\frac{5}{4}\left(\frac{5}{2}\alpha_1\mu^2\nu\sqrt{d}c_s+\frac{1}{32\mu\nu\sqrt{d}}\right)\biggr\}\frac{\sigma^\ast q^k}{\sqrt{n}},\label{ineq:const}
\end{align}
where in~\eqref{ineq:const} we use~\eqref{eqn:sigma} and the following inequality
\begin{align*}
	\frac{\sigma^\ast\mu}{\sqrt{n}}+\|{(\bM^k)}^T\bu_1\|_{\ell^\infty}\leq\frac{\sigma^\ast\mu}{\sqrt{n}}\left[1+\left(\frac{1}{8\mu}+2\alpha_1(1+\nu^2d)c_s\right)\right]\leq\frac{\sigma^\ast\mu}{\sqrt{n}}\left[1+\left(\frac{1}{8}+\frac{1}{8}\right)\right]\leq\frac{5}{4}\frac{\sigma^\ast\mu}{\sqrt{n}},
\end{align*}
whenever $\alpha_1\leq\frac{1}{16c_s(1+\nu^2d)}$. Now for
$\alpha_1,\alpha_2$ sufficiently small (in the order of
$\mathcal{O}(d^{-1})$,$\mathcal{O}(d^{-\frac{1}{2}})$,
respectively) and $d\leq\frac{\sqrt{m}}{32\mu^2\nu\kappa}$, $\exists q<1$
such that~\eqref{ineq:const} becomes
\[
l_{k+1}\leq\frac{1}{16\mu\nu\sqrt{d}}\frac{\sigma^\ast q^{k+1}}{\sqrt{n}}.
\]
On the other hand, using~\eqref{eqn:lerr}~\eqref{eqn:ehu1inf}~\eqref{eqn:ehv1inf}~\eqref{eqn:uehv},
\begin{align}
	&\|\bL^{k+1}-\bL^\ast\|_{\ell^\infty}\nonumber\\
	\leq&\frac{\sigma^\ast|\bu_1^T\bM^k\bv_1|}{\sigma_1}\frac{\mu^2}{\sqrt{mn}}+\frac{\sigma^\ast}{\sigma_1}\frac{\mu}{\sqrt{m}}\left\|{(\bM^k)}^T\bu_1\right\|_{\ell^\infty}+\frac{\sigma^\ast\|\bM^k\bv_1\|_{\ell^\infty}}{\sigma_1}\frac{\mu}{\sqrt{n}}+\frac{\|\bM^k\bv_1\|_{\ell^\infty}\left\|{(\bM^k)}^T\bu_1\right\|_{\ell^\infty}}{\sigma_1}\nonumber\\
	\leq&\frac{\sigma^\ast q^k\mu^2}{\sqrt{mn}}\frac{\sigma^\ast}{\sigma_1}\biggl\{\frac{1}{16\mu\nu\sqrt{d}}+\frac{1}{4\mu}+4\alpha_1\nu^2dc_s+2(\alpha_1+\alpha_2)c_s\nonumber\\
		&+c_s(\sqrt{\alpha_1\alpha_2}+\alpha_1\nu\sqrt{d})+\left[\frac{1}{8\mu}+2\alpha_1(1+\nu^2d)c_s\right]\left[\frac{1}{8\mu}+2(\alpha_2+\alpha_1\nu^2d)c_s\right]\biggr\}\nonumber\\
	\leq&\frac{\sigma^\ast q^{k+1}\mu^2}{\sqrt{mn}},\label{eqn:lkinf}
\end{align}
again for some $\alpha_1,\alpha_2$ sufficiently small (same orders as in~\eqref{ineq:const}) and some $q<1$. Under the same conditions,
\begin{align*}
	\|\bH^k\|_{\ell^\infty}&\leq\max_{i,j}|\be_i^T\bQ_j\bQ_j^T(\bL^k-\bL^\ast+\bS^k-\bS^\ast)\be_j|\\
	&\stackrel{\mbox{\tiny Lemma 1}}{\leq}\nu\sqrt{\frac{d}{m}}(l_k+\alpha_1\nu\sqrt{md}s_k)\leq\frac{3}{4}\frac{\sigma^\ast q^{k+1}\mu^2}{\sqrt{mn}},
\end{align*}
and $|{(\bL^\ast-\bL^{k+1}+\bH^k)}_{ij}|\leq\zeta_{k+1},\forall(i, j)\not\in\supp(\bS^\ast)$. From step~\ref{supdate1_supp} in Algorithm~\ref{alg:single_supp}, ${(\bS^{k+1})}_{ij}=0$ and $\supp(\bS^{k+1})\subset\supp(\bS^\ast)$. Let $\bN^k=\bL^\ast-\bL^{k+1}+\bH^k+\bS^\ast$, then
\begin{align}
	|{(\bE^{k+1})}_{ij}|=|{(\bS^\ast-\bS^{k+1})}_{ij}|=
	\begin{cases}
		|{(\bL^{k+1}-\bL^\ast-\bH^k\pm\zeta_{k+1})}_{ij}|, & \text{ if }|{(\bN^k)}_{ij}|>\zeta_{k+1},\\
		|{(\bS^\ast)}_{ij}|, & \text{ if }|{(\bN^k)}_{ij}|\leq\zeta_{k+1}.
	\end{cases}\nonumber
\end{align}
In both cases $|{(\bE^{k+1})}_{ij}|\leq\|\bL^{k+1}-\bL^\ast\|_{\ell^\infty}+\|\bH^k\|_{\ell^\infty}+\zeta_{k+1}$ and
\[
s_{k+1}\leq2\zeta_{k+1}=\frac{2\beta_1q^{k+1}}{\sqrt{mn}}\sigma_1\leq5\widetilde{\mu}\frac{q^{k+1}}{\sqrt{mn}}\sigma^\ast.
\]
\qedhere

\section{Analysis of Bundle Robust Alignment}\label{sec:analysis_bundle}
In this section we extend our analysis to the case of multiple overlapping
regions and derive convergence properties of Algorithm~\ref{alg:region}. We
first introduce the notations and assumptions, analogous to A.1 to A.6 in
Section~\ref{subsec:analysis}. We then state an extension to
Theorem~\ref{thm:convergence} in Section~\ref{subsec:thm_multi}, followed by
the proof. Likewise, the key induction step is postponed to
Section~\ref{subsec:key_multi} to avoid obscuring the main flow of analysis.

\begin{table}
	\centering
	\caption{Partial list of symbols and their concise meanings used in Section~\ref{sec:analysis_bundle} and Section~\ref{sec:bundle_alignment}. Please refer to the text for detailed explanations. For some other standard definitions, please refer to the {\bf Notation} part of Section~\ref{sec:robust_alignment}.}\label{tab:def_multi}
	\begin{tabular}{cm{3.5cm}m{10.8cm}}
		\toprule
		\bf Symbols & \bf Space &  \bf Meanings\\
		\toprule
		$s_1,s_2$ & $\mathbb{R}^{h_1\times w_1\times 128},\mathbb{R}^{h_2\times w_2\times 128}$ & SIFT images~\cite{liu_sift_2011}\\
		\midrule
		$L$ & Integers & Maximum pixel displacements\\
		\midrule
		$\bw(\bp)$ & ${\{-L,-L+1,\dots,L\}}^2$ & Discretized motion vector at pixel $\bp$\\
		\midrule
		$m$ & Integers & Number of pixels on the canvas\\
		\midrule
		$n$ & Integers & Number of images\\
		\midrule
		$m_r$ & Integers & Number of pixels in the $r$-th region\\
		\midrule
		$n_r$ & Integers & Number of images contributing to the $r$-th region\\
		\midrule
		$\alpha,\beta_0,\beta_1,\eta,\kappa,\zeta_r,q$ & $\mathbb{R}$ & Positive constant parameters\\
		\midrule
		$\bD_r$ & $\mathbb{R}^{m_r\times n_r}$ & Data matrix for the $r$-th region\\
		\midrule
		$\bL_r$	  & $\mathbb{R}^{m_r\times n_r}$ & Rank-1 component in the $r$-th region \\
		\midrule
		$\bS_r$  & $\mathbb{R}^{m_r\times n_r}$ & Sparse component in the $r$-th region\\
		\midrule
		$\bJ_{r,i}$	  & $\mathbb{R}^{m_r\times d}$ & Jacobian of image $i$ in the $r$-th region\\
		\midrule
		$\tau^{(u)}$ & $\mathbb{R}^{d\times n}$ & Transformation parameters for cell $u$; plane homographies unless otherwise stated\\
		\midrule
		$\Delta\tau^{(u)}$ & $\mathbb{R}^{d\times n}$ & Incremental transformation parameters for cell $u$\\
		\midrule
		$C_1,C_2$ & Integers & Number of cells in row and column\\
		\bottomrule
	\end{tabular}
\end{table}

\begin{algorithm}
	\renewcommand{\algorithmicrequire}{\textbf{Input:}}
	\renewcommand{\algorithmicensure}{\textbf{Output:}}
	\caption{Alternating Minimization for Solving Linearized Subproblem (Multiple Overlapping Case)}
	
	\begin{algorithmic}[1]
		\REQUIRE{${\{\bD_r\in\mathbb{R}^{m_r\times n_r}\}}_r$, ${\{\bJ_{r,i}\in\mathbb{R}^{m_r\times d}\}}_{r,i}$,${\{\beta_{r,0}\}}_r$,$\beta_1$,$q>0$.}
		\FORALL{$r$}
		\STATE{$\bL_r^0\gets 0$, $\bS_r^0\gets\cS_{\zeta_r^0}(\bD_r)$ where $\zeta_r^0=\beta_{r,0}\frac{\|\bD_r\|_2}{\sqrt{m_{r}n_r}}$.}
		\ENDFOR
		\STATE{$\dt^0\gets\sum_r\sum_{i\in\cI_r}{\left(\sum_{l\in\bar{\cI}_i}\frac{\bJ_{l, i}^T\bJ_{l, i}}{\zeta_l^0}\right)}^\dag\frac{\bJ_{r, i}^T(\bS_r^0-\bD_r)\be_{f_r^i}\be_i^T}{\zeta_r^0}$.}
		\FOR{$k=1,2, \dots, K$}
		\FORALL{$r$}
		\STATE{$\bL_r^{k}\gets\cT_1\left\{\bD_r+\sum_{i\in\mathcal{I}_r}\bJ_{r,i}\Delta\tau_i^{k-1}\be_{f_r^i}^T-\bS_r^{k-1}\right\}$,\label{lupdate_supp}}
		\STATE{$\zeta_r^{k}\gets\beta_1\frac{q^{k}}{\sqrt{m_{r}n_r}}\|\bL_r^{k}\|_2$,}
		\STATE{$\bS_r^{k}\gets\cS_{\zeta_r^{k}}\left\{\bD_r+\sum_{i\in\mathcal{I}_r}\bJ_{r,i}\Delta\tau_i^{k-1}\be_{f_r^i}^T-\bL_r^{k}\right\}$.\label{supdate_supp}}
		\ENDFOR
		\STATE{$\dt^{k}\gets\sum_{r,i\in\cI_r}\!\!{\left(\sum_{l\in\bar{\cI}_i}\!\!\frac{\bJ_{l, i}^T\bJ_{l, i}}{\zeta_l^{k}}\right)}^\dag\!\frac{\bJ_{r,i}^T(\bL_r^{k}+\bS_r^{k}-\bD_r)\be_{f_r^i}\be_i^T}{\zeta_r^{k}}$.}
		\ENDFOR
		\ENSURE{$\widehat{\bL_r}\gets\bL_r^K$, $\widehat{\bS_r}=\bS_r^K$, $\widehat{\dt}=\dt^K$.}
	\end{algorithmic}\label{alg:region_supp}
\end{algorithm}

\subsection{Notations}\label{subsec:notations}
We copy Algorithm~\ref{alg:region} in the paper here as Algorithm~\ref{alg:region_supp} for
easier reference. We also include a summary of the notations and symbols in Table~\ref{tab:def_multi}. Let $\bL_r^\ast, \bS_r^\ast$ and $\dt^\ast$ be the true model
parameters we aim to recover and
$\bD_r=\bL_r^\ast+\bS_r^\ast-\sum_{i\in\cI_r}\bJ_{r,i}\Delta\tau_i^\ast\be_{f_r^i}^T$
be the data matrices. We assume $\bS_r^\ast$ has a fraction of at most
$\alpha_r,\alpha_r^\prime$ non-zeros in each column and row. Let
$\bL_r^{k+1}=\sigma_r\bu_r\bv_r^T$ and
$\bL_r^\ast=\sigma^\ast_r\bu^\ast_r{\bv^\ast}_r^T$ be the singular value
decomposition of $\bL_r^{k+1}$ and $\bL^\ast_r$, where
$\bu_r,\bu_r^\ast\in\mathbb{R}^{m_r},\bv_r,\bv_r^\ast\in\mathbb{R}^{n_r}$ are
unit vectors.  Let $\bJ_{r,i}=\bQ_{r,i}\bR_{r,i} (i\in\cI_r)$ be the (reduced)
QR decomposition where $\bQ_{r,i}\in\mathbb{R}^{m_r\times d}$ and
$\bR_{r,i}\in\mathbb{R}^{d\times d}$; following the same reasoning as in
Section~\ref{subsec:analysis}, we define positive constants $\mu,\nu,\delta,\kappa,\gamma$ as
follows
\[
    \max_{1\leq i\leq m, j\in\cI_r}\|\bQ_{r,j}^T\be_i\|_{\ell^2}\leq\nu\sqrt{\frac{d}{m_r}},
\]
\[
    \|\bu^\ast_r\|_{\ell^\infty}\leq\frac{\mu}{\sqrt{m_r}}, \|\bv_r^\ast\|_{\ell^\infty}\leq\frac{\mu}{\sqrt{n_r}},
\]
\[
    \frac{1}{n_r-1}\sum_{j\neq i}\|\bQ_{r,j}^T\bQ_{r,i}\|_2\leq \delta,
\]
\[
    \max_{i\in\cI_r}\|\bQ_{r,i}^T\bu_r^\ast\|_{\ell^2}\leq\kappa\sqrt{\frac{d}{m_r}},\\
\]
\[
	\max_{i\in\cI_r}\|\bR_{r,i}\dt^\ast\be_i\|_{\ell^2}\leq\gamma\frac{\sigma_r^\ast}{\sqrt{n_{r}d}}.
\]
Also, $\cU_r:=\bigcup_{i\in\cI_r}\bar{\cI}_i$, $\omega_r:=\sum_{t\in\cU_r}\sqrt{\frac{\sigma_t^\ast}{\sigma_r^\ast}}\sqrt[4]{\frac{m_{t}n_r}{m_{r}n_t}}$.
Let $\bP_{r, t}^{i, k}=\bR_{r, i}{\left[\sum_{l\in\bar{\cI}_i}\frac{\bR_{l, i}^T\bR_{l, i}}{\zeta_l^k}\right]}^{-1}\frac{\bR_{t, i}^T}{\zeta_t^k}$. Let
$\bE_r^k=\bS_r^\ast-\bS_r^k$, $\bF_r^k=\sum_{i\in\cI_r}\sum_{t\in\bar{\cI_i}}\bQ_{r,i}\bP_{r,t}^{i,k}\bQ_{t,i}^T(\bS_t^k-\bS_t^\ast)\be_{f_t^i}\be_{f_r^i}$, $\bG_r^k=\sum_{i\in\cI_r}\sum_{t\in\bar{\cI_i}}\bQ_{r,i}\bP_{r,t}^{i,k}\bQ_{t,i}^T(\bL_t^k-\bL_t^\ast)\be_{f_t^i}\be_{f_r^i}^T$,
$\bH_r^k=\bF_r^k+\bG_r^k$ and $\bM_r^k=\bH_r^k+\bE_r^k$. It is easy to check that
\begin{align*}
    \bH_r^k=\sum_{i\in\cI_r}\bJ_{r, i}(\dt^k-\dt^\ast)\be_i\be_{f_r^i}^T.
\end{align*}
Therefore, steps~\ref{lupdate_supp} and~\ref{supdate_supp} in Algorithm~\ref{alg:region_supp} can be rewritten as
\begin{align}
    \bL_r^{k+1}&\gets\cT_1\left\{\bM_r^k+\bL_r^\ast\right\},\label{lupdate1_supp}\\
    \bS_r^{k+1}&\gets\cS_{\zeta_r^{k+1}}\left\{\bL_r^\ast-\bL_r^{k+1}+\bH_r^k+\bS_r^\ast\right\}.\label{eqn:supdate1_supp}
\end{align}
Finally, we define real sequences ${\{l_{r,k}\}}_{k},{\{s_{r,k}\}}_{k}$ as follows
\[
    l_{r,k}=\max_{i\in\cI_r}\|\bQ_{r,i}^T(\bL_r^k-\bL_r^\ast)\be_{f_r^i}\|_{\ell^2},s_{r,k}=\|\bE_r^k\|_{\ell^\infty}.
\]

\subsection{Main Result}\label{subsec:thm_multi}
We are now ready to state the theorem about convergence of Algorithm~\ref{alg:region_supp}. Basically it asserts that, under certain conditions, the sequences $\bL^k_r,\bS^k_r$ and $\dt^k$ generated by Algorithm~\ref{alg:region_supp} converge to the underlying true model parameters $\bL^\ast_r,\bS^\ast_r$ and $\dt^\ast$ in a linear rate.

\begin{theorem}
	Let $\tilde{\alpha_r}=\sum_{t\in\cU_r}\alpha_t\sqrt{\frac{\sigma_t^\ast}{\sigma_r^\ast}}\sqrt[4]{\frac{m_{t}n_r}{m_{r}n_t}}$. There exist constants
    $C_{\alpha_1}=\mathcal{O}\left(\frac{1}{d}\right)$, $C_{\alpha_2}=\mathcal{O}\left(\frac{1}{\sqrt{d}}\right)$, $C_d=\mathcal{O}(\sqrt{m})$ and
    $C_\delta$, $C_q$, such that if $\tilde{\alpha_r}\leq C_{\alpha_1}$, $\alpha_r^\prime\leq C_{\alpha_2}$, $d\leq C_d$, $\delta\leq C_\delta$, $C_q\leq q<1$,
    then $\forall\beta_{r,0}\in\left[\frac{\tilde{\mu}\sigma_r^\ast}{\|\bD_r\|_2},\frac{2\tilde{\mu}\sigma_r^\ast}{\|\bD_r\|_2}\right]$, $\beta_1\in\left[2\tilde{\mu},
    2.2\tilde{\mu}\right]$ where $\tilde{\mu}=\mu^2+\gamma\nu$, each $\bS_r^k
    (k=0,1,\dots,K)$ in algorithm~\ref{alg:region_supp} has the property
    $\supp(\bS_r^k)\subset\supp(\bS_r^\ast)$; furthermore, $\forall\varepsilon>0$, $K\geq\max_r\log_q\left(\frac{\varepsilon}{\beta_{r,0}\|\bD_r\|_2}\right)$,
    $\|\widehat{\bL_r}-\bL_r^\ast\|_F\leq\varepsilon$,
	$\|\widehat{\bS_r}-\bS_r^\ast\|_{\ell^\infty}\leq\frac{5\varepsilon}{\sqrt{m_{r}n_r}}$ and
    $\|\widehat{\dt}-\dt^\ast\|_F\leq M\varepsilon$ for some $M>0$.
\end{theorem}

\begin{proof}
    We prove it by induction. Define $l_0=\min_r\frac{1}{16\omega_r\mu\nu\sqrt{d}}$, $s_0=5(\mu^2+\nu\sqrt{d})$. For $k=0$, $\forall(i,
	j)\not\in\supp(\bS_r^\ast)$, ${(\bS_r^\ast)}_{ij}=0$ and
    \begin{align*}
		|{(\bD_r)}_{ij}|&\leq\|\bL_r^\ast\|_{\ell^\infty}+\|\sum_{i\in\cI_r}^n\bJ_{r,i}\dt^\ast\be_i\be_{f_r^i}^T\|_{\ell^\infty}\\
        &\leq\sigma_r^\ast\|\bu_r^\ast\|_{\ell^\infty}\|\bv_r^\ast\|_{\ell^\infty}+\max_{i,j}|\be_i^T\bQ_{r, i}\bR_{r, i}\dt^\ast\be_j|\\
		&\leq\sigma_r^\ast\frac{\mu^2}{\sqrt{m_{r}n_r}}+\nu\sqrt{\frac{d}{m_r}}\frac{\sigma_r^\ast}{\sqrt{n_{r}d}}\gamma\leq\zeta_{r}^0,
    \end{align*}
	and thus ${(\bS_r^0)}_{ij}=0$, $\supp(\bS_r^0)\subset\supp(\bS_r^\ast)$.
    Furthermore, for $d\leq\min_r\frac{\sqrt{m_r}}{16\omega_r\mu^2\nu\kappa}$,
	$l_{r,0}=\sigma_r^\ast\max_i\|\bQ_{r, i}^T\bu_r^\ast\|_{\ell^2}|{(\bv_r^\ast)}_i|\leq l_0\frac{\sigma_r^\ast}{\sqrt{n_r}}$,
	$\|\bL_r^0-\bL_r^\ast\|_{\ell^\infty}=\|\bL_r^\ast\|_{\ell^\infty}\leq\frac{\sigma_r^\ast\mu^2}{\sqrt{m_{r}n_r}}$
    and
	$s_{r,0}=\|\bE_r^0\|_{\ell^\infty}\leq\zeta_r^0+\|\bL_r^\ast\|_{\ell^\infty}+\|\sum_{i\in\cI_r}\bJ_{r, i}\dt^\ast\be_i\be_{f_r^i}^T\|_{\ell^\infty}\leq s_0\frac{\sigma_r^\ast}{\sqrt{m_{r}n_r}}$.
    Assume that $l_{r,k}\leq l_0\frac{\sigma_r^\ast q^k}{\sqrt{n}}$,
	$\supp(\bS_r^k)\subset\supp(\bS_r^\ast)$, $\frac{\zeta_r^k}{\zeta_t^k}\leq\frac{2\sigma_r^\ast}{\sigma_t^\ast}\sqrt{\frac{m_{t}n_t}{m_{r}n_r}}$ and $s_{r,k}\leq s_0\frac{\sigma_r^\ast q^k}{\sqrt{m_{r}n_r}}$; our goal is to show
    $l_{r,k+1}\leq l_0\frac{\sigma_r^\ast q^{k+1}}{\sqrt{n_r}}$,
	$\supp(\bS_r^{k+1})\subset\supp(\bS_r^\ast)$, $\frac{\zeta_r^{k+1}}{\zeta_t^{k+1}}\leq\frac{2\sigma_r^\ast}{\sigma_t^\ast}\sqrt{\frac{m_{t}n_t}{m_{r}n_r}}$ and $s_{r,k+1}\leq s_0\frac{\sigma_r^\ast q^{k+1}}{\sqrt{m_{r}n_r}}$. We delay this procedure
    to Section~\ref{subsec:key_multi} due to its length.
    For
	$K\geq\max_r\log_q\left(\frac{\varepsilon}{\beta_{r,0}\|\bD_r\|_2}\right)$, as proved in~(\ref{eqn:linf1}),
	$\|\bL_r^K-\bL_r^\ast\|_{\ell^\infty}\leq\frac{\mu^2\sigma_r^\ast q^K}{\sqrt{m_{r}n_r}}\leq\beta_{r, 0}\frac{\|\bD_r\|_2q^K}{\sqrt{m_{r}n_r}}$
	and $\|\bS_r^K-\bS_r^\ast\|_{\ell^\infty}\leq5\beta_{r,0}\frac{\|\bD_r\|_2q^K}{\sqrt{m_{r}n_r}}$,
    we have $\|\bL_r^K-\bL_r^\ast\|_F\leq\varepsilon$,
	$\|\bS_r^K-\bS_r^\ast\|_{\ell^\infty}\leq\frac{5\varepsilon}{\sqrt{m_{r}n_r}}$;
	using Lemma~\ref{lemm:ql21} in Section~\ref{subsec:key_multi}:
    \begin{align*}
        &\|\dt^K-\dt^\ast\|_F\\
		=&\sqrt{\sum_{i=1}^n\left\|\sum_{r\in\bar{\cI}_i}\bR_{r, i}^{-1}\bP_{r, r}^{i, K}\bQ_{r, i}^T(\bL_r^K-\bL_r^\ast+\bS_r^K-\bS_r^\ast)\be_{f_r^i}\right\|_{\ell^2}^2}\\
		\leq&\sqrt{\sum_{i=1}^n{\left[\sum_{r\in\bar{\cI}_i}\|\bR_{r, i}^{-1}\|_2(l_{r, k}+\alpha_r\nu\sqrt{m_{r}d}s_{r, k})\right]}^2}\\
		\leq&\sqrt{\sum_{i=1}^n|\bar{\cI}_i|\sum_{r\in\bar{\cI}_i}\|\bR_{r, i}^{-1}\|_2^2{(l_{r, k}+\alpha_r\nu\sqrt{m_{r}d}s_{r, k})}^2}\\
		\leq&\sqrt{\sum_{i=1}^n\sum_{r\in\bar{\cI_i}}\max_i|\bar{\cI_i}|{\left[\max_r\|\bR_{r,i}\|_2(l_0+\alpha_r\nu\sqrt{d}s_0)\right]}^2\frac{\varepsilon^2}{\tilde{\mu}^2n_r}}\\
		\leq&\frac{\varepsilon}{\tilde{\mu}}\sqrt{R\max_i|\bar{\cI_i}}|\left[\max_r\|\bR_{r,i}\|_2(l_0+\alpha_r\nu\sqrt{d}s_0)\right]
    \end{align*}
	where $R$ is the number of overlapping regions.
\end{proof}

\subsection{Proof of the Key Induction Step}\label{subsec:key_multi}

From~\eqref{lupdate1_supp}
\begin{align}
    \sigma_r\bu_r&=\sigma^\ast_r({\bv^\ast_r}^T\bv_r)\bu^\ast_r+\bM_r^k\bv_r,\label{eqn:uv1}\\
    \sigma_r\bv_r^T&=\sigma_r^\ast(\bu_r^T\bu_r^\ast){\bv_r^\ast}^T+\bu_r^T\bM_r^k\label{eqn:vu1}.
\end{align}
Note $\bu_r^T\bu_r=1$, we have
\begin{equation}
	\frac{\sigma_r^\ast({\bv_r^\ast}^T\bv_r)({\bu_r^\ast}^T\bu_r)}{\sigma_r}-1=-\frac{\bu_r^T\bM_r^k\bv_r}{\sigma_r},\label{eqn:uMv}
\end{equation}
and by combining~\eqref{eqn:uv1}\eqref{eqn:vu1}\eqref{eqn:uMv} together,
\begin{align}
	&\bL_r^{k+1}-\bL^\ast_r\nonumber\\
	=&\,\sigma_r\bu_r\bv_r^T-\sigma_r^\ast\bu_r^\ast{\bv_r^\ast}^T\nonumber\\
    =&\,-\frac{\bu_r^T\bM_r^k\bv_r}{\sigma_r}\sigma_r^\ast\bu_r^\ast{\bv_r^\ast}^T+\frac{\sigma_r^\ast({\bv_r^\ast}^T\bv_r)\bu_r^\ast\bu_r^T\bM_r^k}{\sigma_r}+\frac{\sigma_r^\ast(\bu_r^T\bu_r^\ast)\bM_r^k\bv_r{\bv_r^\ast}^T}{\sigma_r}+\frac{\bM_r^k\bv_r\bu_r^T\bM_r^k}{\sigma_r}.\label{eqn:lerr1}
\end{align}
Plug~\eqref{eqn:vu1} into~\eqref{eqn:uv1} to cancel the $\bM_r^K\bv_r$ term:
\begin{align}
	&\sigma_r^2\bu_r-\bM_r^k{(\bM_r^k)}^T\bu_r=\sigma_r\sigma_r^\ast\bu_r^\ast({\bv_r^\ast}^T\bv_r)+\sigma_r^\ast(\bu_r^T\bu_r^\ast)\bM_r^k\bv_r^\ast.\label{eqn:u1}
\end{align}
We first derive a tight approximation of $\sigma_r$ by $\sigma_r^\ast$; to this end, we first show
\begin{lemma}
    \[
		\forall r, t\in\bar{\mathcal{I}}_i, \|\bP^{i, k}_{r, t}\|_2\leq \sqrt{\frac{\sigma_r^\ast}{\sigma_t^\ast}}\sqrt[4]{\frac{m_{t}n_t}{m_{r}n_r}}
    \]\label{lemm:p2}
\end{lemma}

\begin{proof}
	Let $\bW_i={\left[\sum_{l\in\bar{\mathcal{I}}_i}\frac{\bR_{l, i}^T\bR_{l, i}}{\zeta_l^k}\right]}^{-\frac{1}{2}}$. For $r=t$, since for $\bA, \bB$ positive semidefinite, $\|\bA+\bB\|_2\geq\|\bA\|_2$, we have
    \[
		\left\|\bW_i\frac{\bR_{r, i}^T\bR_{r, i}}{\zeta_r^k}\bW_i\right\|_2\leq\left\|\bW_i\left[\sum_{l\in\bar{\mathcal{I}}_i}\frac{\bR_{l,i}^T\bR_{l,i}}{\zeta_l^k}\right]\bW_i\right\|_2=1=\sqrt{\frac{\sigma_r^\ast}{\sigma_t^\ast}}\sqrt[4]{\frac{m_{r}n_r}{m_{r}n_r}};
    \]
	for $r\neq t$, by the Arithmetic-Geometric Mean Inequality $\|\bA^T\bB\|_2\leq\frac{1}{2}\|\bA\bA^T+\bB\bB^T\|_2$~\cite{bhatia_matrix_1997}, taking $\bA=\frac{\bW_i\bR_{r, i}^T}{\sqrt{\zeta_r^k}}$ and $\bB=\frac{\bW_i\bR_{t, i}^T}{\sqrt{\zeta_t^k}}$ gives (notice $r, t\in\bar{\mathcal{I}}_i$)
    \begin{align*}
        \left\|\frac{\bR_{r, i}}{\sqrt{\zeta_r^k}}\bW_i^2\frac{\bR^T_{t, i}}{\sqrt{\zeta_t^k}}\right\|_2&\leq\frac{1}{2}\left\|\bW_i\left(\frac{\bR^T_{r, i}\bR_{r, i}}{\zeta_r^k}+\frac{\bR_{t, i}^T\bR_{t, i}}{\zeta_t^k}\right)\bW_i\right\|_2\\
                                                                                                              &\leq\frac{1}{2}\left\|\bW_i\left(\sum_{l\in\bar{\mathcal{I}}_i}\frac{\bR_{l,i}^T\bR_{l,i}}{\zeta_l^k}\right)\bW_i\right\|_2=\frac{1}{2},
    \end{align*}
    and since $\|\bA\bA^T\|_2=\|\bA^T\bA\|_2$, we obtain
    \[
		\|\bP^{i, k}_{r, t}\|_2=\left\|\frac{\bR_{r, i}}{\sqrt{\zeta_r^k}}\bW_i^2\frac{\bR^T_{t, i}}{\sqrt{\zeta_t^k}}\right\|_2\sqrt{\frac{\zeta_r^k}{\zeta_t^k}}\leq\frac{1}{2}\sqrt{\frac{2\sigma_r^\ast}{\sigma_t^\ast}}\sqrt[4]{\frac{m_{t}n_t}{m_{r}n_r}}\leq\sqrt{\frac{\sigma_r^\ast}{\sigma_t^\ast}}\sqrt[4]{\frac{m_{t}n_t}{m_{r}n_r}}.
    \]
\end{proof}

\noindent From Lemma~\ref{lemm:p2}, $\sum_{t\in\bar{\cI}_i}\|\bP_{r, t}^{i, k}\|_2\frac{\sigma_t^\ast}{\sqrt{n_t}}\leq\omega_r\frac{\sigma_r^\ast}{\sqrt{n_r}}$ and $\sum_{t\in\bar{\cI}_i}\|\bP_{r, t}^{i, k}\|_2\alpha_t\frac{\sigma_t^\ast}{\sqrt{n_t}}\leq\tilde{\alpha_r}\frac{\sigma_r^\ast}{\sqrt{n_r}}$; furthermore,
\begin{lemma}
    \[
        \|\bQ_{r,i}^T\bE_r^k\be_l\|_{\ell^2}\leq\alpha_r\nu\sqrt{d}s_0\frac{\sigma_r^\ast q^k}{\sqrt{n_r}},\forall i\in\mathcal{I}_r, l = 1, 2, \dots, n_r.
    \]\label{prop:qle21}
\end{lemma}

\begin{proof}
    Since $\bE_r^k$ has at most $\alpha_r$ fraction of non-zeros in each column, and using Cauchy-Schwartz inequality,
\begin{align*}
	\|\bQ_{r,i}^T\bE_r^k\be_l\|_{\ell^2}^2&=\sum_{j_1=1}^{m_r}\sum_{j_2=1}^{m_r}\be_l^T{(\bE_r^k)}^T\be_{j_1}\be_{j_1}^T\bQ_{r,i}\bQ_{r,i}^T\be_{j_2}\be_{j_2}^T\bE_r^k\be_l\\
										  &\leq\sum_{j_1,j_2}|{(\bE_r^k)}_{j_1, l}||{(\bE_r^k)}_{j_2, l}|{\left(\max_{1\leq p\leq m_r}\|\bQ_{r,i}^T\be_p\|_{\ell^2}\right)}^2\\
										  &\leq{\left(\nu\alpha_r\sqrt{m_{r}d}s_{r,k}\right)}^2\leq{\left(\alpha_r\nu\sqrt{d}s_0\frac{\sigma_r^\ast q^k}{\sqrt{n_r}}\right)}^2.
\end{align*}
\end{proof}
\noindent Combining Lemma~\ref{lemm:p2} and~\ref{prop:qle21}, we bound $\|\bF_r^k\|_F$ as:
\begin{align*}
	\|\bF_r^k\|_F&\leq\sqrt{n_r}\max_i\|\bF_r^k\be_i\|_{\ell^2}\leq\sqrt{n_r}\sum_{t\in\cU_r}\|\bP^{i, k}_{r, t}\|_2\alpha_t\nu\sqrt{d}s_0\frac{\sigma_t^\ast q^k}{\sqrt{n_t}}\\
    &\leq\tilde{\alpha_r}\nu\sqrt{d}s_0\sigma_r^\ast q^k\leq\tilde{\alpha_r}\nu\sqrt{d}s_0\sigma_r^\ast,
\end{align*}
similarly we have $\|\bG_r^k\|_F\leq l_0\omega_r\sigma_r^\ast$; finally, note that
\begin{align*}
	\|\bE^k_r\|_2\leq\sqrt{\|\bE_r^k\|_1\|\bE_r^k\|_\infty}\leq\sqrt{\alpha_r\alpha_r^\prime m_{r}n_r}s_{r, k}\leq\sqrt{\alpha_r\alpha_r^\prime}s_0\sigma_r^\ast,
\end{align*}
\begin{align}
    |\sigma_r-\sigma_r^\ast|&\leq\|\bM_r^k\|_2\leq\|\bE^k_r\|_2+\|\bH_r^k\|_F\leq\sigma_r^\ast\left[l_0\omega_r+s_0(\tilde{\alpha_r}\nu\sqrt{d}+\sqrt{\alpha_r\alpha_r^\prime})\right]\leq\frac{\sigma_r^\ast}{8},\label{eqn:eh21}
\end{align}
whenever $\tilde{\alpha_r}\leq\frac{1}{32s_0\nu\sqrt{d}}$ and $\alpha_r^\prime\leq\frac{1}{32s_0}$. Thus
\begin{equation}
    \frac{8}{9}\leq\frac{\sigma_r^\ast}{\sigma_r}\leq\frac{8}{7}.\label{eqn:lambda1}
\end{equation}
We next prove that $\bu_r$ and $\bv_r$ are also incoherent vectors. Indeed, from~(\ref{eqn:u1}), taking $\|\cdot\|_{\ell^\infty}$ on both sides yields
\begin{align}
	&\sigma_r^2\|\bu_r\|_{\ell^\infty}-\|\bM_r^k{(\bM_r^k)}^T\bu_r\|_{\ell^\infty}\leq\sigma_r\sigma_r^\ast\|\bu_r^\ast\|_{\ell^\infty}+\sigma_r^\ast\|\bM_r^k\bv_r^\ast\|_{\ell^\infty}\label{eqn:uincor1}.
\end{align}
To bound its individual terms, we show that
\begin{lemma}
    \[
		\|{(\bM_r^k)}^T\bu_r\|_{\ell^\infty}\leq\left[l_0\nu\sqrt{d}\omega_r+s_0(1+\nu^2d)\tilde{\alpha_r}\right]\sigma_r^\ast q^k\sqrt{\frac{m_r}{n_r}}{\|\bu_r\|_{l^\infty}}.
    \]\label{lemm:ehtinf1}
\end{lemma}
\begin{proof}
For $i\in\cI_r$, using Cauchy-Schwartz inequality,
\begin{align*}
    \|\bQ_{r,i}^T\bu_r\|_{\ell^2}^2&=\sum_{j_1=1}^{m_r}\sum_{j_2=1}^{m_r}\bu_r^T\be_{j_1}\be_{j_1}^T\bQ_{r,i}\bQ_{r,i}^T\be_{j_2}\be_{j_2}^T\bu_r\\
								   &\leq m_r^2{\left(\nu\sqrt{\frac{d}{m_r}}\right)}^2\|\bu_r\|_{\ell^\infty}^2={(\nu\sqrt{m_{r}d}\|\bu_r\|_{\ell^\infty})}^2.
\end{align*}
Lemma~\ref{lemm:ehtinf1} comes from adding the following inequalities:
\begin{align*}
	\left|\be_{f_r^i}^T{(\bG_r^k)}^T\bu_r\right|&=\left|\sum_{t\in\bar{\cI}_i}\be_{f_t^i}^T{(\bL_t^k-\bL^\ast_t)}^T\bQ_{t,i}{\bP^{i, k}_{r,t}}^T\bQ_{r,i}^T\bu_r\right|\leq\sum_{t\in\cU_r}l_{t, k}\left\|{\bP^{i, k}_{r, t}}^T\right\|_2\left\|\bQ^T_{r, i}\bu_r\right\|_{\ell^2}\leq l_0\nu\sqrt{d}\omega_r\sigma_r^\ast q^k\sqrt{\frac{m_r}{n_r}}\|\bu_r\|_{\ell^\infty},\\
	\left|\be_{f_r^i}^T{(\bF_r^k)}^T\bu_r\right|&\leq\sum_{t\in\cU_r}\alpha_t\nu\sqrt{d}s_0\frac{\sigma_t^\ast q^k}{\sqrt{n_t}}\left\|{\bP^{i, k}_{r,t}}^T\right\|_2\left\|\bQ_{r,i}^T\bu_r\right\|_{\ell^2}\leq s_0\nu^2d\tilde{\alpha_r}\sigma_r^\ast q^k\sqrt{\frac{m_r}{n_r}}\|\bu_r\|_{\ell^\infty},\\
	\left\|{(\bE_r^k)}^T\bu_r\right\|_{\ell^\infty}&\leq\left\|\bE_r^k\right\|_1\|\bu_r\|_{\ell^\infty}\leq s_0\tilde{\alpha_r}\sigma_r^\ast q^k\sqrt{\frac{m_r}{n_r}}\|\bu_r\|_{\ell^\infty}.
\end{align*}
\end{proof}
\noindent Notice $q^k\leq 1$; using Lemma~\ref{lemm:ehtinf1} with $\tilde{\alpha_r}\leq\frac{1}{8s_0(1+\nu^2d)}$:
\begin{equation}
	\left\|{(\bM_r^k)}^T\bu_r\right\|_{\ell^\infty}\leq\left[l_0\nu\sqrt{d}\omega_r+s_0(1+\nu^2d)\tilde{\alpha_r}\right]\sigma_r^\ast\sqrt{\frac{m_r}{n_r}}\|\bu_r\|_{\ell^\infty}\leq\frac{1}{4}\sqrt{\frac{m_r}{n_r}}\sigma_r^\ast\|\bu_r\|_{\ell^\infty},\label{eqn:ehehtu1}
\end{equation}
\begin{lemma}
    $\forall\bv\in\mathbb{R}^{n_r}$,
    \[
        \left\|\bM_r^k\bv\right\|_{\ell^\infty}\leq\left[s_0(\alpha_r^\prime+\tilde{\alpha_r}\nu^2d)+l_0\nu\sqrt{d}\omega_r\right]\sigma_r^\ast q^k\sqrt{\frac{n_r}{m_r}}\|\bv\|_{\ell^\infty}.
    \]\label{lemm:ehinf1}
\end{lemma}
\begin{proof}
    For $i=1, 2, \dots, m_r$, the following three inequalities give the desired result when added together,
    \begin{align*}
		|\be_i^T\bG_r^k\bv|&\leq\sum_{j\in\cI_r}\sum_{t\in\bar{\cI}_j}\left\|\be_i^T\bQ_{r,j}\right\|_{\ell^2}\left\|\bP^{i, k}_{r, t}\right\|_2l_{t, k}\left|{(\bv)}_{f_r^j}\right|\leq l_0\omega_r\nu\sqrt{d}\sigma_r^\ast q^k\sqrt{\frac{n_r}{m_r}}\|\bv\|_{\ell^\infty},\\
        \left\|\bE_r^k\bv\right\|_{\ell^\infty}&\leq\left\|\bE_r^k\right\|_\infty\|\bv\|_{\ell^\infty}\leq s_0\alpha_r^\prime\sigma_r^\ast q^k\sqrt{\frac{n_r}{m_r}}\|\bv\|_{\ell^\infty},\\
		\left|\be_i^T\bF_r^k\bv\right|&\leq \sum_j\sum_t\left\|\be_i^T\bQ_{r,j}\right\|_{\ell^2}\left\|\bP^{j, k}_{r, t}\right\|_2\left(\alpha_t\nu\sqrt{d}s_0\frac{\sigma_t^\ast q^k}{\sqrt{n_t}}\right)\|\bv\|_{\ell^\infty}\leq s_0\nu^2d\tilde{\alpha_r}\sigma_r^\ast q^k\sqrt{\frac{n_r}{m_r}}\|\bv\|_{\ell^\infty}.
    \end{align*}
\end{proof}
\noindent Invoking Lemma~\ref{lemm:ehtinf1},~\ref{lemm:ehinf1} with $\tilde{\alpha_r}\leq\frac{1}{16s_0\nu^2d},\alpha_r^\prime\leq\frac{1}{8s_0}$:
\begin{align}
	\left\|\bM_r^k{(\bM_r^k)}^T\bu_r\right\|_{\ell^\infty}&\leq\frac{\sigma_r^\ast}{4}\sqrt{\frac{n_r}{m_r}}\left\|{(\bM_r^k)}^T\bu_r\right\|_{\ell^\infty}\leq\frac{{\sigma_r^\ast}^2}{16}\|\bu_r\|_{\ell^\infty}\nonumber,\\
    \left\|\bM_r^k\bv_r^\ast\right\|_{\ell^\infty}&\leq\frac{\sigma_r^\ast}{4}\sqrt{\frac{n_r}{m_r}}\|\bv_r^\ast\|_{\ell^\infty}\leq\frac{\mu\sigma_r^\ast}{4\sqrt{m_r}}.\label{eqn:ehv1}
\end{align}
Plugging~(\ref{eqn:ehehtu1}) and~(\ref{eqn:ehv1}) into~(\ref{eqn:uincor1}) to obtain (similarly for $\bv_r$)
\begin{equation}
    \|\bu_r\|_{\ell^\infty}\leq \frac{2\mu}{\sqrt{m_r}}, \|\bv_r\|_{\ell^\infty}\leq \frac{2\mu}{\sqrt{n_r}}.\label{eqn:u1incor1}
\end{equation}
Now using Lemma~\ref{lemm:ehtinf1},~\ref{lemm:ehinf1} again, together with~(\ref{eqn:u1incor1}):
\begin{align}
	\left\|{(\bM_r^k)}^T\bu_r\right\|_{\ell^\infty}&\leq\left[\frac{1}{8}+2\tilde{\alpha_r}\mu s_0(1+\nu^2d)\right]\frac{\sigma_r^\ast q^k}{\sqrt{n_r}},\label{eqn:ehu1inf1}\\
    \left\|\bM_r^k\bv_r\right\|_{\ell^\infty}&\leq\left[\frac{1}{8}+2s_0\mu(\alpha_r^\prime+\tilde{\alpha_r}\nu^2d)\right]\frac{\sigma_r^\ast q^k}{\sqrt{m_r}}.\label{eqn:ehv1inf1}
\end{align}
We then proceed to bound $l_{r, k+1}$ in subsequent steps:
\begin{lemma}
    For $i\in\mathcal{I}_r$, $n_r\geq 5\omega_r\mu^2$, $0<\delta\leq\frac{1}{20\omega_r\mu^2}$,
    \[
        \left\|\bQ_{r,i}^T\bM_r^k\bv_r\right\|_{\ell^2}\leq\left(\frac{l_0}{2\mu}+\frac{5s_0\tilde{\alpha_r}\mu\nu\sqrt{d}}{2}\right)\sigma_r^\ast q^k.
    \]\label{lemm:ql21}
\end{lemma}
\begin{proof}
    Under our choice of $\delta$, $1+(n_r-1)\delta\leq\frac{n_r}{4\omega_r\mu^2}$; thus
\begin{align}
	\left\|\bQ_{r,i}^T\bG_r^k\bv_r\right\|_{\ell^2}&\leq\sum_{j\in\cI_r}\sum_{t\in\bar{\cI}_j}\left\|\bQ_{r, i}^T\bQ_{r, j}\right\|_2\left\|{\bP^{j, k}_{r, t}}^T\right\|_2l_{t, k}\|\bv_r\|_{\ell^\infty}\nonumber\\
												   &\leq[1+(n_r-1)\delta]\omega_{r}l_0\frac{\sigma_r^\ast q^k}{\sqrt{n_r}}\|\bv_r\|_{\ell^\infty}\leq\frac{l_0}{4\mu^2}\sigma_r^\ast q^k\sqrt{n_r}\|\bv_r\|_{\ell^\infty}\leq\frac{l_0}{2\mu}\sigma_r^\ast q^k,\label{eqn:qgv}
\end{align}
and Lemma~\ref{lemm:ql21} comes from adding~(\ref{eqn:qgv}) with the following two inequalities. Note we invoke Lemma~\ref{prop:qle21} in~\eqref{eqn:qfv}:
\begin{align}
    \|\bQ_{r,i}^T\bF_r^k\bv_r\|_{\ell^2}&\leq[1+(n_r-1)\delta]\left(\tilde{\alpha_r}\nu\sqrt{d}s_0\frac{\sigma_r^\ast q^k}{\sqrt{n_r}}\right)\|\bv_r\|_{\ell^\infty}\leq\frac{s_0}{4}\tilde{\alpha_r}\nu\sqrt{d}\sigma_r^\ast q^k\sqrt{n_r}\|\bv_r\|_{\ell^\infty}\leq\frac{s_0}{2}\tilde{\alpha_r}\mu\nu\sqrt{d}\sigma_r^\ast q^k,\label{eqn:qfv}
\end{align}
\begin{align*}
	\|\bQ^T_{r,i}\bE_r^k\bv_r\|_{\ell^2}^2&=\bv_r^T{(\bE_r^k)}^T\bQ_{r,i}\bQ_{r,i}^T\bE_r^k\bv_r=\sum_{j_1=1}^{m_r}\sum_{j_2=1}^{m_r}\bv_r^T{(\bE_r^k)}^T\be_{j_1}\be_{j_1}^T\bQ_{r,i}\bQ_{r,i}^T\be_{j_2}\be_{j_2}^T\bE_r^k\bv_r\\
										  &\leq{\left(\nu\sqrt{\frac{d}{m_r}}\right)}^2{\left(\sum_{i=1}^{n_r}\sum_{j=1}^{m_r}{(\bE_r^k)}_{ji}{(\bv_r)}_i\right)}^2\leq{\left(\nu\sqrt{\frac{d}{m_r}}\right)}^2{\left(\alpha_r m_{r}n_r\|\bE_r^k\|_{\ell^\infty}\|\bv_r\|_{\ell^\infty}\right)}^2\leq{\left(2s_0\tilde{\alpha_r}\mu\nu\sqrt{d}\sigma_r^\ast q^k\right)}^2.\nonumber
\end{align*}
\end{proof}
\noindent Also, working in the same manner as in~(\ref{eqn:eh21}), we obtain
\begin{equation}
    |\bu_r^T\bM_r^k\bv_r|\leq\left[l_0\omega_r+s_0(\tilde{\alpha_r}\nu\sqrt{d}+\sqrt{\alpha_r\alpha_r^\prime})\right]\sigma_r^\ast q^k.\label{eqn:uehv1}
\end{equation}
$\forall i\in\cI_r$, a combination of~\eqref{eqn:lerr1}~\eqref{eqn:ehu1inf1}~\eqref{eqn:uehv1} and Lemma~\ref{lemm:ql21} yields
\begin{align}
    &\|\bQ_{r,i}^T(\bL_r^{k+1}-\bL_r^\ast)\be_{f_r^i}\|_{\ell^2}\nonumber\\
	\leq&\frac{\sigma_r^\ast|\bu_r^T\bM_r^k\bv_r|}{\sigma_r}\frac{\kappa\mu\sqrt{d}}{\sqrt{m_{r}n_r}}+\frac{\sigma_r^\ast\kappa}{\sigma_r}\sqrt{\frac{d}{m_r}}\|{(\bM_r^k)}^T\bu_r\|_{\ell^\infty}+\frac{\|\bQ_i^T\bM_r^k\bv_r\|_{\ell^2}}{\sigma_r}\left[\frac{\sigma_r^\ast\mu}{\sqrt{n_r}}+\|{(\bM_r^k)}^T\bu_r\|_{\ell^\infty}\right]\nonumber\\
	\leq&\frac{8}{7}\frac{\sigma_r^\ast q^k}{\sqrt{n_r}}\biggl\{\kappa\sqrt{\frac{d}{m_r}}\biggl[\frac{1}{16\nu\sqrt{d}}+\frac{1}{8}+2\mu\tilde{\alpha_r}\left(1+\nu^2d+\frac{\nu\sqrt{d}}{2}\right)s_0+\mu s_0\sqrt{\tilde{\alpha_r}\alpha_r^\prime}\biggr]+\frac{5}{4}\left(\frac{5}{2}\tilde{\alpha_r}\mu^2\nu\sqrt{d}s_0+\frac{l_0}{2}\right)\biggr\},\label{ineq:const1}
\end{align}
where in~(\ref{ineq:const1}) we use~(\ref{eqn:lambda1}) and the following inequality
\begin{align*}
	\frac{\sigma_r^\ast\mu}{\sqrt{n_r}}+\|{(\bM_r^k)}^T\bu_r\|_{\ell^\infty}\leq\frac{\sigma_r^\ast\mu}{\sqrt{n_r}}\left\{1+\left[\frac{1}{8\mu}+2\tilde{\alpha_r}(1+\nu^2d)s_0\right]\right\}\leq\frac{\sigma_r^\ast\mu}{\sqrt{n_r}}\left[1+\left(\frac{1}{8}+\frac{1}{8}\right)\right]\leq\frac{5}{4}\frac{\sigma_r^\ast\mu}{\sqrt{n_r}},
\end{align*}
whenever $\tilde{\alpha_r}\leq\frac{1}{16s_0(1+\nu^2d)}$. Now for some $d\leq\frac{\sqrt{m_r}}{32\mu^2\nu\kappa\omega_r}$, $\tilde{\alpha_r},\alpha_r^\prime$ sufficiently small (in the order of $\mathcal{O}(d^{-1})$,$\mathcal{O}(d^{-\frac{1}{2}})$, respectively), $\exists q<1$ such that
\[
    l_{r,k+1}\leq\frac{1}{16\mu\nu\omega_r\sqrt{d}}\frac{\sigma_r^\ast q^{k+1}}{\sqrt{n_r}}=l_0\frac{\sigma_r^\ast q^{k+1}}{\sqrt{n_r}}.
\]
On the other hand, using~(\ref{eqn:lerr1})~(\ref{eqn:ehu1inf1})~(\ref{eqn:ehv1inf1})~(\ref{eqn:uehv1}),
\begin{align}
    &\|\bL_r^{k+1}-\bL_r^\ast\|_{\ell^\infty}\nonumber\\
	\leq&\frac{\sigma_r^\ast|\bu_r^T\bM_r^k\bv_r|}{\sigma_r}\frac{\mu^2}{\sqrt{m_{r}n_r}}+\frac{\sigma_r^\ast}{\sigma_r}\frac{\mu}{\sqrt{m_r}}\|{(\bM_r^k)}^T\bu_r\|_{\ell^\infty}+\frac{\sigma_r^\ast\|\bM_r^k\bv_r\|_{\ell^\infty}}{\sigma_r}\frac{\mu}{\sqrt{n_r}}+\frac{\|\bM_r^k\bv_r\|_{\ell^\infty}\|{(\bM_r^k)}^T\bu_r\|_{\ell^\infty}}{\sigma_r}\nonumber\\
	\leq&\frac{\sigma_r^\ast q^k\mu^2}{\sqrt{m_{r}n_r}}\frac{\sigma_r^\ast}{\sigma_r}\biggl\{l_0\omega_r+\frac{1}{4\mu}+4\tilde{\alpha_r}\nu^2ds_0+2(\tilde{\alpha_r}+\alpha_r^\prime)s_0+(\tilde{\alpha_r}\nu\sqrt{d}+\sqrt{\tilde{\alpha_r}\alpha_r^\prime})s_0\nonumber\\
    +&\left[\frac{1}{8\mu}+2\tilde{\alpha_r}(1+\nu^2d)s_0\right]\left[\frac{1}{8\mu}+2(\alpha_r^\prime+\tilde{\alpha_r}\nu^2d)s_0\right]\biggr\}\nonumber\\
	\leq&\frac{7}{8}\frac{\sigma_r^\ast q^{k+1}\mu^2}{\sqrt{m_{r}n_r}}\leq\frac{\beta_1}{2}\frac{\sigma_r q^{k+1}}{\sqrt{m_{r}n_r}},\label{eqn:linf1}
\end{align}
again for some $\tilde{\alpha_r},\alpha_r^\prime$ sufficiently small (same orders as~(\ref{ineq:const1})) and some $q<1$. Under the same conditions,
\begin{align*}
	|{(\bH_r^k)}_{ij}|&\leq\sum_{t\in\bar{\cI}_j}\left|\be_i^T\bQ_{r,j}\bP_{r, t}^{j, k}\bQ_{t, j}^T(\bL_t^k-\bL_t^\ast+\bS_t^k-\bS_t^\ast)\be_{f_t(j)}\right|\\
					  &\leq\frac{\sigma_r^\ast q^k}{\sqrt{m_{r}n_r}}\nu\sqrt{d}(l_0\omega_r+s_0\tilde{\alpha_r}\nu\sqrt{d})\leq\frac{\sigma_r q^{k+1}}{\sqrt{m_{r}n_r}}\frac{\beta_1}{2},
\end{align*}
and thus $|{(\bL_r^\ast-\bL_r^{k+1}+\bH_r^k)}_{ij}|\leq\zeta_r^{k+1},\forall(i, j)\not\in\supp(\bS_r^\ast)$. From~(\ref{eqn:supdate1_supp}), ${(\bS_r^{k+1})}_{ij}=0$ and $\supp(\bS_r^{k+1})\subset\supp(\bS_r^\ast)$. Let $\bN_r^k=\bL_r^\ast-\bL_r^{k+1}+\bH_r^k+\bS_r^\ast$, then
\begin{align*}
	|{(\bE_r^{k+1})}_{ij}|=|{(\bS_r^\ast-\bS_r^{k+1})}_{ij}|=
    \begin{cases}
		|{(\bL_r^{k+1}-\bL_r^\ast-\bH_r^k\pm\zeta_r^{k+1})}_{ij}| & \text{ if }|{(\bN_r^k)}_{ij}|>\zeta_r^{k+1},\\
		|{(\bS_r^\ast)}_{ij}| & \text{ if }|{(\bN_r^k)}_{ij}|\leq\zeta_r^{k+1},
    \end{cases}
\end{align*}
in both cases we have $|{(\bE_r^{k+1})}_{ij}|\leq\|\bL_r^{k+1}-\bL_r^\ast\|_{\ell^\infty}+\|\bH_r^k\|_{\ell^\infty}+\zeta_r^{k+1}$ and
\[
	s_{r,k+1}\leq2\zeta_r^{k+1}=\frac{2\beta_1q^{k+1}}{\sqrt{m_{r}n_r}}\sigma_r\leq s_0\frac{q^{k+1}}{\sqrt{m_{r}n_r}}\sigma_r^\ast.
\]
\qedhere

\section{Additional Experimental Results}\label{sec:stitch}
This section includes additional experimental results that cannot be included
in the paper due to space constraints.  \tcr{In Fig.~\ref{fig:rasl_align_skyscraper} we further justify the rank-1 constraint by comparing with RASL over the {\it skyscraper\/} dataset~\cite{chang_shape-preserving_2014}, under the same settings as Section~\ref{subsec:rasl}.} In Fig.~\ref{fig:apartments_aligned} we
show aligned and overlayed images for the {\it apartments\/}
dataset~\cite{gao_constructing_2011}. It can be observed that CPW, SPHP and
APAP attempt to align the images according to the thicket, and thus introduce
significant misalignments on the facet of the apartment building. In contrast,
BRAS reliably and successfully achieves a more sensible alignment.

Fig.~\ref{fig:railtracks_stitched} shows stitched images for the {\it
railtracks\/} dataset~\cite{zaragoza_as-projective-as-possible_2014}. While
AutoStitch, ICE, CPW and SPHP either bend the cranes or distort the railtracks,
APAP and BRAS are both able to deliver high-quality stitched images.

\tcr{Fig.~\ref{fig:skyscraper} to Fig.~\ref{fig:couch} show additional image
stitching results on different image sets. These cover a large variety of
interesting scenarios, including multiple images in Fig.~\ref{fig:skyscraper},
Fig.~\ref{fig:forest}, and Fig.~\ref{fig:rooftops}, and multiple homographies
in Fig.~\ref{fig:temple}, Fig.~\ref{fig:carpark} and Fig.~\ref{fig:couch}, etc.}

As was discussed in Section~\ref{subsec:large_objects}, a
particularly challenging case of image alignment is when there are large moving
objects. Fig.~\ref{fig:catabus_aligned} shows one such image set and the aligned
images by APAP, CPW, and BRAS\@. Feature based methods must largely rely on the
moving object features for alignment and the large motion means that the
accuracy of the alignment is fundamentally limited. 

We observe that, APAP and CPW fail to align the static backgrounds accurately,
because they are heavily influenced by the bus motion. \tcr{In contrast, BRAS
achieves good alignment accuracy overall. We note that BRAS being a pixel based method obtains information from the entire image and therefore can automatically adapt to both foreground and background motion. Feature based methods on the other hand must rely on whatever features are detected, which in cases such as this one can be inadequate for accurate alignment and subsequent stitching.}

\tcr{Finally, we further verify the effectiveness of the rank-1 constraint in BRAS by comparing it with a multi-cell extension of~\eqref{eqn:pairwise}, i.e., by integrating the model discussed in Section~\ref{subsec:multiple} into~\eqref{eqn:pairwise}, in Fig.~\ref{fig:multi_cell}. Correspondingly, a smoothness term similar to~\eqref{eqn:formulate_multiple} is added. Evidently, a simple multi-cell extension (without enforcing the rank-1 constraint) yields unsatisfactory result. Furthermore, we observe that, experimentally the multi-cell extension is not stable and frequently diverges, while BRAS reliably performs accurate alignment.}

\begin{figure*}
	\centering
	\subfloat[RASL~\cite{peng_rasl:_2012}]{%
		\begin{tikzpicture}
			\node {\includegraphics[height=0.56\textheight]{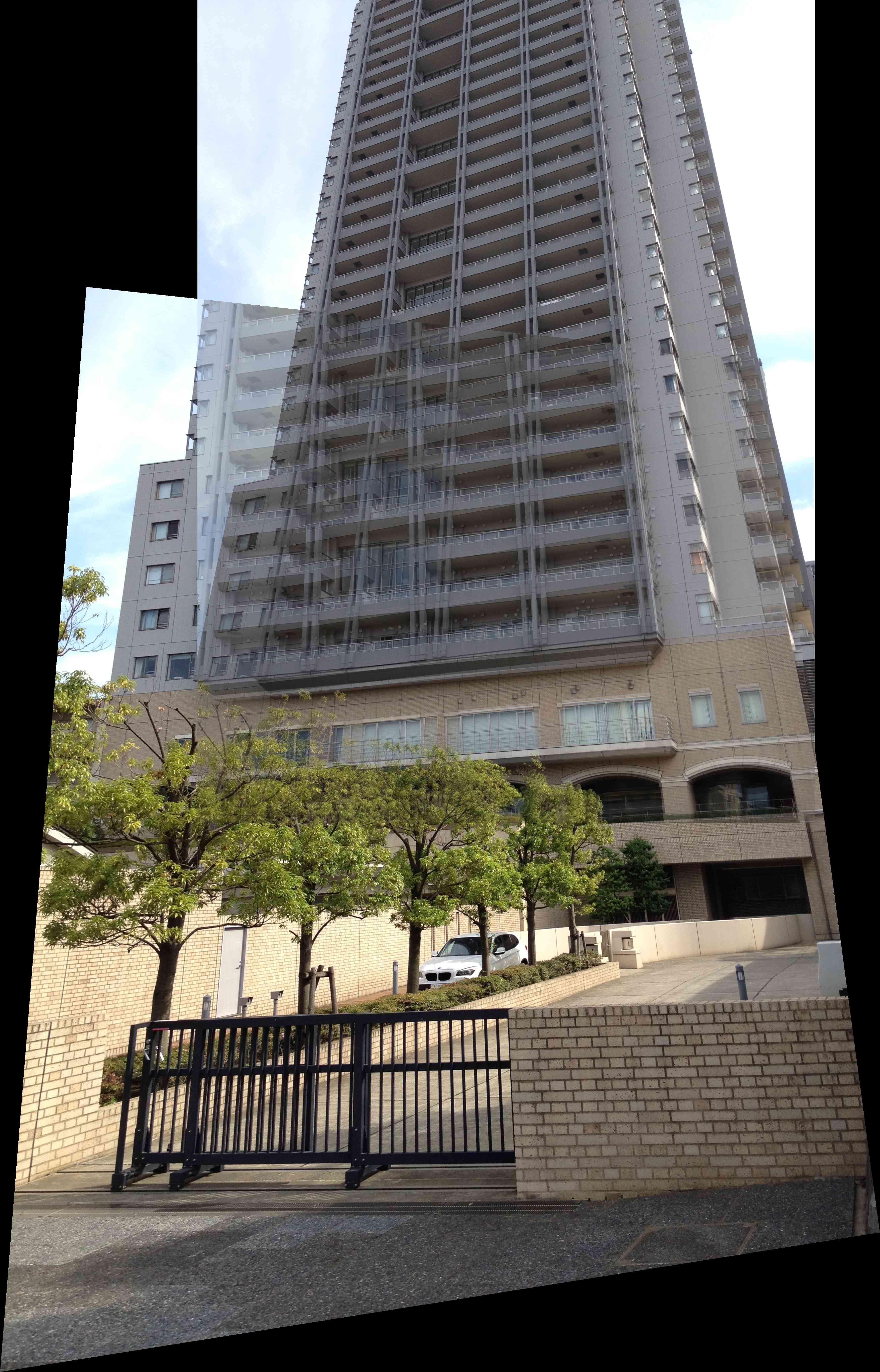}};
			\draw[thick,red] (0.4,1.8) ellipse (2.5cm and 1.5cm);
		\end{tikzpicture}
	}\hspace{-4mm}
	\subfloat[BRAS]{
		\begin{tikzpicture}
			\node {\includegraphics[height=0.56\textheight]{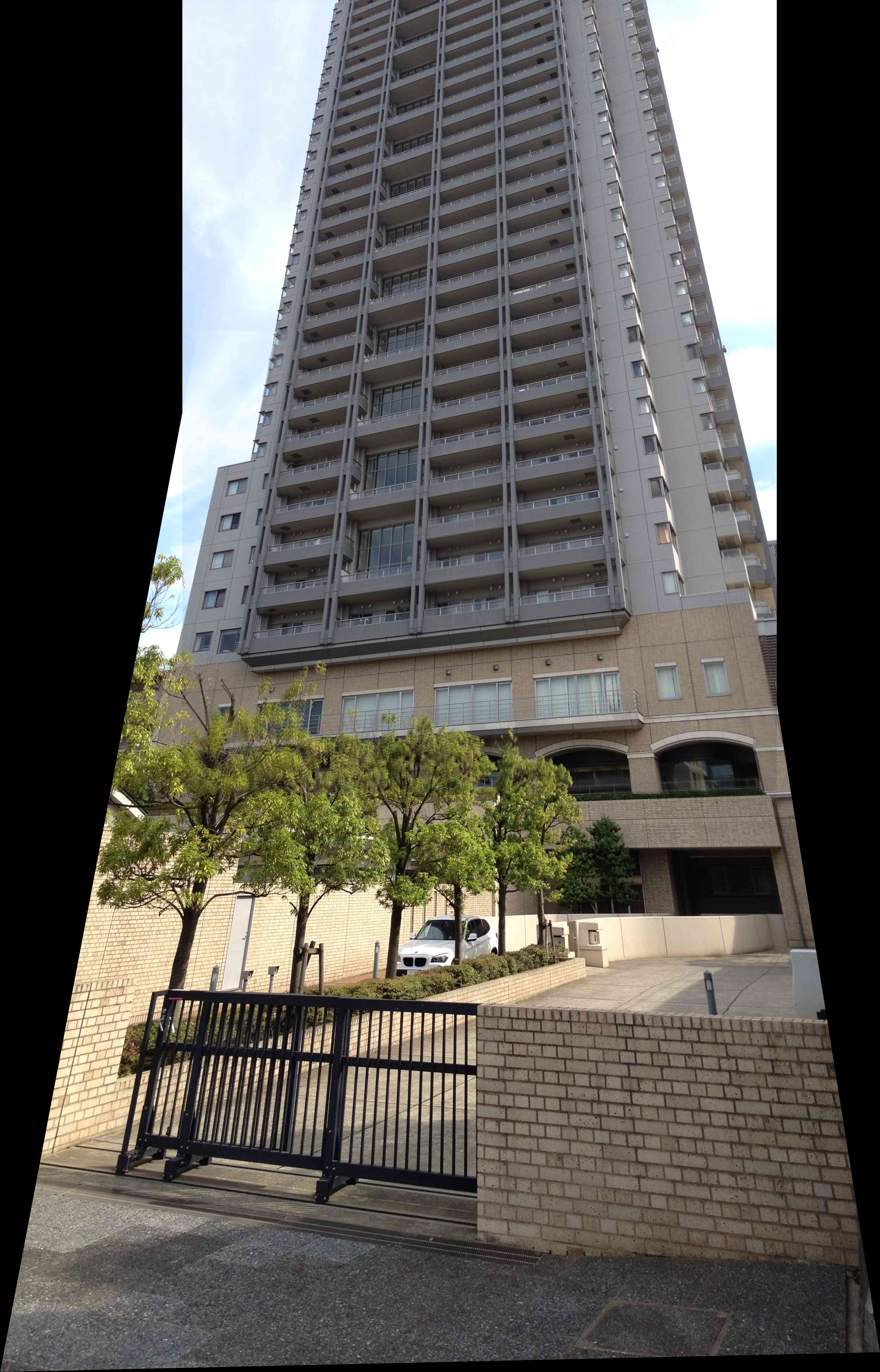}};
		\end{tikzpicture}
	}
	\caption{\tcr{Additional example to verify the effectiveness of the rank-1 constraint: comparison on the \emph{skyscraper} dataset~\cite{chang_shape-preserving_2014} between RASL and BRAS\@. The aligned images are overlayed and misaligned regions are highlighted in red circles. BRAS aligns the images significantly better, as can be seen by the reduction in ghosting artifacts.}}\label{fig:rasl_align_skyscraper}
\end{figure*}

\begin{figure}
	\centering
	\subfloat[CPW~\cite{hu_multi-objective_2015}\label{fig:apartments_CPW_aligned}] {%
		\begin{tikzpicture}[zoomboxarray, zoomboxes below, zoomboxarray rows=1]
		\node [image node] {\includegraphics[height=0.19\textheight]{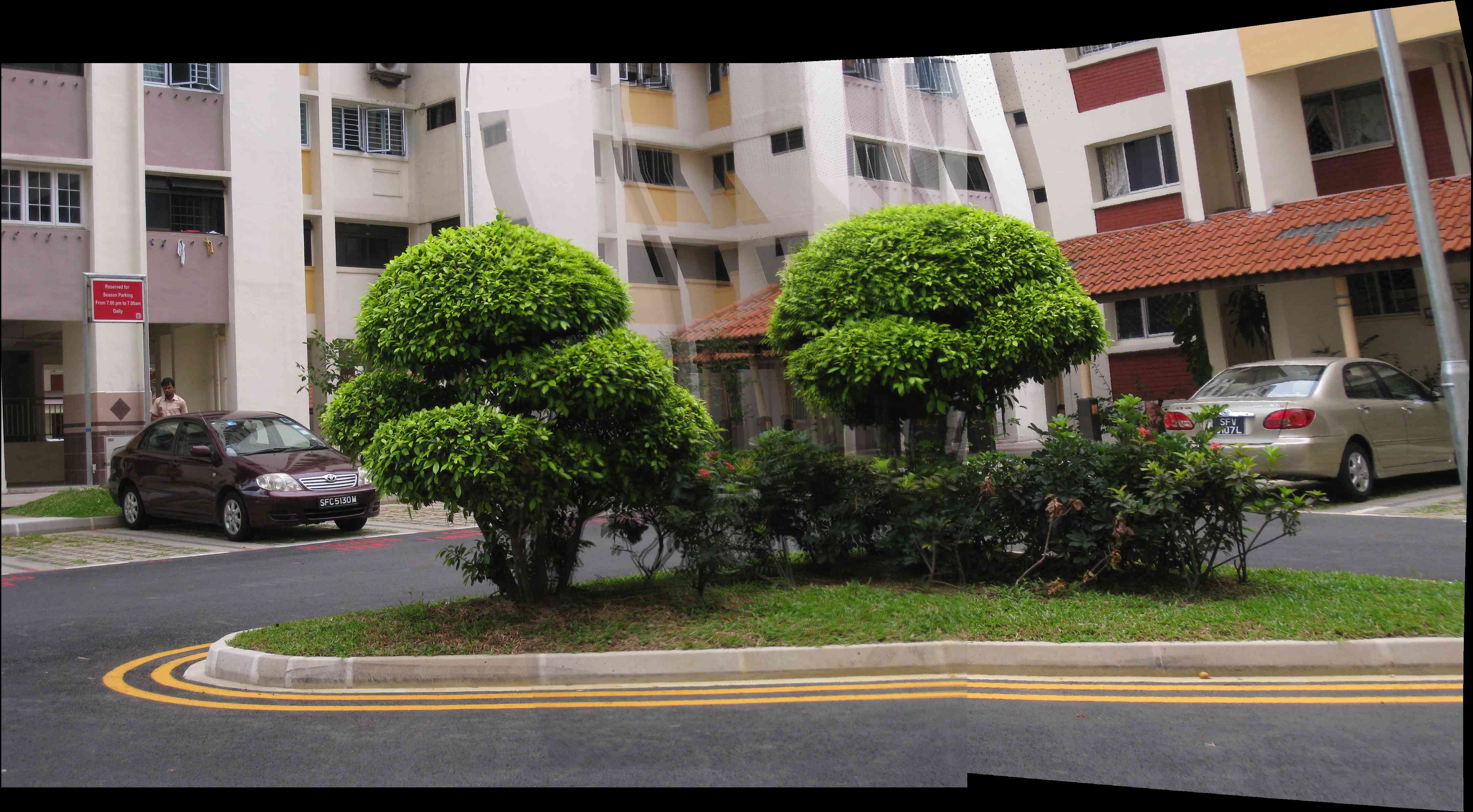}};
		\zoombox[color code=green]{0.47,0.73}
		\zoombox[color code=blue]{0.62,0.73}
		\end{tikzpicture}
	}\,\!
	\subfloat[SPHP~\cite{chang_shape-preserving_2014}\label{fig:apartments_SPHP_aligned}] {%
		\begin{tikzpicture}[zoomboxarray, zoomboxes below, zoomboxarray rows=1]
		\node [image node] {\includegraphics[height=0.19\textheight]{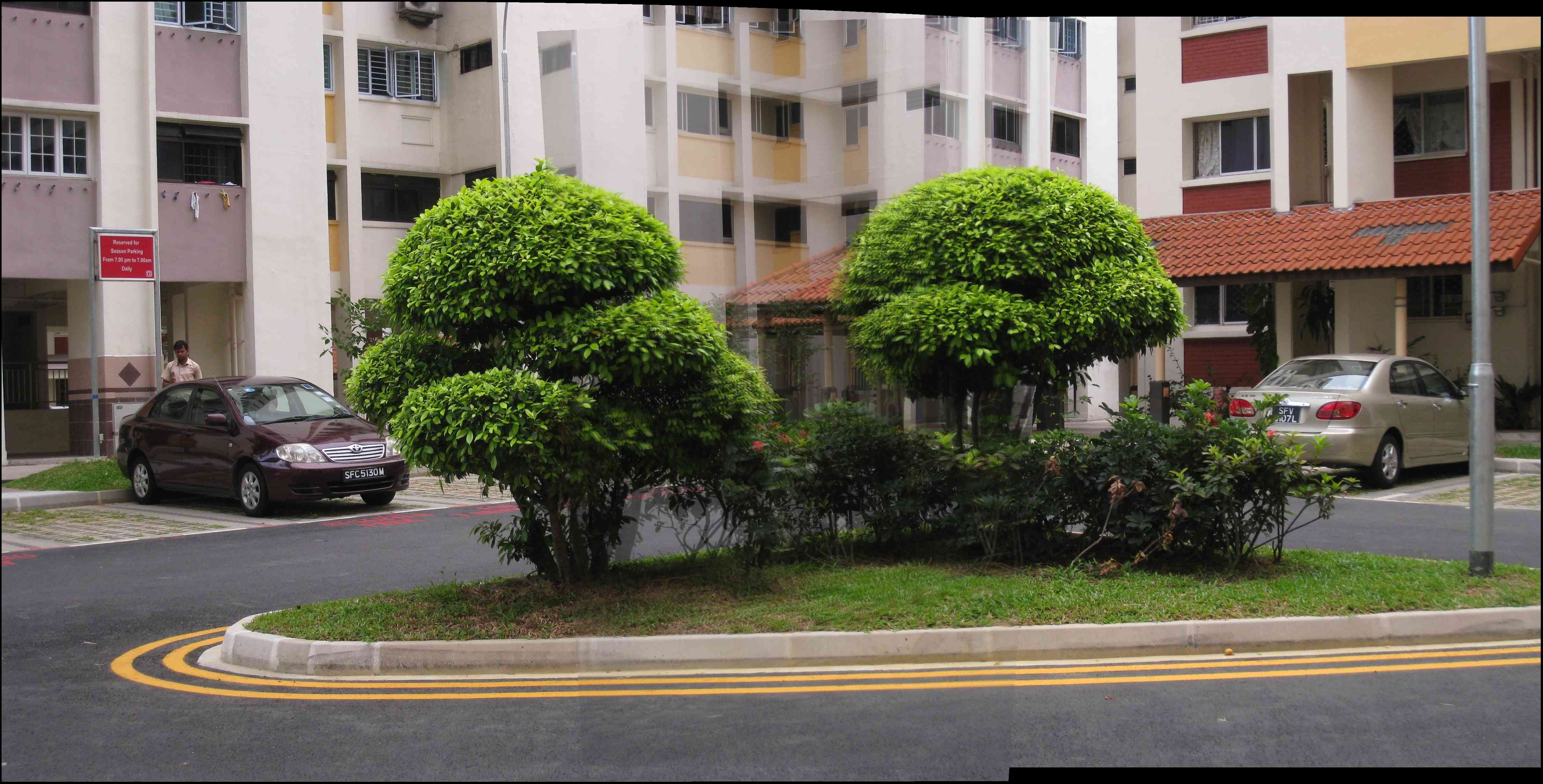}};
		\zoombox[color code=green]{0.5,0.73}
		\zoombox[color code=blue]{0.65,0.77}
		\end{tikzpicture}
	}\\
	\subfloat[APAP~\cite{zaragoza_as-projective-as-possible_2014}\label{fig:apartments_APAP_aligned}] {%
		\begin{tikzpicture}[zoomboxarray, zoomboxes below, zoomboxarray rows=1]
		\node [image node] {\includegraphics[height=0.186\textheight]{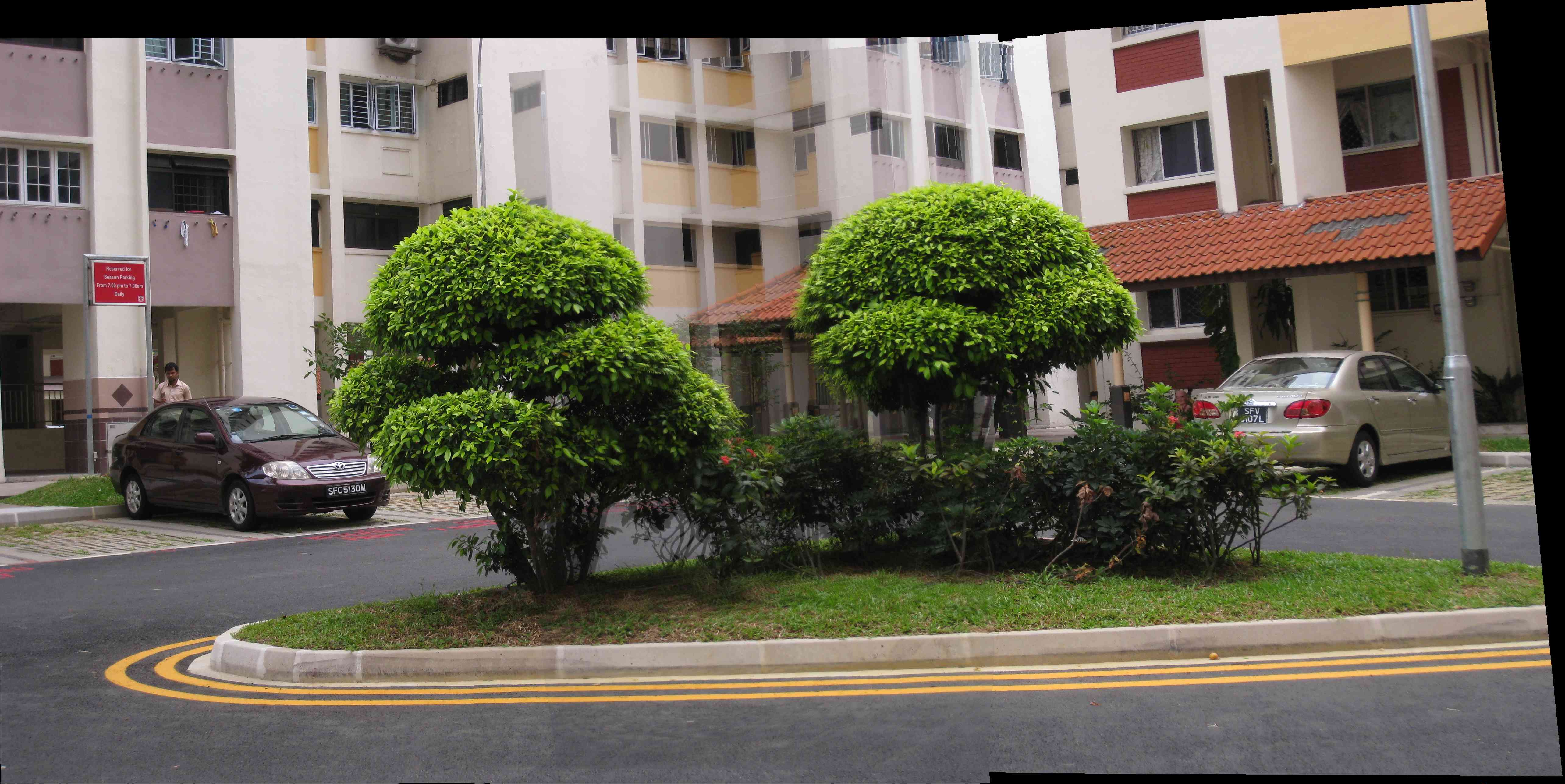}};
		\zoombox[color code=green]{0.5,0.75}
		\zoombox[color code=blue]{0.65,0.79}
		\end{tikzpicture}
	}\,\!
	\subfloat[BRAS\label{fig:apartments_BRAS_aligned}] {%
		\begin{tikzpicture}[zoomboxarray, zoomboxes below, zoomboxarray rows=1]
		\node [image node] {\includegraphics[height=0.186\textheight]{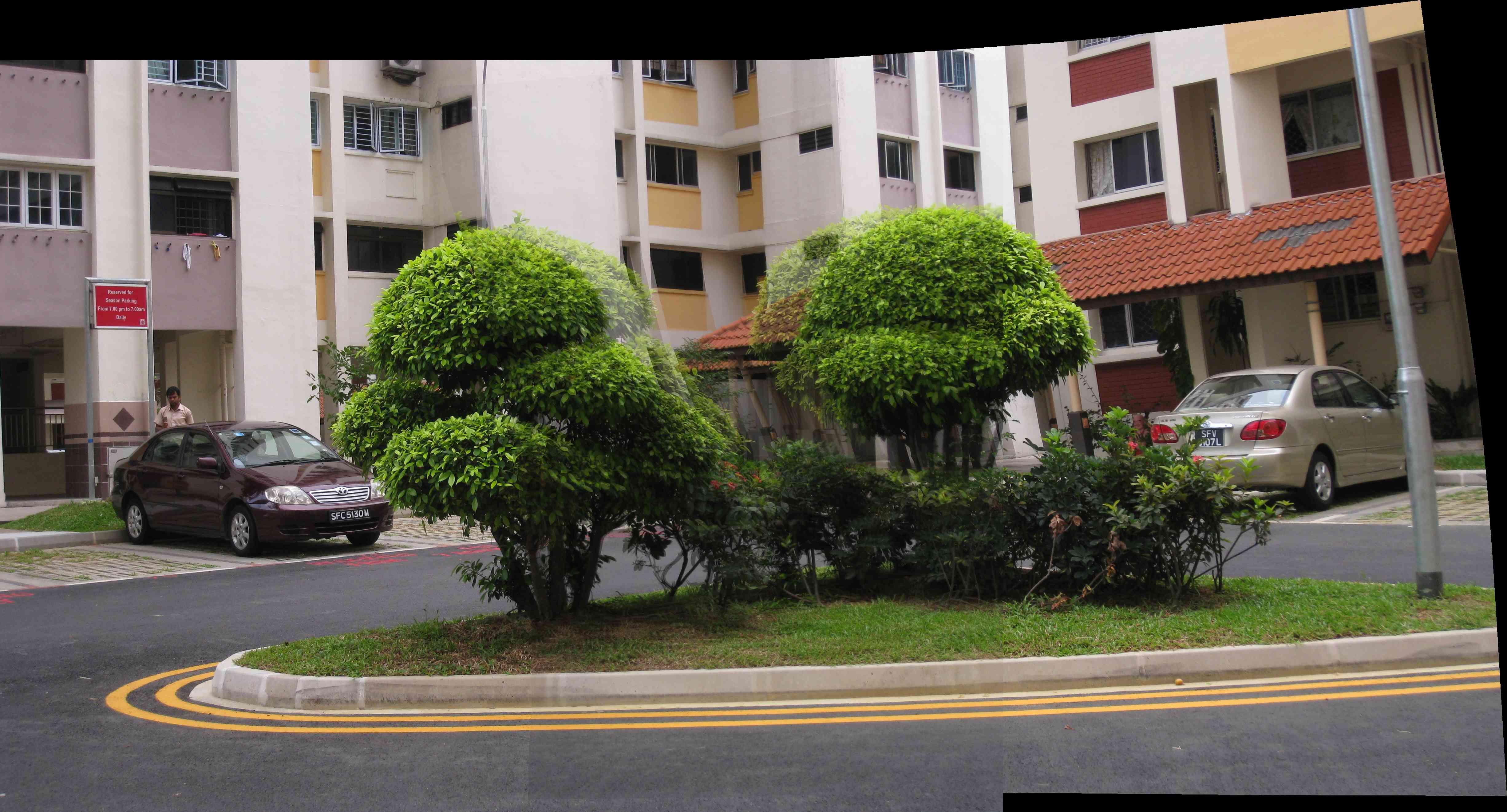}};
		\zoombox[color code=green]{0.45,0.73}
		\zoombox[color code=blue]{0.6,0.77}
		\end{tikzpicture}
	}\\	
	\caption{Aligned and overlayed images on the {\it apartments\/} dataset~\cite{gao_constructing_2011}. The misalignments from \protect\subref{fig:apartments_CPW_aligned} CPW, \protect\subref{fig:apartments_SPHP_aligned} SPHP and \protect\subref{fig:apartments_APAP_aligned} APAP on the facets of the building can be clearly seen in the insets at the bottom.\label{fig:apartments_aligned}}
\end{figure}

\begin{figure}
	\centering
	\subfloat[AutoStitch~\cite{brown_automatic_2007}] {%
		\begin{tikzpicture}[zoomboxarray, zoomboxes below, zoomboxarray rows=1]
		\node [image node] {\includegraphics[height=0.13\textheight]{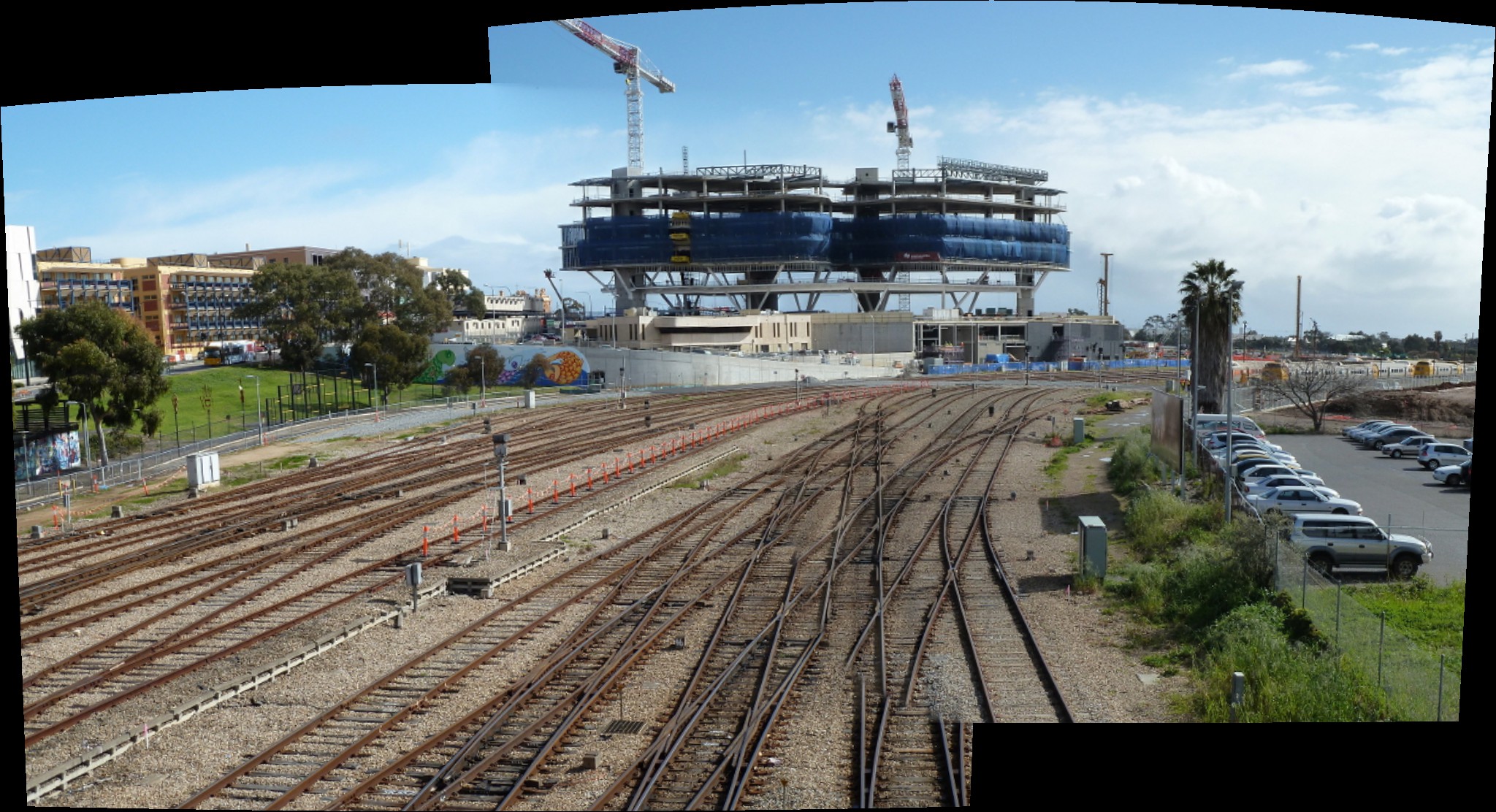}};
		\zoombox[color code=green]{0.4,0.7}
		\zoombox[color code=blue]{0.6,0.2}
		\end{tikzpicture}
	}\,\!
	\subfloat[ICE~\cite{ICE}] {%
		\begin{tikzpicture}[zoomboxarray, zoomboxes below, zoomboxarray rows=1]
		\node [image node] {\includegraphics[height=0.13\textheight]{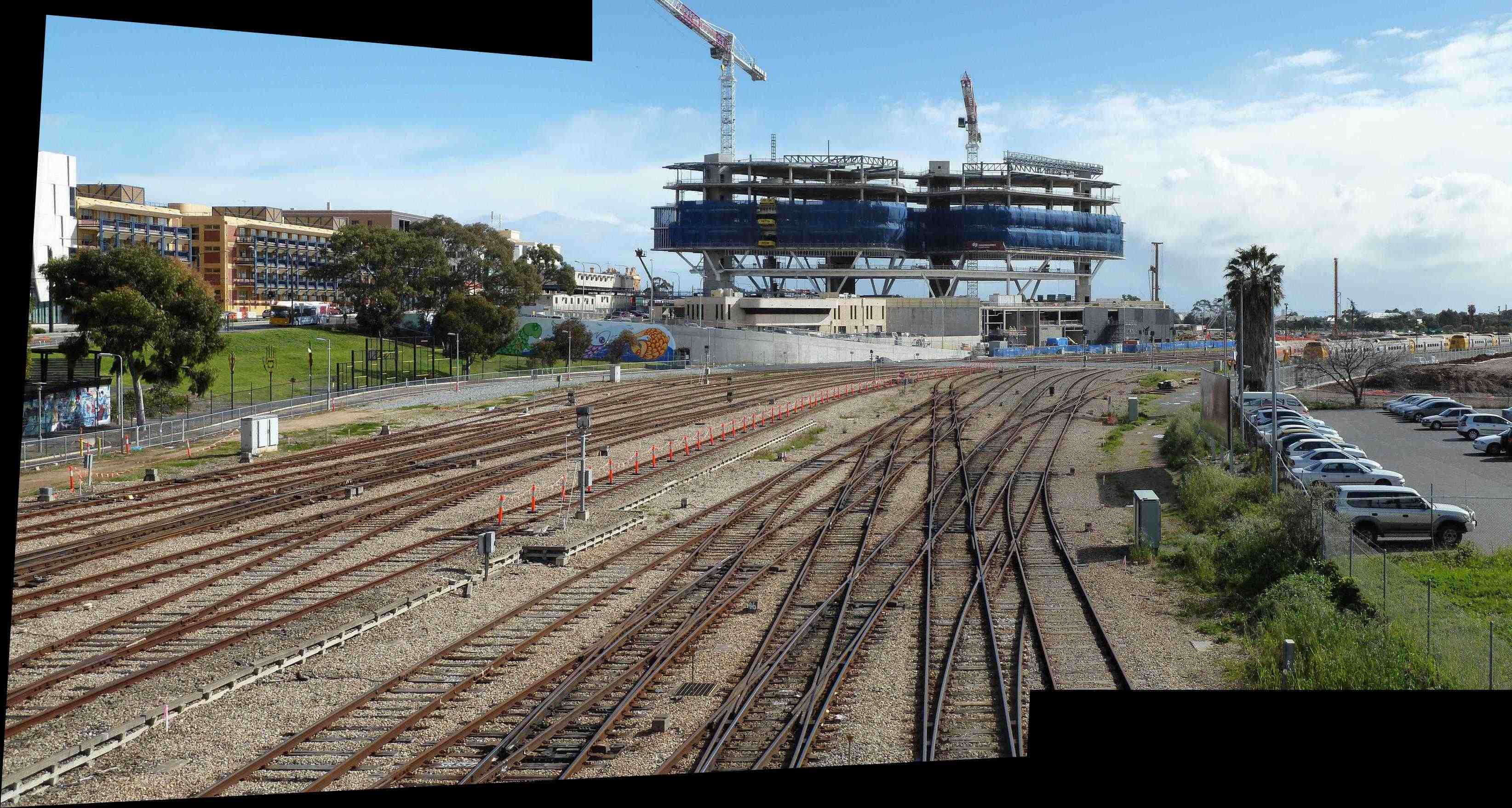}};
		\zoombox[color code=green]{0.46,0.72}
		\zoombox[color code=blue]{0.63,0.235}
		\end{tikzpicture}
	}\,\!
	\subfloat[CPW~\cite{hu_multi-objective_2015}] {%
		\begin{tikzpicture}[zoomboxarray, zoomboxes below, zoomboxarray rows=1]
		\node [image node] {\includegraphics[height=0.13\textheight]{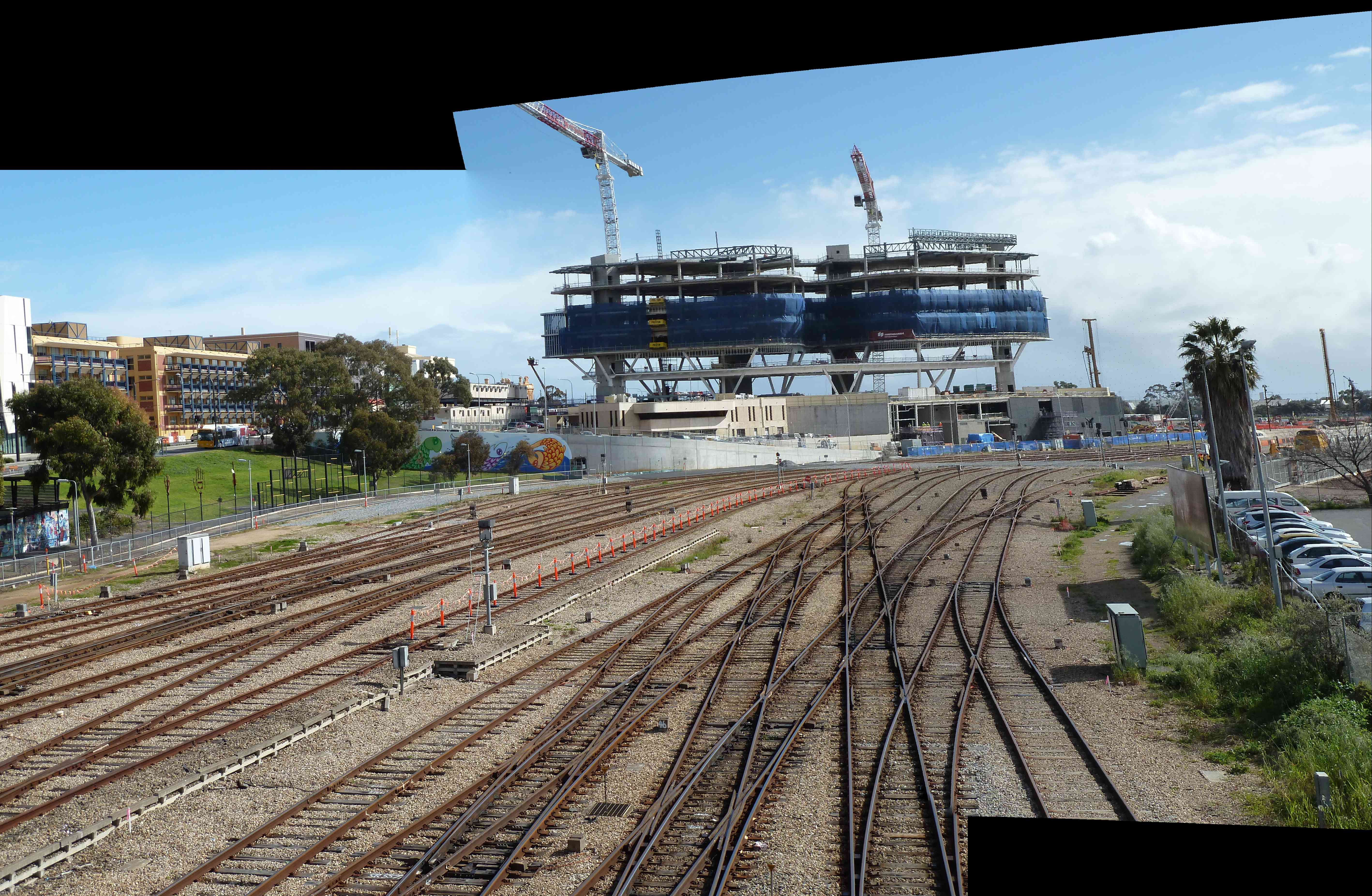}};
		\zoombox[color code=green]{0.43,0.7}
		\zoombox[color code=blue]{0.66,0.2}
		\end{tikzpicture}
	}\\
	
	\subfloat[SPHP~\cite{chang_shape-preserving_2014}] {%
		\begin{tikzpicture}[zoomboxarray, zoomboxes below, zoomboxarray rows=1]
		\node [image node] {\includegraphics[height=0.13\textheight]{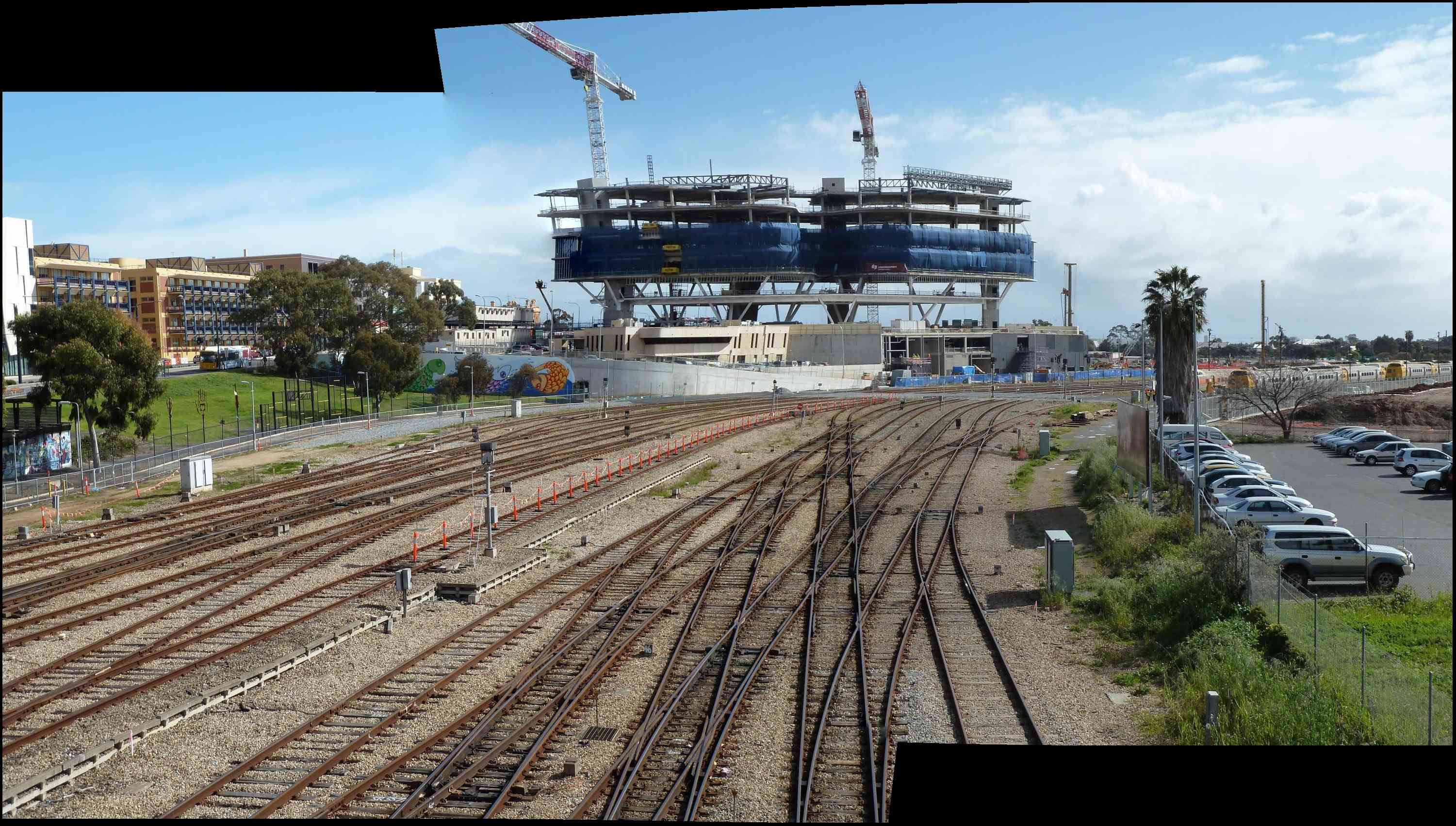}};
		\zoombox[color code=green]{0.4,0.7}
		\zoombox[color code=blue]{0.565,0.19}
		\end{tikzpicture}
	}\,\!
	\subfloat[APAP~\cite{zaragoza_as-projective-as-possible_2014}] {%
		\begin{tikzpicture}[zoomboxarray, zoomboxes below, zoomboxarray rows=1]
		\node [image node] {\includegraphics[height=0.13\textheight]{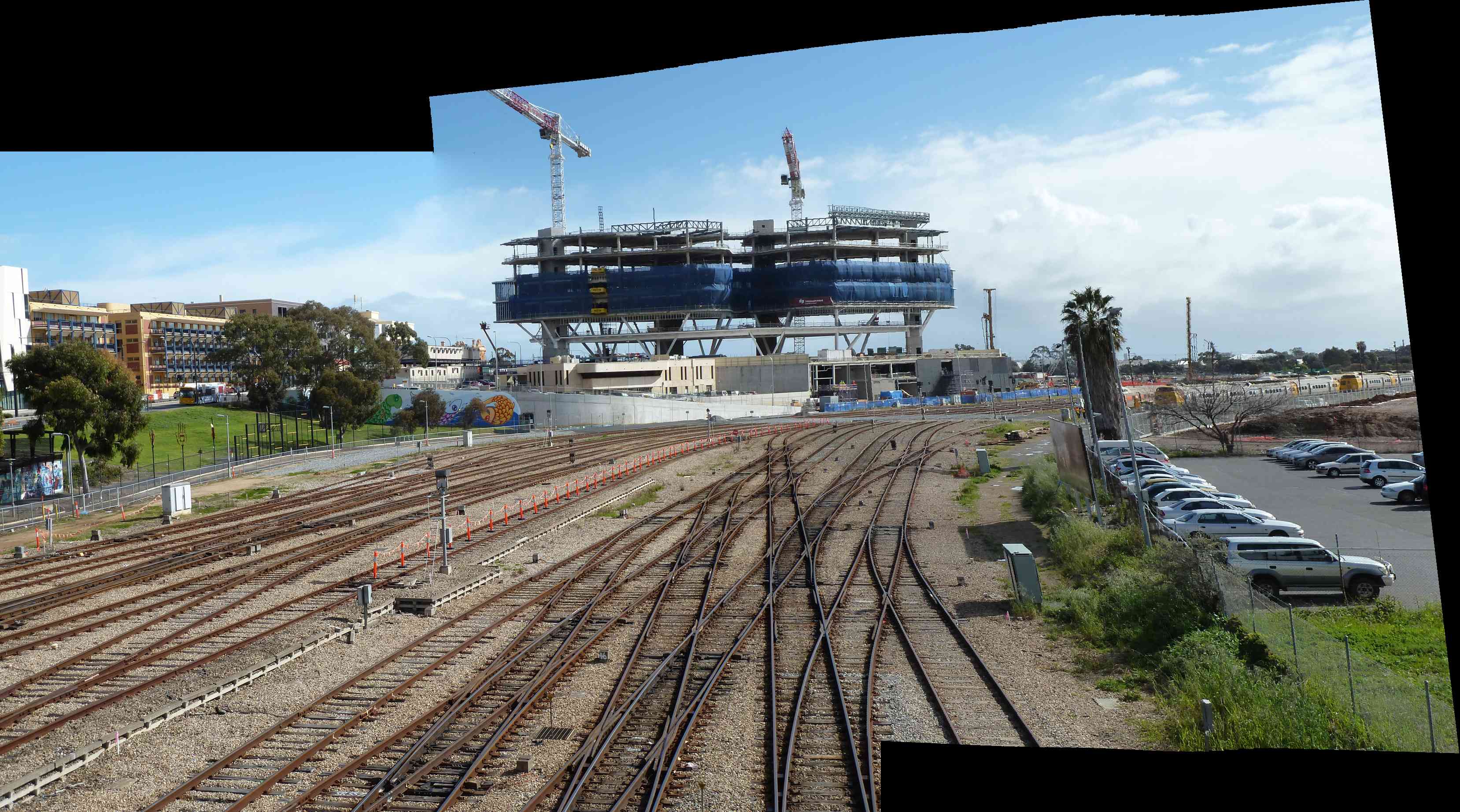}};
		\zoombox[color code=green]{0.357,0.64}
		\zoombox[color code=blue]{0.57,0.18}
		\end{tikzpicture}
	}\,\!
	\subfloat[BRAS] {%
		\begin{tikzpicture}[zoomboxarray, zoomboxes below, zoomboxarray rows=1]
		\node [image node] {\includegraphics[height=0.13\textheight]{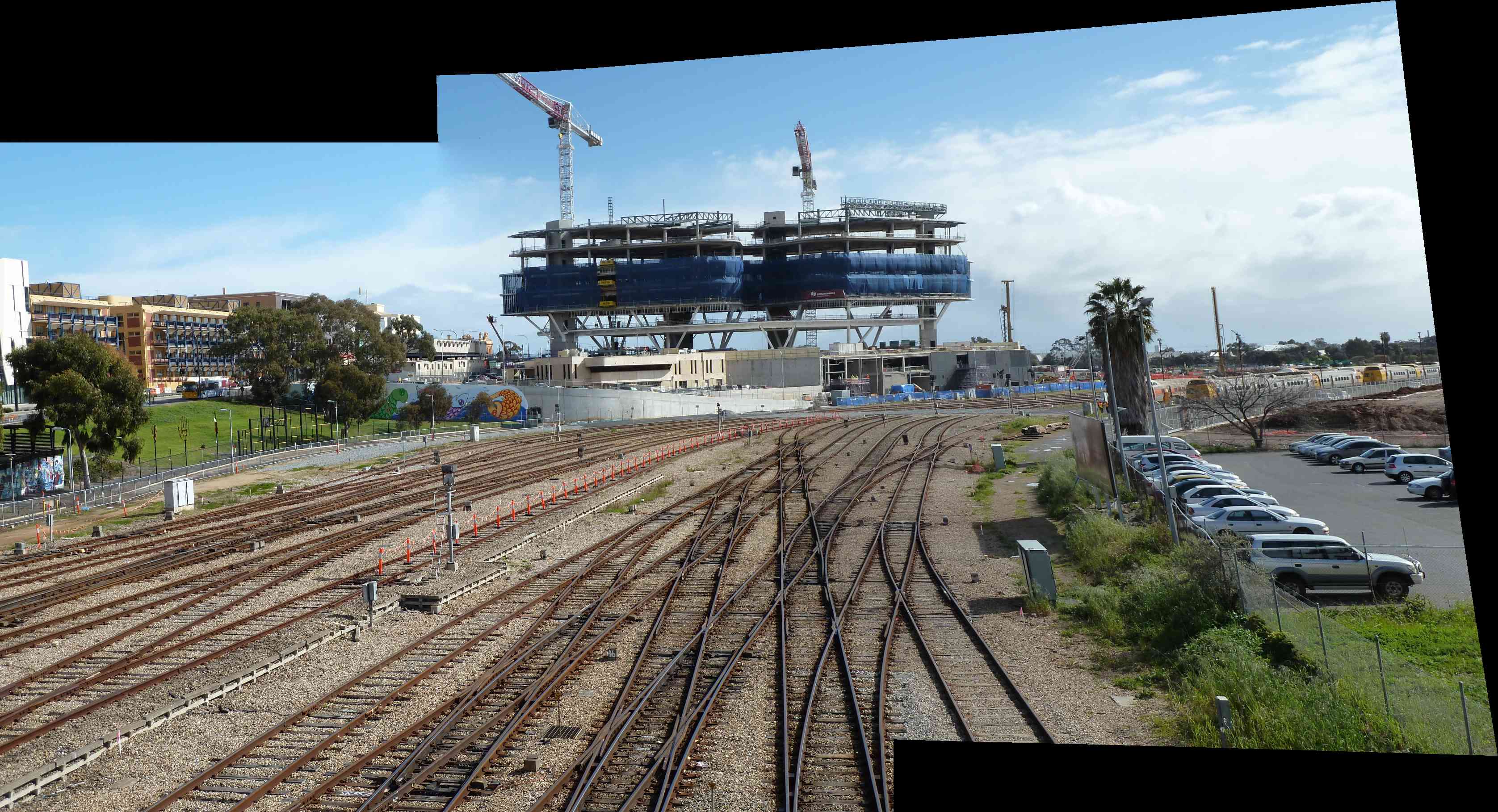}};
		\zoombox[color code=green]{0.35,0.63}
		\zoombox[color code=blue]{0.565,0.185}
		\end{tikzpicture}
	}\\
	\caption{Stitched images on the {\it railtracks\/} dataset~\cite{zaragoza_as-projective-as-possible_2014}. Regions with artifacts in some algorithms are magnified in the insets at the bottom.}\label{fig:railtracks_stitched}
\end{figure}

\begin{figure}
	\subfloat[Input\label{fig:catabus_input}]{%
		\begin{tikzpicture}
		\node (input01) {\includegraphics[width=0.23\textwidth]{figs/catabus_01}};
		\node[right=0em of input01] {\includegraphics[width=0.23\textwidth]{figs/catabus_02}};
		\end{tikzpicture}
	}\hspace{-0.2cm}
	\subfloat[CPW~\cite{hu_multi-objective_2015}\label{fig:catabus_CPW_aligned}]{%
		\begin{tikzpicture}
		\node {\includegraphics[width=0.48\textwidth]{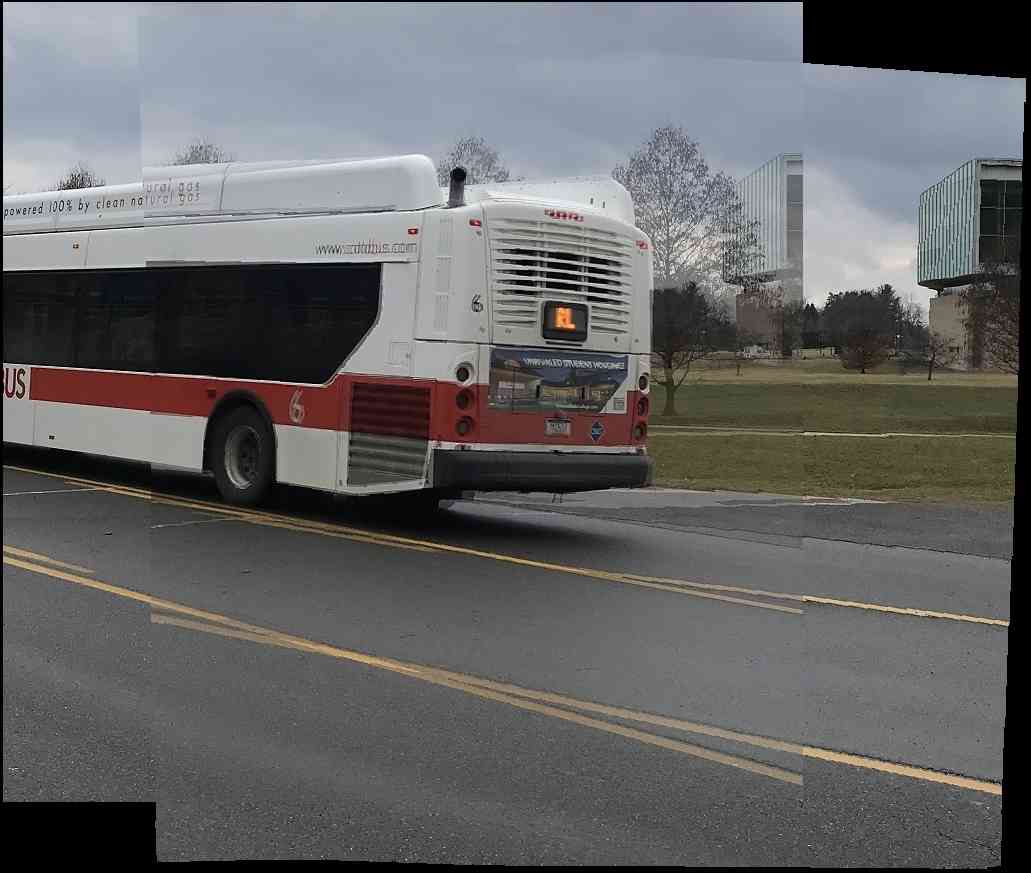}};
		\end{tikzpicture}
	}\\
	\subfloat[APAP~\cite{zaragoza_as-projective-as-possible_2014}\label{fig:catabus_APAP_aligned}]{%
		\begin{tikzpicture}
		\node {\includegraphics[width=0.48\linewidth]{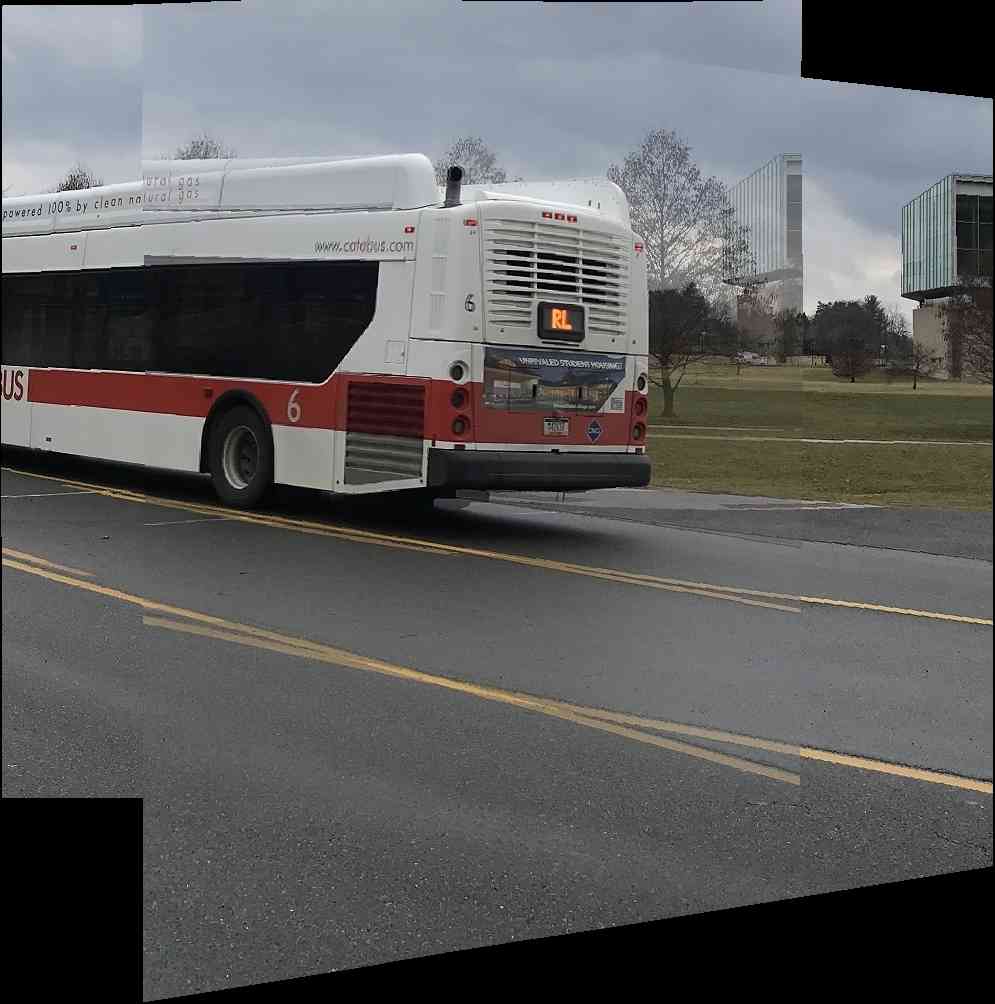}};
		\end{tikzpicture}
	}\hspace{-0.2cm}
	\subfloat[BRAS\label{fig:catabus_BRAS_aligned}]{%
		\begin{tikzpicture}
		\node {\includegraphics[width=0.48\linewidth]{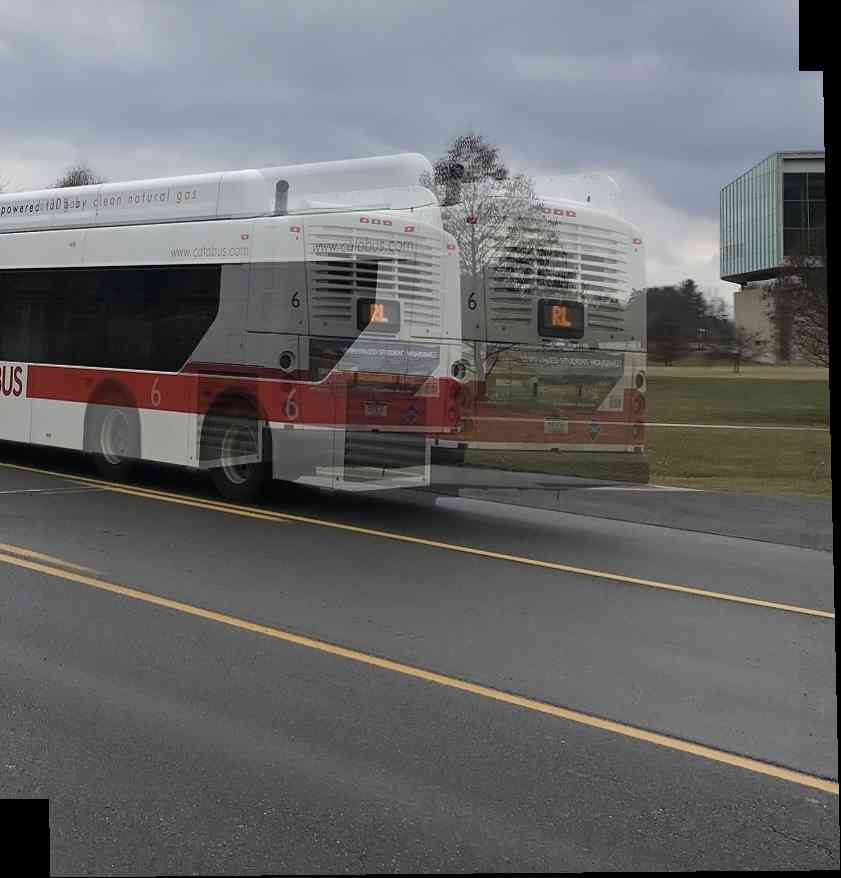}};
		\end{tikzpicture}
	}
	\caption{Alignment results in the presence of large moving objects. \tcr{Feature-based methods \protect\subref{fig:catabus_CPW_aligned} CPW, and \protect\subref{fig:catabus_APAP_aligned} APAP are guided by a dominant number of moving object features and thus misalign the background, while \protect\subref{fig:catabus_BRAS_aligned} BRAS, being pixel based, absorbs information from the whole image and performs alignment more favorable for subsequent stitching.}}\label{fig:catabus_aligned}
\end{figure}

\tcred{
\subsection{Stitched Results for Long Image Sequences}
\label{subsec:wide_fov}
In this section we present two examples from the \emph{CMU}~\cite{CMU} dataset, for aligning and stitching images covering a wide field of view, in Fig.~\ref{fig:CMU0} and Fig.~\ref{fig:CMU1}. The planar homography model is inapplicable to this case and we thus pre-project some of the images into cylindrical coordinates before feeding them into BRAS algorithm. We observe that BRAS does a competent job of aligning and stitching, whereas many state of the art methods such as APAP and CPW are inapplicable to such a scenario.}


\begin{figure*}
	\centering
	\subfloat {\includegraphics[width=0.9\textwidth]{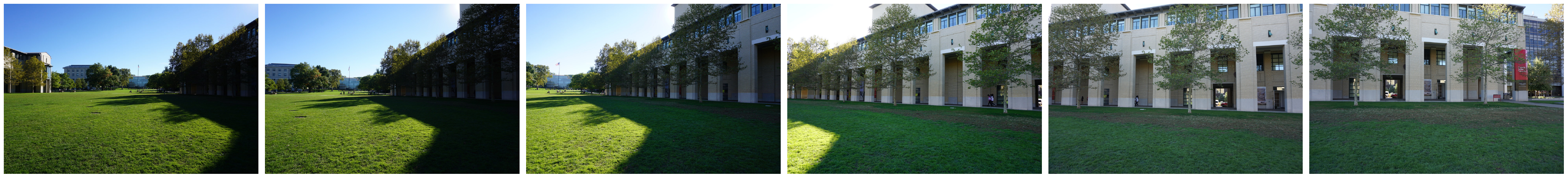}}\\ \vspace{-2mm}
	\subfloat {\includegraphics[width=0.9\textwidth]{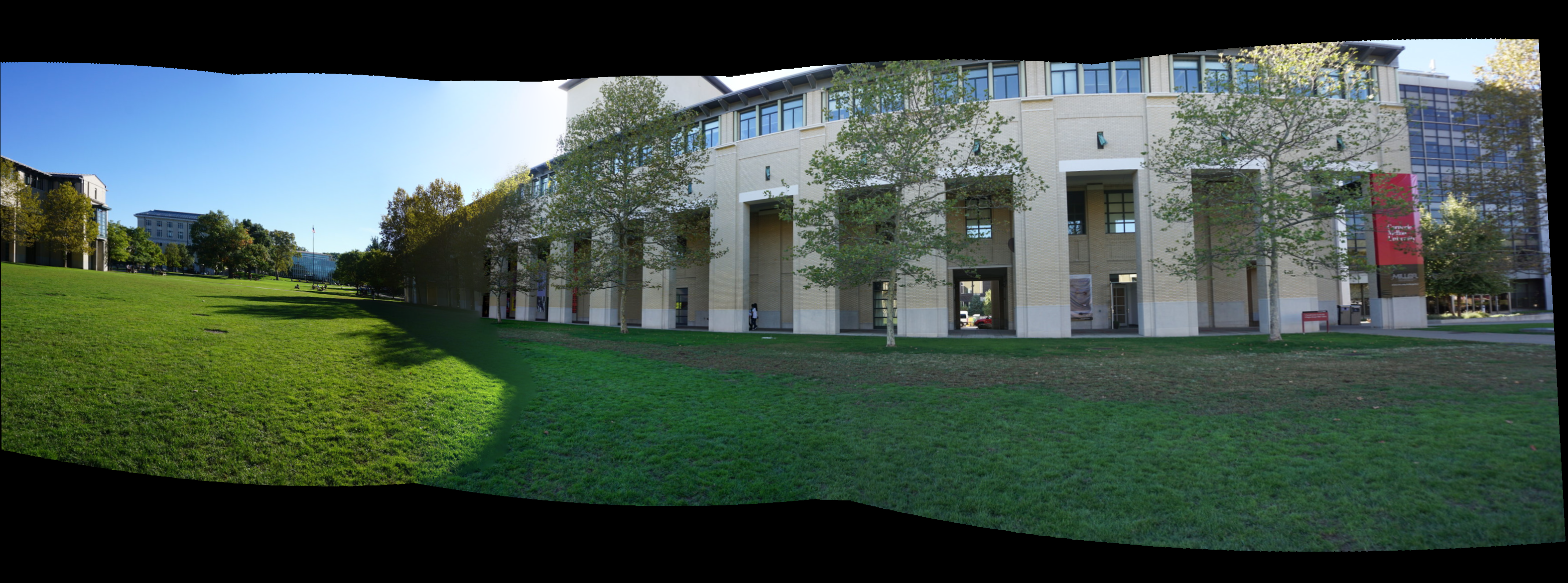}} \vspace{-2mm}
	\caption{\tcred{Stitched image with large number of images: Top: $6$ images from the {\it CMU0} dataset~\cite{CMU}. Bottom: These images are aligned and stitched by BRAS accurately, which verifies its effectiveness in the case with a long sequence of images.}}\label{fig:CMU0}
\end{figure*}

\begin{figure*}
	\centering
	\subfloat {\includegraphics[width=0.9\textwidth]{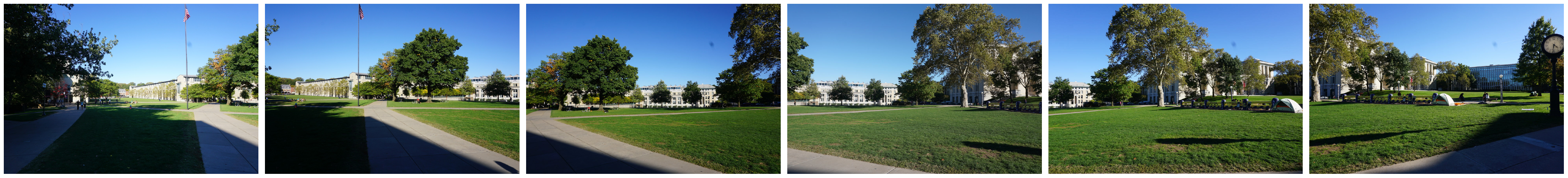}}\\ \vspace{-2mm}
	\subfloat {\includegraphics[width=0.9\textwidth]{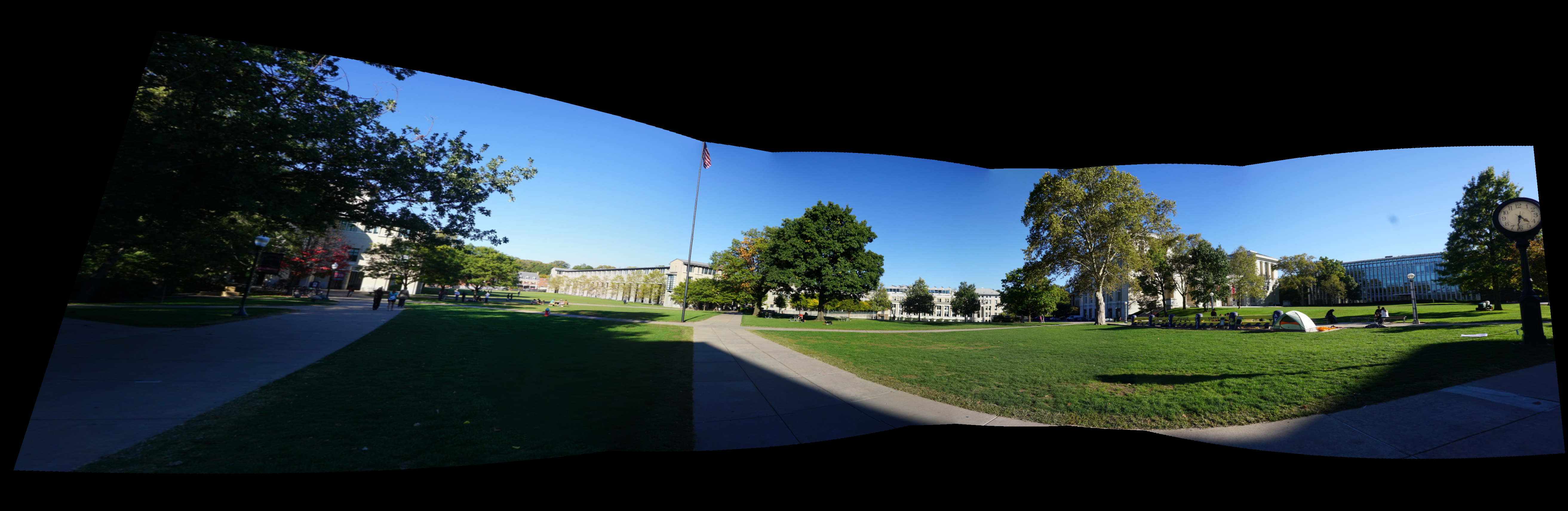}} \vspace{-2mm}
	\caption{\tcred{Stitched image with large number of images: Top: $6$ images from the {\it CMU1} dataset~\cite{CMU}. Bottom: These images are aligned and stitched by BRAS accurately, which verifies its effectiveness in the case with a long sequence of images.}}\label{fig:CMU1}
\end{figure*}


\begin{figure*}
	\centering
	\subfloat[AutoStitch~\cite{brown_automatic_2007}] {%
		\begin{tikzpicture}[zoomboxarray, zoomboxes below, zoomboxarray height=0.15\textheight, zoomboxes yshift=-0.9]
		\node [image node] {\includegraphics[height=0.25\textheight]{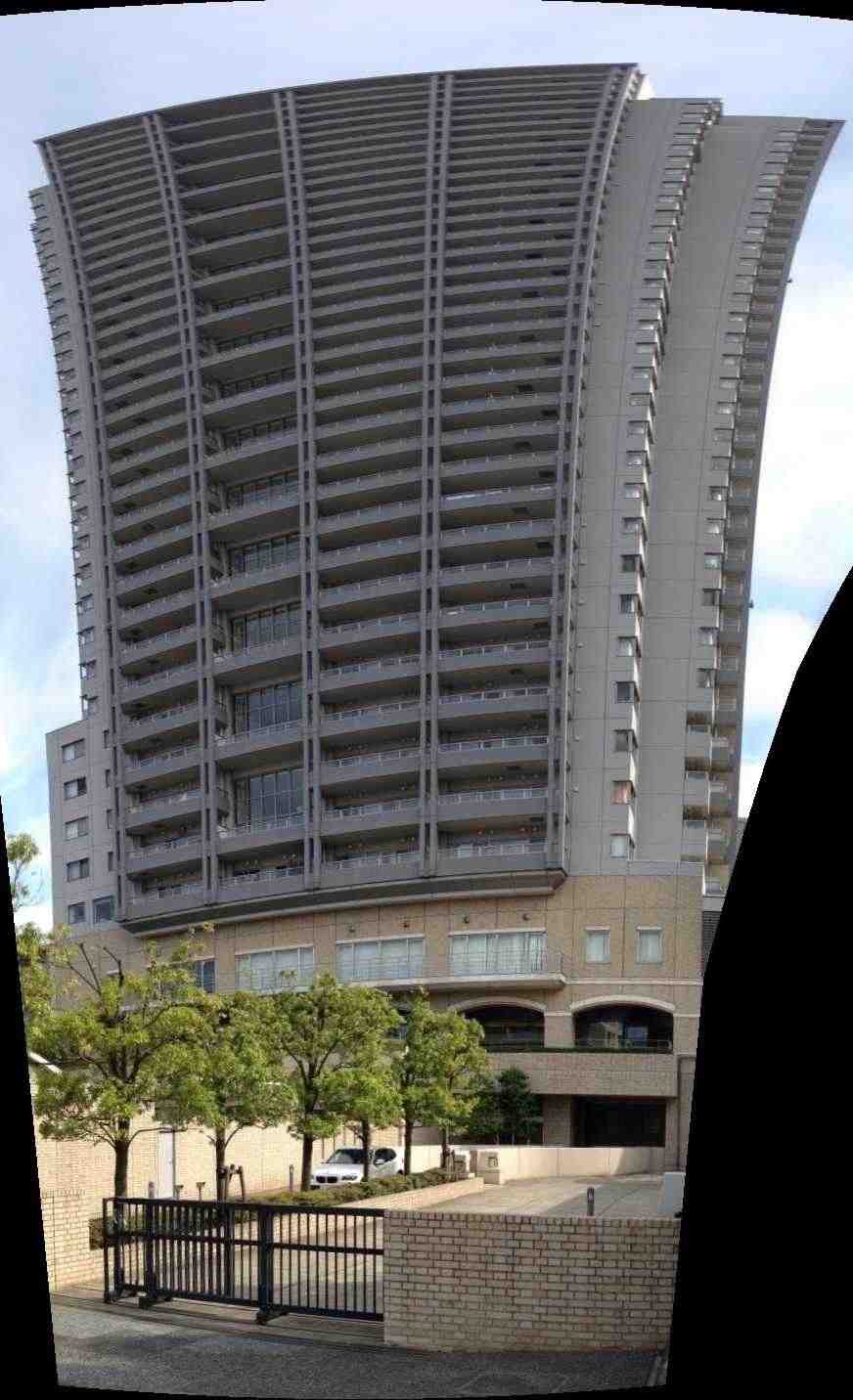}};
		\zoombox[color code=red]{0.38,0.58}
		\zoombox[color code=cyan]{0.6,0.85}
		\zoombox[color code=blue]{0.8,0.33}
		\zoombox[color code=green]{0.85,0.53}
		\end{tikzpicture}
	}\,\!
	\subfloat[ICE~\cite{ICE}] {%
		\begin{tikzpicture}[zoomboxarray, zoomboxes below, zoomboxarray height=0.15\textheight, zoomboxes yshift=-0.9, execute at end picture={\draw[thick,red] (0.8,-1) circle (.4cm);}, execute at end picture={\draw[thick,red] (2.7,-3) circle (.4cm);}]
		\node [image node] {\includegraphics[height=0.25\textheight]{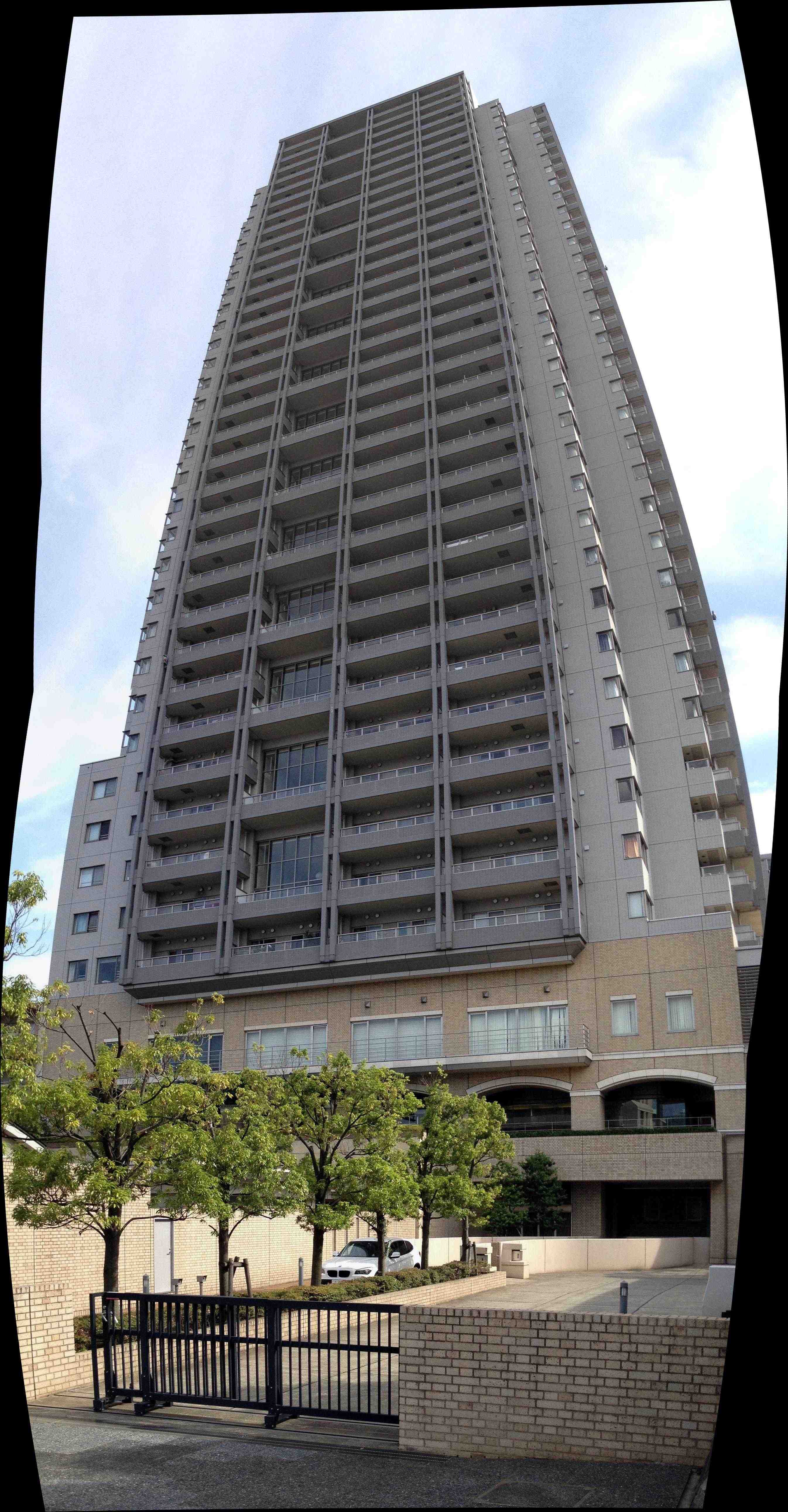}};
		\zoombox[color code=red]{0.43,0.58}
		\zoombox[color code=cyan]{0.6,0.85}
		\zoombox[color code=blue]{0.9,0.33}
		\zoombox[color code=green]{0.9,0.55}
		\end{tikzpicture}
	}\,\!
	\subfloat[SPHP~\cite{chang_shape-preserving_2014}] {%
		\begin{tikzpicture}[zoomboxarray, zoomboxes below, zoomboxarray height=0.15\textheight, zoomboxes yshift=-0.9, execute at end picture={\draw[thick,red] (1.05,-3.2) circle (.4cm);}]
		\node [image node] {\includegraphics[height=0.25\textheight]{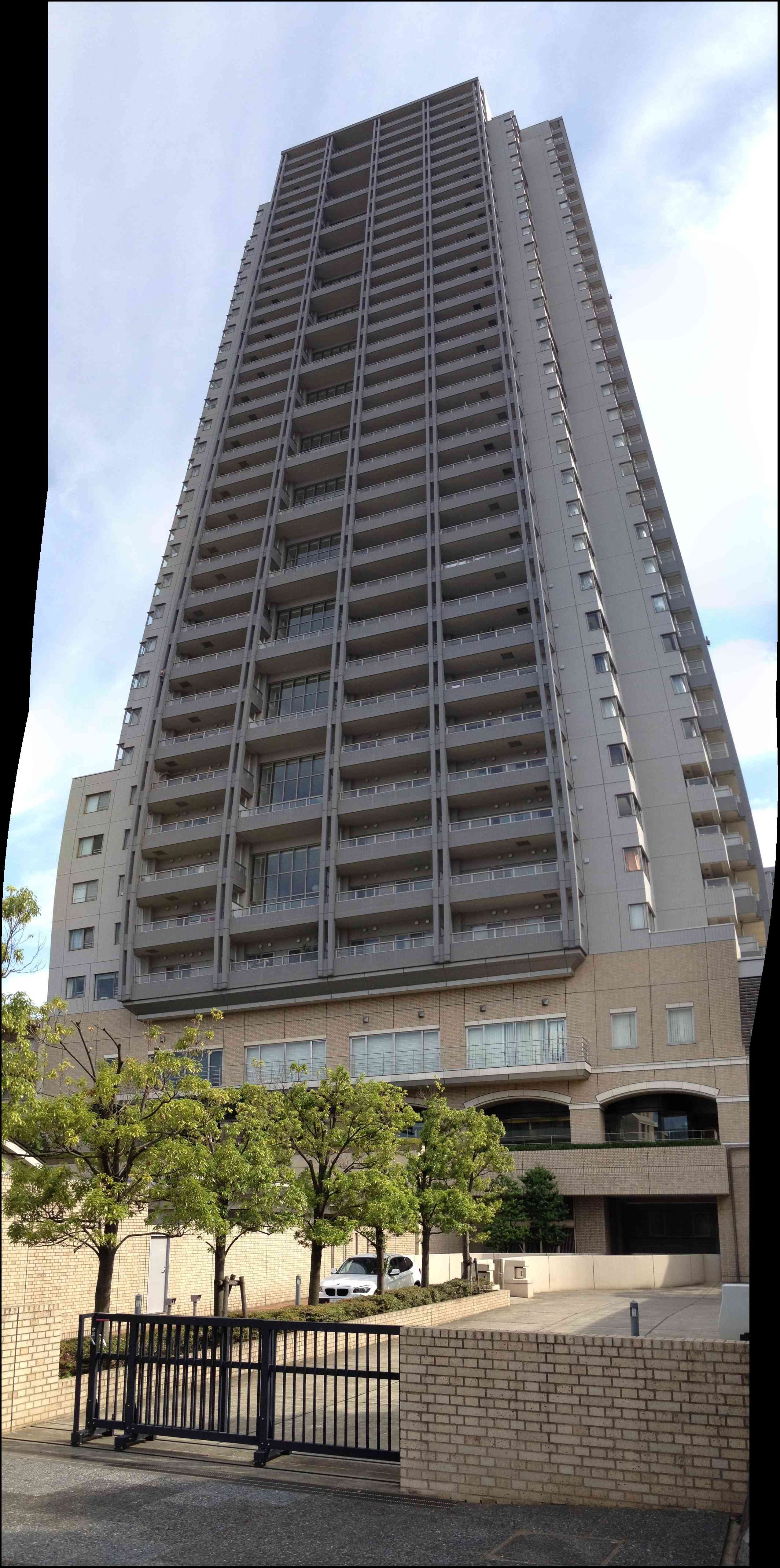}};
		\zoombox[color code=red]{0.43,0.56}
		\zoombox[color code=cyan]{0.6,0.85}
		\zoombox[color code=blue]{0.9,0.35}
		\zoombox[color code=green]{0.9,0.55}
		\end{tikzpicture}
	}\,\!
	\subfloat[APAP~\cite{zaragoza_as-projective-as-possible_2014}] {%
		\begin{tikzpicture}[zoomboxarray, zoomboxes below, zoomboxarray height=0.15\textheight, zoomboxes yshift=-0.9, execute at end picture={\draw[thick,red] (2.5,-1) circle (.4cm);}]
		\node [image node] {\includegraphics[height=0.25\textheight]{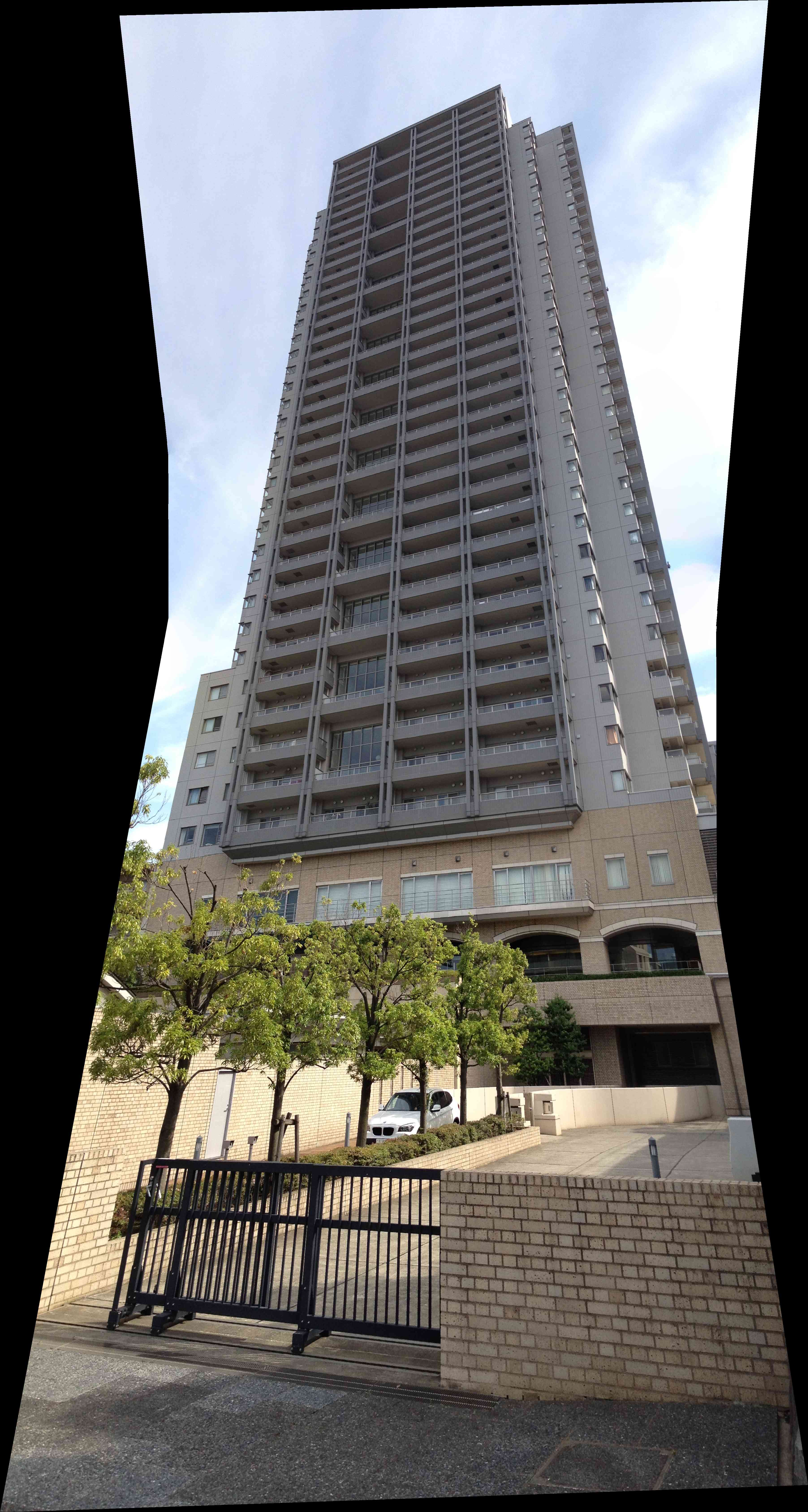}};
		\zoombox[color code=red]{0.48,0.62}
		\zoombox[color code=cyan]{0.6,0.85}
		\zoombox[color code=blue]{0.86,0.43}
		\zoombox[color code=green]{0.83,0.58}
		\end{tikzpicture}
	}\,\!
	\subfloat[BRAS] {%
		\begin{tikzpicture}[zoomboxarray, zoomboxes below, zoomboxarray height=0.15\textheight, zoomboxes yshift=-0.9]
		\node [image node] {\includegraphics[height=0.25\textheight]{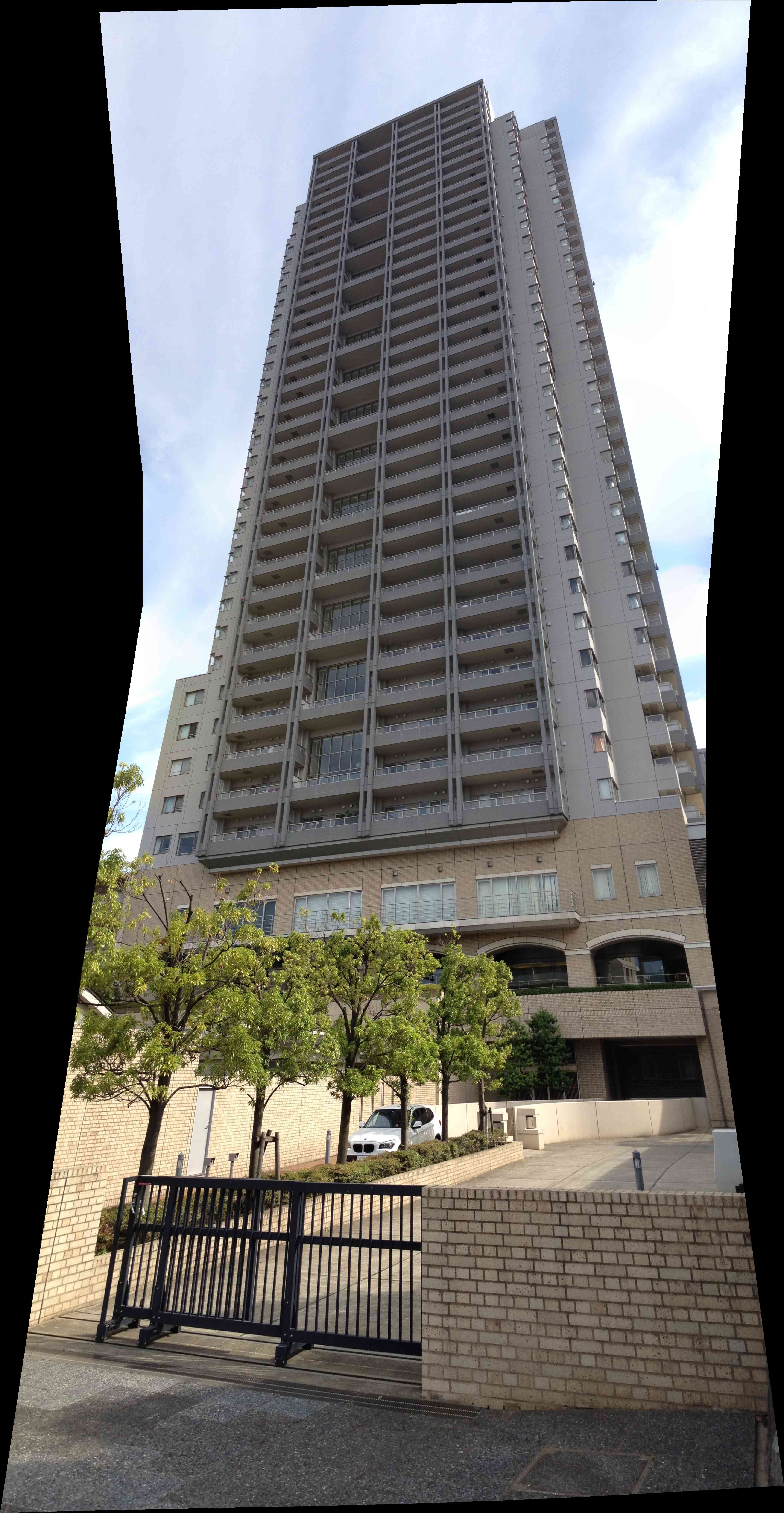}};
		\zoombox[color code=red]{0.43,0.58}
		\zoombox[color code=cyan]{0.6,0.85}
		\zoombox[color code=blue]{0.85,0.42}
		\zoombox[color code=green]{0.83,0.58}
		\end{tikzpicture}
	}\\
	\caption{\tcr{Stitched images on the {\it skyscraper\/} dataset~\cite{chang_shape-preserving_2014}. Regions with artifacts in some algorithms are magnified in boxes. Red circles in the insets highlight the artifacts.}}\label{fig:skyscraper}
\end{figure*}
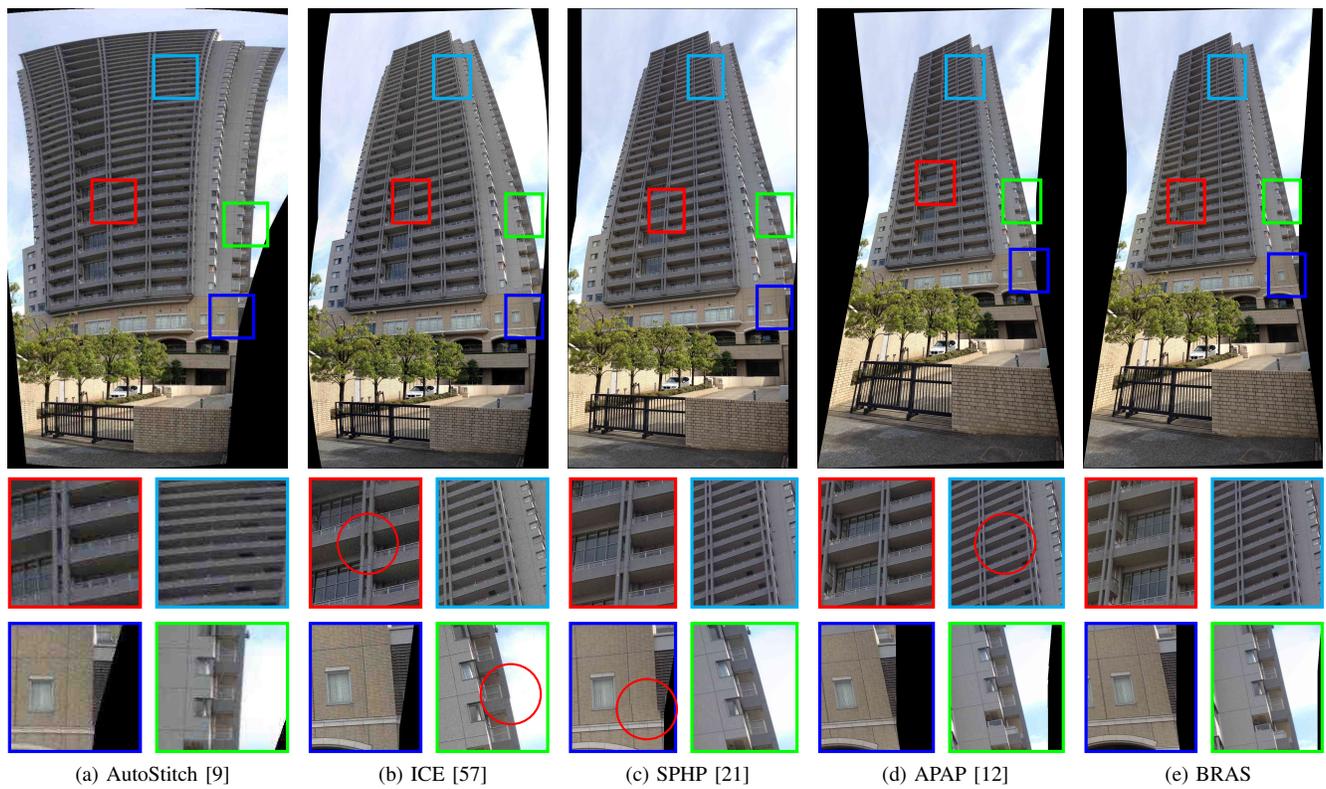

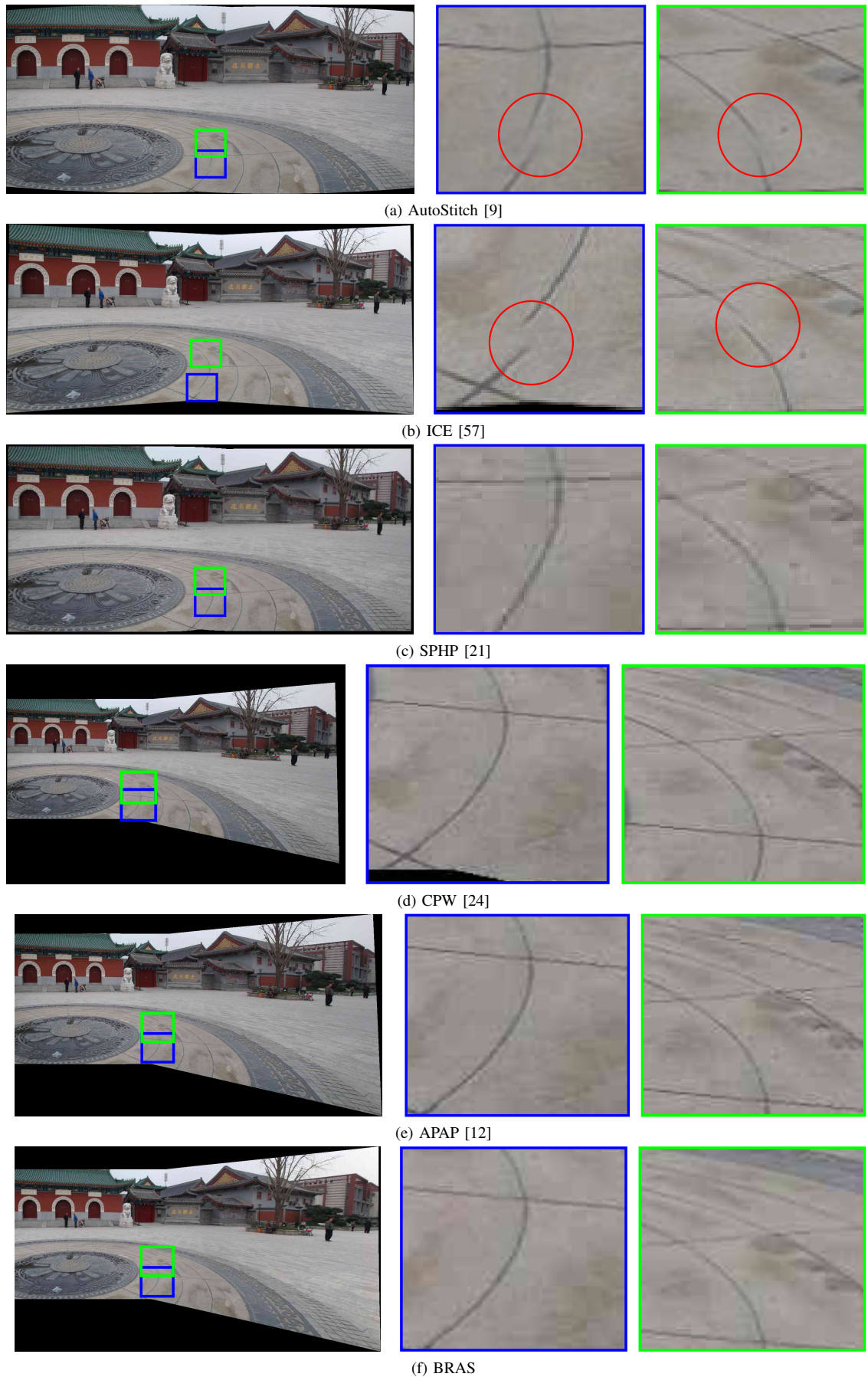
\begin{figure*}
	\centering
	\subfloat[AutoStitch~\cite{brown_automatic_2007}] {%
		\resizebox{0.83\textwidth}{!}{\begin{tikzpicture}[zoomboxarray, zoomboxarray rows=1, zoomboxarray width=0.4\textwidth, zoomboxes xshift=1.05, execute at end picture={\draw[thick,red] (9,1) circle (.7cm);}, execute at end picture={\draw[thick,red] (12.7,1) circle (.7cm);}]
			\node [image node] {\includegraphics[height=0.13\textheight]{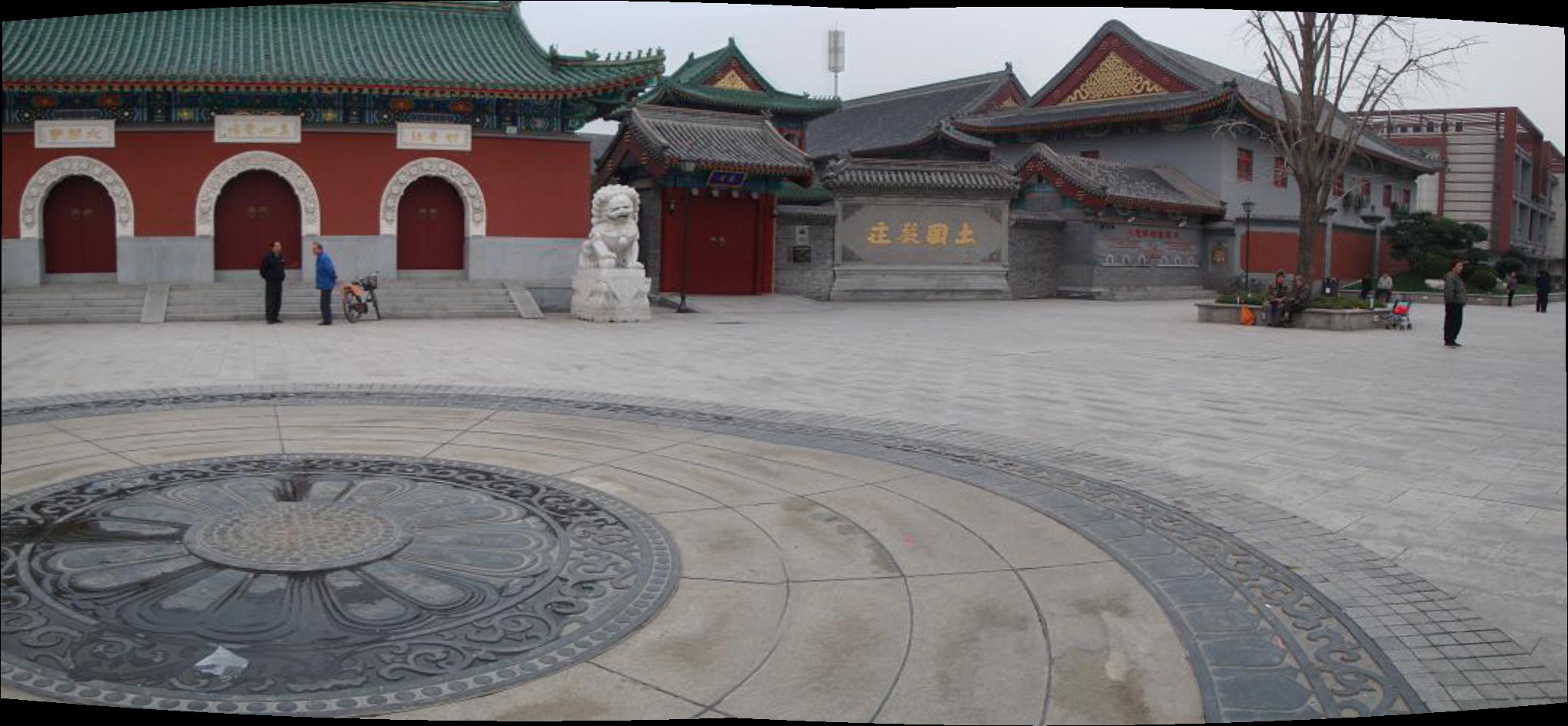}};
			\zoombox[color code=blue, magnification=7]{0.5,0.16} 
			\zoombox[color code=green, magnification=7]{0.5,0.27}
			\end{tikzpicture}
		}
	}\\\vspace{-3mm}
	\subfloat[ICE~\cite{ICE}] {%
		\resizebox{0.83\textwidth}{!}{\begin{tikzpicture}[zoomboxarray, zoomboxarray rows=1, zoomboxarray width=0.4\textwidth, zoomboxes xshift=1.05, execute at end picture={\draw[thick,red] (8.8,1.2) circle (.7cm);}, execute at end picture={\draw[thick,red] (12.6,1.5) circle (.7cm);}]
			\node [image node] {\includegraphics[height=0.13\textheight]{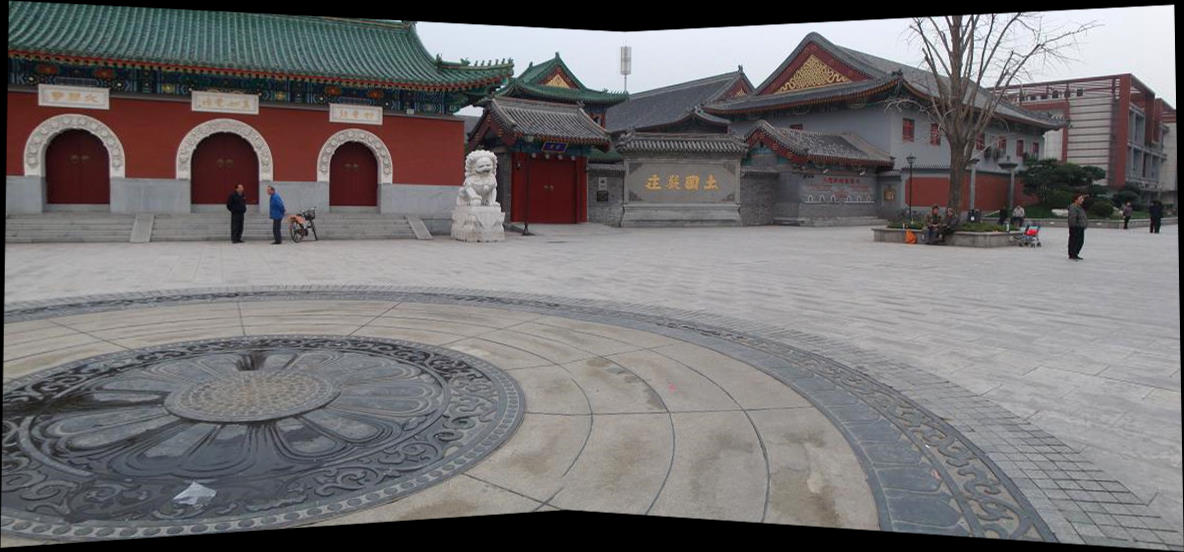}};
			\zoombox[color code=blue, magnification=7]{0.48,0.14}
			\zoombox[color code=green, magnification=7]{0.49,0.32}
			\end{tikzpicture}
		}
	}\\\vspace{-3mm}
	\subfloat[SPHP~\cite{chang_shape-preserving_2014}] {%
		\resizebox{0.83\textwidth}{!}{\begin{tikzpicture}[zoomboxarray, zoomboxarray rows=1, zoomboxarray width=0.4\textwidth, zoomboxes xshift=1.05]
			\node [image node] {\includegraphics[height=0.13\textheight]{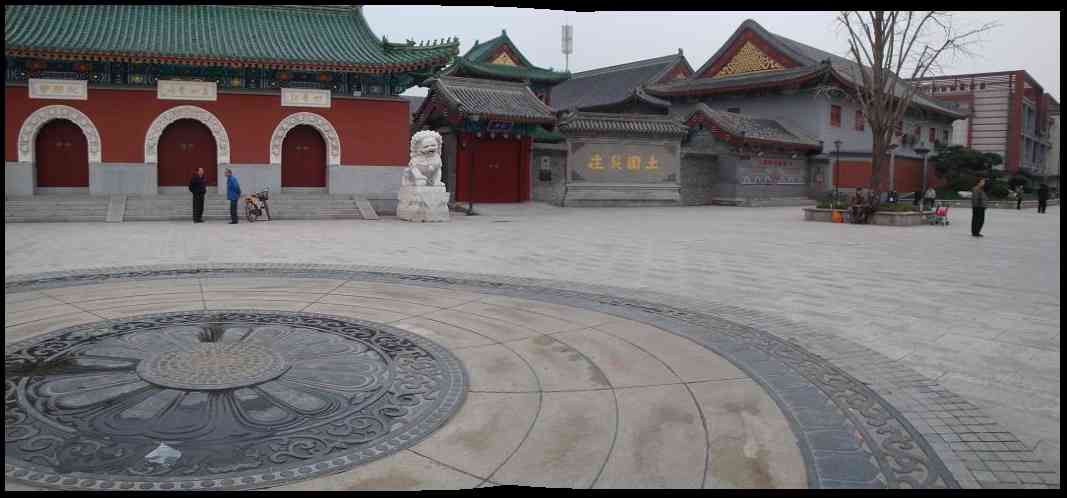}};
			\zoombox[color code=blue, magnification=7]{0.5,0.17}
			\zoombox[color code=green, magnification=7]{0.5,0.28}
			\end{tikzpicture}
		}
	}\\\vspace{-3mm}
	\subfloat[CPW~\cite{hu_multi-objective_2015}] {%
		\resizebox{0.83\textwidth}{!}{\begin{tikzpicture}[zoomboxarray, zoomboxarray rows=1, zoomboxarray width=0.4\textwidth, zoomboxes xshift=1.25]	
			\node [image node] {\includegraphics[height=0.13\textheight]{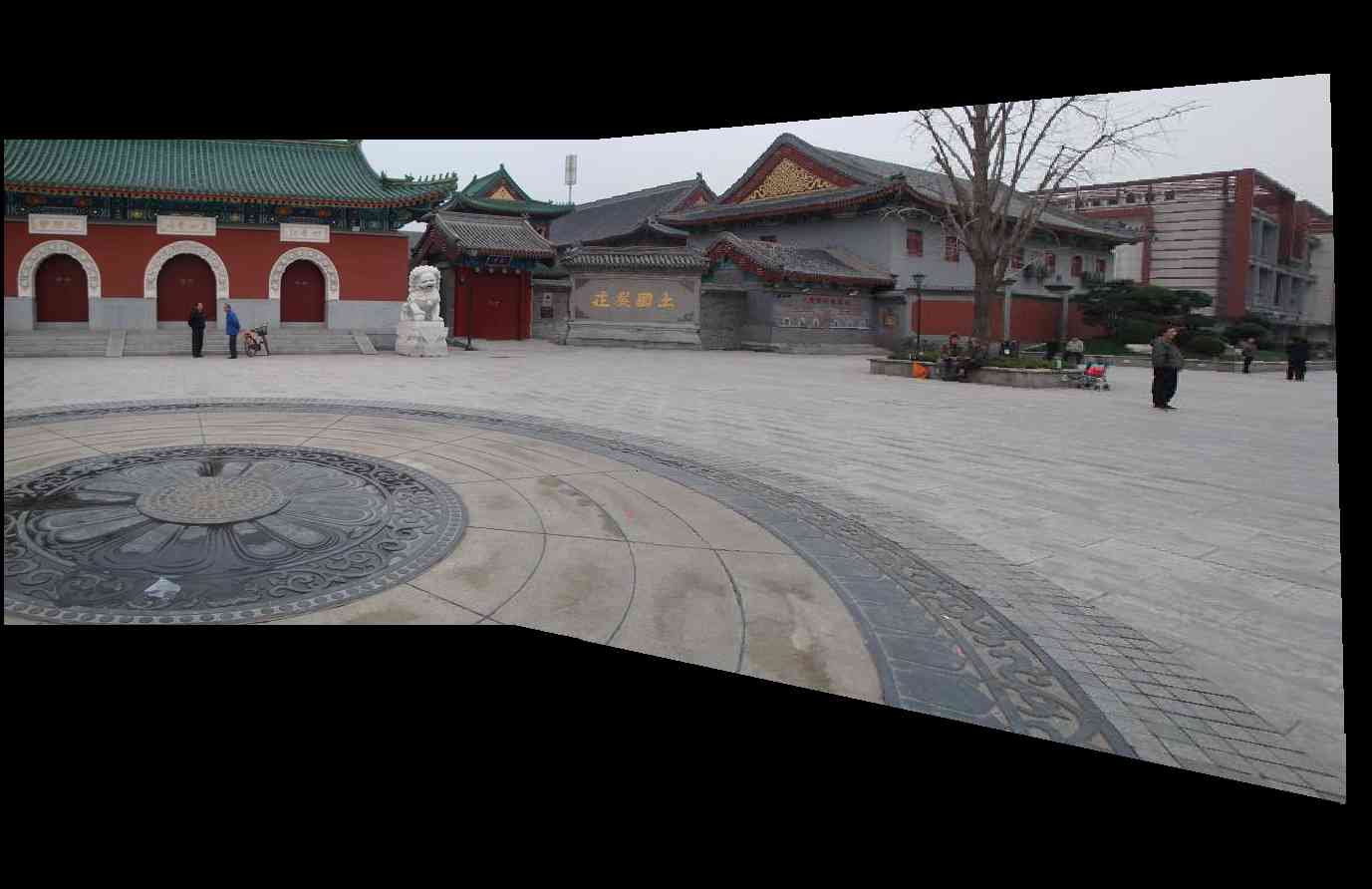}};
			\zoombox[color code=blue, magnification=7]{0.39,0.36}
			\zoombox[color code=green, magnification=7]{0.39,0.44}
			\end{tikzpicture}
		}
	}\\\vspace{-3mm}
	\subfloat[APAP~\cite{zaragoza_as-projective-as-possible_2014}] {%
		\resizebox{0.83\textwidth}{!}{
			\begin{tikzpicture}[zoomboxarray, zoomboxarray rows=1, zoomboxarray width=0.4\textwidth, zoomboxes xshift=1.15]
			\node [image node] {\includegraphics[height=0.13\textheight]{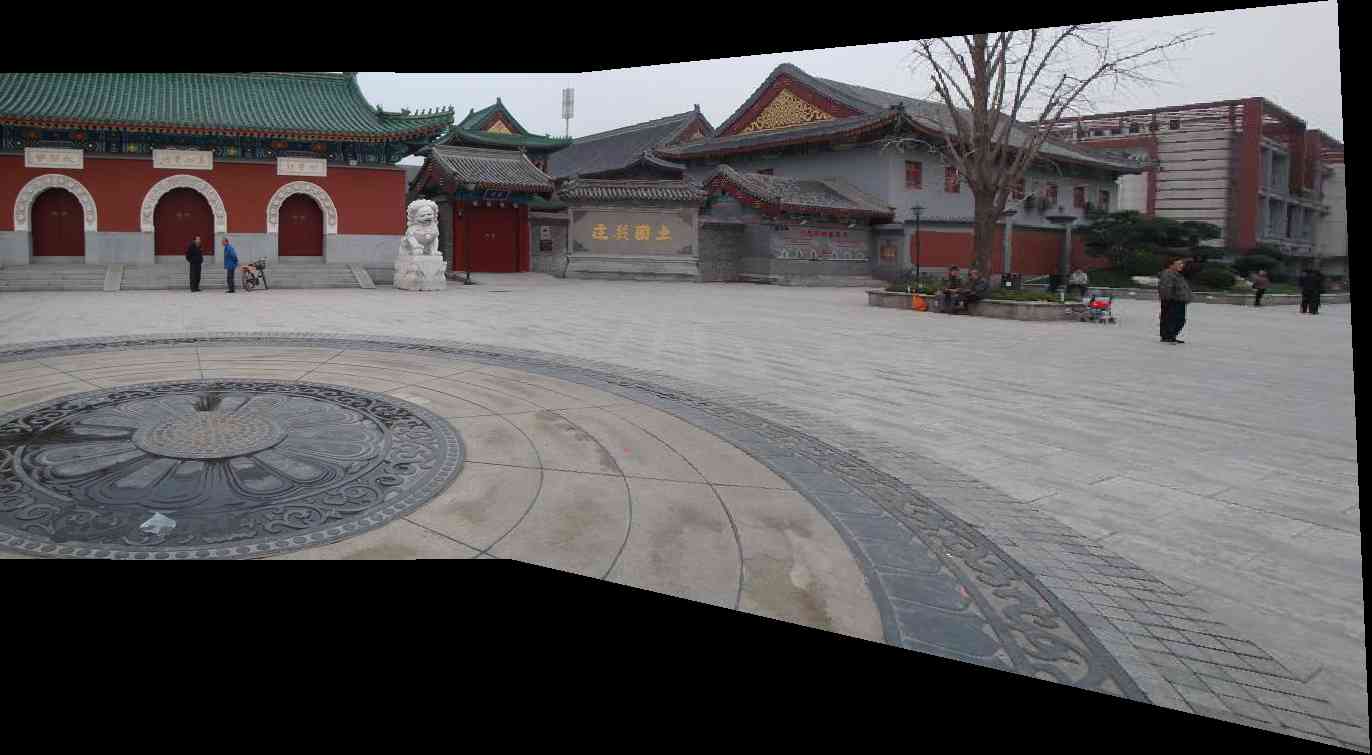}};
			\zoombox[color code=blue, magnification=7]{0.39,0.34}
			\zoombox[color code=green, magnification=7]{0.39,0.44}
			\end{tikzpicture}
		}
	}\\\vspace{-3mm}
	\subfloat[BRAS] {%
		\resizebox{0.83\textwidth}{!}{
			\begin{tikzpicture}[zoomboxarray, zoomboxarray rows=1, zoomboxarray width=0.4\textwidth, zoomboxes xshift=1.15]
			\node [image node] {\includegraphics[height=0.13\textheight]{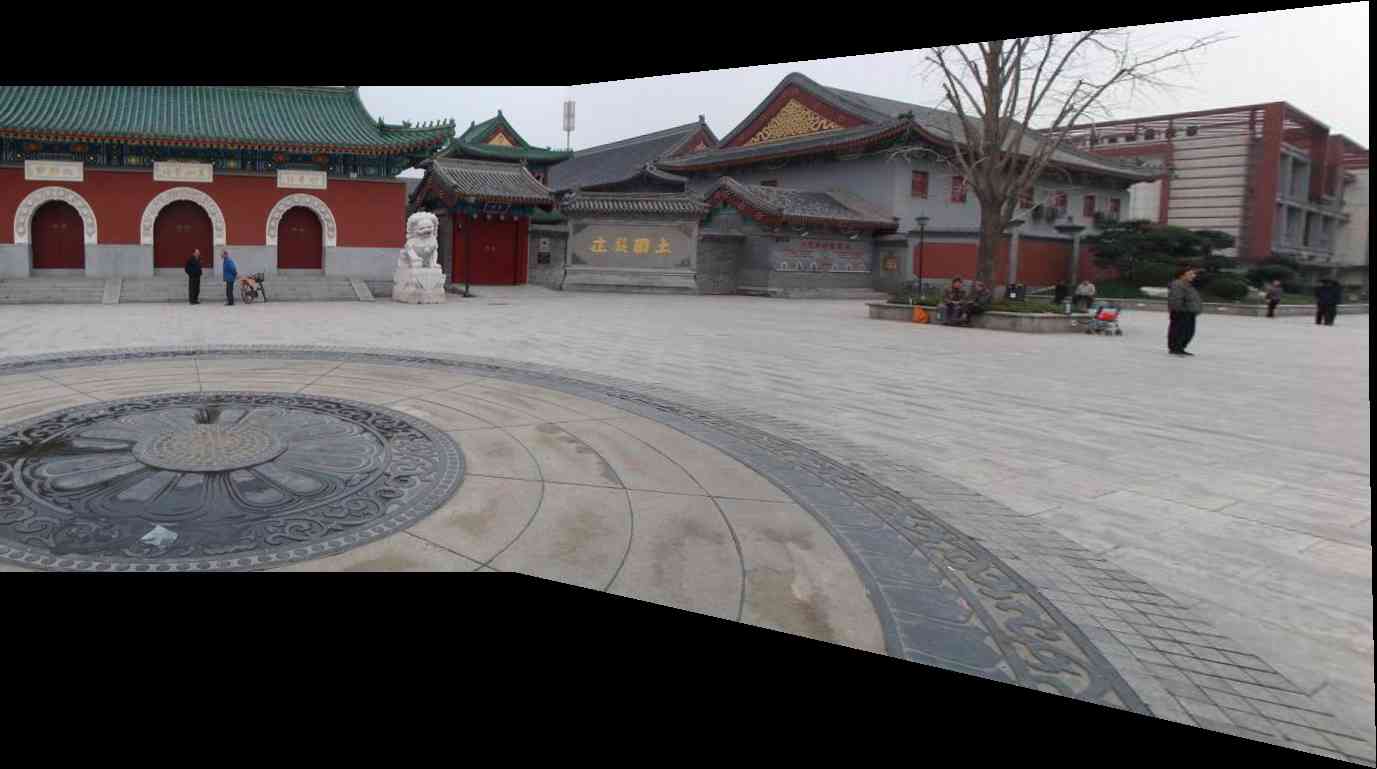}};
			\zoombox[color code=blue, magnification=7]{0.39,0.34}
			\zoombox[color code=green, magnification=7]{0.39,0.44}
			\end{tikzpicture}
		}
	}\\\vspace{-1mm}
	\caption{\tcr{Stitched images on the {\it temple\/} dataset~\cite{gao_constructing_2011}. Red circles in the insets highlight artifacts.}}\label{fig:temple}
\end{figure*}

\begin{figure*}
	\centering
	\subfloat[AutoStitch~\cite{brown_automatic_2007}] {%
		\resizebox{0.9\textwidth}{!}{
			\begin{tikzpicture}[zoomboxarray , zoomboxarray rows=1, zoomboxarray width=0.4\textwidth, zoomboxes xshift=1.05]
			\node [image node] {\includegraphics[height=0.15\textheight]{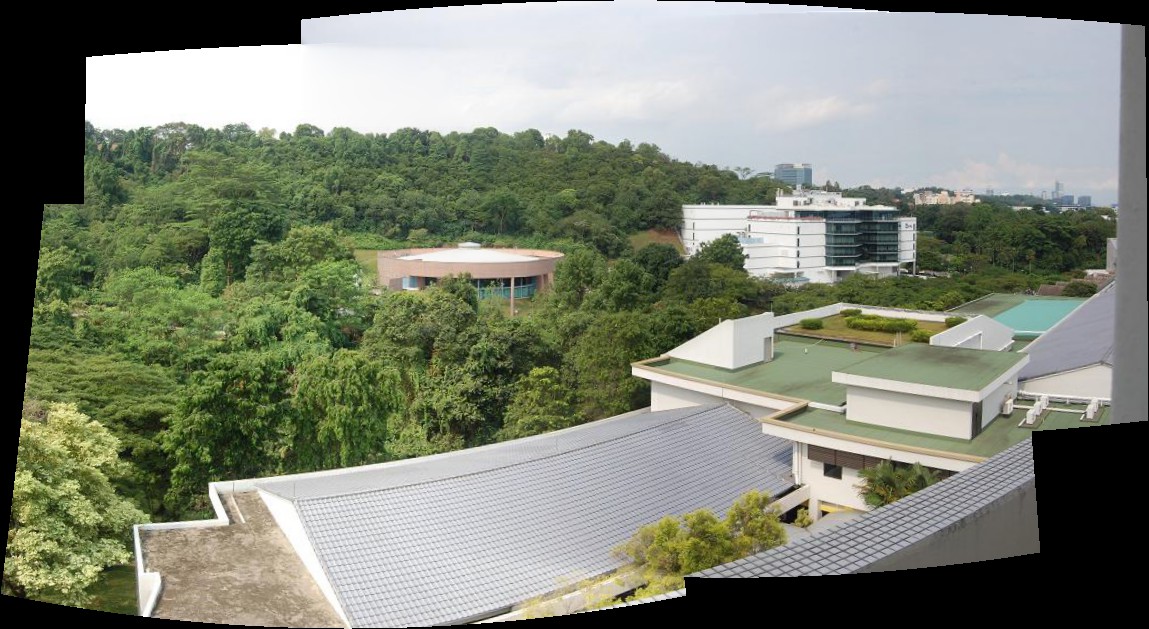}};
			\zoombox[color code=blue, magnification=7]{0.58,0.4}
			\zoombox[color code=green, magnification=6]{0.75,0.62}
			\end{tikzpicture}
		}
	}\,\!
	\subfloat[ICE~\cite{ICE}] {%
		\resizebox{0.9\textwidth}{!}{
			\begin{tikzpicture}[zoomboxarray , zoomboxarray rows=1, zoomboxarray width=0.4\textwidth, zoomboxes xshift=1.1]
			\node [image node] {\includegraphics[height=0.15\textheight]{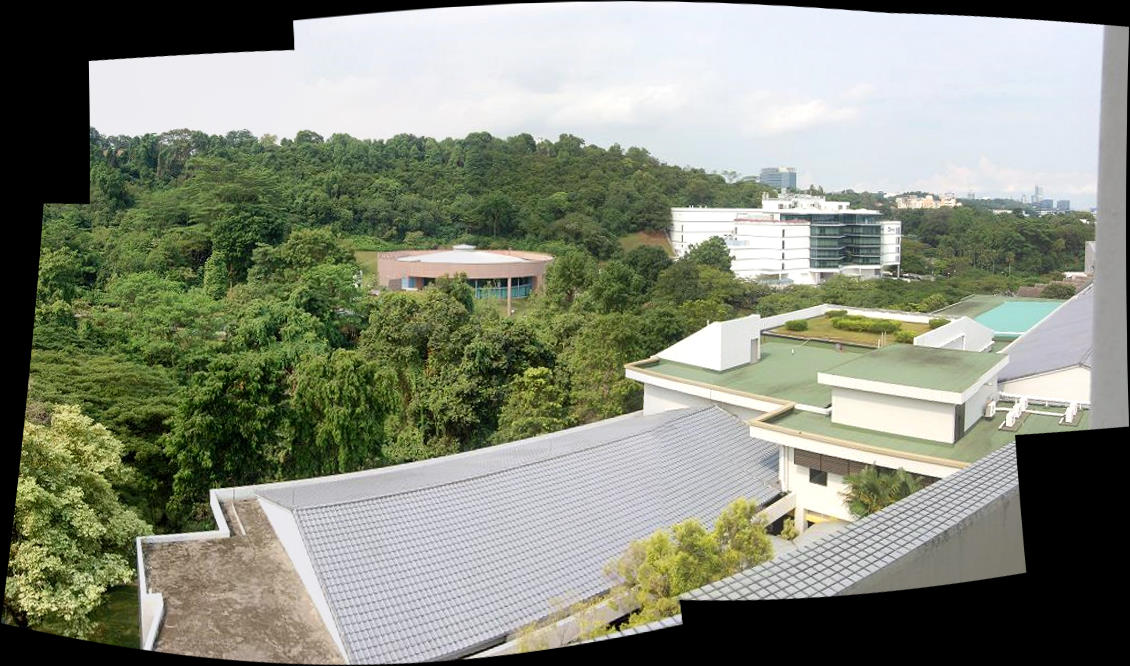}};
			\zoombox[color code=blue, magnification=7]{0.58,0.42}
			\zoombox[color code=green, magnification=6]{0.75,0.62}
			\end{tikzpicture}
		}
	}\,\!
	\subfloat[SPHP~\cite{chang_shape-preserving_2014}] {%
		\resizebox{0.9\textwidth}{!}{
			\begin{tikzpicture}[zoomboxarray , zoomboxarray rows=1, zoomboxarray width=0.4\textwidth, zoomboxes xshift=1.05, execute at end picture={\draw[thick,red] (8.9,0.95) circle (.5cm);}]
			\node [image node] {\includegraphics[height=0.15\textheight]{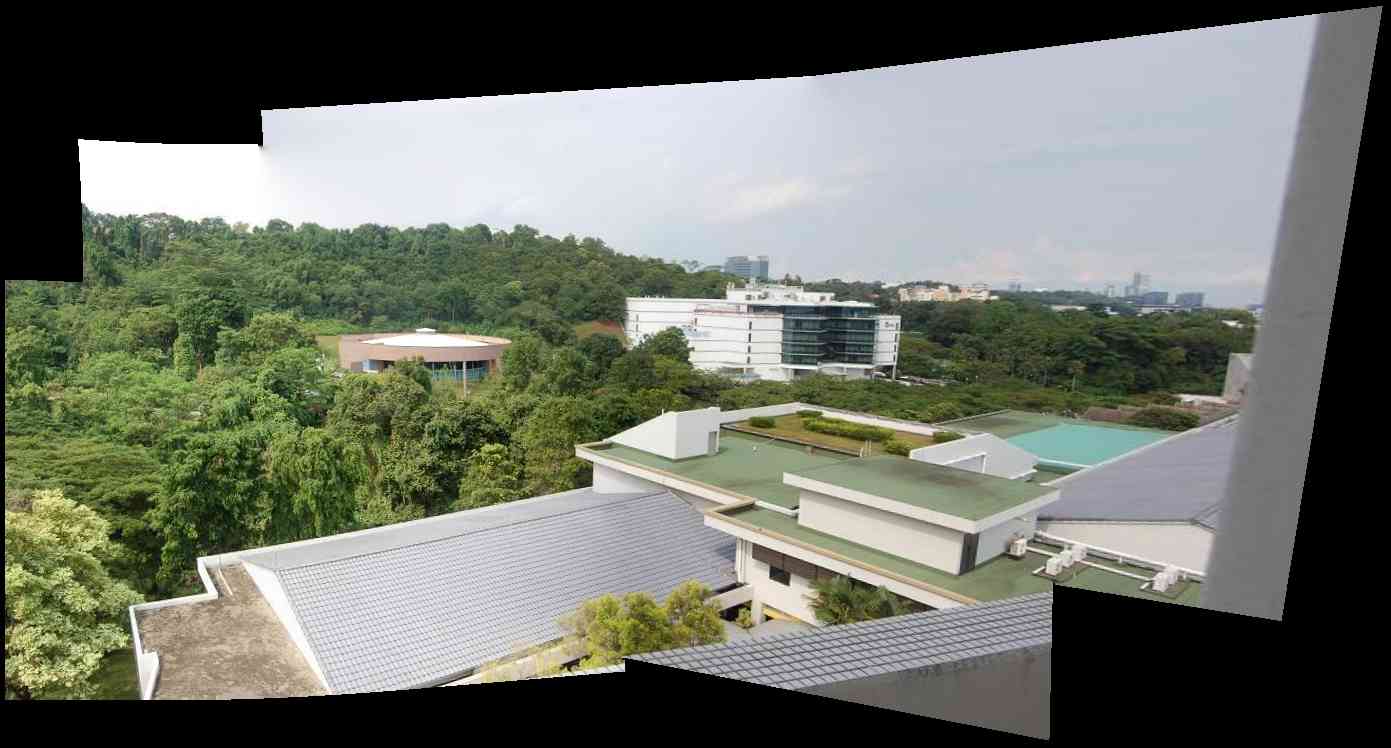}};
			\zoombox[color code=blue, magnification=7]{0.43,0.38}
			\zoombox[color code=green, magnification=6]{0.6,0.53}
			\end{tikzpicture}
		}
	}\,\!
	\subfloat[APAP~\cite{zaragoza_as-projective-as-possible_2014}] {%
		\resizebox{0.9\textwidth}{!}{
			\begin{tikzpicture}[zoomboxarray , zoomboxarray rows=1, zoomboxarray width=0.4\textwidth, zoomboxes xshift=1.08, execute at end picture={\draw[thick,red] (8.7,2.4) circle (.7cm);}, execute at end picture={\draw[thick,red] (12,2) circle (.7cm);}]
			\node [image node] {\includegraphics[height=0.15\textheight]{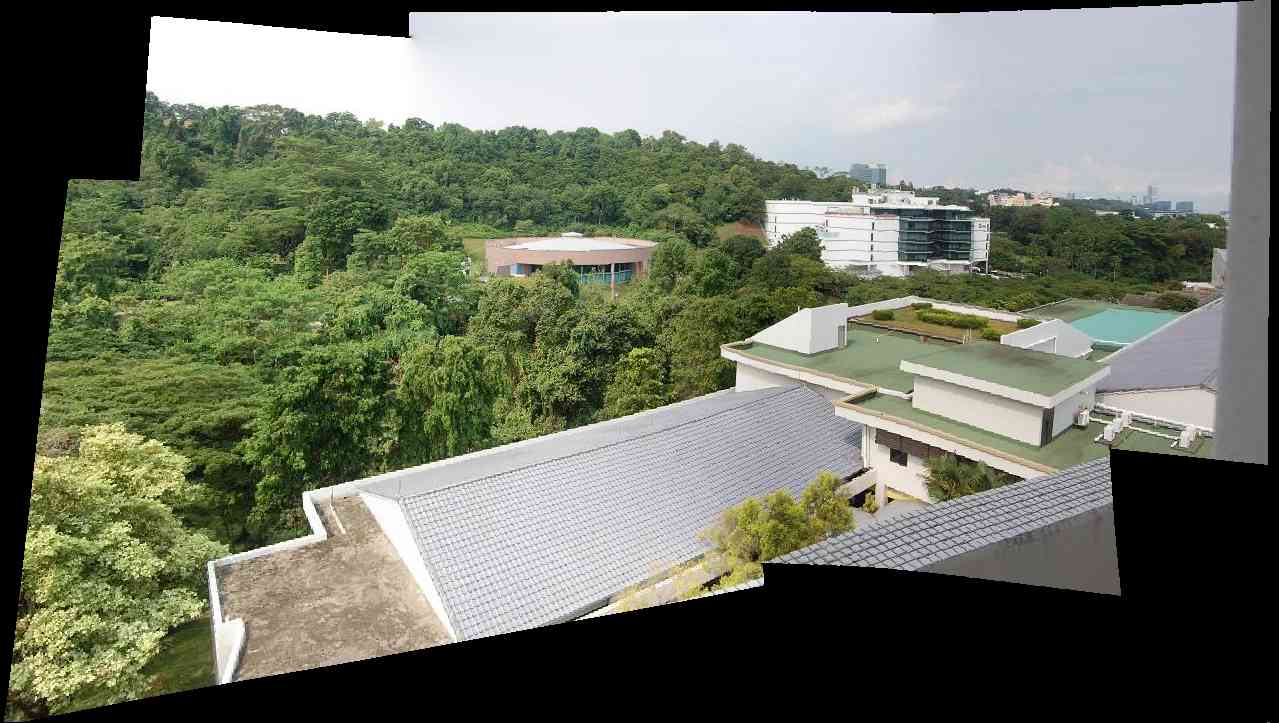}};
			\zoombox[color code=blue, magnification=7]{0.58,0.5}
			\zoombox[color code=green, magnification=6]{0.73,0.65}
			\end{tikzpicture}
		}
	}\,\!
	\subfloat[BRAS] {%
		\resizebox{0.9\textwidth}{!}{
			\begin{tikzpicture}[zoomboxarray , zoomboxarray rows=1, zoomboxarray width=0.4\textwidth, zoomboxes xshift=1.1]
			\node [image node] {\includegraphics[height=0.15\textheight]{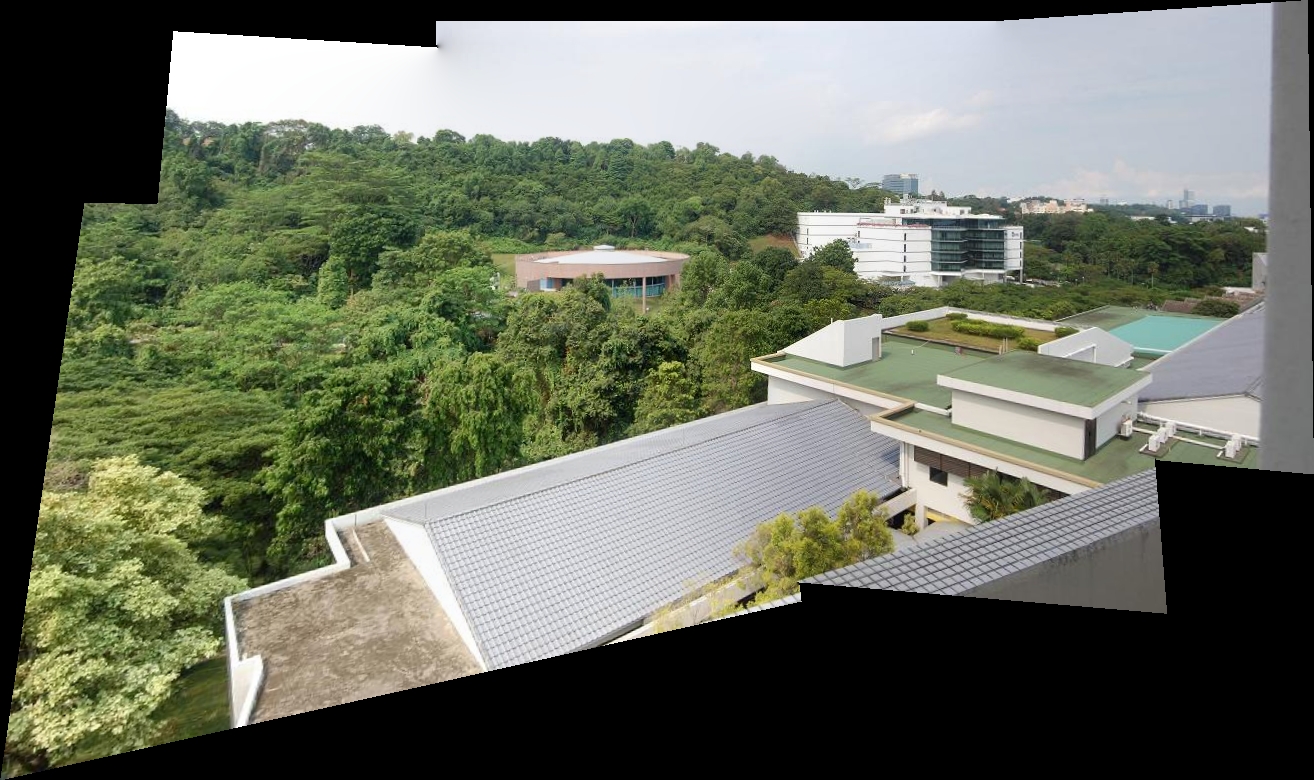}};
			\zoombox[color code=blue, magnification=7]{0.58,0.5}
			\zoombox[color code=green, magnification=6]{0.73,0.65}
			\end{tikzpicture}
		}
	}\\
	\caption{\tcr{Stitched images on the {\it forest\/} dataset~\cite{zaragoza_as-projective-as-possible_2014}. Red circles in the insets highlight artifacts.}}\label{fig:forest}
\end{figure*}

\begin{figure*}
	\centering
	\subfloat[AutoStitch~\cite{brown_automatic_2007}] {%
		\resizebox{0.95\textwidth}{!}{
			\begin{tikzpicture}[zoomboxarray, zoomboxarray rows=1, zoomboxarray width=0.4\textwidth, zoomboxes xshift=.99, execute at end picture={\draw[thick,red] (9.6,2.2) circle (.7cm);}]
			\node [image node] {\includegraphics[height=0.15\textheight]{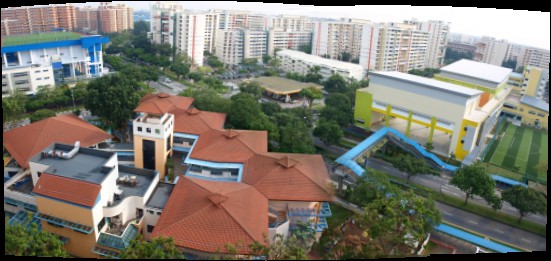}};
			\zoombox[color code=blue]{0.6,0.80}
			\zoombox[color code=green]{0.6,0.55}
			\end{tikzpicture}
		}
	}\,\!
	\subfloat[ICE~\cite{ICE}] {%
		\resizebox{0.95\textwidth}{!}{
			\begin{tikzpicture}[zoomboxarray, zoomboxarray rows=1, zoomboxarray width=0.4\textwidth, zoomboxes xshift=0.99, execute at end picture={\draw[thick,red] (9.9,1) circle (.7cm);}, execute at end picture={\draw[thick,red] (9,2) circle (.7cm);}, execute at end picture={\draw[rotate around={150:(13,1.8)}, thick,red] (13,1.8) ellipse (.8cm and 1.5cm);}]
			\node [image node] {\includegraphics[height=0.15\textheight]{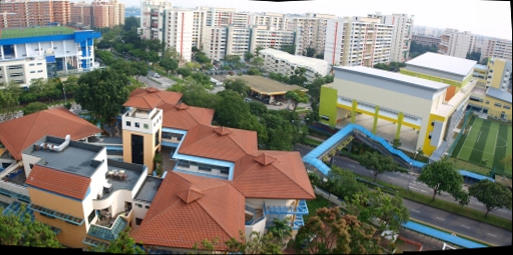}};
			\zoombox[color code=blue]{0.6,0.82}
			\zoombox[color code=green]{0.6,0.55}
			\end{tikzpicture}
		}
	}\,\!
	\subfloat[SPHP~\cite{chang_shape-preserving_2014}] {%
		\resizebox{0.95\textwidth}{!}{
			\begin{tikzpicture}[zoomboxarray, zoomboxarray rows=1, zoomboxarray width=0.4\textwidth, zoomboxes xshift=1.02, execute at end picture={\draw[rotate around={90:(9,2.5)}, thick,red] (9,2.5) ellipse (.6cm and 1.2cm);}]
			\node [image node] {\includegraphics[height=0.15\textheight]{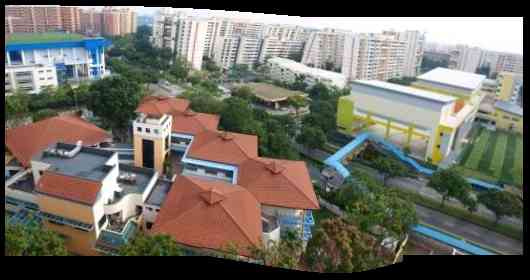}};
			\zoombox[color code=blue]{0.6,0.78}
			\zoombox[color code=green]{0.6,0.55}
			\end{tikzpicture}
		}
	}\,\!
	\subfloat[APAP~\cite{zaragoza_as-projective-as-possible_2014}] {%
		\resizebox{0.95\textwidth}{!}{
			\begin{tikzpicture}[zoomboxarray, zoomboxarray rows=1, zoomboxarray width=0.4\textwidth, zoomboxes xshift=0.97, execute at end picture={\draw[rotate around={0:(9.5,2.3)}, thick,red] (9.5,2.3) ellipse (.6cm and .9cm);}]
			\node [image node] {\includegraphics[height=0.15\textheight]{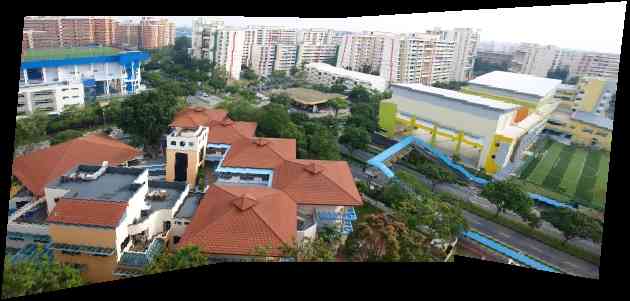}};
			\zoombox[color code=blue]{0.56,0.78}
			\zoombox[color code=green]{0.56,0.55}
			\end{tikzpicture}
		}
	}\,\!
	\subfloat[BRAS] {%
		\resizebox{0.95\textwidth}{!}{%
			\begin{tikzpicture}[zoomboxarray, zoomboxarray rows=1, zoomboxarray width=0.4\textwidth, zoomboxes xshift=0.97]
			\node [image node] {\includegraphics[height=0.15\textheight]{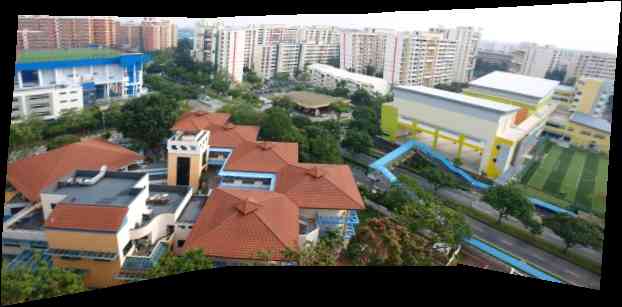}};
			\zoombox[color code=blue]{0.56,0.78}
			\zoombox[color code=green]{0.56,0.55}
			\end{tikzpicture}
		}
	}\\
	\caption{\tcr{Stitched images on the {\it rooftops\/} dataset~\cite{lin_smoothly_2011}. Red circles in the insets highlight artifacts.}}\label{fig:rooftops}
\end{figure*}

\begin{figure*}
	\centering
	\subfloat[AutoStitch~\cite{brown_automatic_2007}] {%
		\resizebox{0.8\textwidth}{!}{%
			\begin{tikzpicture}[zoomboxarray, zoomboxarray rows=1, zoomboxarray width=0.4\textwidth, zoomboxes xshift=1.1, execute at end picture={\draw[thick,red] (8,2.7) circle (.5cm);}]
				\node [image node] {\includegraphics[height=0.13\textheight]{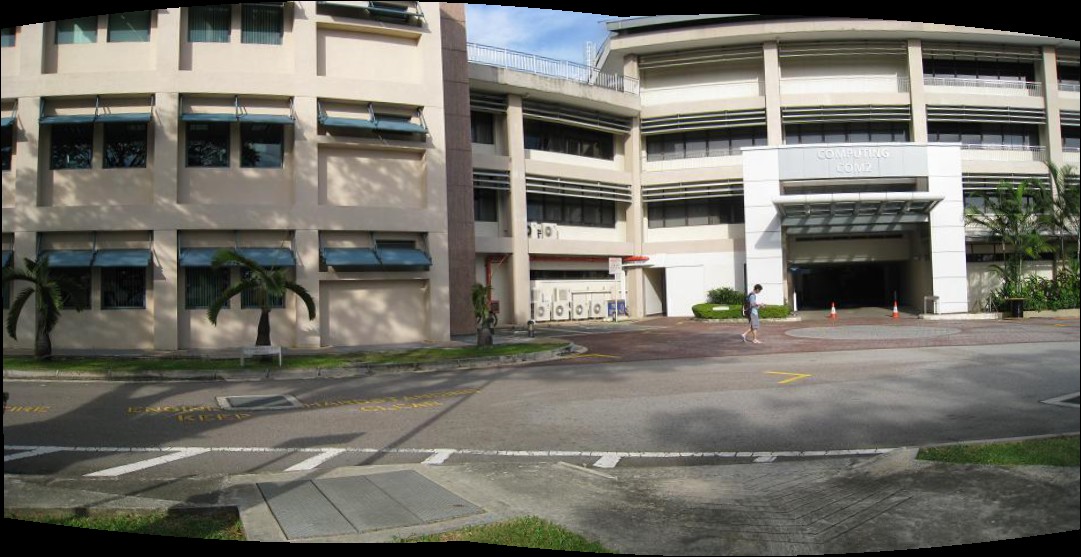}};
				\zoombox[color code=blue, magnification=7]{0.57,0.9}
				\zoombox[color code=green, magnification=6]{0.45,0.2}
			\end{tikzpicture}
		}
	}\\\vspace{-2mm}
	\subfloat[ICE~\cite{ICE}] {%
		\resizebox{0.8\textwidth}{!}{%
			\begin{tikzpicture}[zoomboxarray, zoomboxarray rows=1, zoomboxarray width=0.4\textwidth, zoomboxes xshift=1.1, execute at end picture={\draw[thick,red] (7.9,2.6) circle (.5cm);}]
				\node [image node] {\includegraphics[height=0.13\textheight]{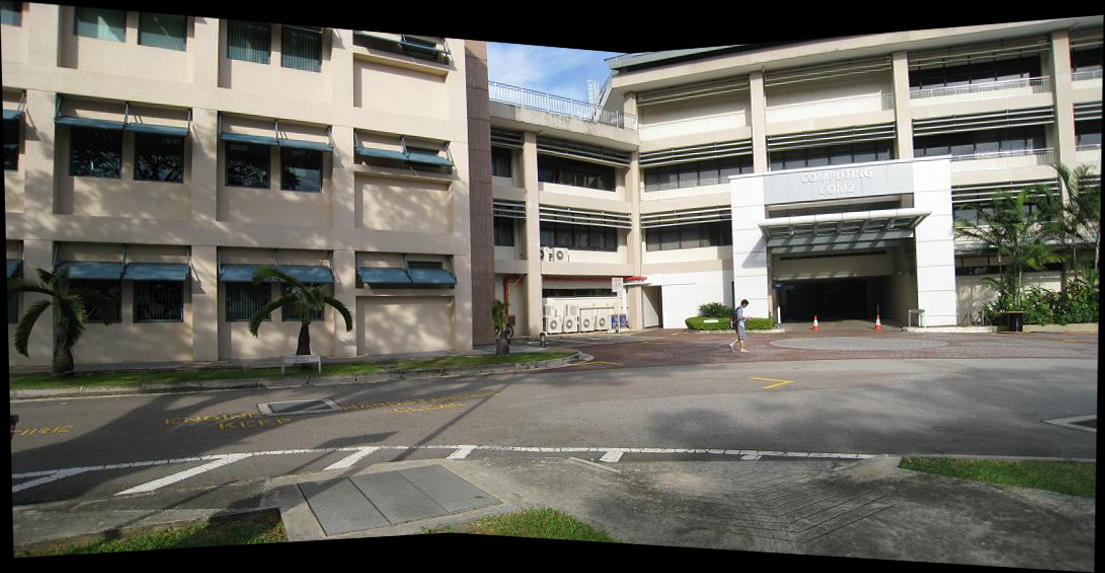}};
				\zoombox[color code=blue, magnification=7]{0.57,0.85}
				\zoombox[color code=green, magnification=6]{0.45,0.2}
			\end{tikzpicture}
		}
	}\\\vspace{-2mm}
	\subfloat[SPHP~\cite{chang_shape-preserving_2014}] {%
		\resizebox{0.8\textwidth}{!}{%
			\begin{tikzpicture}[zoomboxarray, zoomboxarray rows=1, zoomboxarray width=0.4\textwidth, zoomboxes xshift=1.1, execute at end picture={\draw[thick,red] (7.9,2.7) circle (.5cm);}]
				\node [image node] {\includegraphics[height=0.13\textheight]{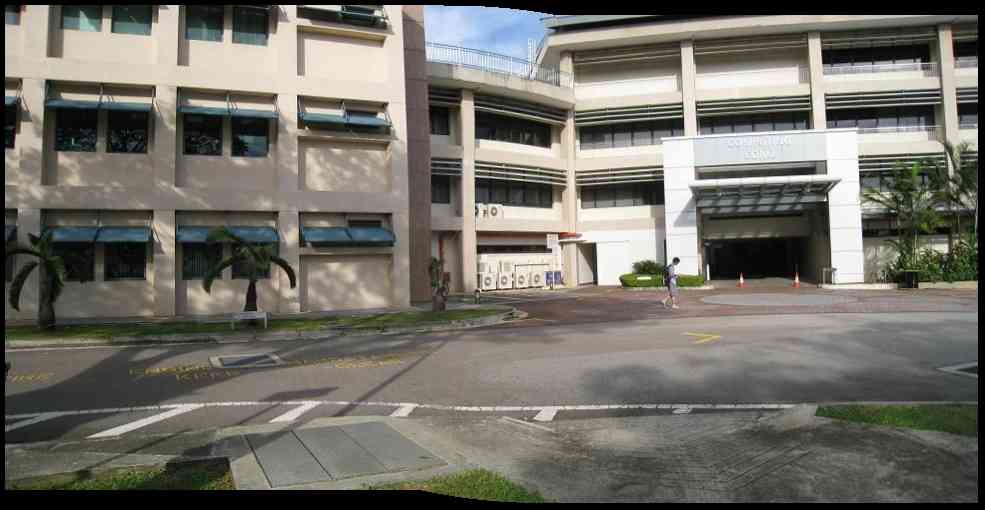}};
				\zoombox[color code=blue, magnification=7]{0.57,0.91}
				\zoombox[color code=green, magnification=6]{0.45,0.2}
			\end{tikzpicture}
		}
	}\\\vspace{-2mm}
	\subfloat[CPW~\cite{hu_multi-objective_2015}] {%
		\resizebox{0.8\textwidth}{!}{%
			\begin{tikzpicture}[zoomboxarray, zoomboxarray rows=1, zoomboxarray width=0.4\textwidth, zoomboxes xshift=1.23, execute at end picture={\draw[thick,red] (7.9,2.6) circle (.5cm);}, execute at end picture={\draw[thick,red] (10,2) circle (.7cm);}]
				\node [image node] {\includegraphics[height=0.13\textheight]{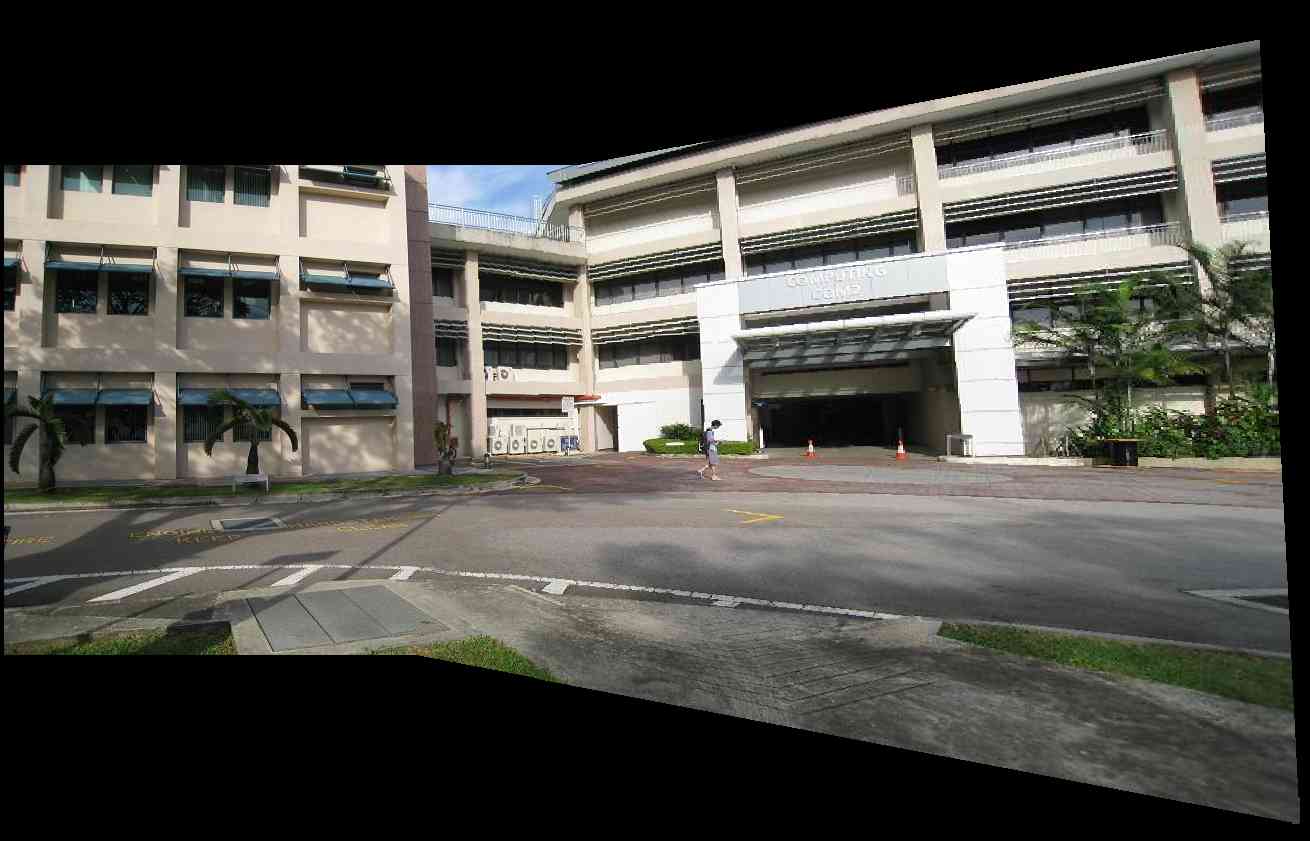}};
				\zoombox[color code=blue, magnification=7]{0.41,0.75}
				\zoombox[color code=green, magnification=6]{0.36,0.29}
			\end{tikzpicture}
		}
	}\\\vspace{-2mm}
	\subfloat[APAP~\cite{zaragoza_as-projective-as-possible_2014}] {%
		\resizebox{0.8\textwidth}{!}{%
			\begin{tikzpicture}[zoomboxarray, zoomboxarray rows=1, zoomboxarray width=0.4\textwidth, zoomboxes xshift=1.2, execute at end picture={\draw[thick,red] (7.5,2.7) circle (.5cm);}]
				\node [image node] {\includegraphics[height=0.13\textheight]{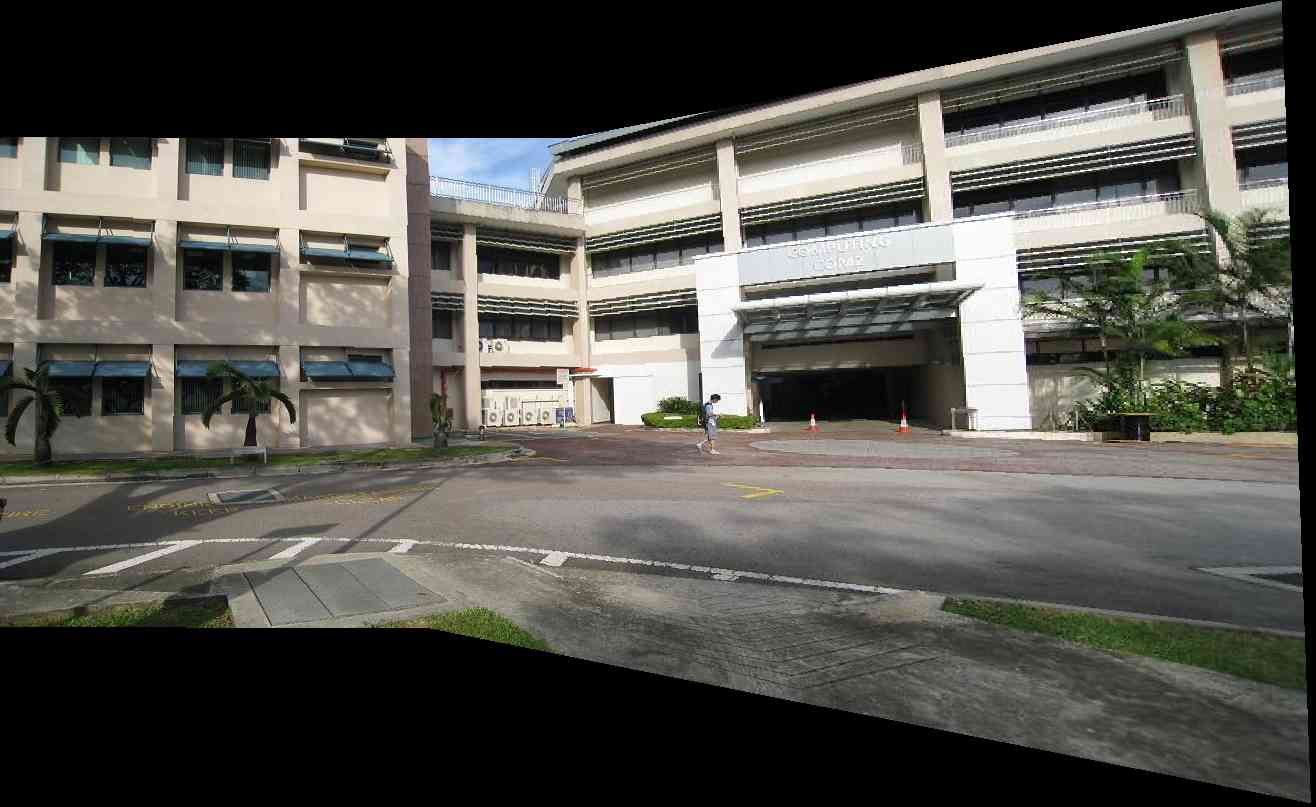}};
				\zoombox[color code=blue, magnification=7]{0.43,0.77}
				\zoombox[color code=green, magnification=6]{0.35,0.3}
			\end{tikzpicture}
		}
	}\\\vspace{-2mm}
	\subfloat[BRAS] {%
		\resizebox{0.8\textwidth}{!}{%
			\begin{tikzpicture}[zoomboxarray, zoomboxarray rows=1, zoomboxarray width=0.4\textwidth, zoomboxes xshift=1.2]
				\node [image node] {\includegraphics[height=0.13\textheight]{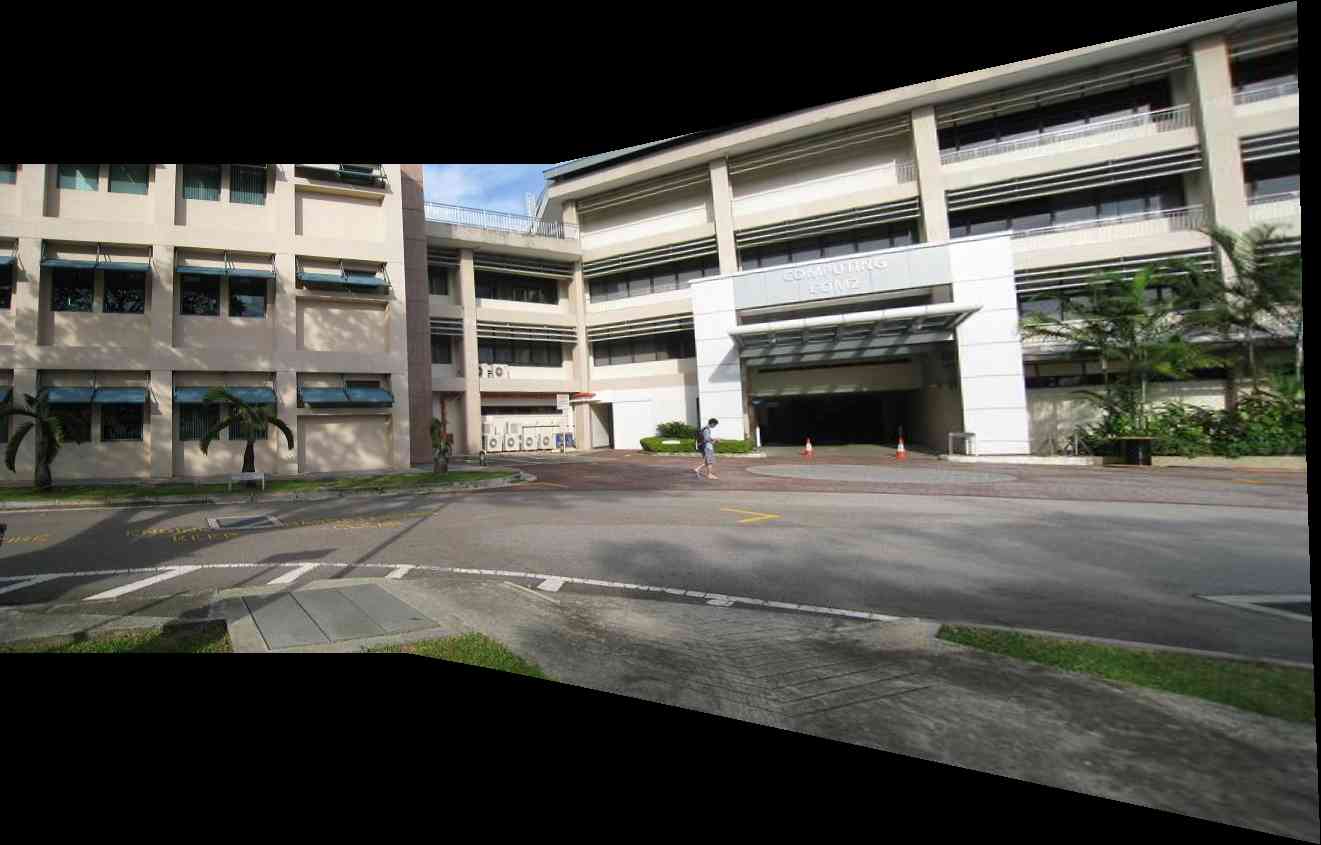}};
				\zoombox[color code=blue, magnification=7]{0.43,0.76}
				\zoombox[color code=green, magnification=6]{0.35,0.3}
			\end{tikzpicture}
		}
	}\\\vspace{-1mm}
	\caption{\tcr{Stitched images on the {\it carpark\/} dataset~\cite{gao_constructing_2011}. Red circles in the insets highlight artifacts.}}\label{fig:carpark}
\end{figure*}
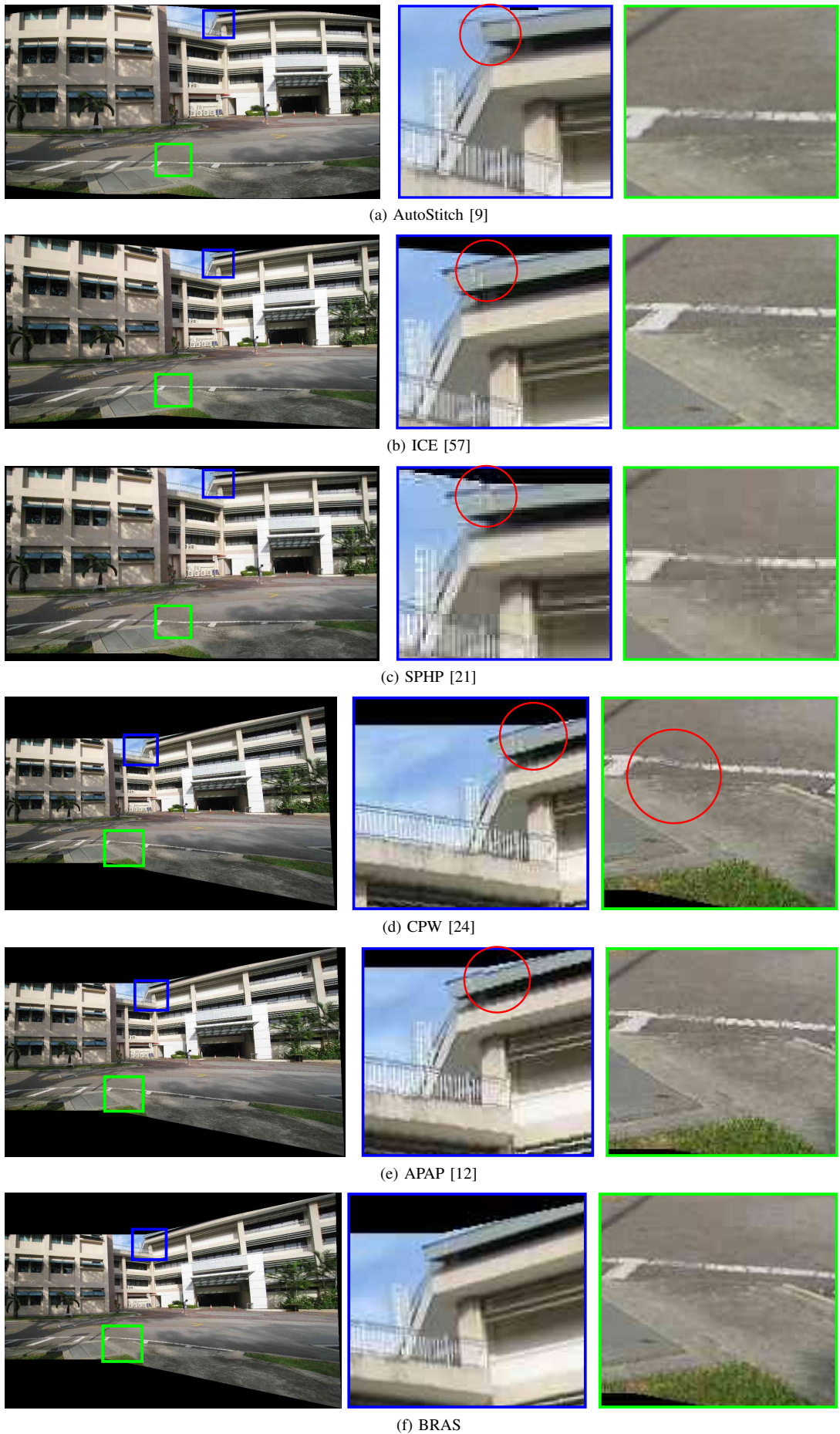

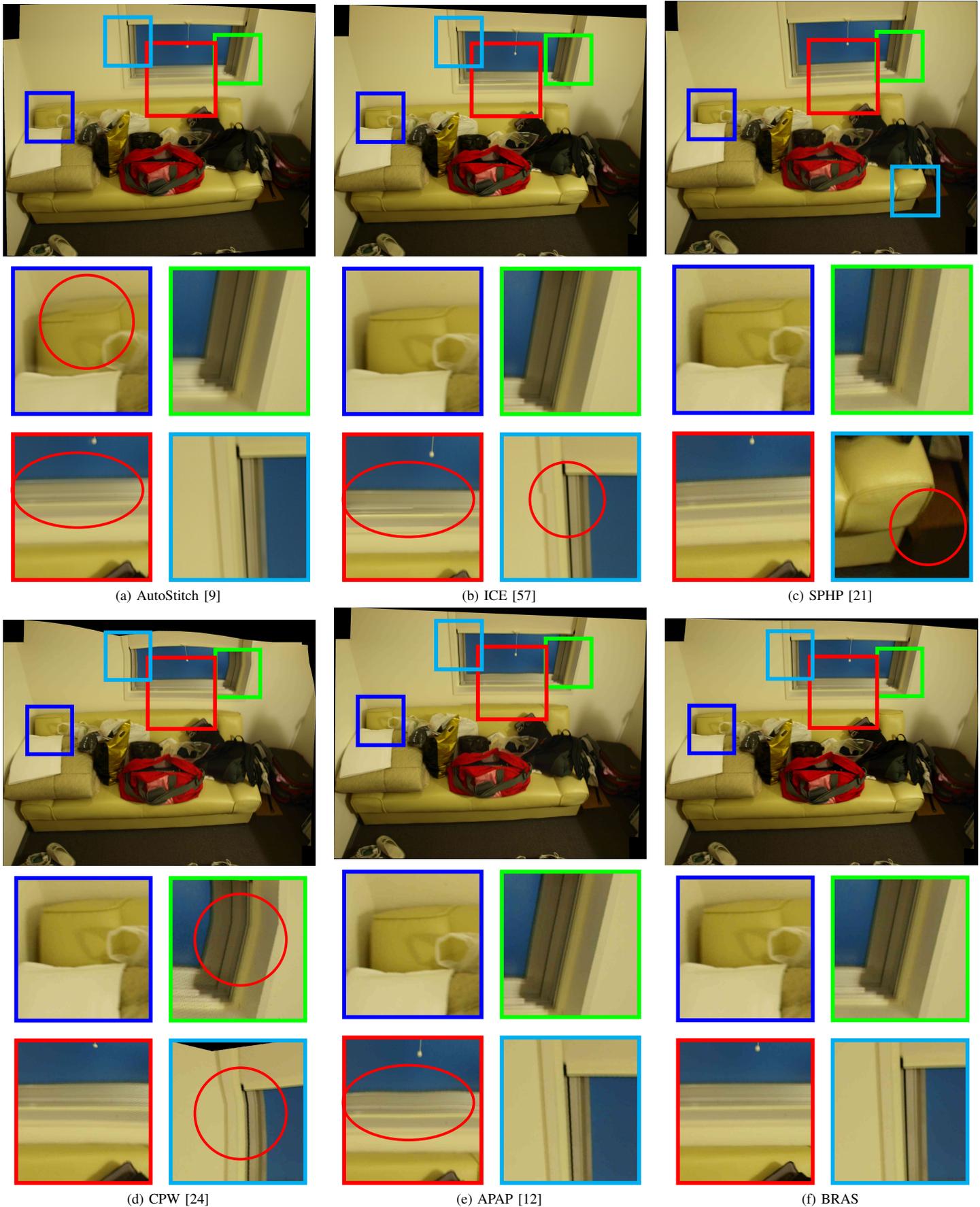
\begin{figure*}
	\centering
	\subfloat[AutoStitch~\cite{brown_automatic_2007}] {%
		\resizebox{0.35\textwidth}{!}{%
			\begin{tikzpicture}[zoomboxarray, zoomboxarray rows=1, zoomboxarray width=3.2cm, zoomboxarray height=1.6cm, zoomboxes xshift=0, zoomboxes yshift=-.8, execute at end picture={\draw[thick,red] (.8,-2.5) ellipse (.7cm and .4cm);}, execute at end picture={\draw[thick,red] (.9,-.7) circle (.5cm);}]
			\node [image node] {\includegraphics[height=0.11\textheight]{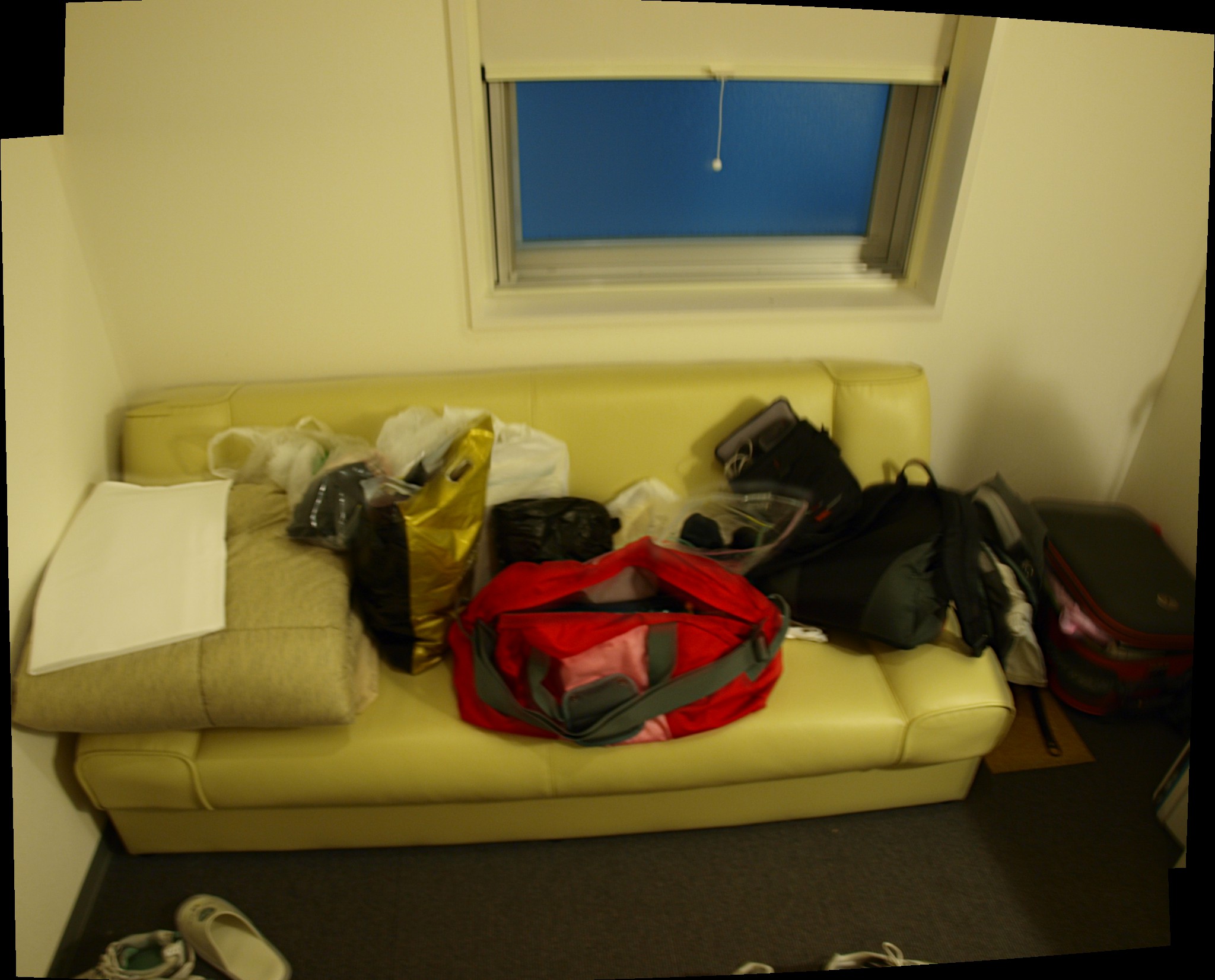}};\vspace{-1cm}
			\zoombox[color code=blue, magnification=3]{0.15,0.55}
			\zoombox[color code=green, magnification=3]{0.75,0.78}
			\zoombox[color code=red, magnification=2]{0.57,0.7}
			\zoombox[color code=cyan, magnification=3]{0.4,0.85}
			\end{tikzpicture}
		}
	}
	\subfloat[ICE~\cite{ICE}] {%
		\resizebox{0.35\textwidth}{!}{%
			\begin{tikzpicture}[zoomboxarray, zoomboxarray rows=1, zoomboxarray width=3.2cm, zoomboxarray height=1.6cm, zoomboxes xshift=0, zoomboxes yshift=-.8, execute at end picture={\draw[thick,red] (2.5,-2.6) circle (.4cm);}, execute at end picture={\draw[thick,red] (.8,-2.6) ellipse (.7cm and .4cm);}]
			\node [image node] {\includegraphics[height=0.11\textheight]{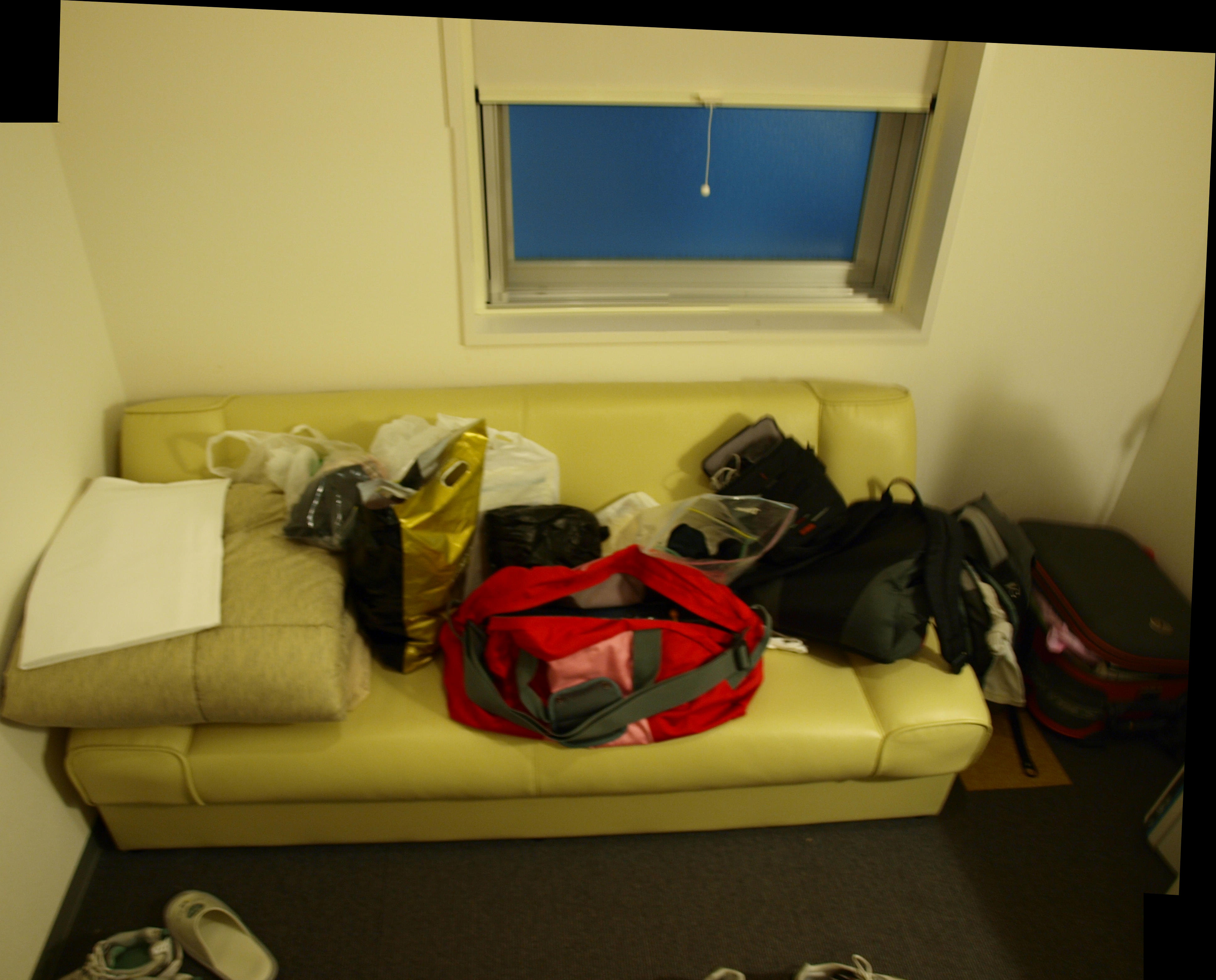}};\vspace{-1cm}
			\zoombox[color code=blue, magnification=3]{0.15,0.55}
			\zoombox[color code=green, magnification=3]{0.75,0.78}
			\zoombox[color code=red, magnification=2]{0.55,0.7}
			\zoombox[color code=cyan, magnification=3]{0.4,0.85}
			\end{tikzpicture}
		}
	}
	\subfloat[SPHP~\cite{chang_shape-preserving_2014}] {%
		\resizebox{0.35\textwidth}{!}{%
			\begin{tikzpicture}[zoomboxarray, zoomboxarray rows=1, zoomboxarray width=3.2cm, zoomboxarray height=1.6cm, zoomboxes xshift=0, zoomboxes yshift=-.8, execute at end picture={\draw[thick,red] (2.8,-2.9) circle (.4cm);}]
			\node [image node] {\includegraphics[height=0.11\textheight]{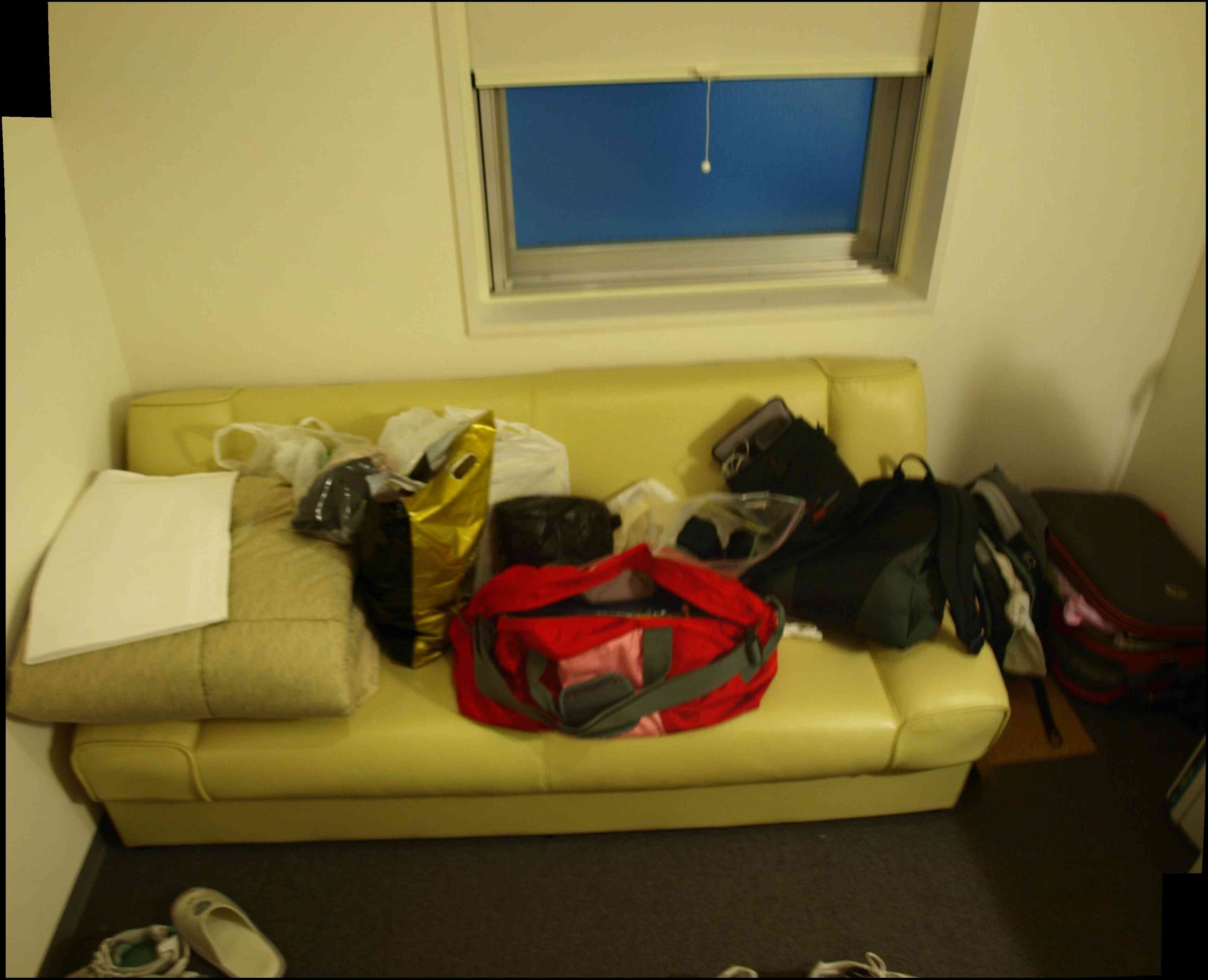}};
			\zoombox[color code=blue, magnification=3]{0.15,0.55}
			\zoombox[color code=green, magnification=3]{0.75,0.78}
			\zoombox[color code=red, magnification=2]{0.57,0.7}
			\zoombox[color code=cyan, magnification=3]{0.8,0.25}
			\end{tikzpicture}
		}
	}\\\vspace{-3mm}
	\subfloat[CPW~\cite{hu_multi-objective_2015}] {%
		\resizebox{0.35\textwidth}{!}{%
			\begin{tikzpicture}[zoomboxarray, zoomboxarray rows=1, zoomboxarray width=3.2cm, zoomboxarray height=1.6cm, zoomboxes xshift=0, zoomboxes yshift=-0.8, execute at end picture={\draw[thick,red] (2.6,-0.8) circle (.5cm);}, execute at end picture={\draw[thick,red] (2.6,-2.7) circle (.5cm);}]
			\node [image node] {\includegraphics[height=0.11\textheight]{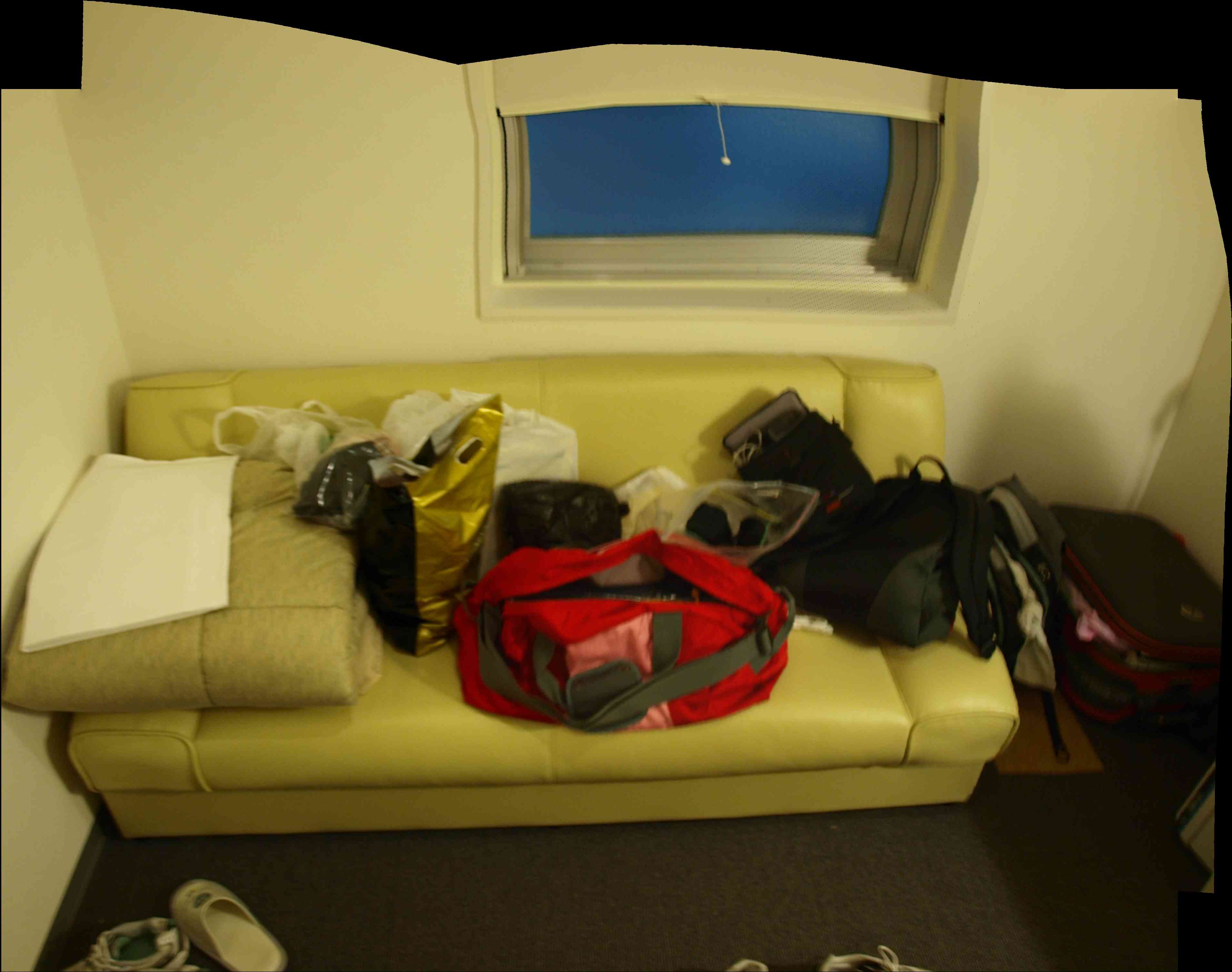}};
			\zoombox[color code=blue, magnification=3]{0.15,0.55}
			\zoombox[color code=green, magnification=3]{0.75,0.78}
			\zoombox[color code=red, magnification=2]{0.57,0.7}
			\zoombox[color code=cyan, magnification=3]{0.4,0.85}
			\end{tikzpicture}
		}
	}
	\subfloat[APAP~\cite{zaragoza_as-projective-as-possible_2014}] {%
		\resizebox{0.35\textwidth}{!}{%
			\begin{tikzpicture}[zoomboxarray, zoomboxarray rows=1, zoomboxarray width=3.2cm, zoomboxarray height=1.6cm, zoomboxes xshift=0, zoomboxes yshift=-0.8, execute at end picture={\draw[thick,red] (.8,-2.6) ellipse (.7cm and .4cm);}]
			\node [image node] {\includegraphics[height=0.11\textheight]{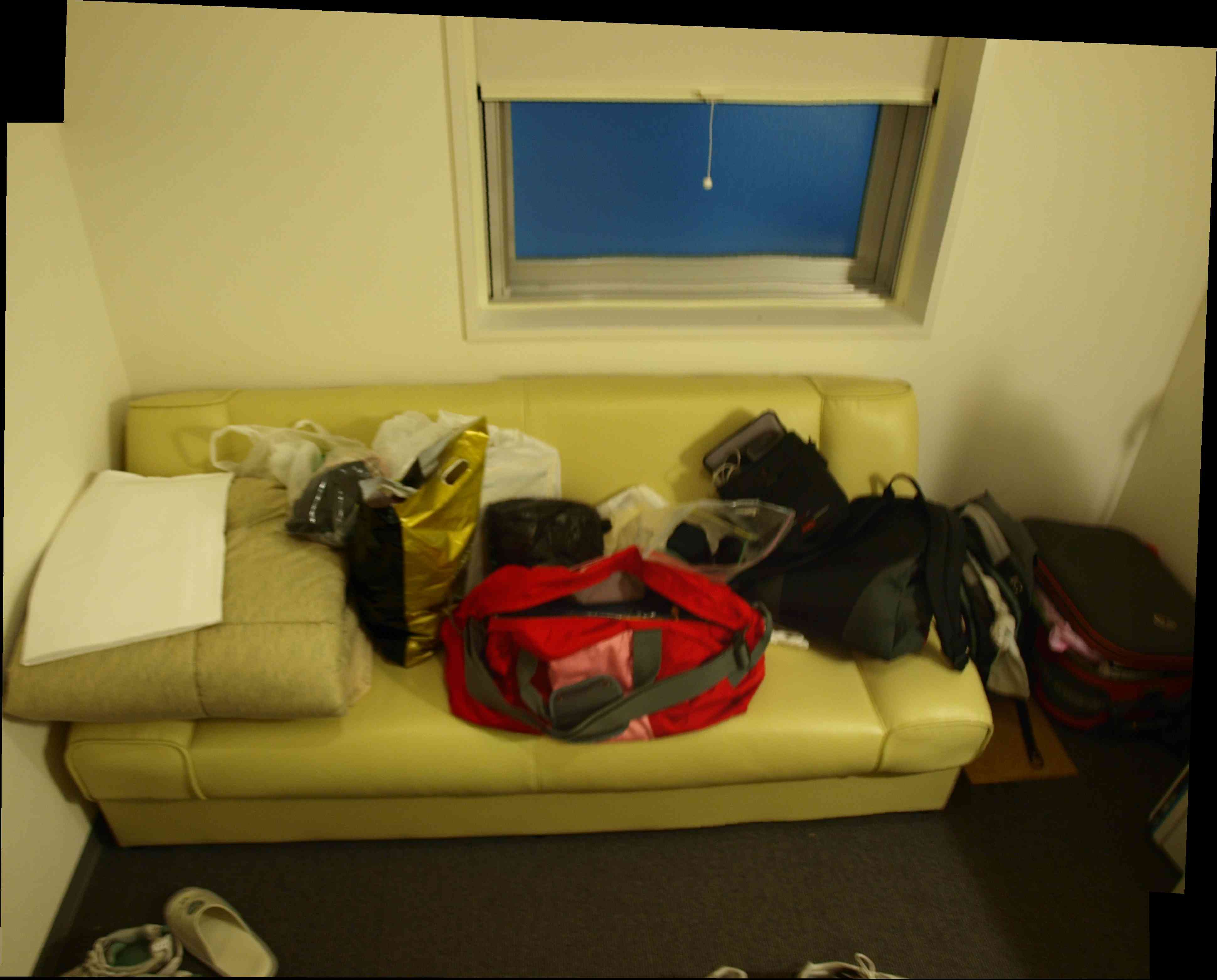}};
			\zoombox[color code=blue, magnification=3]{0.15,0.55}
			\zoombox[color code=green, magnification=3]{0.75,0.78}
			\zoombox[color code=red, magnification=2]{0.57,0.7}
			\zoombox[color code=cyan, magnification=3]{0.4,0.85}
			\end{tikzpicture}
		}
	}
	\subfloat[BRAS] {%
		\resizebox{0.35\textwidth}{!}{%
			\begin{tikzpicture}[zoomboxarray, zoomboxarray rows=1, zoomboxarray width=3.2cm, zoomboxarray height=1.6cm, zoomboxes xshift=0, zoomboxes yshift=-.8]
			\node [image node] {\includegraphics[height=0.11\textheight]{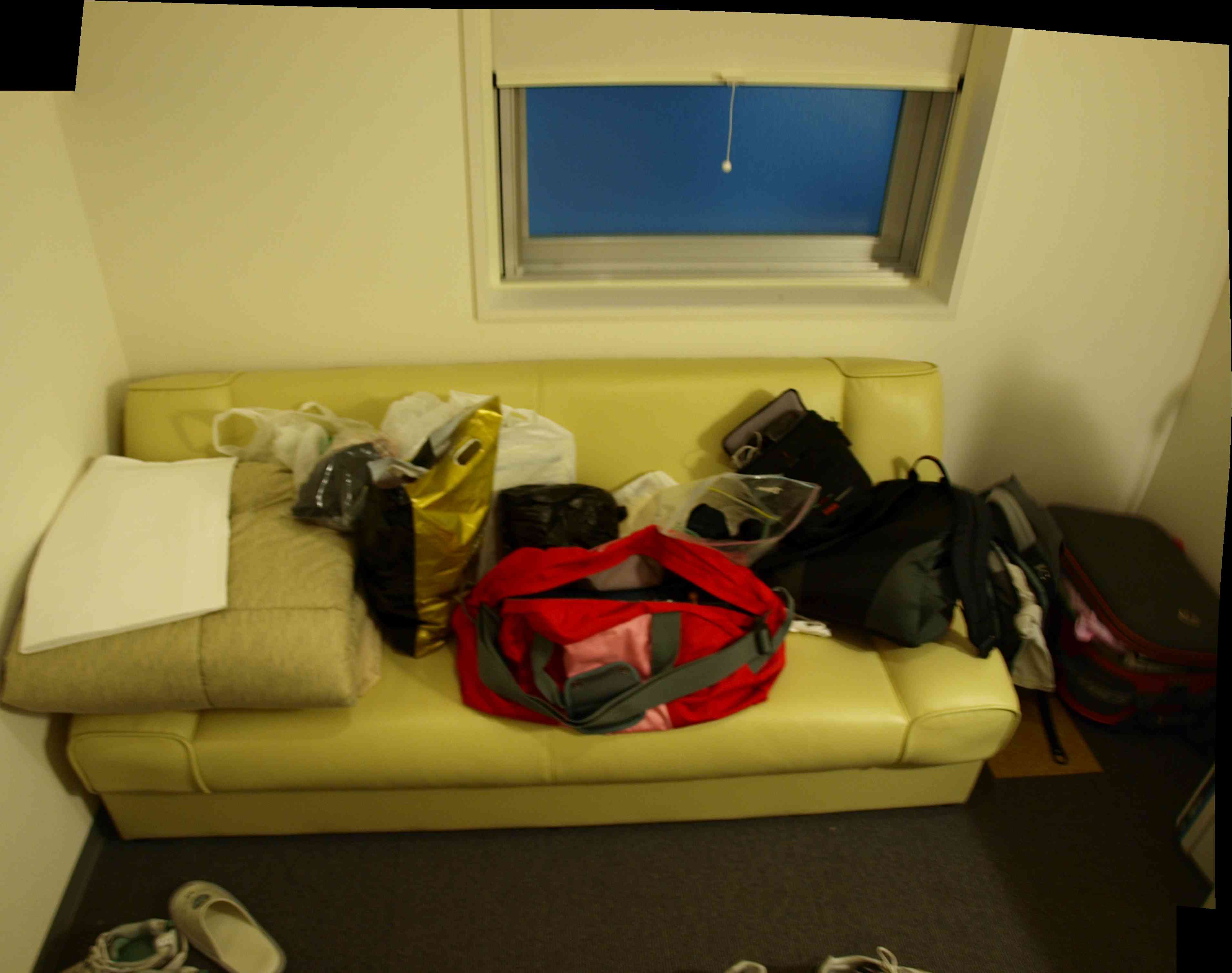}};
			\zoombox[color code=blue, magnification=3]{0.15,0.55}
			\zoombox[color code=green, magnification=3]{0.75,0.78}
			\zoombox[color code=red, magnification=2]{0.57,0.7}
			\zoombox[color code=cyan, magnification=3]{0.4,0.85}
			\end{tikzpicture}
		}
	}\\\vspace{-1mm}
	\caption{\tcr{Stitched images on the {\it couch\/} dataset~\cite{gao_constructing_2011}. Red circles in the insets highlight artifacts.}}\label{fig:couch}
\end{figure*}

\begin{figure*}
	\centering
	\subfloat[\label{subfig:bras_railtracks}]{\includegraphics[width=0.49\textwidth]{figs/railtracks_BRAS_aligned}}\,\!
	\subfloat[\label{subfig:multicell_railtracks}]{\includegraphics[width=0.49\textwidth]{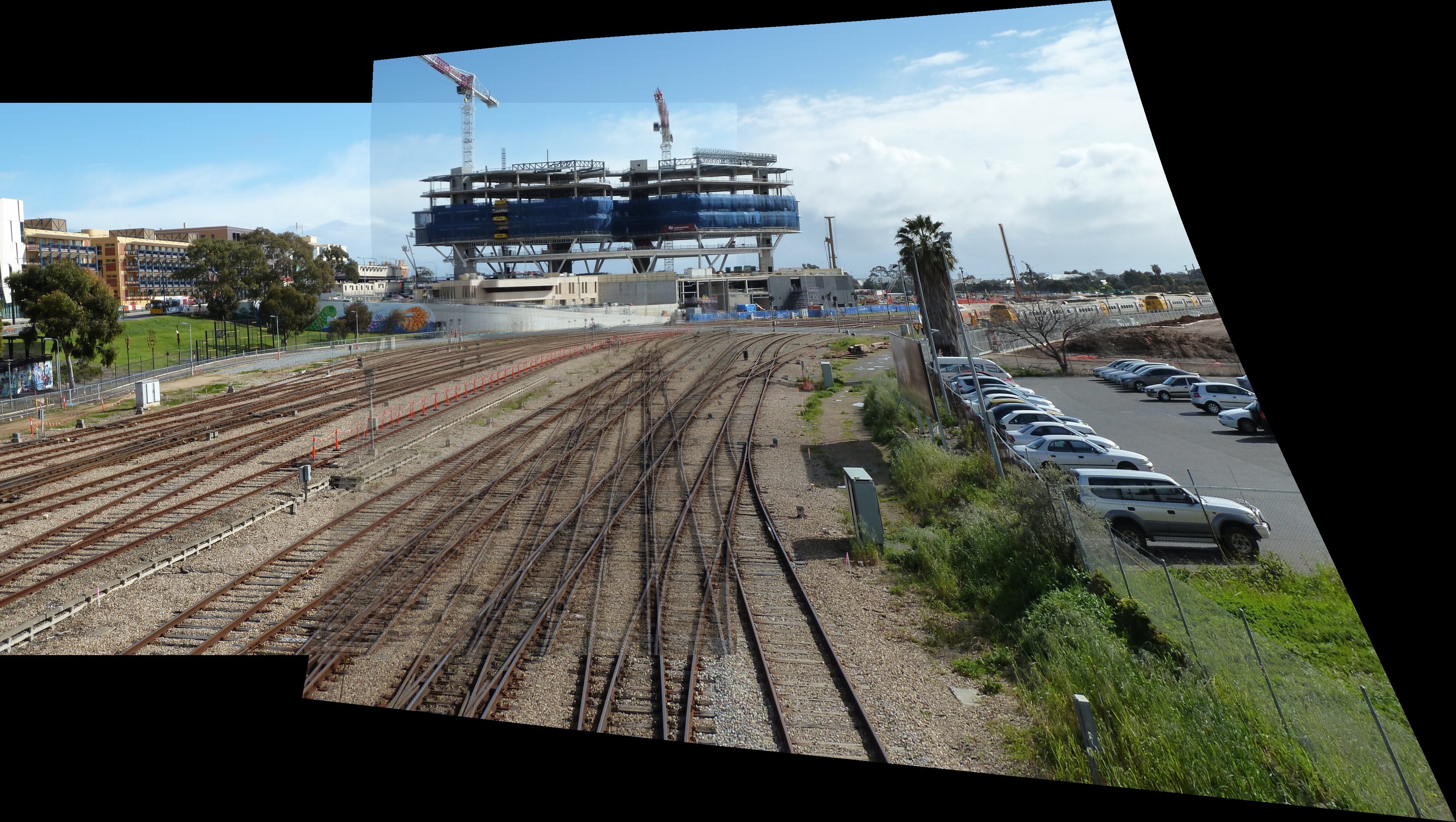}}
	\caption{\tcr{Comparison between \protect\subref{subfig:bras_railtracks} BRAS and \protect\subref{subfig:multicell_railtracks} a multi-cell extension to~\eqref{eqn:pairwise} on the \emph{railtracks} image set. Note that the only difference between these two methods is on whether the rank-1 constraint is imposed and clearly the rank-1 constraint improves the alignment performance.}}\label{fig:multi_cell}
\end{figure*}

\end{document}